\documentclass[10pt,a4paper]{book}
\usepackage[utf8]{inputenc}
\usepackage{amsmath}
\usepackage{amsfonts}
\usepackage{amssymb}
\usepackage{graphicx}
\usepackage{color}
\usepackage{amsthm}
\usepackage{dsfont}
\usepackage{hyperref}
\title{Lecture notes}

\newcommand{\ba}{\textbf{a}}
\newcommand{\bfb}{\textbf{b}}
\newcommand{\bc}{\textbf{c}}

\newcommand{\bff}{\textbf{f}}

\newcommand{\bu}{\textbf{u}}
\newcommand{\bv}{\textbf{v}}
\newcommand{\bw}{\textbf{w}}
\newcommand{\bx}{\textbf{x}}
\newcommand{\by}{\textbf{y}}

\newcommand{\bA}{\textbf{A}}
\newcommand{\bB}{\textbf{B}}
\newcommand{\bC}{\textbf{C}}

\newcommand{\bI}{\textbf{I}}
\newcommand{\bK}{\textbf{K}}
\newcommand{\bM}{\textbf{M}}

\newcommand{\bS}{\textbf{S}}
\newcommand{\bX}{\textbf{X}}

\newcommand{\cA}{\mathcal{A}}
\newcommand{\cB}{\mathcal{B}}
\newcommand{\cC}{\mathcal{C}}
\newcommand{\cD}{\mathcal{D}}
\newcommand{\cE}{\mathcal{E}}
\newcommand{\cF}{\mathcal{F}}
\newcommand{\cH}{\mathcal{H}}
\newcommand{\cK}{\mathcal{K}}
\newcommand{\cL}{\mathcal{L}}
\newcommand{\cM}{\mathcal{M}}
\newcommand{\cP}{\mathcal{P}}
\newcommand{\cR}{\mathcal{R}}
\newcommand{\cS}{\mathcal{S}}
\newcommand{\cT}{\mathcal{T}}
\newcommand{\cU}{\mathcal{U}}
\newcommand{\cV}{\mathcal{V}}

\newcommand{\cX}{\mathcal{X}}
\newcommand{\cY}{\mathcal{Y}}
\newcommand{\cZ}{\mathcal{Z}}

\newcommand{\E}{\mathbb{E}}
\newcommand{\F}{\mathbb{F}}
\newcommand{\K}{\mathbb{K}}
\newcommand{\N}{\mathbb{N}}
\newcommand{\Z}{\mathbb{Z}}
\renewcommand{\P}{\mathbb{P}}
\newcommand{\R}{\mathbb{R}}

\newcommand{\olsigma}{\overline{\sigma}}

\newcommand{\ERM}{\widehat{f}}
\newcommand{\bayes}{f^*}

\newcommand{\norm}[1]{\left\|#1\right\|}
\newcommand{\cro}[1]{\left[#1\right]}
\newcommand{\paren}[1]{\left(#1\right)}
\newcommand{\set}[1]{\left\{#1\right\}}
\newcommand{\absj}[1]{\left|#1\right|}
\newcommand{\psh}[2]{\left\langle#1,#2\right\rangle}

\newcommand{\PEInf}[1]{\left\lfloor#1\right\rfloor} 

\newcommand{\MOM}[2]{\text{MOM}_{#1}{\cro{#2}}}

\newcommand{\muh}{\widehat{\mu}}
\newcommand{\Kh}{\widehat{K}}
\newcommand{\Ih}{\widehat{I}}
\newcommand{\kh}{\widehat{k}}

\newcommand{\rmd}{ {\rm d}}

\DeclareMathOperator*{\gauss}{N}

\newtheorem{theorem}{Theorem}
\newtheorem{lemma}[theorem]{Lemma}
\newtheorem{proposition}[theorem]{Proposition}
\newtheorem{corollary}[theorem]{Corollary}
\newtheorem{definition}[theorem]{Definition}
\newtheorem{remark}[theorem]{Remark}

\DeclareMathOperator{\argmin}{\rm{argmin}}
\DeclareMathOperator{\argmax}{\textrm{argmax}}
\DeclareMathOperator{\Ent}{\textrm{Ent}}
\DeclareMathOperator{\Tr}{\textrm{Tr}}
\date{}
\begin{document}
	
\title{
 {\Huge{Selected topics on robust statistical learning theory}}\\
Lecture Notes}
\author{Matthieu Lerasle\\
 Paris-Saclay university}
	\maketitle
	\tableofcontents

\chapter{Introduction}
\section{Statistical learning}
These notes gather some results dealing with robustness issues in statistical learning.
Most of the results lie within the framework introduced by Vapnik \cite{MR1719582}, see also \cite{MR2291502}.
Given a dataset $\cD_N=(Z_1,\ldots,Z_N)$, where each $Z_i$ belongs to a measurable space $\cZ$, the goal is to infer from $\cD_N$ relevant informations regarding the stochastic mechanism that generated $\cD_N$.
To proceed, assume first that all data have the same (unknown) distribution $P$ and let $Z$ denote a random variable with distribution $P$ independent of $\cD_N$. 
Choose a set of parameters $F$ and a real valued function $\ell:F\times \cZ\to \R$, $(f,z)\mapsto \ell_f(z)$, $\ell$ is called the \emph{loss}.
Based on this loss, the \emph{risk} of any parameter $f\in F$ is defined as the integral of $\ell_f$ with respect to the distribution $P$:
\[
\forall f\in F,\qquad P\ell_f:=\E_{Z\sim P}[\ell_f(Z)]\enspace.
\]
The goal is to infer from $\cD_N$ the ``best" parameter $f^*$ in $F$ which is the one minimizing the risk: 
\[
f^*\in \argmin_{f\in F}P\ell_f\enspace.
\]
Hereafter, such a minimizer is assumed to exist to simplify notations. 
The interested reader can check that all results pertain if $P\ell_{f^*}$ is replaced by $\inf_{f\in F}P\ell_f$ in the following.
$f^*$ is unknown as it depends on $P$, it is usually called the \emph{oracle} as it is the parameter that would have chosen someone knowing the distribution $P$.
It cannot be used as an estimator, it is rather an ideal that any procedure tries to mimic.
Indeed, most of the material presented here aims at bounding the \emph{excess risk} of any estimator $\ERM\in F$ defined by 
\[
\cE(\ERM)=P[\ell_{\ERM}-\ell_{f^*}]=\E[\ell_{\ERM}(Z)-\ell_{f^*}(Z)|\cD_N]\enspace.
\]
For any $f\in F$, $\cE(f)=P[\ell_{f}-\ell_{f^*}]$ measures by how much $f$ fails to minimize $P\ell_f$.
It is worth noticing that $\cE(\ERM)$ is a random variable, the integral defining the risk being with respect to the random variable $Z\sim P$ that is \emph{independent} of $\cD_N$.
Bounding $\cE(\ERM)$ from above means here finding $\Delta_{N,\delta}(F)$ such that
\[
\P(\cE(\ERM)\leqslant \Delta_{N,\delta}(F))\geqslant 1-\delta\enspace.
\]
This type of result will be refered to as \emph{oracle inequality} as it compares the risk of the estimator $P\ell_{\ERM}$ with the one of an oracle $P\ell_{\bayes}=\inf_{f\in F}P\ell_f$.
This problem covers many classical problems in statistics and learning, we present here some basic examples, other will follow in the notes.
\paragraph{Univariate mean estimation}
In this example, given real valued random variables $Z_1,\ldots,Z_N$ with common distribution $P$, the goal is to infer the expectation $P[Z]=\E_{Z\sim P}[Z]$.

Set $F=\R$ and let $\ell_f(z)=(z-f)^2$ so, if $P[Z^2]<\infty$, then for any $f\in \R$, the expectation $f^*=P[Z]$ satisfies
\[
P\ell_f=\E[(Z-f)^2]=(f-f^*)^2+\E[(Z-f^*)^2]=(f-f^*)^2+P\ell_{f^*}\enspace.
\]
It follows that $f^*$ is the unique minimizer of $P\ell_f$ over $F$.
This example is simultaneously the simplest one can imagine and a natural building block for any learning procedure. 
Chapter~\ref{Chap:UME} is therefore dedicated to this elementary problem. 

\paragraph{Multivariate mean estimation}	Assume now that data $Z\in \R^d$ and, setting $F=\R^d$, the goal is to estimate $f^*=P[Z]$. 
Let $\|\cdot\|$ denote the Euclidean norm and let $\ell_f(z)=\|f-z\|^2$.
For any $f\in F$, it holds
\begin{align*}
P\ell_f&=\E[\|Z-f\|^2]\\
&=\E[\|Z-f^*\|^2]+2\E[(Z-f^*)^T(f-f^*)]+\|f-f^*\|^2\\
&=P\ell_{f^*}+\|f-f^*\|^2\enspace. 
\end{align*}

Here the second equality follows by linearity of the expectation and $\E[Z-f^*]=0$.
It follows that $f^*$ is the unique minimizer of $P\ell_f$ over $F$.
This example allows to understand the central role of uniform concentration inequalities to bound the excess risk of estimators. 
Chapter~\ref{Chap:MME} is dedicated to this problem.

\paragraph{Regression}
While previous problems are typical examples of \emph{unsupervised} learning tasks where data are not labeled, regression is arguably the most classical example of \emph{supervised} learning task where data are labeled: $Z=(X,Y)$ with $X$ the input or feature taking values in a measurable space $\cX$ and $Y$ is the output or label takes value in a subset $\cY\subset \R$.
The goal is to \emph{predict} $Y$ from $X$.
The purpose of regression is to estimate the regression function defined as any function $f^*$ such that, for any bounded measurable function $\varphi:\cX\to\R$, 
\[
\E[Y\varphi(X)]=\E[f^*(X)\varphi(X)]\enspace.
\]
Assume that $P[Y^2]<\infty$, let $F=L^2(P_X)$ and $\ell_f(x,y)=(y-f(x))^2$.
For any $f\in F$,
\begin{align*}
P\ell_f&=\E[(Y-f^*(X))^2]+2\E[(Y-f^*(X))(f^*(X)-f(X))]+\E[(f^*(X)-f(X))^2]\\
&=P\ell_{f^*}+\E[(f^*(X)-f(X))^2]\enspace. 
\end{align*}
It follows that $f^*$ is $P_X$-almost surely the unique minimizer of $P\ell_f$.
An important difference with the previous examples is that the ``natural" set of parameters $F$ is here infinite dimensional.
To bound properly the risk of the estimators, it is necessary to consider strict subsets $F_0\subset F$ and consider only estimators taking values in $F_0$.
This implies that, rather than the regression function $f^*$, the estimators are more natural estimators of the ``local" oracle 
\[
f^*_0\in \argmin_{f\in F_0}P\ell_f\enspace,
\]
provided that such function exists.
Chapter~\ref{Chap:LSR} is dedicated to the least-squares regression problem.

\subsubsection{Empirical risk minimisation}
One of the most classical algorithm in statistical learning is empirical risk minimization, see \cite{MR1719582}, which considers the estimator $\widehat{f}_{\text{erm}}$ of $f^*$ defined by 
\[
\widehat{f}_{\text{erm}}\in\argmin_{f\in F}P_N\ell_f,\qquad \text{where}\qquad P_N\ell_f:=\frac1N\sum_{i=1}^N\ell_f(Z_i)\enspace.
\]
One of the reasons explaining the success of this estimator is that is minimax optimal in many problems.
Minimax optimal rates can usually be proved for the ERM in problems where data are assumed \emph{independent, identically distributed and sub-Gaussian}.
In the univariate and multivariate mean estimation problems, the empirical risk minimizer $\widehat{f}_{\text{erm}}$ is simply the empirical mean $N^{-1}\sum_{i=1}^NZ_i$. 
In these examples, data are called sub-Gaussian if the Laplace transform of $P$ is bounded from above by the one of Gaussian random variable.
In the univariate mean estimation problem, this means that there exists $\sigma^2>0$ such that
\begin{equation}\label{eq:GaussAss}
\forall s>0,\qquad \log\E[e^{s(Z-\E[Z])}]\leqslant \frac{s^2\sigma^2}2\enspace. 
\end{equation}
Under this assumption, Markov's inequality ensures that, for any $t>0$, and $s>0$,
\begin{align*}
 \P\bigg(\frac1N\sum_{i=1}^NZ_i>\E[Z]+t\bigg)&=\P(e^{sN^{-1}\sum_{i=1}^N(Z_i-\E[Z])}>e^{st})\\
 &\leqslant e^{-st+\log\E[e^{sN^{-1}\sum_{i=1}^N(Z_i-\E[Z])}]}\\
 &=e^{-st+\sum_{i=1}^N\log\E[e^{sN^{-1}(Z_i-\E[Z])}]}\\
 &=e^{-st+\frac{s^2\sigma^2}{2N}}\enspace.
\end{align*}
Optimizing over $s$ yields
\[
\forall t>0,\qquad \P\bigg(\frac1N\sum_{i=1}^NZ_i>\E[Z]+t\bigg)\leqslant e^{-Nt^2/2\sigma^2}\enspace,
\]
or, equivalently 
\[
\forall t>0,\qquad \P\bigg(\frac1N\sum_{i=1}^NZ_i>\E[Z]+\sigma\sqrt{\frac{2t}N}\bigg)\leqslant e^{-t}\enspace.
\]
Sub-Gaussian deviations of the empirical mean are central in the analysis of ERM.
The sub-Gaussian deviation inequality only involves moments of order $1$ and $2$ of the $Z_i$, and an interesting question is whether this inequality remains valid if the $Z_i$ are only assumed to have finite moments of order $2$.
As explained in Chapter~\ref{Chap:UME}, sub-Gaussian deviations of the empirical mean only holds under the sub-Gaussian assumption \eqref{eq:GaussAss}.
When $Z$ is only assumed to have $2$ moments, one cannot essentially do better than Chebishev's inequality (see \cite[Proposition 6.2]{MR3052407} that is recalled in Proposition~\ref{prop:EmpMeanNotGood}) which states that
\[
\forall t>0,\qquad \P\bigg(\frac1N\sum_{i=1}^NZ_i>\E[Z]+\sigma\sqrt{\frac{2t}N}\bigg)\leqslant \frac1t\enspace.
\]
Providing estimators of the mean with sub-Gaussian deviations 
in a relaxed setting where $P$ is only assumed to a finite second moment is one of the guidelines in these notes.
This example shows the importance of evaluating estimators through their \emph{deviation} properties rather than in expectation.
Actually, in expectation
\[
\E[\cE(\widehat{f}_{\text{erm}})]=\E[P\ell_{\widehat{f}_{\text{erm}}}-P\ell_{f^*}]=\E[(\widehat{f}_{\text{erm}}-f^*)^2]=\frac{\sigma^2}N\enspace.
\]
This result is true if $Z$ has a finite moment of order $2$ and does not improve if $Z$ is Gaussian.
This is why these notes focus on Catoni's point of view, see \cite{MR3052407}, evaluating estimators by their deviation properties and proving oracle inequalities.

\section{Robustness}
Robustness is a classical topic in statistics that has been around since the seminal works of Hampel \cite{hampel1968contribution, hampel1971general, hampel1973robust,hampel1974influence, hampel1975beyond}, Huber \cite{huber1964robust, huber1967behavior} and Tukey \cite{tukey1960survey, tukey1962future,tukey74a,tukey74b}, see the classical textbook \cite{MR2488795} for an overview. 
Informally, an estimator is called robust if it behaves nicely even when data are not i.i.d. and sub-Gaussian.
This holds for example, when data are i.i.d. but satisfy only weak moment assumptions like the existence of a second moment only.
A large part of these notes deals with this issue.
An extensive literature has also been studied the case where the dataset is ``close" to the ideal  setup, but may have been corrupted.
This includes the following well known examples.

\paragraph{Model misspecification}
In statistics, this means that the distribution $P$ of $Z$ does not lie into the statistical model $\cP$ where the estimator $\hat{P}$ of $P$ lies.
A classical example of Birg\'e \cite{MR2219712} is the following: assume that $P$ is the mixture ${\rm d}P(x)=(1-1/N){\bf 1}_{x\in [0,1]}+(1/N)\delta_{x=N^2}$ and that the statistical model is the set of uniform distributions $\cP=\{\cU[0,t], t>0\}$.
The distribution $P$ is ``close" to the model $\cP$ since, for example, the Hellinger distance between $P$ and the uniform distribution $\cU([0,1])$ is bounded from above as follows:
\[
h^2(P,\cU([0,1]))\leqslant \frac1{N}\enspace.
\]
However, one of the most classical ERM in statistics, the maximum likelihood estimator, has positive probability to be the distribution $\cU[0,N^2]$ which is a very poor estimator of $P$.

\paragraph{Huber's contamination model.}
In this model, see \cite{MR2488795}, it is assumed that data are i.i.d. with common distribution 
\[
{\rm d} P=(1-\epsilon){\rm d} P_I+\epsilon {\rm d} P_O\enspace.
\]
$P_I$ is the distribution of \emph{inliers}, most of the sample is drawn from this distribution.
$P_I$ is the distribution on which one wants to make assumptions.
$P_O$ is the distribution of \emph{outliers}. 
These are data corrupting the dataset that may have nothing to do with the learning task.
Birg\'e's example is a particular instance of the Huber contamination problem where most of the data is drawn from the uniform distribution on $[0,1]$ but some data may be equal to $N^2$.
Usually, very few assumptions are granted on the outliers distribution.
However, in this model, these data are always independent and independent from the inliers.

\paragraph{The $O\cup I$ frameworks}
In this setting introduced in \cite{MOM1}, there exists a partition (unknown to the statistician) of $\{1,\ldots,N\}$ in two blocks $O$ and $I$.
Data $(Z_i)_{i\in O}$ are the \emph{outliers}, nothing is assumed on these data. 
Data $(Z_i)_{i\in I}$ are the inliers on which one may grant some assumptions.
This model is closely related to the $\epsilon$-contamination model while being slightly different:
\begin{itemize}
 \item Outliers may not be independent, nor independent from the other data $(Z_i)_{i\in I}$.
 This allows ``aggressive" outliers which can look at the dataset to corrupt it.
 \item The proportion of outliers is fixed in the $O\cup I$, equal to $|O|/N$, it is random in the Huber contamination model (although concentrated around $\epsilon$).
\end{itemize}

The main challenges in robust statistics are to \emph{resist} and \emph{detect} outliers.
Resist means looking for procedures that behave in the $\epsilon$-contamination model as well as ``good" estimators such that MLE do when $P=P_I$.
Detect means identifying outliers (think about fraud detection for example).

Of course, in both Huber's contamination model and $O\cup I$ frameworks, it is possible to consider situations where, besides being contaminated, the ``inliers" (those distributed as $P_I$ in Huber's contamination's model and data $(Z_i)_{i\in I}$ in the $O\cup I$ frameworks) only satisfy moment assumptions.
In these notes, I will mostly consider the situation where data are i.i.d. hence, not contaminated (see however Section~\ref{sec:Outliers}).
It is an interesting exercise to check if the different results extend to contaminated settings and which proportion of outliers is tolerated by different methods.

\section{What are these notes about}
The notes are an attempt to extract important \emph{principles} underlying the construction and theoretical analysis of estimators that are referred to as ``robust".
The main task is to build estimators that satisfy the same oracle inequalities as the ERM does when data have sub-Gaussian behavior in a relaxed setting where the Gaussian assumption is replaced by moment hypotheses.
These principles are divided in four main categories.
\begin{itemize}
\item The median-of-means principle allows to build estimators of univariate mean estimation achieving sub-Gaussian deviations, see Chapter \ref{Chap:UME}.
This is arguably the simplest construction allowing to achieve such results.
 \item The minmax principle allows to build from estimators of increments $P[\ell_f-\ell_g]$ (which are univariate means), estimators of ``oracles", see Chapters~\ref{Chap:MME} and~\ref{Chap:MinMax}. 
 The idea of using pairwise comparisons or tests to build estimators goes back to the works of Le Cam and Birgé, the minmax principle is an elegant formulation of this construction which makes a bridge between Birgé/Le Cam's construction and the ERM of Vapnik.
 \item The homogeneity lemma reduces the analysis of minmax estimators to deviation bounds of MOM processes on localized classes, see Chapter~\ref{Chap:MinMax}.
 The homogeneity lemma is an alternative to peeling arguments that can be used when deviation properties cannot be obtained at any confidence level.
  \item The \emph{small ball method} allows to prove (uniform and sub-Gaussian) deviation inequalities of the median-of-mean processes, under weak assumptions but only up to a confidence level that decreases geometrically with the sample size $N$, see Chapter~\ref{Chap:ConcIn}.
\end{itemize}
The combination of these principles allows to prove oracle inequalities simultaneously for the ERM in the sub-Gaussian framework (hence, providing the relevant benchmarks) and for robust alternatives such as minmax MOM estimators.
These procedures, thanks to the median step, naturally resist to a small proportion of outliers in the dataset. 
Chapter~\ref{Sec:RhoEst} presents $\rho$-estimators of \cite{MR3565484, MR3595933,BarBir2017RhoAgg2}.
This presentation is not exhaustive, it stresses some links between this construction and both the minmax principle and the homogeneity lemma.
It should be seen as an invitation to learn this powerful theory

These notes do not cover many important development, they focus on very particular learning tasks and very particular robustness issues, which correspond to problems I have been mostly interested in regarding this subject.
I hope that they will convince some readers to contribute to this rapidly growing literature.
In particular, Chapter~\ref{Chap:CompEst} presents (way too) briefly the literature on robust procedures that are computable in polynomial time.
In particular, numerically efficient methods are not discussed here.
In this direction, important results appeared recently.
In particular, \cite{LLVZ:2019} presents a first algorithm which produces an estimator of multivariate mean with optimal sub-Gaussian deviation rates using a spectral algorithm rather than SDP relaxations as the one presented in Chapter~\ref{Chap:CompEst}.
This new algorithm should behave much better numerically than its concurrent.
Moreover, only the problem of multivariate mean expectation is considered in the note from this computational perspective.
Going from this problem to more generic learning problems is well understood, see in particular \cite{PSBR:2018}.
This material should also be added in future version of these notes.

\chapter{Estimation of univariate means}\label{Chap:UME}
This chapter focuses on one of the most simple problem in statistics where we want to estimate the expectation $\mu_P=P[X]$ of a distribution $P$ from the observation of an i.i.d. sample $\cD_N=(X_1,\ldots,X_N)$ of real valued random variables with common distribution $P$.
These estimators are natural building blocks for more general learning tasks in the following chapters.
We first establish the behaviour of the empirical mean from a \emph{deviation} point of view.
We prove that it achieves good subexponential deviation bounds when $X$ is Gaussian and that Chebyshev's inequality is essentially sharp when $X$ is only assumed to have a bounded second moment.
Then, we study alternative estimators that achieve sub-Gaussian deviation inequalities when $X$ has only $2$ finite moments.

\paragraph{Notation}
All along the chapter, $\cP_2$ denotes the class of all probability distributions on $\R$ with finite second moment and $\cP\subset \cP_2$.
For any $P\in \cP_2$, $\mu_P$ denotes the expectation of $P$ and $\sigma_P^2$ its variance. 
$X_1,\ldots,X_N$ denotes an i.i.d. sample and for any $P\in\cP_2$, $\P=P^{\otimes N}$.
An estimator $\muh$ of $\mu_P$ is a real valued random variable $\muh=F(\cD_N)$, where $F:\R^N\to\R$ is a measurable function.

\section{Empirical mean}
The arguably most simple estimator of $\mu_P$ is the empirical mean 
\[
P_NX=\overline{X}_N=\frac1N\sum_{i=1}^NX_i\enspace.
\] 
\subsection{Lower bounds}
The empirical mean plays an important role in these notes in the case where the random variables are Gaussian.
The reason is that the deviation of the empirical mean in this example are somehow extremal as can be seen from the following result.
\begin{proposition}\label{prop:GaussIsOpt}\cite[Proposition 6.1]{MR3052407}
Assume that $\cP$ contains all Gaussian distributions $\gauss(\mu,\sigma^2)$. For any estimator $\muh$ of $\mu_P\in\R$, any $t>0$, there exists $P\in \cP$ such that the empirical mean $\overline{X}_N=N^{-1}\sum_{i=1}^NX_i$ satisfies either
\[
\P\paren{\muh-\mu_P>t}\geqslant \P\paren{\overline{X}_N>\mu_P+t}\quad \text{or}\quad \P\paren{\muh-\mu_P<-t}\geqslant \P\paren{\overline{X}_N<\mu_P-t}\enspace.
\]
\end{proposition}
\begin{proof}
For any $i\in\{-t,t\}$, let $P_i$ denote the Gaussian distribution with variance $1$ and respective mean $\mu_{P_i}=i$. By construction
\begin{multline*}
\P_{t}\paren{\muh\leqslant \mu_{P_{t}}-t}+\P_{-t}\paren{\muh\geqslant\mu_{P_{-t}}+t}=\P_{t}\paren{\muh\leqslant 0}+\P_{-t}\paren{\muh\geqslant 0}\\
 \geqslant (\P_{-t}\wedge\P_{t})\paren{\muh\leqslant 0}+(\P_{-t}\wedge\P_{t})\paren{\muh\geqslant 0}\geqslant |\P_{-t}\wedge\P_{t}|\enspace.
\end{multline*}
Here $\P_{-t}\wedge\P_{t}$ denotes the measure whose density is the minimum between those of $\P_{-t}$ and $\P_{t}$ and $ |\P_{-t}\wedge\P_{t}|$ is its total variation. 
Now, $\P_{i}$ has density 
\[
\frac1{(2\pi)^{n/2}}e^{-\frac1{2}\|\bx-\mu_{P_i}{\bf 1}\|^2}\enspace,
\]
therefore, $\P_{-t}\wedge\P_{t}$ has density $\P_{t}$ for any $\bx\in \R^N$ such that
\[
\|\bx-\mu_{P_t}{\bf 1}\|^2\geqslant \|\bx-\mu_{P_{-t}}{\bf 1}\|^2\quad\text{that is, such that}\quad \overline{x}_N=N^{-1}\sum_{i=1}^Nx_i\leqslant 0\enspace.
\]
Therefore,
\begin{multline*}
|\P_{-t}\wedge\P_{t}|=\P_t\paren{\overline{X}_N\leqslant 0}+\P_{-t}\paren{\overline{X}_N\geqslant 0}\\
=\P_{t}\paren{\overline{X}_N\leqslant\mu_{P_{t}}-t}+\P_{-t}\paren{\overline{X}_N\geqslant\mu_{P_{-t}}+t}\enspace.
\end{multline*}
Overall,
\begin{multline*}
\P_{t}\paren{\muh\leqslant \mu_{P_{t}}-t}+\P_{-t}\paren{\muh\geqslant\mu_{P_{-t}}+t}\\
 \geqslant\P_{t}\paren{\overline{X}_N\leqslant\mu_{P_{t}}-t}+\P_{-t}\paren{\overline{X}_N\geqslant\mu_{P_{-t}}+t}\enspace.
\end{multline*}
This implies the result.
\end{proof}
\subsection{Upper bounds in the sub-Gaussian case}
In order to establish the benchmark for future estimators, recall the following upper bound on the deviations of the empirical mean in the Gaussian case.
\begin{proposition}\label{prop:GaussIsSuGauss}
 If $X\sim \gauss(\mu,\sigma^2)$, then the empirical mean $P_NX$ satisfies
 \[
 \forall t>0,\qquad \P\bigg(|P_NX-\mu|>\sigma\sqrt{\frac{2t}N}\bigg)\leqslant e^{-t}\enspace.
 \]
\end{proposition}
\begin{proof}
Since $X\sim \gauss(\mu,\sigma^2)$, $ P_NX\sim\gauss(\mu,\sigma^2/N)$ and $\sqrt{N}(P_NX-\mu)/\sigma\sim\gauss(0,1)$.
The Gaussian distribution satisfies
\begin{align*}
1-\Phi(x)&=\int_{x}^{+\infty}e^{-u^2/2}\frac{{\rm d}u}{\sqrt{2\pi}}=\int_{0}^{+\infty}e^{-(u+x)^2/2}\frac{{\rm d}u}{\sqrt{2\pi}}\\
&\leqslant e^{-x^2}\int_{0}^{+\infty}e^{-u^2/2}\frac{{\rm d}u}{\sqrt{2\pi}}=\frac{e^{-x^2/2}}{2}\enspace. 
\end{align*}
Therefore,
\[
\forall x>0,\qquad \P\bigg(\frac{\sqrt{N}(P_NX-\mu)}{\sigma}>x\bigg)\leqslant \frac{e^{-x^2/2}}{2}\enspace. 
\]
This is equivalent to 
\[
\forall t>0,\qquad \P\bigg(P_NX-\mu>\sigma\sqrt{\frac{2t}N}\bigg)\leqslant \frac{e^{-t}}{2}\enspace. 
\]
Applying this inequality to $-X_i$ yields 
\[
\forall t>0,\qquad \P\bigg(P_NX-\mu<-\sigma\sqrt{\frac{2t}N}\bigg)\leqslant \frac{e^{-t}}{2}\enspace. 
\]
The result follows therefore from a union bound.
\end{proof}
The result on Gaussian distributions naturally extends to any sub-Gaussian distribution, thanks to Hoeffding's inequality.
Let $\sigma>0$. Recall that a random variable $X$ is called $\sigma$-sub-Gaussian if, for any $s>0$, 
\[
\E[e^{s(X-\E[X])}]\leqslant e^{\sigma^2s^2/2}\enspace.
\]
A Gaussian random variable with variance $\sigma^2$ is $\sigma$-sub-Gaussian.
Another important are bounded variables as shown by the following result.
\begin{lemma}[Hoeffding's Lemma]
 If $X\in [a,b]$, then $X$ is $(b-a)/2$-sub-Gaussian.
\end{lemma}
Hoeffding's Lemma is proved in Lemma~\ref{lem:Hoeff} in the following Chapter.

\medskip 

Deviation properties of Sub-Gaussian random variables are easy to get from the Chernoff bound.
Let $X$ denote a random variable and, for any $s$ for which it make sense, let $\psi(s)=\log\E[e^{s(X-\E[X])}]$.
Chernoff bound is an upper bound on the deviation of the variable $X$.
Let $t>0$, then, by Markov's inequality, for any $s\geqslant 0$ such that $\psi(s)$ is well defined,
\begin{equation}\label{eq:CB}
\P(X-\E[X]>t)=\P(e^{s(X-\E[X])}>e^{st})\leqslant e^{-st+\psi(s)}\enspace. 
\end{equation}
When $X$ is $\sigma$-sub-Gaussian, $\psi(s)\leqslant s^2\sigma^2/2$, for all $s\geqslant 0$, hence,
\[
\P(X-\E[X]>t)\leqslant e^{-st+s^2\sigma^2/2}\enspace.
\]
As this holds for any $s>0$, one can apply it to $s=t/\sigma^2$ and we obtain
\begin{equation}\label{eq:CBSG}
\forall t>0,\qquad \P(X-\E[X]>t)\leqslant e^{-\frac{t^2}{2\sigma^2}}\enspace. 
\end{equation}

%
The empirical mean of independent sub-Gaussian random variables is sub-Gaussian, as shown by the following inequality.
\begin{lemma}
If $X_1,\ldots,X_n$ are independent random variables and if, for any $i\in\{1,\ldots,n\}$, $X_i$ is $\sigma_i$-sub-Gaussian, then $n^{-1}\sum_{i=1}^nX_i$ is $n^{-1}\sqrt{\sum_{i=1}^n\sigma_i^2}$-sub-Gaussian.
\end{lemma}
\begin{proof}
Assume without loss of generality that each $\E[X_i]=0$.
Let $s>0$, then
\[
\E[e^{\frac{s}n\sum_{i=1}^nX_i}]=\prod_{i=1}^n\E[e^{sX_i}]=e^{\frac{s^2}{2n^2}\sum_{i=1}^n\sigma_i^2}\enspace.
\]
\end{proof}
Together with \eqref{eq:CBSG}, we obtain the following corollary.
\begin{corollary}\label{cor:SGEM}
  If $X_1,\ldots,X_n$ are independent random variables and if, for any $i\in\{1,\ldots,n\}$, $X_i$ is $\sigma_i$-sub-Gaussian, then
  \[
 \forall t>0,\qquad   \P\bigg(\frac1n\sum_{i=1}^n(X_i-\E[X_i])>t\bigg)\leqslant e^{-\frac{n^2t^2}{2\sum_{i=1}^n\sigma_i^2}}\enspace.
  \]
\end{corollary}
In the particular case of finite valued random variables discussed in Hoeffding's lemma, this corollary yields the standard version of Hoeffding's inequality.
\begin{corollary}[Hoeffding's inequality]\label{cor:HoeffIneq}
   If $X_1,\ldots,X_n$ are independent random variables and if, for any $i\in\{1,\ldots,n\}$, $X_i$ takes values in $[a_i,b_i]$, then
  \[
\forall t>0,\qquad   \P\bigg(\frac1n\sum_{i=1}^n(X_i-\E[X_i])>t\bigg)\leqslant e^{-\frac{2n^2t^2}{\sum_{i=1}^n(b_i-a_i)^2}}\enspace.
  \]

\end{corollary}

\subsection{Sub-Gaussian estimators}
All along these notes, we build estimators achieving the same deviation rates as the empirical mean in the (sub-)Gaussian case, since these rates are somehow extremal from Proposition~\ref{prop:GaussIsOpt}. 
Proposition~\ref{prop:GaussIsSuGauss} suggests a first definition of ``good" estimators of $\mu_P$.
%
%
\begin{definition}[sub-Gaussian estimator \cite{MR3576558}]\label{def:subgauss}
Let $A\in[0,+\infty]$, $B,C\geqslant 0$. An estimator $\muh$ of $\mu_P$ is called $(A,B,C)$-sub-Gaussian over $\cP$ if, for any $P\in \cP$,
\[
\forall t\in(0, A),\qquad \P\paren{|\muh-\mu_P|>B\sigma_P\sqrt{\frac{1+t}N}}\leqslant Ce^{-t}\enspace.
\]
\end{definition}
Of course, $(A,B,C)$-sub-Gaussian estimators with $A=+\infty$ are the most desirable.
%
Proposition~\ref{prop:GaussIsSuGauss} shows that the empirical mean is $(+\infty,\sqrt{2},1)$-sub-Gaussian over the class $\cP_{\text{gauss}}=\{\gauss(\mu,\sigma^2), \ \mu \in \R,\sigma^2>0\}$ and Corollary~\ref{cor:SGEM} shows that this result extends to the class of sub-Gaussian distributions.
As the empirical mean is satisfying on $\cP_{\text{gauss}}$, we may wonder if it is also the case on $\cP_2$.
The following proposition proves that this is unfortunately not true and that Chebyshev's inequality is sharp in general.
\begin{proposition}\label{prop:EmpMeanNotGood}\cite[Proposition 6.2]{MR3052407}
 For any $\sigma^2$ and $t>0$, there exists a distribution $P\in \cP_2$ with variance $\sigma^2_P=\sigma^2$ (and mean $\mu_P=0$) such that the empirical mean $\overline{X}_N=N^{-1}\sum_{i=1}^NX_i$ satisfies
 \[
 \P\paren{\overline{X}_N\geqslant t}=\P\paren{\overline{X}_N\leqslant -t}\geqslant \frac{\sigma^2}{2t^2N}\paren{1-\frac{\sigma^2}{t^2N^2}}^{N-1}\enspace.
 \]
\end{proposition}
\begin{remark}
Proposition~\ref{prop:EmpMeanNotGood} implies in particular that, for any value of $N$ and $t>1/4$, there exists a distribution $P=P_{N,t}$ such that
\[
\P\bigg(\overline{X}_N-\mu_P>\sigma\sqrt{\frac{2t}N}\bigg)\geqslant \frac{e^{-1/(2t)}}{4t}\geqslant \frac{1}{4e^2t}\enspace.
\]
In words, this means that, for any constants $B$ and $C$, there exists $A=f(B,C)$ such that  the empirical mean is not an $(A,B,C)$ sub-Gaussian estimator.
\end{remark}

\begin{proof}
 Consider the distribution taking values in $\{-Nt,0,Nt\}$ such that 
\[
\P(Nt)=\P(-Nt)=\frac{1-\P(0)}2=\frac{\sigma^2}{2N^2t^2}\enspace.
\]
This distribution is centered with variance $\sigma^2$. As this distribution is symmetric,
\[
\P\paren{\overline{X}_N\geqslant t}=\P\paren{\overline{X}_N\leqslant -t}\enspace.
\] 
Moreover, 
\begin{align*}
\P\paren{\overline{X}_N\geqslant t}&\geqslant \P\paren{\overline{X}_N=  t}\\
&\geqslant \P\paren{\exists! i\in\{1,\ldots,N\}: X_i=Nt,\; \forall j\ne i,\; X_j=0}\enspace. 
\end{align*}
It is clear that this last event has probability $\frac{\sigma^2}{2t^2N}\paren{1-\frac{\sigma^2}{t^2N^2}}^{N-1}$ as desired.
\end{proof}

\section{Level-dependent sub-Gaussian estimators}
Proposition~\ref{prop:EmpMeanNotGood} implies that, for any choice of $B$ and $C$, there exists a constant $f(B,C)$ such that, for any $N\geqslant 1$ and any $A\geqslant f(B,C)$, the empirical mean is not an $(A,B,C)$-sub-Gaussian estimator over $\cP_2$.
The question is therefore if there exist estimators $\muh$, constants $B$ and $C$ and a sequence $A_N\to\infty$ such that, for any $N\geqslant 1$, $\muh$ is a $(A_N,B,C)$ sub-Gaussian estimator over $\cP_2$.
The following result shows that, actually, this problem cannot be solved over $\cP_2$.

Let $\theta>0$ and let $\text{Po}_{\theta}$ denote the Poisson distribution such that
\[
\text{Po}_{\theta}(k)=\frac{\theta^k}{k!}e^{-\theta},\qquad \forall k\in \Z_+\enspace.
\]
For any $\theta>0$, the expectation and the variance of $\text{Po}_{\theta}$ are equal to $\theta$.
\begin{theorem}\label{thm:NoMultP2}
 Assume that $\cP$ contains all Poisson's distributions. Then, for any $(B,C)$, there exists $A=f(B,C)$ such that for any $N\geqslant 1$, there does not exist $(A,B,C)$-sub-Gaussian estimators over $\cP$.
\end{theorem}
%
%

\begin{proof}
Let $\muh$, $A,B,C$ such that $\muh$ is $(A,B,C)$-sub-Gaussian over $\cP$. 
 Let $\theta_1=\sqrt{1/N}$ and $\theta_2=\sqrt{R/N}$. Let $P_i=\text{Po}_{\theta_i^2}$, $\P_i=P_i^{\otimes N}$ for any $i\in\{1,2\}$. 
The sub-Gaussian hypothesis on $\muh$ implies that
\begin{gather*}
 \P_1\paren{|\muh-\theta_1^2|>B\theta_1\sqrt{\frac{1+t}N}}\leqslant Ce^{-t},\qquad \forall t\leqslant A\enspace,\\
 \P_2\paren{|\muh-\theta_2^2|>B\theta_2\sqrt{\frac{1+t}N}}\leqslant Ce^{-t},\qquad \forall t\leqslant A\enspace.
\end{gather*}
Denote by $\Omega=\{\sum_{i=1}^NX_i=R\}$. 
Define $h(x)=(x/2)\log (x/2)$. Applying Stirling's formula as $R\to\infty$ shows that there exists $R_0$ such that, for any $R\geqslant R_0$,
\begin{gather*}
\P_1(\Omega)=e^{-1}\frac{1}{R!}\sim \frac{e^{-R\log R-2R}}{\sqrt{2\pi R}}\geqslant e^{-h(R)} \enspace,\\
\P_2(\Omega)=e^{-R}\frac{R^R}{R!}\sim \frac{1}{\sqrt{2\pi R}}\geqslant \frac1{4\sqrt{R}} \enspace.
\end{gather*}
We deduce from these estimates that
\begin{gather*}
 \P_2\paren{\muh<\frac{R-B\sqrt{R(1+t)}}N|\Omega}\leqslant 4C\sqrt{R}e^{-t}\enspace.
\end{gather*}
If $A\geqslant \log(8C\sqrt{R})$, one can apply this inequality with $t=\log(8C\sqrt{R})$ to get 
\begin{gather*}
 \P_2\paren{\muh<\frac{R-B\sqrt{R(1+\log(8C\sqrt{R}))}}N|\Omega}\leqslant \frac{1}{2}\enspace.
\end{gather*}
This is equivalent to
\[
 \P_2\paren{\muh\geqslant\frac{R-B\sqrt{R(1+\log(8C\sqrt{R}))}}N|\Omega}\geqslant \frac{1}{2}\enspace.
\]
Now, we apply the following Poisson's trick
\[
\cD_1\bigg(X_1,\ldots,X_N\bigg|\sum_{i=1}^NX_i=R\bigg)=\cD_2\bigg(X_1,\ldots,X_N\bigg|\sum_{i=1}^NX_i=R\bigg)\enspace.
\]
In particular,
\[
 \P_1\paren{\muh\geqslant\frac{R-B\sqrt{R(1+\log(8C\sqrt{R}))}}N|\Omega}\geqslant \frac{1}{2}\enspace.
\]
Therefore, by the estimate on $\P_2(\Omega)$, 
\[
 \P_1\paren{\muh\geqslant\frac{R-B\sqrt{R(1+\log(8C\sqrt{R}))}}N}\geqslant \frac{e^{-h(R)}}{2}\enspace.
\]
Now, for any $R$ larger than a fixed $R_0(B,C)$, 
\[
R-B\sqrt{R(1+\log(8C\sqrt{R}))}\geqslant 1+B\sqrt{1+R^2/(2B^2)}\enspace.
\]
Pick $R\geqslant R_0(B,C)$ and $A\geqslant R^2/(2B^2)$, the sub-Gaussian property of $\muh$ applied with $t=R^2/(2B^2)$ yields
\[
 \P_1\paren{\muh\geqslant\frac{R-B\sqrt{R(1+\log(8C\sqrt{R}))}}N}\leqslant Ce^{-R^2/(2B^2)}\enspace.
\]
In other words, $Ce^{-R^2/(2B^2)}\geqslant e^{-h(R)}/2$ which is possible only if $R\leqslant R_1(B,C)$. 
In conclusion, the result is possible only if $R\leqslant R_0(B,C)\vee R_1(B,C)$ which means that an $(A,B,C)$ sub-Gaussian $\muh$ exists only if $A\leqslant R^2/(2B^2)\vee\log(8C\sqrt{R})$, for some $R\leqslant R_0(B,C)\vee R_1(B,C)$.
\end{proof}

Theorem~\ref{thm:NoMultP2} shows that the notion of sub-Gaussian estimators is a bit too constraining to work on $\cP_2$.
Hereafter, the following relaxation of the sub-Gaussian property will be used extensively.

\begin{definition}
Let $t\in (0,1)$. A level-dependent estimator $\muh_t$ of $\mu$ is a function of the data and $t$: $\muh_t=F(\cD_N,t)$, where $F$ is a measurable map $\R^{N+1}\to\R$.

The (level-dependent) estimator $\muh_t$ of $\mu_P$ is called $(B,C)$-sub-Gaussian at level $t$ over $\cP$ if, for any $P\in \cP$, 
\[
\P\paren{|\muh_t-\mu_P|>B\sigma_P\sqrt{\frac{1+t}N}}\leqslant Ce^{-t}\enspace.
\]
\end{definition}
The notion of level dependent estimator may seem surprising at first sight.
It is however a key concept in these notes.
The first reason is that one can build level-dependent estimators over the class $\cP_2$ up to levels $t\asymp N$ as we will see in the following section.

\section{Median-Of-Means estimators}
%
%
%
This section introduces a basic example of level-dependent sub-Gaussian estimators, called median-of-means estimators (MOM).
These estimators date back at least from the textbook \cite{MR702836} although they have been around before.
For example, a similar construction also appeared independently in \cite{MR762855}.
These estimators are used systematically in these notes to build robust extensions of ERM.

Let $K$ and $b$ such that $N=Kb$ and let $B_1,\ldots,B_K$ denote a partition of $\{1,\ldots,N\}$ into subsets of cardinality $b$.
For any $k\in \{1,\ldots,K\}$, let $P_{B_k}X=b^{-1}\sum_{i\in B_k}X_i$.
The MOM estimators of $\mu_P$ are defined by 
\[
\MOM{K}{X}\in\text{median}\set{P_{B_k}X,\; k\in\{1,\ldots,K\}}\enspace.
\]
The following result shows that $\MOM{K}{X}$ is a level dependent estimator over $\cP_2$ for a proper choice of $t\asymp K$.
\begin{proposition}\label{pro:MOMRBase}
For any $K$, $\epsilon>0$ and $P\in \cP_2$,
 \[
 \P\paren{|\MOM{K}{X}-\mu_P|>\epsilon}\leqslant e^{-2K(1/2-\sigma^2K/(N\epsilon^2))^2}\enspace.
 \]
It follows that, for any $\delta>0$, choosing $\epsilon=\sigma_P\sqrt{(2+\delta)K/N}$ yields
\[
\P\paren{|\MOM{K}{X}-\mu|>\sigma_P\sqrt{(2+\delta)\frac{K}N}}\leqslant e^{-\frac{\delta^2K}{2(2+\delta)^2}}\enspace.
\]
Choosing $\delta=2$ yields
\[
\P\paren{|\MOM{K}{X}-\mu|>2\sigma_P\sqrt{\frac{K}N}}\leqslant e^{-K/8}\enspace.
\]
$\MOM{K}{X}$ is a $(4\sqrt{2},1)$-sub-Gaussian estimator at level $K/8$. 
\end{proposition}
\begin{proof}
Fix $\epsilon>0$. 
The first analysis of MOM estimators is based on the remark that, if there are more than $K/2$ blocks $B_k$ such that $|P_{B_k}X-\mu_P|\leqslant \epsilon$, then $|\MOM{K}{X}-\mu_P|\leqslant \epsilon$. 
Formally,
\begin{align*}
\{|\MOM{K}{X}-\mu_P|\leqslant \epsilon\}&\supset \set{|\{k\in\{1,\ldots,K\} : |P_{B_k}X-\mu_P|\leqslant \epsilon\}|\geqslant \frac K2} \\
&=\bigg\{\sum_{k\in\{1,\ldots,K\}}{\bf 1}_{\{ |P_{B_k}X-\mu_P|> \epsilon\}}< \frac K2\bigg\}\enspace.
\end{align*}
Denote by $p_{\epsilon}=\P\paren{|P_{B_k}X-\mu_P|> \epsilon}$, $Y_k={\bf 1}_{\{ |P_{B_k}X-\mu_P|> \epsilon\}}-p_\epsilon$. This implies
\begin{align*}
 \P\paren{|\MOM{K}{X}-\mu_P|\leqslant \epsilon}\geqslant 1-\P\paren{\sum_{i=1}^KY_k\geqslant K(1/2-p_\epsilon)}\enspace.
\end{align*}
As $(Y_k)_{k=1,\ldots,K}$ are independent random variables bounded by $1$, by Hoeffding's inequality (recalled in \eqref{eq:HoeffdingIneq}),
\[
 \P\paren{|\MOM{K}{X}-\mu_P|\leqslant \epsilon}\geqslant 1-e^{2(1/2-p_\epsilon)^2K}\enspace.
\]
The proof is concluded since, by Chebishev's inequality,
\[
p_{\epsilon}\leqslant \frac{\sigma^2_PK}{N\epsilon^2}\enspace.
\]
\end{proof}

The previous elementary result shows that median-of-means estimators are level-dependent sub-Gaussian estimators. 
It is based on a very basic first argument that easily generalise to other frameworks.
However, the result can be refined using Gaussian approximation and slightly stronger hypotheses.
The following result is due to Minsker and Strawn \cite{Minsker-Strawn2017}.
It shows that, under slightly stronger assumptions on $P$, MOM estimators are also sub-Gaussian estimators (not level-dependent). Let 
\[
\cP_3^\gamma=\{P\in \cP_2 : P[|X|^3]<\infty\text{ and }\sigma^{-3}\E[|X-\mu|^3]\leqslant \gamma\}\enspace.
\]

\begin{theorem}\label{thm:Minsker-Strawn}
For any $P\in \cP^\gamma_3$ and any $t>0$ such that $0.5\gamma\sqrt{K/N}+\sqrt{t/2K}\leqslant 1/3$, 
\[
\P\paren{|\MOM{K}{X}-\mu_P|>\sigma_P\paren{1.5\gamma\frac{ K}N+3\sqrt{\frac{t}{2N}}}}\leqslant 4e^{-t}\enspace.
\]
\end{theorem}
\begin{remark}
 As long as $K\leqslant \sqrt{N}$, this result implies that 
 \[
 \forall t\lesssim \sqrt{N},\qquad \P\paren{|\MOM{\sqrt{N}}{X}-\mu_P|>\sigma_P\paren{(1.5\gamma+3)\sqrt{\frac{1+t}{2N}}}}\leqslant 4e^{-t}\enspace.
 \]
 In other words, there exists a constant $C(\gamma)$ such that the estimator $\MOM{\sqrt{N}}{X}$ is $(O(\sqrt{N}), C(\gamma),4)$-sub-Gaussian over $\cP_3^\gamma$.
 Theorem~\ref{thm:Minsker-Strawn} is not in contradiction with Theorem~\ref{thm:NoMultP2} since $\cP_3^\gamma$ does not contain all Poisson's distributions.
\end{remark}
\begin{proof}
Denote by 
\begin{gather*}
Q^{(b)}_K(t)=\frac1K\sum_{k=1}^K{\bf 1}_{\set{\sqrt{b}\frac{P_{B_k}X-\mu}{\sigma_P}> t}}\enspace.
\end{gather*}
The goal is to find deterministic quantities $t_-$ and $t_+$ such that, the event $\Omega_{t_-,t_+}$ has large probability, where
\[
\Omega_{t_-,t_+}=\bigg\{1-Q^{(b)}_K(t_+)\geqslant \frac12,\quad Q^{(b)}_K(t_-)\geqslant \frac12\bigg\}\enspace.
\]
Indeed, on $\Omega_{t_-,t_+}$, it holds
\[
\text{median}\paren{\sqrt{b}\frac{P_{B_k}X-\mu}{\sigma_P},k\in\{1,\ldots,K\}}\in[t_-,t_+]\enspace.
\]
By homogeneity and translation invariance of the median, this implies that, on $\Omega_{t_-,t_+}$,
\begin{equation}\label{eq:BoundMOMOnOmegat}
 \sigma_P\frac{t_-}{\sqrt{b}}\leqslant \MOM{K}{X}-\mu_P\leqslant \sigma_P\frac{t_+}{\sqrt{b}}\enspace.
\end{equation}
Fix $t\in \R$, to bound $Q^{(b)}_K(t)$, let us first introduce
\[
Q^{(b)}(t)=\P\paren{\sqrt{b}\frac{P_{B_1}X-\mu}{\sigma}> t}\enspace.
\]
By Hoeffding's inequality,
\begin{equation}\label{eq:Hoeffding1}
\forall x>0,\qquad \P\paren{|Q^{(b)}_K(t)-Q^{(b)}(t)|>\sqrt{\frac{x}{2K}}}\leqslant 2e^{-x}\enspace. 
\end{equation}
Hence, for any $t_-, t_+$ in $\R$ such that 
\begin{equation}\label{cdt1}
Q^{(b)}(t_+)\leqslant 1/2-\sqrt{\frac{x}{2K}},\qquad  Q^{(b)}(t_-)\geqslant 1/2+\sqrt{\frac{x}{2K}}\enspace.
\end{equation}
A union bound in \eqref{eq:Hoeffding1} shows that
\[
\P\paren{\Omega_{t_-,t_+}}\geqslant 1-4e^{-x}\enspace.
\]
Therefore, \eqref{eq:BoundMOMOnOmegat} holds for these values of $t_-,t_+$ with probability at least $1-4e^{-x}$.
To evaluate $t_-$, $t_+$ in \eqref{cdt1}, introduce now
\[
Q(t)=1-\Phi(t)=\int_{t}^{+\infty}e^{-x^2/2}\frac{\rmd x}{\sqrt{2\pi}}\enspace.
\]
By Berry-Essen theorem, 
\begin{equation}\label{eq:BEThm}
\big\|Q^{(b)}-Q\big\|_{\infty}\leqslant 0.5\gamma\sqrt{\frac K{N}}\enspace. 
\end{equation}
Therefore, \eqref{cdt1} is fulfilled if $t_-$ and $t_+$ satisfy
\[
Q(t_+)\leqslant 1/2-\sqrt{\frac{x}{2K}}-0.5\gamma\sqrt{\frac K{N}},\qquad  Q(t_-)\geqslant 1/2+\sqrt{\frac{x}{2K}}+0.5\gamma\sqrt{\frac K{N}}\enspace.
\]
Using the mean valued theorem, for any $t\in (0,\sqrt{\log(9/2\pi)})$, $|Q'(t)|=e^{-t^2/2}/\sqrt{2\pi}\geqslant 1/3$, therefore
\[
Q(t)\leqslant Q(0)-\frac{t}3= \frac12-\frac{t}{3}\enspace.
\]
Therefore, \eqref{cdt1} is fulfilled if
\[
t_+=3\sqrt{\frac{x}{2K}}+1.5\gamma\sqrt{\frac K{N}},\qquad t_-=-t_+\enspace.
\]
\end{proof}


Proposition~\ref{pro:MOMRBase} shows that, when $K\geqslant8t$, $\MOM{K}{X}$ is a level-dependent sub-Gaussian estimators at level $t$. 
In particular, as $K$ can be equal to $N$, there exist level-dependent sub-Gaussian estimators at levels $t$ that might be of order $N$. The following result shows that this rate cannot be improved in general. 
Let $\lambda\in \R$ and let $\text{La}_{\lambda}$ denote the Laplace distribution with density
\[
f_\lambda(x)=\frac12e^{-|x-\lambda|},\qquad \forall x\in \R\enspace.
\]
This distribution has expectation $\lambda$ and variance $2$.

\begin{proposition}\label{pro:Nisgood}
 Assume that $\cP$ contains all Laplace distributions. Then, for any $B$ and $C$ there exists a constant $f(B,C)$ such that, for any $N\geqslant 1$, there does not exists a $(B,C)$ level-dependent sub-Gaussian estimator at level $t\geqslant f(B,C)N$.
\end{proposition}
\begin{proof}
Proceed by contradiction and let $\muh_t$ denote such an estimator. Let $P_1=\text{La}_0$ and $P_2=\text{La}_{\lambda}$ and $t=N\lambda^2/(4B^2\sigma_{P_i}^2)=N\lambda^2/(8B^2)$. By the triangular inequality,
\[
f_2(x_1,\ldots,x_N)\leqslant e^{\lambda N}f_1(x_1,\ldots,x_N)\enspace.
\]
Therefore 
\[
\P_2\bigg(\muh_t>\frac{\lambda}2\bigg)\leqslant e^{\lambda N}\P_1\bigg(\muh_t>\frac{\lambda}2\bigg)\enspace.
\]
Now $\lambda=P_2X$, so, by the sub-Gaussian property of $\muh_t$,
\[
\P_2\bigg(\muh_t>\frac{\lambda}2\bigg)=\P_2\bigg(\muh_t-\mu_{P_2}>B\sigma_{P_2}\sqrt{\frac{t}N}\bigg)\geqslant 1-Ce^{-t}=1-Ce^{-N\lambda^2/(8B^2)}\enspace.
\]
Likewise, the sub-Gaussian property of $\muh_t$ yields
\[
\P_1\bigg(\muh_t>\frac{\lambda}2\bigg)\leqslant\P_1\bigg(|\muh_t-\mu_{P_1}|<B\sigma_{P_1}\sqrt{\frac{t}N}\bigg)\leqslant Ce^{-t}=Ce^{-N\lambda^2/(8B^2)}\enspace.
\]
Overall, this yields
\[
1-Ce^{-N\lambda^2/(8B^2)}\leqslant Ce^{-N\lambda(1-\lambda)/(8B^2)}\enspace.
\]
Whatever the value of $N\geqslant 1$, this relationship is absurd for any $\lambda\geqslant \lambda_0(B,C)$ which implies that the existence of $\muh_t$ is absurd for any $t\geqslant N\lambda_0(B,C)^2/(8B^2)$.
\end{proof}

\section{$M$-estimators}

This section introduces an alternative to MOM estimators which is extremely popular in robust statistics. 
These estimators are known as $M$-estimators. 
The asymptotic of these estimators is well known and an overview of these results can be found in  \cite{MR2488795}. 
Recall that
\[
\mu\in\argmin_{\nu\in\R}\E[(X-\nu)^2],\quad P_NX\in\argmin_{\nu\in \R}\sum_{i=1}^N(X_i-\nu)^2\enspace.
\]
The principle of $M$-estimation is to replace the function $x\mapsto x^2$ in this formulation by another function $\Psi$ and build 
\[
\muh\in\argmin_{\nu\in\R}\sum_{i=1}^N\Psi(X_i-\nu)\enspace.
\]
The most famous example of $M$-estimator used to estimate $\mu_P$ is given by Huber's function 
\[
\Psi_c(x)=
\begin{cases}
 \frac{x^2}{2}&\text{if} |x|\leqslant c\\
 c|x|-\frac{c^2}{2}&\text{if} |x|> c
\end{cases}
\enspace.
\]
This function is continuously differentiable, with derivative $\psi_c(x)=x{\bf 1}_{|x|\leqslant c}+c\text{sign}(x){\bf 1}_{|x|>c}$, $\Psi_c$ is convex and $c$-Lipshitz. 
Huber's estimators interpolate between the empirical mean that would be obtained for $\Psi=x^2$ and the empirical median that would be obtained for $\Psi=|x|$.
In this section, we study the Huber estimators defined either by
\begin{equation}\label{eq:DefHuber}
 \muh_c\in\argmin_{\nu\in \R}\sum_{i=1}^N\Psi_c(X_i-\nu)
\end{equation}
or as a solution of the equation 
\begin{equation}\label{eq:DefHuber2}
 P_N\psi_c(\cdot-\nu)=\frac1N\sum_{i=1}^N\psi_c(X_i-\nu)=0\enspace.
\end{equation}
Using the formulation \eqref{eq:DefHuber2}, it is clear that these estimators are particular instances of the following larger family of $Z$-estimators introduced by \cite{MR3052407}.
Let $\psi:\R\to\R$ denote any continuous and non-decreasing function such that
\[
-C\log\bigg(1-x+\frac{x^2}2\bigg)\leqslant \psi(x)\leqslant C\log\bigg(1+x+\frac{x^2}2\bigg)\enspace.
\]
Let $\alpha>0$ and define $\muh_\alpha$ as any solution of the equation
\begin{equation}\label{def:EstCato}
\sum_{i=1}^N\psi[\alpha(X_i-\mu)]=0\enspace. 
\end{equation}
The following result establishes the sub-Gaussian behavior of these estimators.
\begin{theorem}
 Pick $\alpha=\sigma_P^{-1}\sqrt{t/N}$, the estimator $\muh_{\alpha}$ defined in \eqref{def:EstCato} satisfies
 \[
 \P\paren{|\muh_\alpha-\mu|>\sigma_P\sqrt{\frac{2}{1-2\epsilon}\frac tN}}\leqslant 2e^{-t}\enspace,
 \]
 for any $\epsilon\in (0,1/2)$ such that
\begin{equation}\label{cdt:alpha}
 \frac{\alpha^2\sigma_P^2}2+\frac{t}{N}=\frac{3t}{2N}\leqslant \epsilon\enspace.
\end{equation}
\end{theorem}
\begin{remark}
 As $N\to\infty$, the constant $\epsilon$ can be chosen as small as desired in \eqref{cdt:alpha}, so Catoni's construction shows that almost optimal constant $\sqrt{2}$ can be achieved by $t$-dependent sub-Gaussian estimators up to levels of order $N$.

Besides $t$, Catoni's estimators are sub-Gaussian if $\sigma_P$ is \emph{known}, that is it can be used on the classes $\cP_2^{\sigma^2,L\sigma^2}$ of distributions $P\in\cP_2$ with variance $\sigma_P\in [\sigma^2,L\sigma^2]$. It yields optimal constants if $L=1$. Otherwise, choosing for example $\alpha=\sqrt{t/N}$, Catoni's estimators are \emph{weakly} sub-Gaussian in the sense that the variance $\sigma_P$ in Definition~\ref{def:subgauss} is replaced by a larger quantity, here $1+\sigma_P^2$.
\end{remark}
\begin{proof}
 All along the proof, denote, for any $\mu\in\R$, by 
 \[
 Z_\alpha(\mu)=\frac{1}{N\alpha}\sum_{i=1}^N\psi[\alpha(X_i-\mu)]\enspace.
 \]
 First, by independence of $X_i$, for any $s\in\{-1,1\}$,
\[
 \E\bigg[e^{s\alpha N Z_{\alpha}(\mu)/C}\bigg]\leqslant \prod_{i=1}^N\E\bigg[e^{s\psi[\alpha(X_i-\mu)/C}\bigg]\enspace.
\] 
Second, the definition of $\psi$ implies that, for any $s\in\{-1,1\}$,
\[
 \E\bigg[e^{s\alpha N Z_{\alpha}(\mu)/C}\bigg]\leqslant \prod_{i=1}^N\bigg(1+s\alpha(\mu_P-\mu)+\frac{\alpha^2}2[\sigma_P^2+(\mu_P-\mu)^2]\bigg)\enspace.
\] 
By the inequality $1+x\leqslant e^x$, it follows that, for any $s\in\{-1,1\}$,
\begin{equation}\label{lem:Laplace}
 \E\bigg[e^{s\alpha N Z_{\alpha}(\mu)/C}\bigg]\leqslant e^{ N\big[s\alpha(\mu_P-\mu)+\frac{\alpha^2}2[\sigma_P^2+(\mu_P-\mu)^2]\big]}\enspace.
\end{equation}
Fix $t>0$ and, for any $\mu\in\R$, let 
\begin{gather*}
U_\alpha(\mu,t)=C(\mu_P-\mu)+\frac{\alpha C}{2}\big[\sigma^2_P+(\mu_P-\mu)^2\big]+\frac{Ct}{N\alpha}\enspace,\\
L_\alpha(\mu,t)=C(\mu_P-\mu)-\frac{\alpha C}{2}\big[\sigma^2_P+(\mu_P-\mu)^2\big]-\frac{Ct}{N\alpha}\enspace.
\end{gather*}
Fix $t>0$.
Then, using the inequality $\P(sZ_{\alpha}(\mu)>u)\leqslant e^{-u}\E[e^{N\alpha sZ_{\alpha}(\mu)}]$ respectively with $u=U_\alpha(\mu,t)$, $s=1$ and $u=L_\alpha(\mu,t)$, $s=-1$ yields
\begin{equation}\label{eq:ProbBound}
 \P\paren{Z_{\alpha}(\mu)<U_\alpha(\mu,t)}\geqslant 1-e^{-t},\qquad \P\paren{Z_{\alpha}(\mu)>L_\alpha(\mu,t)}\geqslant 1-e^{-t}\enspace.
\end{equation}
 
By \eqref{cdt:alpha}, the smallest solution $\mu_+$ of the equation $U_\alpha(\mu,t)=0$ and the largest solution $\mu_-$ of $L_\alpha(\mu,t)=0$ satisfy
\begin{gather*}
\mu_+\leqslant \mu+\frac1{\sqrt{1-2\epsilon}}\bigg(\frac{\alpha\sigma_P^2}2+\frac{t}{\alpha N}\bigg)\enspace,\\
\mu_- \geqslant \mu-\frac1{\sqrt{1-2\epsilon}}\bigg(\frac{\alpha\sigma_P^2}2+\frac{t}{\alpha N}\bigg)\enspace.
\end{gather*}
Consider the event 
\[
\Omega=\big\{L_\alpha(\mu_-,t)<Z_{\alpha}(\mu_-),\; Z_{\alpha}(\mu_+)<U_\alpha(\mu_+,t)\big\}\enspace.
\]
By \eqref{eq:ProbBound}, $\P(\Omega)\geqslant 1-2e^{-t}$. 
As the map $\mu\mapsto Z_{\alpha}(\mu)$ is non-increasing, on $\Omega$,
$Z_{\alpha}(\mu_+)<U_\alpha(\mu_+,t)=0=Z_{\alpha}(\muh)$, so $\muh\leqslant \mu_+$. Likewise $\muh\geqslant \mu_-$.
It follows that
\[
\P\paren{\mu_-<\muh<\mu_+}\geqslant \P(\Omega)\geqslant 1-2e^{-t}\enspace.
\]
This concludes the proof.
\end{proof}

\section{Level free sub-Gaussian estimators}
Theorem~\ref{thm:Minsker-Strawn} showed that $\MOM{\sqrt{N}}{X}$ is a level free $(A_N,B(\gamma),C)$-sub-Gaussian estimator over $\cP_3^\gamma$ with $A_N$ of order $\sqrt{N}$. 
The purpose here is to present a method to derive level free estimators from level dependent ones, provided, for example that informations on the variance are available.
The central tool is due to Lepski.
\begin{theorem}\label{thm:LepMeth1}
 Assume that, for any $K$ in a finite set $\cK$, there exists a confidence interval $\Ih_K$ such that 
\begin{itemize}
 \item[(i)] for any $K$ and $K'$ in $\cK$ such that $K\leqslant K'$, $|\Ih_K|\leqslant |\Ih_{K'}|$,
 \item[(ii)] $\P\paren{\mu\in \Ih_K}\geqslant 1-\alpha_K$.
\end{itemize}
Then, if one defines 
\[
\Kh=\min\set{K\in \cK : \cap_{J\in \cK, J\geqslant K}\Ih_J\ne \emptyset},\qquad \muh\in \Ih_{\widehat{K}}\enspace, 
\]
we have
\[
\forall K\in \cK,\qquad \P\paren{|\muh-\mu|>2|\Ih_K|}\leqslant \sum_{J\in \cK, J\geqslant K}\alpha_J\enspace.
\]
\end{theorem}
\begin{proof}
 For any $K\in \cK$, denote by $\cK_K=\{J\in \cK: J\geqslant K\}$.
 Fix $K\in \cK$ and consider the event $\Omega=\{\mu\in \cap_{J\in \cK_K}\Ih_J\}$. A union bound grants that 
 \[
 \P(\Omega)\geqslant 1-\sum_{J\in\cK_K}\alpha_J\enspace.
 \]
 On $\Omega$, $\cap_{J\in \cK_K}\Ih_J\ne \emptyset$, therefore, $\Kh\leqslant K$ and there exists $\mu_0\in \cap_{J\in \cK_{\widehat{K}}}\Ih_J$.
 As $\mu_0,\muh\in \Ih_{\widehat{K}}$, $|\mu_0-\mu|\leqslant |\Ih_{\widehat{K}}|$ and as $\widehat{K}\leqslant K$, $|\Ih_{\widehat{K}}|\leqslant |\Ih_{K}|$, so $|\mu_0-\muh|\leqslant |\Ih_{K}|$.
 Moreover, as $\widehat{K}\leqslant K$ and $\mu_0\in \cap_{J\in \cK_{\widehat{K}}}\Ih_J$, $\mu_0\in \cap_{J\in \cK_{K}}\Ih_J$ and as $\mu\in \cap_{J\in \cK_{K}}\Ih_J$, $|\mu_0-\mu|\leqslant |\Ih_K|$. 
 Hence,
 \[
 |\muh-\mu|\leqslant |\muh-\mu_0|+|\mu_0-\mu|\leqslant 2|\Ih_K|\enspace.
 \]
\end{proof}

We are now in position to prove the result. 
\begin{theorem}\label{thm:ABCSub}
For any $\sigma^2>0$ and $L\geqslant 1$, there exists an $((N/2-1)/8,8\sqrt{2L},9)$-sub-Gaussian estimator on $\cP_2^{[\sigma^2,L\sigma^2]}$.
\end{theorem}
\begin{proof}
For any $K=1,\ldots,N/2$, let $b=\PEInf{N/K}$ and let $\MOM{K}{X}$ denote the MOM estimators based on $X_1,\ldots,X_{bK}$.
Define, for any $K\in \{1,\ldots,N/2\}$, the intervals
\[
\Ih_K=\cro{\MOM{K}{X}\pm2\sigma\sqrt{L\frac{K}N}}\supset \cro{\MOM{K}{X}\pm2\sigma_P\sqrt{\frac{K}N}}\enspace.
\]
Proposition~\ref{pro:MOMRBase} shows that the intervals $\Ih_K$ satisfy Condition (i) of Theorem~\ref{thm:LepMeth1} with $|\Ih_K|=2\sigma\sqrt{LK/N}$ and Condition (ii) with $\alpha_K=e^{-K/8}$. 
It follows that, if 
\[
\Kh=\min\set{K\in \{1,\ldots,N/2\} : \cap_{J=K}^{N/2}\Ih_J\ne \emptyset},\qquad \muh=\MOM{\Kh}{X}\enspace, 
\]
the estimator $\muh$ satisfies, for any $K\in\{1,\ldots.N/2\}$,
\[
\P\paren{|\muh-\mu_P|>4\sigma\sqrt{\frac{LK}{N}}}\leqslant \sum_{J=K}^{+\infty}e^{-J/8}\leqslant \frac{e^{-K/8}}{1-e^{-1/8}}\enspace.
\]
Fix $x\in(0,(N/2-1)/8)$ and choose $K=\PEInf{8x}+1$. It follows from this result that
\[
\P\paren{|\muh-\mu_P|>8\sigma_P\sqrt{\frac{2L(1+x)}{N}}}\leqslant 9e^{-x}\enspace.
\]
\end{proof}


\chapter{Concentration/deviation inequalities}\label{Chap:ConcIn}
Concentration inequalities evaluate the probability that random variables deviate from their expectation by more than a given threshold.
They are natural tools to show deviation properties of estimators.
They have been widely used in statistics since the 1990's and their introduction for model selection by Birg\'e and Massart \cite{MR1240719}. 
This chapter presents useful concentration inequalities for the following chapters.
We briefly present the entropy method and recall sufficient results to establish Bousquet's version of Talagrand's concentration inequality for suprema of empirical processes.
All the material of Sections~\ref{Sec:EntMeth} and~\ref{Sec:TalIneq} is borrowed from \cite{BouLugMass13} that the interested reader is invited to read to learn much more on concentration inequalities.
Section~\ref{Sec:PacBayes} presents a PAC-Bayesian inequality that will be used to analyse $M$-estimators for multivariate mean estimation.
This result is borrowed from \cite{CatGiu2017} where PAC-Bayesian approaches are developed in various other learning problems.
Finally, Section~\ref{sec:DevMOMProc} presents the result that will be the most useful in these notes, which is a deviation result for suprema of MOM processes.
This result is obtained using the small ball approach, following arguments originally introduced in \cite{LugosiMendelson2016}.

All along the chapter, $X=(X_1,\ldots,X_N)$ denotes a vector of independent random variables taking values in a measurable space $\cX$. 
For any $i\in \{1,\ldots,N\}$, $X^{(i)}=(X_1,\ldots,X_{i-1},X_{i+1},\ldots,X_N)$ and $\E^{(i)}$ denote expectation conditionally on $X^{(i)}$. 
The function $\Phi:x\mapsto x\log(x)$ for any $x>0$ is extended by continuity $\Phi(0)=0$.
For any positive random variable $Y$ such that $\E[\Phi(Y)]<\infty$, the entropy of $Y$ is defined by $\Ent(Y)=\E[\Phi(Y)]-\Phi(\E[Y])$.
The conditional entropies are defined, for any $i\in\{1,\ldots,N\}$ by $\Ent^{(i)}(Y)=\E^{(i)}[\Phi(Y)]-\Phi(\E^{(i)}[Y])$.
$f$ denotes a measurable map $\cX^n\to [0,+\infty)$ and $Z=f(X)=f(X_1,\ldots,X_N)$.

\section{The entropy method}\label{Sec:EntMeth}
The entropy method is a series of steps introduced by Ledoux \cite{MR1849347} that allows to establish concentration inequalities for $Z=f(X)$ around its expectation $\E[Z]$. 
The starting point is the Chernoff bound.
Assume that $Z\leqslant 1$ so, for any $s>0$, the log Laplace-transform $\psi(s)=\log(\E[e^{s(Z-\E[Z])}])$ is well defined.
For any $t>0$, by Markov's inequality, it holds that
\[
\forall s>0,\qquad \P(Z-\E[Z]>t)=\P(e^{s(Z-\E[Z])}>e^{st})\leqslant e^{-st+\psi(s)}\enspace.
\]
Introduce the Fenchel-Legendre transform of $Z$, $\psi^*(t)=\sup_{s>0}\{st-\psi(s)\}$.
Optimizing over $s>0$ in the previous bound shows the Chernoff bound
\[
\forall t>0,\qquad \P(Z-\E[Z]>t)\leqslant e^{-\psi^*(t)}\enspace.
\]
Chernoff's bound shows that one can bound the deviation probabilities of $Z-\E[Z]$ by bounding from bellow the Fenchel-Legendre transform $\psi^*(t)$ of $Z$, which can be done by bounding from above the log-Laplace transform $\psi(s)$ of $Z$.
As important examples, basic analysis allows to check the following result.
\begin{lemma}\label{lem:BoundsFLT}
Let $\sigma>0$. The random variable $Z$ is called $\sigma$-sub-Gaussian if $\psi(s)\leqslant s^2\sigma^2/2$. If $Z$ is $\sigma$-sub-Gaussian, $\psi^*(t)\geqslant t^2/(2\sigma^2)$. In particular,
\[
\forall t>0,\qquad \P(Z-\E[Z]>t)\leqslant e^{-t^2/(2\sigma^2)}\enspace.
\]
Let $\nu>0$, $\phi(s)=e^s-1-s$ and $h(t)=(1+t)\log(1+t)-t$. The random variable $Z$ is called $\nu$-sub-Poissonian if $\psi(s)\leqslant \nu\phi(s)$. If $Z$ is $\nu$-sub-Poissonian, $\phi^*(t)\geqslant \nu h(t/\nu)$. In particular, 
\[
\forall t>0,\qquad \P(Z-\E[Z]>t)\leqslant e^{-\nu h(t/\nu)}\enspace.
\]
\end{lemma}
Thanks to Chernoff's bound, concentration inequalities follow from upper bounds on $\psi(s)$.
The idea of the entropy method is to obtain these bound by bounding from above the entropy of $e^{s(Z-\E[Z])}$. 
The method can be summarized in the following lemma.
\begin{lemma}\label{lem:EntMeth}
 The entropy satisfies
\begin{equation}\label{eq:Ent}
 \Ent(e^{sZ})=\E[e^{sZ}](s\psi'(s)-\psi(s))\enspace. 
\end{equation}
Therefore, if there exists a function $g$ such that
\begin{equation}\label{eq:CdtEnt}
 \Ent(e^{sZ})\leqslant g(s)\E[e^{sZ}]\enspace,
\end{equation}
then, the log-Laplace transform of $Z$ satisfies
\begin{equation}\label{eq:UBpsi}
 \psi(s)\leqslant s\int_0^s\frac{g(t)}{t^2}\rmd t\enspace.
\end{equation}
For example, if \eqref{eq:CdtEnt} holds with $g(s)= \sigma^2s^2/2$, then $\psi(s)\leqslant s^2\sigma^2/2$, so $Z$ is $\sigma$-sub-Gaussian.
\end{lemma}
\begin{proof}
Notice that $\Ent(e^{s(Z-\E[Z])})=\E[e^{-s\E[Z]}]\Ent(e^{sZ})$.
Thus, Equation~\eqref{eq:Ent} is equivalent to
\[
\Ent(e^{s(Z-\E[Z])})=\E[e^{s(Z-\E[Z])}](s\psi'(s)-\psi(s))\enspace.
\]
As 
\[
\psi'(s)=\frac{\E\big[(Z-\E[Z])e^{s(Z-\E[Z])}\big]}{\E[e^{s(Z-\E[Z])}]}\enspace,
\]
we have
\begin{align*}
\Ent(e^{s(Z-\E[Z])})&=\E[e^{s(Z-\E[Z])}\log(e^{s(Z-\E[Z])})]-\E[e^{s(Z-\E[Z])}]\log(\E[e^{s(Z-\E[Z])}])\\
&=\E[e^{s(Z-\E[Z])}]\bigg(s\frac{\E\big[(Z-\E[Z])e^{s(Z-\E[Z])}\big]}{\E[e^{s(Z-\E[Z])}]}-\psi(s)\bigg)\\
&=\E[e^{s(Z-\E[Z])}](s\psi'(s)-\psi(s))\enspace. 
\end{align*}
This shows Equation~\eqref{eq:Ent}.

The entropy condition \eqref{eq:CdtEnt} is equivalent to
\begin{equation}\label{eq:CdtEnt2}
 \Ent(e^{s(Z-\E[Z])})\leqslant g(s)\E[e^{s(Z-\E[Z])}]\enspace.
\end{equation}
Under this condition, the function $\psi$ satisfies the following differential inequality
\[
s\psi'(s)-\psi(s)\leqslant g(s)\enspace.
\]
Dividing by $s^2$ on both sides shows that
\[
\bigg(\frac{\psi(s)}{s}\bigg)'\leqslant \frac{g(s)}{s^2}\enspace.
\]
As $\psi(0)=\psi'(0)=0$, the function $u:s\mapsto \psi(s)/s$ can be extended continuously in $0$ by defining $u(0)=0$ and the previous inequality implies that
\begin{equation*}
 \psi(s)\leqslant s\int_0^s\frac{g(t)}{t^2}\rmd t\enspace.
\end{equation*}
\end{proof}
The Entropy lemma is well known in the sub-Gaussian case where it is referred to as Herbst's argument, which is both simple and elegant while surprisingly powerful.
The entropy method (Lemma~\ref{lem:EntMeth}) shows that bounding the entropy from above \emph{can be useful}.
The success of the method comes from the fact that \emph{it is actually possible} to obtain such upper bounds.
An important reason is the sub-additivity property of the entropy which allows to bound the entropy of functions depending on only one variable $X_i$.
This property is shown in the following section.

\subsection{Sub-additivity of the entropy}
To prove the sub-additivity property, we need a first variational formula for the entropy.
\begin{theorem}[Duality formula of entropy]\label{thm:DualFormEnt}
 Let $Y$ denote a positive random variable such that $\E[\Phi(Y)]<\infty$ and let $\cU$ denote the set of real valued random variables $U$ such that $\E[e^U]=1$. Then
\begin{equation}\label{eq:DualFormEnt1}
 \Ent(Y)=\sup_{U\in \cU}\E[UY]\enspace.
\end{equation}
 Equivalently, let $\cT$ denote the set of non negative and integrable random variables, then
 \[
 \Ent(Y)=\sup_{T\in \cT}\E[Y(\log (T)-\log(\E[T]))]\enspace.
 \]
\end{theorem}
\begin{proof}
The second part being a direct consequence of the first one, it is sufficient to show the first part.
Let $U\in \cU$, then 
\begin{align*}
\Ent(Y)-\E[UY]&=\E[Ye^{-U}\log( Ye^{-U}) e^{U}]-\E[Ye^{-U} e^U]\log(\E[Ye^{-U} e^U])\\
 &=\E[\Phi(Ye^{-U})e^{U}]-\Phi(\E[Ye^{-U} e^U])\enspace.
\end{align*}
If $P'$ denotes the measure such that $P'({\rm d}u)=e^{u}P({\rm d} u)$ (note that this is a probability measure), then $\Ent(Y)-\E[UY]$ is the entropy of $Ye^{-U}$ with respect to the measure $P'$. 
Hence,  $\Ent(Y)-\E[UY]\geqslant 0$, so the right-hand side of \eqref{eq:DualFormEnt1} is smaller than the left-hand side.

Conversely, if $U=\log (Y)-\log(\E[Y])$, then $\E[e^U]=1$ so $U\in \cU$ and $\E[UY]=\Ent(Y)$. 
This proves the second inequality in \eqref{eq:DualFormEnt1} and therefore the theorem.
\end{proof}

The Duality formula is used to prove the sub-additivity property. 
The idea is to bound the entropy $\Ent(Z)$ of any function $Z=f(X)$ by the entropies of ``simpler" functions depending on a single variable $X_i$ only.
Recall that $X^{(i)}=(X_1,\ldots,X_{i-1},X_{i+1},\ldots,X_N)$, $\E^{(i)}=\E[\cdot|X^{(i)}]$ and $\Ent^{(i)}(Z)=\E^{(i)}[\Phi(Z)]-\Phi( \E^{(i)}[Z])$.
By conditioning on $X^{(i)}$, $\Ent^{(i)}(Z)$ is therefore the entropy of $Z$ with respect to $X_i$ only, while $X^{(i)}$ is left fixed.
The sub-additivity property bounds the entropy $\Ent(Z)$ from above using the simpler entropies $\Ent^{(i)}(Z)$.

\begin{theorem}\label{thm:SubAddEnt}[Sub-additivity of entropy]
If $Z>0$, then
 \[
 \Ent(Z)\leqslant\E\bigg[\sum_{i=1}^N\Ent^{(i)}(Z)\bigg]\enspace.
 \]
\end{theorem}
\begin{proof}
 Introduce $\E_i=\E[\cdot|X_1,\ldots,X_i]$, $\E_0=\E$. As $\E_N[Z]=Z$, it holds 
 \[
\log (Z)-\log(\E[Z])= \sum_{i=1}^N(\log(\E_i[Z])-\log(\E_{i-1}[Z]))\enspace,
 \]
 hence
\begin{equation}\label{eq:Subadd1}
 Z(\log (Z)-\log(\E[Z]))=\sum_{i=1}^NZ(\log(\E_i[Z])-\log(\E_{i-1}[Z]))\enspace.
\end{equation}
 Now by independence of $X_i$ and $X_1,\ldots,X_{i-1}$,
\begin{equation}\label{eq:Subadd2}
 \E^{(i)}[\E_i[Z]]=\E_{i-1}[Z]\enspace.
\end{equation}
Plugging \eqref{eq:Subadd2} into \eqref{eq:Subadd1} yields
\begin{equation}\label{eq:Subadd3}
 Z(\log (Z)-\log(\E[Z]))=\sum_{i=1}^NZ(\log(\E_i[Z])-\log(\E^{(i)}[\E_i[Z]]))\enspace.
\end{equation}
The second duality formula in Theorem~\ref{thm:DualFormEnt} applied conditionally on $X^{(i)}$ with $Y=Z$ and $T=\E_i[Z]$ implies that
\begin{equation}\label{eq:Subadd4}
 \E^{(i)}[Z(\log(\E_i[Z])-\log(\E^{(i)}[\E_i[Z]]))]\leqslant \Ent^{(i)}(Z)\enspace.
\end{equation}
Therefore, taking expectation in \eqref{eq:Subadd3} and using \eqref{eq:Subadd4} shows that
\begin{align*}
 \Ent(Y)&=\E\bigg[\sum_{i=1}^N\E^{(i)}[Z(\log(\E_i[Z])-\log(\E^{(i)}[\E_i[Z]]))]\bigg]\\
 &\leqslant \E\bigg[\sum_{i=1}^N\Ent^{(i)}(Z)\bigg]\enspace.
\end{align*}
This proves the theorem.
\end{proof}

\subsection{Bounded difference inequality}
The entropy method shows that concentration derives from upper bounds on the entropy.
Sub-additivity of the entropy shows that it is sufficient to bound the entropy of functions depending on one of the $X_i$ only.
One can bound the entropy of a function of one $X_i$ only if the function takes value in a compact space.
This is the purpose of Hoeffding's lemma.
\begin{lemma}[Hoeffding's lemma]\label{lem:Hoeff}
 Let $X_0$ denote a random variable taking values in $[a,b]$ and let $\psi(s)=\log\E[e^{s(X_0-\E[X_0])}]$. Then
 \[
 \psi(s)\leqslant \frac{s^2(b-a)^2}8,\qquad \Ent(e^{s X_0})\leqslant \frac{s^2(b-a)^2}8\E[e^{s X_0}]\enspace.
 \]
\end{lemma}
\begin{proof}
Assume, without loss of generality, that $\E[X_0]=0$.
Check that $\psi(0)=\psi'(0)=0$ and that 
\begin{equation}\label{eq:psi2Hoeff}
 \psi''(s)=\E\bigg[X_0^2\frac{e^{sX_0}}{\E[e^{sX_0}]}\bigg]-\bigg(\E\bigg[X_0\frac{e^{sX_0}}{\E[e^{sX_0}]}\bigg]\bigg)^2\enspace.
\end{equation}
 As $e^{sX_0}/\E[e^{sX_0}]$ is non-negative with expectation with respect to the measure $\E$ equal $1$, one can consider the measure $\F$ such that 
\[
\frac{\rmd\F}{\rmd\E}(x)=\frac{e^{sx}}{\E[e^{sX_0}]}\enspace.
\]
Equation~\eqref{eq:psi2Hoeff} shows that
\[
\psi''(s)=\text{Var}_{\F}(X_0)=\text{Var}_{\F}\bigg(X_0-\frac{a+b}2\bigg)\enspace.
\]  
As $X_0$ takes value in $[a,b]$ $\F$-a.s., $|X_0-(a+b)/2|\leqslant (b-a)/2$ $\F$-a.s. so 
\[
\text{Var}_{\F}\bigg(X_0-\frac{a+b}2\bigg)\leqslant \frac{(b-a)^2}4\enspace.
\]
Integrating twice shows that
\begin{align*}
\psi(s)&=\psi(s)-\psi(0)=\int_0^s\psi'(t)\rmd t= \int_0^s(\psi'(t)-\psi'(0))\rmd t=\int_0^s\int_0^t\psi''(u)\rmd u\rmd t\\
&\leqslant \int_0^s\int_0^t\frac{(b-a)^2}4\rmd u\rmd t=\int_0^s\frac{(b-a)^2}4t\rmd t=\frac{(b-a)^2}8\enspace.
\end{align*}
For the second inequality, note that
 \[
 s\psi'(s)-\psi(s)=\int_0^s u\psi''(u){\rm d}u\leqslant \frac{s^2(b-a)^2}8\enspace.
 \]
 Plugging this bound into \eqref{eq:Ent} gives the second inequality.
\end{proof}
The association of the sub-additivity of entropy with Hoeffding's lemma is useful when the functions $x_i\mapsto f(x)$ have bounded range. 
This property of the function is known as the bounded difference property of $f$. 

\begin{definition}[Bounded difference property]
 Let $\bc=(c_1,\ldots,c_N)$ denote a vector of positive real numbers. 
 The set $\cB(\bc)$ is the set of functions $f:\cX^N\to \R$ such that, for any $x=(x_1,\ldots,x_N)$ and $y=(y_1,\ldots,y_N)$ in $\cX^N$, 
 \[
 |f(x)-f(y)|\leqslant \sum_{i=1}^Nc_i{\bf 1}_{\{x_i\ne y_i\}}\enspace.
 \]
\end{definition}
The bounded difference property is a Lipschitz property of $f$ with respect to the Hamming distance.
It implies that the functions $x_i\mapsto f(x)$ have range with length at most $c_i$.
The bounded difference inequality provides the concentration inequality satisfied by $f(X)$ when $f$ has bounded differences.
\begin{theorem}[Bounded Difference Inequality, BDI]
 Assume that $\bc\in\R_+^N$, $f\in \cB(\bc)$ and let $\sigma^2=\|\bc\|^2/4$. $Z$ is $\sigma$-sub-Gaussian, in particular,
 \[
\forall t>0,\qquad   \P(Z-\E[Z]>t)\leqslant e^{-t^2/(2\sigma^2)}\enspace.
 \]
\end{theorem}

\begin{proof}
 By sub-additivity of the entropy,
 \[
 \Ent(e^{s Z})\leqslant \E\bigg[\sum_{i=1}^N\Ent^{(i)}(e^{s Z})\bigg]\enspace.
 \]
 As $f\in \cB(\bc)$, conditionally on $X^{(i)}$, $Z$ belongs to a set with range at most $c_i$. 
 By the second part of Hoeffding's lemma,
 \[
 \frac{\Ent^{(i)}(e^{s Z})}{\E^{(i)}[e^{s Z}]}\leqslant \frac{s^2c_i^2}{8}\enspace.
 \]
Summing up over $i\in\{1,\ldots,N\}$ and taking expectation yields
\begin{align*}
  \Ent(e^{s Z})& \leqslant \E\bigg[\sum_{i=1}^N\Ent^{(i)}(e^{s Z})\bigg]\leqslant \E\bigg[\sum_{i=1}^N\frac{s^2c_i^2}{8}\E^{(i)}[e^{s Z}]\bigg]\\
  &=\sum_{i=1}^N\frac{s^2c_i^2}{8}\E[e^{s Z}]=\frac{s^2 \sigma^2}{2}\E[e^{s Z}]\enspace.
\end{align*}
 Herbst's argument, see Lemma~\ref{lem:EntMeth} in the sub-Gaussian case, concludes the proof.
\end{proof}
The bounded difference inequality is of particular interest when $f$ is the supremum of bounded empirical processes.
Let $X=(X_1,\ldots,X_N)$ denote independent $\cX$-valued random variables, each $x_i\in \cX$ being a vector $x_i=(x_{i,t})_{t\in T}$. 
Assume that 
\[
\forall t\in T,\qquad \E[X_{i,t}]=0,\quad\text{and}\quad X_{i,t}\in [a_i,b_i]\enspace.
\]
For any $x=(x_1,\ldots,x_N)\in\cX^N$, let 
\[
f(x)=\sup_{t\in T}\frac1N\sum_{i=1}^Nx_{i,t}\enspace.
\]
It is clear that $f\in \cB(\bc)$, with $c_i=(b_i-a_i)/N$, therefore, the BDI applies to $f$ and yields the following concentration inequality for suprema of empirical processes.
\begin{equation}\label{eq:SupEmpProc}
 \forall u>0,\qquad \P\bigg(\sup_{t\in T}\frac1N\sum_{i=1}^NX_{i,t}>\E\big[\sup_{t\in T}\frac1N\sum_{i=1}^NX_{i,t}\big]+u\bigg)\leqslant e^{-\frac{2N^2u^2}{\sum_{i=1}^N(b_i-a_i)^2}}\enspace.
\end{equation}
In particular, if each $X_{i,t}\in[a_i,a_i+1]$,
\begin{equation}\label{eq:SupEmpProc2}
 \forall u>0,\qquad \P\bigg(\sup_{t\in T}\frac1N\sum_{i=1}^NX_{i,t}>\E\big[\sup_{t\in T}\frac1N\sum_{i=1}^NX_{i,t}\big]+u\bigg)\leqslant e^{-2Nu^2}\enspace,
\end{equation}
or equivalently
\begin{equation*}
 \forall u>0,\qquad \P\bigg(\sup_{t\in T}\frac1N\sum_{i=1}^NX_{i,t}>\E\big[\sup_{t\in T}\frac1N\sum_{i=1}^NX_{i,t}\big]+\sqrt{\frac{u}{2N}}\bigg)\leqslant e^{-u}\enspace.
\end{equation*}
Another classical application of \eqref{eq:SupEmpProc} is when $T$ is reduced to a singleton. 
In that case, the result, known as Hoeffding's inequality, see Corollary~\ref{cor:HoeffIneq}, states that, if $X_1,\ldots,X_N$ are independent random variables taking values respectively in $[a_i,b_i]$, then
\begin{equation}\label{eq:HoeffdingIneq}
 \forall u>0,\qquad \P\bigg(\frac1N\sum_{i=1}^N(X_{i}-\E[X_i])>u\bigg)\leqslant e^{-\frac{2N^2u^2}{\sum_{i=1}^N(b_i-a_i)^2}}\enspace.
\end{equation}

\subsection{Gaussian concentration inequality}
The second application of the entropy method is the Gaussian concentration inequality, whose proof also uses Herbst's argument but coupled with the Gaussian logarithmic Sobolev inequality.
These inequalities bound the entropy of $f^2(X)$ for regular functions $f$ by some variance-like term.
To establish this result, start with the basic log-Sobolev inequality for Rademacher random variables.
\begin{theorem}[Rademacher logarithmic Sobolev inequality]
Let $X$ denote a vector of independent Rademacher random variables.
For any $i\in\{1,\ldots,N\}$, let $\bar{X}^{(i)}=(X_1,\ldots,X_{i-1},-X_i,X_{i+1},\ldots,X_N)$ and 
 \[
 \cE(f)=\frac12\E\bigg[\sum_{i=1}^N(f(X)-f(\bar{X}^{(i)}))^2\bigg]\enspace.
 \]
 Then 
 \[
 \Ent(f^2(X))\leqslant \cE(f)\enspace.
 \]
\end{theorem}
\begin{proof}
 By sub-additivity of the entropy,
 \[
 \Ent(f^2(X))\leqslant \E\bigg[\sum_{i=1}^N\Ent^{(i)}(f^2(X))\bigg]\enspace.
 \]
 Hence, it is sufficient to show that
 \[
 \Ent^{(i)}(f^2(X))\leqslant \frac12(f(X)-f(\bar{X}^{(i)}))^2\enspace.
 \]
 Given $X^{(i)}$, $f(X)$ can take two values, say $a$ and $b$, each with probability $1/2$, so it is sufficient to show that, for any $a,b$,
 \[
a^2\log a^2+b^2\log b^2-(a^2+b^2)\log\bigg(\frac{a^2+b^2}2\bigg)\leqslant (a-b)^2\enspace.
 \]
 We may assume without loss of generality that $a$ and $b$ are non-negative and that $a>b$. Therefore, if 
 \[
 h(a)=a^2\log a^2+b^2\log b^2-(a^2+b^2)\log\bigg(\frac{a^2+b^2}2\bigg)-(a-b)^2\enspace,
 \]
 it is sufficient to show that $h(b)=h'(b)=0$, which is obvious and that $h$ is concave, which follows from basic calculus.
\end{proof}
The Rademacher log-Sobolev inequality is sufficient to derive the Gaussian log-Sobolev inequality.
This is then the main tool to prove the Gaussian concentration inequality.

\begin{theorem}[Gaussian log-Sobolev inequality]
Let $X\sim \gauss(0,I_N)$ and $f:\R^N\to\R$ be continuously differentiable, then
\[
\Ent(f^2)\leqslant 2\E\big[\|\nabla f(X)\|^2\big]\enspace.
\] 
\end{theorem}
\begin{proof}
 Assume first that $N=1$. If $\E[f'(X)^2]=\infty$, the result is trivial so we can assume that $\E[f'(X)^2]<\infty$. 
 By standard density arguments, one can assume furthermore that $f$ is twice continuously differentiable with bounded support.
 Under this assumption, let $K$ denote the sup-norm of $f''$.
 Let $\varepsilon_1,\ldots,\varepsilon_n$ denote i.i.d. Rademacher random variables. Define $S_n=\sum_{i=1}^n\varepsilon_i/\sqrt{n}$. 
 By the Rademacher logarithmic Sobolev inequality,
\begin{equation}\label{eq:RadLogSob}
 \Ent(f^2(S_n))\leqslant \frac12\sum_{j=1}^n\bigg(f(S_n)-f\big(S_n-\frac{2\varepsilon_j}{\sqrt{n}}\big)\bigg)^2\enspace.
\end{equation}
As $f$ is uniformly bounded and continuous, by the central limit theorem, the left-hand side in \eqref{eq:RadLogSob} satisfies
 \[
\lim_{n\to\infty} \Ent(f^2(S_n))=\Ent(f^2(X))\enspace.
 \]
On the other hand, for any $j\in \{1,\ldots,n\}$, by a Taylor expansion,
 \[
| f(S_n-2\varepsilon_j/\sqrt{n})-f(S_n)|\leqslant \frac2{\sqrt{n}}|f'(S_n)|+\frac{2K}n\enspace.
 \]
 Thus,
 \[
 \frac 14\sum_{j=1}^n\bigg(f\big(S_n-\frac{2\varepsilon_j}{\sqrt{n}}\big)-f(S_n)\bigg)^2\leqslant f'(S_n)^2+\frac{2K}{\sqrt{n}}|f'(S_n)|+\frac{K^2}n\enspace.
 \]
 By the central limit theorem, it follows that
 \[
 \limsup_{n\to\infty}\frac14\sum_{j=1}^n\bigg(f\big(S_n-\frac{2\varepsilon_j}{\sqrt{n}}\big)-f(S_n)\bigg)^2\leqslant \E[f'(X)^2]\enspace.
 \]
 Hence, the result for $N=1$ follows by taking limits in \eqref{eq:RadLogSob}.
 To extend the results in dimension $N\geqslant1$, apply sub-additivity of entropy to get
 \[
 \Ent(f^2(X))\leqslant \E\bigg[\sum_{i=1}^N\Ent^{(i)}(f^2(X))\bigg]\enspace.
 \]
The result for $N=1$ shows that
 \[
\Ent^{(i)}(f^2(X))\leqslant 2\E^{(i)}[(\partial_if(X))^2 ]\enspace.
 \]
 Hence, $ \Ent(f^2(X))\leqslant 2\E[\sum_{i=1}^N(\partial_if(X))^2]$ and the proof is concluded since $\|\nabla f(X)\|^2=\sum_{i=1}^N(\partial_if(X))^2$.
 \end{proof}
 Together with Herbst's argument, the Gaussian log-Sobolev inequality shows the Gaussian concentration inequality which is the main result of this section.
\begin{theorem}[Borel's Gaussian concentration inequality]
 Assume that $f$ is $L$-Lipschitz, that is $|f(x)-f(y)|\leqslant L\|x-y\|$ for any $x$ and $y$ in $\R^N$. Then $Z=f(X)$ is $L$-sub-Gaussian, that is, for any $s\in \R$,
 \[
 \log(\E[e^{s(f(X)-\E[f(X)])}])\leqslant \frac{s^2L^2}2\enspace.
 \]
 In particular, 
 \[
\forall u>0,\qquad  \P\big(f(X)-\E[f(X)]>u\big)\leqslant e^{-u^2/(2L^2)}\enspace.
 \]
\end{theorem}
\begin{proof}
Using standard density argument, one may assume that $f$ is differentiable with gradient bounded by $L$ and that $\E[f(X)]=0$. 
The Gaussian log-Sobolev inequality applied with $f=e^{s f/2}$ shows that
\begin{align*}
 \Ent(e^{s f})&\leqslant 2\E[\|\nabla e^{s f(X)/2}\|^2]=\frac{s^2}2\E[e^{s f(X)}\|\nabla f(X)\|^2]\leqslant \frac{s^2L^2}2\E[e^{s f(X)}]\enspace.
\end{align*}
The proof is concluded by Herbst's argument.
\end{proof}
Borel's inequality can be applied to show concentration for suprema of Gaussian processes.
\begin{theorem}[Concentration for suprema of Gaussian processes]\label{lem:supGauss}
Let $(X_t)_{t\in T}$ denote a collection of Gaussian random variables $\gauss(\mu_t,\sigma_t^2)$ indexed by a separable set $T$. 
Let $\sigma^2=\sup_{t\in T}\sigma_t^2$.
 \[
\forall u>0,\qquad  \P\big(\sup_{t\in T}(X_t-\mu_t)>\E[\sup_{t\in T}(X_t-\mu_t)]+u\big)\leqslant e^{-u^2/2\sigma^2}\enspace.
 \]
\end{theorem}
\begin{proof}
 Assume that $T$ is finite, the extension to separable sets follows by density arguments. 
Denote $T=\{1,\ldots,d\}$, $Y=(X_t-\mu_t)_{t\in T}$ is a centered Gaussian vector. 
Denote by $\Sigma$ its covariance matrix and $A=\Sigma^{1/2}$ a symmetric positive semi-definite square-root of $\Sigma$. 
$Y$ has the distribution of $A X$, where $X\sim \gauss(0,I_d)$.
Define the function $f:\R^d\to\R, \ x\mapsto \sup_{i\in\{1,\ldots,d\}}(Ax)_i$. 
For any $x$ and $y$, it follows that
\[
|f(x)-f(y)|\leqslant \sup_{i\in\{1,\ldots,d\}}(A(x-y))_i\leqslant \norm{x-y}\sup_{\|v\|=1}|(Av)_i|\leqslant \|A\|_{\text{op}}\|x-y\|\enspace.
\]
Now $\|A\|_{\text{op}}=\sigma$, thus $f$ is $\sigma$-Lipschitz and the result follows from Borel's Gaussian concentration inequality.
\end{proof}

\begin{theorem}\label{thm:LipTransfoGauss}
 Let $X=(X_1,\ldots,X_N)$ denote i.i.d. Gaussian random vectors in $\R^d$ and let $T$ denote a set of functions $t:\R^d\to \R$ such that, for all $t\in T$, $t$ is $1$-Lipshitz. Let $Z=\sup_{t\in T}\frac1N\sum_{i=1}^N[t(X_i)-Pt]$ or $\sup_{t\in T}\frac1N\big|\sum_{i=1}^N[t(X_i)-Pt]\big|$.
 Let $\Sigma$ denote the covariance matrix of $X_1$.
 Then
 \[
 \forall u>0,\qquad \P\big(Z>\E[Z]+u\big)\leqslant e^{-Nu^2/2\|\Sigma\|_{\text{op}}}\enspace.
 \]
\end{theorem}
\begin{proof}
 Write $X_i=\mu+AY_i$ with $A=\Sigma^{1/2}$ and $Y_i$ standard Gaussian. 
 Let 
 \[
 f:(\R^d)^n\to\R,\ \bx\mapsto \sup_{t\in T}\frac1N\sum_{i=1}^N(t(\mu+Ax_i)-Pt)\enspace.
 \]
 Then, for any $\bx,\by\in (\R^d)^n$,
\begin{align*}
  f(\bx)-f(\by)&\leqslant \sup_{t\in T}\frac1N\sum_{i=1}^N(t(\mu+Ax_i)-t(\mu+Ay_i))\leqslant \frac1N\sum_{i=1}^NA(x_i-y_i)\\
  &\leqslant \frac{\|A\|_{\text{op}}}N\sum_{i=1}^N\|x_i-y_i\|\leqslant \frac{\|A\|_{\text{op}}}{\sqrt{N}}\|\bx-\by\|\enspace.
\end{align*}
The result follows from Borel's Gaussian concentration inequality as $\|A\|_{\text{op}}^2=\|\Sigma\|_{\text{op}}$.
\end{proof}

\section{Talagrand's concentration inequality}\label{Sec:TalIneq}
In this section, $f_i:\cX^{N-1}\to\R$ denotes any function and let $Z_i=f_i(X^{(i)})$.
Let $\phi(x)=e^{x}-1-x$.
The goal of this section is to establish a concentration result for suprema of empirical processes. 
Let $X=(X_1,\ldots,X_N)$ denote independent $\cX$-valued random variables, each $x_0\in\cX$ being a vector $x_0=(x_{0,t})_{t\in T}$. 
Assume that 
\[
\forall t\in T,\forall i\in\{1,\ldots,N\},\qquad \E[X_{i,t}]=0,\qquad X_{i,t}\leqslant 1,\ \text{a.s.}\enspace.
\]
Talagrand's concentration inequality shows sub-Poissonian deviations of $Z=\sup_{t\in T}X_{i,t}$ above its expectation $\E[Z]$.
It proceeds by bounding from above the log-Laplace transform $\psi(s)$ of $Z$, using the entropy $\Ent(e^{sZ})$, but in a more involved way than the Herbst's argument.

\subsection{Modified logarithmic Sobolev inequality}
The starting point of this analysis is a modified version of log-Sobolev inequality.
To establish this inequality, the following variational formulation of entropy is useful.
\begin{theorem}\label{thm:VarFormEnt}
 Let $Y$ denote a nonnegative random variable such that $\E[\Phi(Y)]<\infty$.
 Then
 \[
 \text{Ent}(Y)=\inf_{u>0}\E[Y(\log(Y)-\log(u))-(Y-u)]\enspace.
 \]
\end{theorem}
\begin{proof}
Recall that $\Phi$ is convex, so $\Ent(Y)\geqslant 0$ by Jensen's inequality, and 
 \[
 \text{Ent}(Y)=\E[\Phi(Y)-\Phi(\E[Y])]
 \enspace.
 \]
 Then, for any $u>0$, 
\begin{align*}
 \E[\Phi(Y)-\Phi(u)-\Phi'(u)(Y-u)]&=\E[Y\log(Y)-u\log(u)-(1+\log (u))(Y-u)]\\
 &=\E[Y(\log(Y)-\log(u))-(Y-u)]\enspace.
\end{align*}
Thus
\begin{align*}
\E[Y(\log(Y)-\log&(u))-(Y-u)]- \text{Ent}(Y)\\
&=\E[\Phi(Y)-\Phi(u)-\Phi'(u)(Y-u)-(\Phi(Y)-\Phi(\E[Y]))]\\
&=\E[\Phi(\E[Y])-\Phi(u)-\Phi'(u)(Y-u)]\\
&=\Phi(\E[Y])-\Phi(u)-\Phi'(u)(\E[Y]-u)\enspace.
\end{align*}
By convexity of $\Phi$, this last term is always nonnegative and it is clearly null when $u=\E[Y]$.
\end{proof}
The modified log-Sobolev inequality bounds from above the entropy of $Z$ using the increments $Z-Z_i$ via the function $\phi$ rather than the square function.
\begin{theorem}[Modified log-Sobolev inequality]
 For any $s\in \R$,
 \[
\Ent(e^{sZ})\leqslant \sum_{i=1}^n\E[e^{sZ}\phi(s(Z_i-Z))]\enspace.
 \]
\end{theorem}
\begin{proof}
 Basic algebra shows that
\begin{align*}
 e^{sZ}\phi(s(Z_i-Z))&=e^{sZ}(e^{s(Z_i-Z)}-s(Z_i-Z)-1)\\
 &=e^{sZ_i}-e^{sZ}+se^{sZ}(Z-Z_i)\enspace.
\end{align*}
Applying Theorem~\ref{thm:VarFormEnt} conditionally on $X^{(i)}$, to $Y=e^{sZ}$ and $u=e^{sZ_i}$, it follows that
\begin{align*}
\text{Ent}^{(i)}(e^{sZ})&\leqslant \E^{(i)}[e^{sZ}(sZ-sZ_i)-(e^{sZ}-e^{sZ_i})]\\
&=\E^{(i)}[e^{sZ}\phi(s(Z_i-Z))]\enspace.
\end{align*}
Therefore, by sub-additivity of the entropy, see Theorem~\ref{thm:SubAddEnt},
\begin{align*}
\text{Ent}(e^{sZ})&\leqslant \E\bigg[\sum_{i=1}^N\text{Ent}^{(i)}(e^{sZ})\bigg]\leqslant \E\bigg[\sum_{i=1}^N\E^{(i)}[e^{sZ}\phi(s(Z_i-Z))]\bigg]\\
 &=\sum_{i=1}^n\E[e^{sZ}\phi(s(Z_i-Z))]\enspace.
\end{align*}
\end{proof}
\subsection{Bousquet's version of Talagrand's inequality}
%
Define $\sigma^2=\sum_{i=1}^n\sup_{t\in T}\E[X_{i,t}^2]$ and $\nu=2\E[Z]+\sigma^2$.
Recall that $h(u)=(1+u)\log(1+u)-u$.
Talagrand's inequality shows that $Z=\sup_{t\in T}X_{i,t}$ is a $\nu$-sub-Poissonian random   variables. 
The following version of this inequality, with sharp constants, was first established by Bousquet \cite{MR1890640}.
\begin{theorem}[Bousquet's version of Talagrand's concentration inequality]\label{thm:TalIneq}
The random variable $Z-\E[Z]$ is $\nu$-sub-Poissonian, that is, for any $s>0$, $\E[e^{s(Z-\E[Z])}]\leqslant \nu\phi(s)$. Moreover,
 \[
 \forall x>0,\qquad \P(Z>\E[Z]+x)\leqslant e^{-\nu h(t/\nu)}\enspace.
 \]
\end{theorem}
\begin{proof}
 The proof relies on the following result from calculus.
\begin{lemma}\label{lem:Tech}
 For any $s\geqslant 0$ and any $x\leqslant 1$,
 \[
 \frac{\phi(-sx)}{\phi(-s)}\leqslant \frac{x+(x^2/2-x)e^{-sx}}{1-e^{-s}/2}\enspace.
 \]
\end{lemma}
The proof of the lemma is omitted. Going back to the proof of the theorem, define
\[
Z_i=\sup_{t\in T}\sum_{1\leqslant j\leqslant n, j\ne i}X_{j,t}\enspace.
\]
Let also $t_0$ such that $Z=\sum_{1\leqslant i\leqslant n}X_{i,t_0}$ and $t_i$ such that $Z_i=\sum_{1\leqslant j\leqslant n, j\ne i}X_{j,t_i}$.
Remark that $X_{i,t_i}\leqslant Z-Z_i\leqslant X_{i,t_0}\leqslant 1$, so $\E^{(i)}[Z-Z_i]\geqslant \E^{(i)}[X_{i,t_i}]=0$ and
\[
\sum_{i=1}^n Z-Z_i\leqslant Z\enspace.
\]
By the modified logarithmic Sobolev inequality, 
\[
\Ent(e^{sZ})\leqslant \sum_{i=1}^n\E[e^{sZ}\phi(-s(Z-Z_i))]\enspace.
\]
Let 
\[
g(s)=\frac{\phi(-s)}{1-e^{-s}/2}=\frac{e^{-s}-1+s}{1-e^{-s}/2}=\frac{1-e^s+se^s}{e^s-1/2}=\frac{-\phi(s)-s+se^s}{e^s-1/2}\enspace.
\]
Lemma~\ref{lem:Tech} applied with $x=(Z-Z_i)$ implies
\begin{align*}
e^{sZ}\phi(-s(Z-Z_i))&\leqslant \frac{(Z-Z_i)e^{sZ}+((Z-Z_i)^2/2-(Z-Z_i))e^{sZ-s(Z-Z_i)}}{1-e^{-s}/2}\phi(-s)\\
&\leqslant g(s)\bigg(e^{sZ_i}\bigg[\frac{(Z-Z_i)^2}2-(Z-Z_i)\bigg]+(Z-Z_i)e^{sZ}\bigg)\enspace.
\end{align*}
Now, as $Z-Z_i-X_{i,t_i}\geqslant 0$ and $Z-Z_i+X_{i,t_i}-2\leqslant 0$,
\[
(Z-Z_i)^2-2(Z-Z_i)-[X_{i,t_i}^2-2X_{i,t_i}]=(Z-Z_i-X_{i,t_i})(Z-Z_i+X_{i,t_i}-2)\leqslant 0\enspace.
\]
It follows that
\begin{align*}
 \E^{(i)}\bigg[\frac{(Z-Z_i)^2}2-(Z-Z_i)\bigg]\leqslant \E^{(i)}\bigg[\frac{X_{i,t_i}^2}2-X_{i,t_i}\bigg]=\frac{\E^{(i)}[X_{i,t_i}^2]}2\enspace.
\end{align*}
Therefore,
\begin{align*}
 \E^{(i)}[e^{sZ}\phi(-s(Z-Z_i))]&\leqslant g(s)\bigg(\E^{(i)}[(Z-Z_i)e^{sZ}]+\frac12\E^{(i)}[X_{i,t_i}^2]e^{s Z_i}\bigg)\\
 &\leqslant g(s)\bigg(\E^{(i)}[(Z-Z_i)e^{sZ}]+\frac12\sup_{t\in T}\E[X_{i,t}^2]e^{sZ_i}\bigg)\enspace.
\end{align*}
As $\E^{(i)}[Z-Z_i]\geqslant 0$, $Z_i\leqslant \E^{(i)}[Z]$ and, by Jensen's inequality,
\[
e^{s Z_i}\leqslant e^{s\E^{(i)}[Z]}\leqslant \E^{(i)}[e^{sZ}]\enspace,
\]
thus
\begin{align*}
 \E^{(i)}[e^{sZ}\phi(-s(Z-Z_i))]&\leqslant  g(s)\E^{(i)}\bigg[\bigg(Z-Z_i+\frac12\sup_{t\in T}\E[X_{i,t}^2]\bigg)e^{sZ}\bigg]\enspace.
\end{align*}
Summing up over $i$ and taking the expectation, it follows from $\sum_{i=1}^n(Z-Z_i)\leqslant Z$ that
\[
\Ent(e^{sZ})\leqslant g(s)\E\bigg[\bigg(Z+\frac{\sigma^2}2\bigg)e^{sZ}\bigg]=g(s)\E\bigg[\bigg(Z-\E[Z]+\frac{\nu}2\bigg)e^{sZ}\bigg]
\]
By \eqref{eq:Ent}, this can be rewritten
\begin{equation}\label{eq:DiffIneq1}
(s-g(s))\psi'(s)-\psi(s)\leqslant g(s)\frac{\nu}2\enspace. 
\end{equation}
Let $\zeta(s)=\phi(s)+s/2$, so $\zeta'(s)=e^s-1+1/2=e^s-1/2$ and
\[
s-g(s)=s+\frac{\phi(s)+s-se^s}{e^s-1/2}=\frac{se^s-s/2+\phi(s)+s-se^s}{e^s-1/2}=\frac{\phi(s)+s/2}{e^s-1/2}=\frac{\zeta(s)}{\zeta'(s)}\enspace.
\]
In particular thus $g(s)=s-\zeta(s)/\zeta'(s)$ so, multiplying inequality \eqref{eq:DiffIneq1} by $\zeta'(s)$ shows
\[
\zeta(s)\psi'(s)-\zeta'(s)\psi(s)\leqslant (s\zeta'(s)-\zeta(s))\frac{\nu}2\enspace.
\]
Dividing by $\zeta^2(s)$ yields
\[
\frac{\zeta(s)\psi'(s)-\zeta'(s)\psi(s)}{\zeta^2(s)}\leqslant \frac{s\zeta'(s)-\zeta(s)}{\zeta^2(s)}\frac{\nu}2\enspace.
\]
that is 
\[
\bigg(\frac{\psi(s)}{\zeta(s)}\bigg)'\leqslant -\frac{\nu}2\bigg(\frac{s}{\zeta(s)}\bigg)'\enspace.
\]
As $\psi(0)=\psi'(0)=0$, the function $u:s\mapsto \psi(s)/\zeta(s)$ can be continuously extended in $0$ by defining $u(0)=0$.
Therefore, integrating over $s$ shows that
\[
\frac{\psi(s)}{\zeta(s)}\leqslant -\frac{\nu}2\bigg(\frac{s}{\zeta(s)}-\lim_{s\to0}\frac{s}{\zeta(s)}\bigg)=-\frac{\nu}2\bigg(\frac{s}{\zeta(s)}-2\bigg)\enspace.
\]
Finally, multiplying by $\zeta(s)$,
\[
\psi(s)\leqslant -\nu\bigg(\frac{s}2-\zeta(s)\bigg)=\nu\phi(s)\enspace.
\]
This shows the first part of the theorem.
The second part comes then from the first part and Lemma~\ref{lem:BoundsFLT}.
\end{proof}

\section{PAC-Bayesian inequalities}\label{Sec:PacBayes}
Let $X\in \cX$ denote a random variable and let $F$ denote a measurable space.
Let $\Gamma:F\times \cX\to \R$ denote a bounded measurable function.
For any measures $\mu$ and $\rho$ on $F$, let 
\[
K(\rho,\mu)=
\begin{cases}
 \int \log\big(\frac{{\rm d}\rho}{{\rm d}\mu}\big){\rm d}\rho&\ \text{if}\ \rho\ll \mu\\
 +\infty&\ \text{otherwise}
\end{cases}\enspace.
\]
Let $X_1,\ldots,X_N$ denote i.i.d. copies of $X$.
The entropy method is not the only method to show uniform deviation inequalities for empirical processes. 
A famous alternative, that has been fruitfully exploited by Catoni for example, see \cite{MR2163920}, is known as PAC-Bayesian inequality.
The idea is to exploit a variational formula for the Kullback divergence to obtain this uniformity.
\begin{theorem}[PAC-Bayesian inequality]\label{thm:PBIneq}
For any probability measure $\mu$ on $F$, for any $t>0$, with probability $1-e^{-t}$, for any probability measure $\rho$ on $F$,
 \[
 P_N\bigg[\int \Gamma_f {\rm d}\rho(f)\bigg]\leqslant \int \log P\big[e^{\Gamma_f}\big]{\rm d}\rho(f)+\frac{K(\rho,\mu)+t}N\enspace.
 \]
\end{theorem}
\begin{proof}
The proof relies on the following variational formula.
\begin{equation}\label{eq:VarFor}
 \log \int e^h{\rm d}\mu=\sup_{\rho}\int h{\rm d}\rho-K(\rho,\mu)\enspace,
\end{equation}
where the supremum is taken over all probability measures $\rho$ on $F$.
\begin{proof}[Proof of \eqref{eq:VarFor}] Choose $\rho$ such that 
${\rm d}\rho=e^h{\rm d}\mu/\int e^h{\rm d}\mu$.
Then,
\begin{align*}
 \sup_{\rho}\int h{\rm d}\rho-K(\rho,\mu)\geqslant \int (h-h+\log \int e^h{\rm d}\mu){\rm d}\rho=\log \int e^h{\rm d}\mu\enspace.
\end{align*}
In words, the left-hand side of \eqref{eq:VarFor} is smaller than the right-hand side.
Conversely, by Jensen's inequality
\[
\int \bigg(h+\log\frac{{\rm d}\mu}{{\rm d}\rho}\bigg){\rm d}\rho=\int \log \bigg(e^h\frac{{\rm d}\mu}{{\rm d}\rho}\bigg){\rm d}\rho\leqslant \log \int e^h{\rm d}\mu\enspace.
\]
This shows that the right-hand side in \eqref{eq:VarFor} is also smaller than the left-hand side, which concludes the proof of this inequality.
\end{proof}

\medskip

Applying \eqref{eq:VarFor} with $h=N(P_N\Gamma_f-\log Pe^{\Gamma_f})$ yields
\begin{align*}
 \E\bigg[e^{\sup_{\rho} N\int (P_N\Gamma_f-\log Pe^{\Gamma_f}){\rm d}\rho-K(\rho,\mu)}\bigg]&=\E\bigg[\int e^{N(P_N\Gamma_f-\log Pe^{\Gamma_f})}{\rm d}\mu\bigg]\\
 &=\int \E\bigg[e^{N(P_N\Gamma_f-\log Pe^{\Gamma_f})}{\rm d}\mu\bigg]\\
 &=\int\prod_{i=1}^NP\bigg[\frac{e^{\Gamma_f}}{Pe^{\Gamma_f}}\bigg]{\rm d}\mu=1\enspace.
\end{align*}
By the Chernoff bound, any random variable $W$ such that $\E[e^{W}]\leqslant 1$ satisfies
\[
\forall t>0,\qquad \P(W>t)\leqslant e^{-t+\log \E[e^W]}=e^{-t}\enspace.
\]
The result follows by applying this basic inequality to 
\[
W=N\int (P_N\Gamma_f-\log Pe^{\Gamma_f}){\rm d}\rho-K(\rho,\mu)\enspace.
\]
\end{proof}

\section{Deviation of suprema of median-of-means processes}\label{sec:DevMOMProc}
To conclude this chapter, we present two deviation results for suprema of MOM processes.
Both show deviations of this process above a term involving the \emph{Rademacher complexity} of $F$.
Recall that the Rademacher complexity of a class $F$ of functions $f:\cX\to\R$ is defined by 
\[
D(F)=\bigg(\E\bigg[\sup_{f\in F}\bigg\{\frac1{\sqrt{N}}\sum_{i=1}^N\epsilon_if(X_i)\bigg\}\bigg]\bigg)^2\enspace.
\]
The quantity $D(F)$ can easily be evaluated when $F$ is a set linear functionals. 
Let $r>0$, $\|\cdot\|$ denote the Euclidean norm on $\R^d$ and $r\bB=\{\ba\in \R^d:\|\ba\|\leqslant r\}$
\[
F=\{f:\R^d\to\R: \exists \ba\in r\bB,\ f(\bx)=\ba^T\bx\}\enspace.
\]
Let $X_0\in\R^d$ be such that $PX_0=0$, $P\|X_0\|^2<\infty$ and let $\Sigma_P=P[X_0X_0^T]$. In this case,
\begin{align}
\notag D(F)&=\bigg(\E\bigg[\sup_{\ba\in r\bB}\bigg\{\frac1{\sqrt{N}}\sum_{i=1}^N\epsilon_i\ba^TX_i\bigg\}\bigg]\bigg)^2\\
 \notag &=r^2\bigg(\E\bigg[\sup_{\ba\in \bB}\bigg\{\ba^T\bigg(\frac1{\sqrt{N}}\sum_{i=1}^N\epsilon_iX_i\bigg)\bigg\}\bigg]\bigg)^2\\
 \notag&=r^2\bigg(\E\bigg[\bigg\|\frac1{\sqrt{N}}\sum_{i=1}^N\epsilon_iX_i\bigg\|\bigg]\bigg)^2\enspace.
  \end{align}
  By Cauchy-Schwarz inequality,
  \begin{align}
\notag D(F) &\leqslant r^2\E\bigg[\bigg\|\frac1{\sqrt{N}}\sum_{i=1}^N\epsilon_iX_i\bigg\|^2\bigg]\\
\notag  &=\frac{2r^2}N\sum_{1\leqslant i,j\leqslant N}\E\big[\epsilon_i\epsilon_jX_i^TX_j\big]\\
\notag  &=r^2\E\big[X_0^TX_0\big]\\
\notag  &=r^2\E\big[\text{Tr}\big(X_0X_0^T\big)\big]\\
\label{eq:D(F)Lin} &= r^2\text{Tr}(\Sigma)\enspace.
 \end{align}
In particular, when $r=1$ and $\Sigma$ is the identity matrix $\Sigma=\bI_d$, $D(F)$ is the dimension of the state space $D(F)=d$. 
The first result is a deviation for suprema of MOM processes above $\sqrt{D(F)/N}$. 
It is established using the tools introduced by Lugosi and Mendelson \cite{LugosiMendelson2016}.
\begin{theorem}[Concentration for suprema of MOM processes]\label{thm:ConcSupMOM}
 Let $F$ denote a separable set of functions $f:\cX\to\R$ such that $\sup_{f\in F}\sigma^2(f)=\sigma^2<\infty$, where $\sigma^2(f)=\text{Var}(f(X))$. Then, for any $K\in \{1,\ldots,N/2\}$, 
 \[
 \P\bigg(\sup_{f\in F}|\MOM{K}{f}-Pf|\geqslant 128\sqrt{\frac{D(F)}N}\vee 4\sigma\sqrt{\frac{2 K}{N}}\bigg)\leqslant e^{-K/32}\enspace.
 \]
\end{theorem}
\begin{proof}
Assume that $F$ is finite, the general case follows by a standard density argument. 
The basic idea is that, for any $\epsilon>0$,
\[
\sup_{f\in F}|\MOM{K}{f}-Pf|\leqslant\epsilon\quad \text{if}\quad \sup_{f\in F}\sum_{k=1}^K{\bf 1}_{\{|(P_{B_k}-P)f|>\epsilon\}}\leqslant \frac K2\enspace.
\] 
Introduce $\phi$, a $1$-Lipschitz function such that ${\bf 1}_{x\geqslant 2}\leqslant \phi(x)\leqslant {\bf 1}_{x\geqslant 1}$.
We have 
\begin{align*}
&\sup_{f\in F}\sum_{k=1}^K{\bf 1}_{\{|(P_{B_k}-P)f|>\epsilon\}}\leqslant \sup_{f\in F}\sum_{k=1}^K\phi\bigg(\frac{2|(P_{B_k}-P)f|}{\epsilon}\bigg)\\
&\leqslant K\sup_{f\in F}\P\bigg(|(P_{B_1}-P)f|>\frac{\epsilon}2\bigg)+ \sup_{f\in F}\bigg\{\sum_{k=1}^K\phi\bigg(\frac{2|(P_{B_k}-P)f|}{\epsilon}\bigg)-\E\bigg[\phi\bigg(\frac{2|(P_{B_k}-P)f|}{\epsilon}\bigg)\bigg]\bigg\}\enspace.
\end{align*}
The first term in this upper-bound can be bounded from above using Chebyshev's inequality as follows.
\[
\sup_{f\in F}\P\bigg(|(P_{B_k}-P)f|>\frac{\epsilon}2\bigg)\leqslant\frac{4\sigma^2 K}{\epsilon^2 N}\enspace.
\]
Using the bounded difference inequality, the second term is bounded from above, with probability at least $1-e^{-2x^2/K}$, by
\[
\E\bigg[\sup_{f\in F}\bigg\{\sum_{k=1}^K\phi\bigg(\frac{2|(P_{B_k}-P)f|}{\epsilon}\bigg)-\E\bigg[\phi\bigg(\frac{2|(P_{B_k}-P)f|}{\epsilon}\bigg)\bigg]\bigg\}\bigg]+x\enspace.
\]
Using the symmetrization trick, the expectation is now bounded from above by 
\[
2\E\bigg[\sup_{f\in F}\bigg\{\sum_{k=1}^K\epsilon_k\phi\bigg(\frac{2|(P_{B_k}-P)f|}{\epsilon}\bigg)\bigg\}\bigg]\enspace.
\]
By Ledoux and Talagrand's contraction lemma, this term is bounded from above by
\[
\frac{16}\epsilon\E\bigg[\sup_{f\in F}\bigg\{\sum_{k=1}^K\epsilon_k(P_{B_k}-P)f\bigg\}\bigg]\enspace.
\]
By the symmetrization trick, this term is bounded from above by
\[
\frac{32K}{\epsilon }\sqrt{\frac{D(F)}N}\enspace.
\]
Overall, with probability at least $1-e^{-2x^2/K}$, 
\[
\sup_{f\in F}\sum_{k=1}^K{\bf 1}_{\{|(P_{B_k}-P)f|>\epsilon\}}\leqslant \frac{32K}{\epsilon }\sqrt{\frac{D(F)}N}+\frac{4\sigma^2 K}{\epsilon^2 N}+x
\]
Choose $\delta\in1/2$, $\epsilon=128\sqrt{\frac{D(F)}N}\vee \sqrt{32\frac{\sigma^2 K}{ N}}$ and $x=K/8$, this shows that, with probability $1-e^{-K/32}$,
\[
\sup_{f\in F}|\MOM{K}{f}-Pf|\leqslant 128\sqrt{\frac{D(F)}N}\vee 4\sigma\sqrt{\frac{2 K}{ N}}\enspace.
\]
\end{proof}

Some results require the following extension of the previous result whose proof follows exactly the same arguments and is left to the reader.

\begin{theorem}[General concentration bound for suprema of MOM processes]\label{thm:ConcSupMOMg}
 Let $F$ denote a separable set of functions $f:\cX\to\R$ such that $\sup_{f\in F}\sigma^2(f)=\sigma^2<\infty$, where $\sigma^2(f)=\text{Var}(f(X))$. 
 Let $\alpha\in(0,1)$. 
 There exists a constant $c_\alpha$ such that, for any $K\geqslant 1/\alpha$, with probability at least $1-e^{-K/c_\alpha}$, there exists at least $(1-\alpha)K$ blocks $B_k$ where
 \[
\forall f\in F,\qquad |(P_{B_k}-P)f|\leqslant c_\alpha\bigg(\sqrt{\frac{D(F)}N}\vee\sigma\sqrt{\frac{ K}{ N}}\bigg)\enspace.
 \]
\end{theorem}

This general result admits the following corollary that was first proved in \cite{LugosiMendelson2017-2} and that will be used repeatedly in the following.

\begin{corollary}\label{Cor:ConcSupLin}
Assume that $X_1,\ldots,X_N$ are i.i.d. random vectors of $\R^d$, with common distribution $P$ such that $P[\|X\|^2]<\infty$. 
Let $\Sigma=P[(X-PX)(X-PX)^T]$, $\alpha\in(0,1)$ and $r>0$. 
 There exists a constant $c_\alpha$ such that, for any $K\geqslant 1/\alpha$, with probability at least $1-e^{-K/c_\alpha}$, there exists at least $(1-\alpha)K$ blocks $B_k$ where
 \[
\forall \ba\in\R^d:\|\ba\|\leqslant r,\qquad |(P_{B_k}-P)[\ba^T\cdot]|\leqslant c_\alpha r\sqrt{\frac{\text{Tr}(\Sigma)\vee\|\Sigma\|_{\text{op}} K}N}\enspace.
 \]
\end{corollary}
\begin{proof}
 Apply Theorem~\ref{thm:ConcSupMOMg} to the class $F=\{\ba^T\cdot: \|\ba\|\leqslant r\}$.
 By \eqref{eq:D(F)Lin} , $D(F)\leqslant r^2\text{Tr}(\Sigma)$ and, for any $\ba\in\R^d:\|\ba\|\leqslant r$, 
 \[
 \text{Var}(\ba^TX)=\ba^T\Sigma\ba\leqslant r^2\|\Sigma\|_{\text{op}}\enspace.
 \]
 The result follows.
\end{proof}
As for univariate mean estimate, this first analysis can be refined under stronger moments assumptions using Minsker-Strawn's approach. 
Denote by $Q$ the tail function of a standard Gaussian, $n=N/K$ and 
\[
g(n,f)=\sup_{t\in \R}\absj{\P\bigg(\sqrt{n}\frac{(P_{B_1}-P)f}{\sigma(f)}> t\bigg)-Q(t)}\enspace.
\]
Recall that if $F$ is a class of functions such that $P|f|^3<\infty$ for any $f\in F$ and $\sup_{f\in F}P[|f-Pf|^3]/\sigma(f)^3=\gamma<\infty$, then, Berry-Esseen theorem implies that 
\[
\sup_{f\in F}g(n,f):=g(n)\leqslant \frac{\gamma}{\sqrt{n}}\enspace.
\]
The key-point is that one has to use a ``smoothed" version of median of means estimators. 
Define the function 
\begin{equation}\label{eq:DefRho}
\rho(t)=
\begin{cases}
 -1 &\text{ if } t\leqslant -1\\
 t &\text{ if } -1\leqslant t\leqslant 1\\
 1 &\text{ if }  t\geqslant 1
\end{cases} 
\end{equation}
Then, let $\Delta\geqslant \sup_{f\in F}\sigma(f)$ and let $\hat{P}_Kf$ be solution of the equation
\[
\sum_{k=1}^K\rho\bigg(\sqrt{n}\frac{P_{B_k}f-z}{\Delta}\bigg)=0\enspace.
\]

\begin{theorem}\label{Thm:SMOMMinsker}[Minsker's deviation bound for suprema of smoothed MOM processes]
 Assume that $s$ and $K$ satisfy
 \[
300\bigg(\frac{16}{\Delta}\sqrt{\frac{D(F)}N}+\sqrt{\frac{2s}N}+4\frac{g(n)}{\sqrt{n}}\bigg)\leqslant \sqrt{\frac{K}{N}}\enspace,
 \]
 Then
 \[
 \P\bigg(\sup_{f\in F}|\hat{P}_Kf-Pf|\geqslant 300\sqrt{\frac{D(F)}N}+20\Delta\bigg(\sqrt{\frac{2s}N}+4\frac{g(n)}{\sqrt{n}}\bigg)\bigg)\leqslant e^{-s}\enspace.
 \]
\end{theorem}
\begin{remark}
 Assume that $F$ is a class of functions such that $P|f|^3<\infty$ for any $f\in F$ and $\sup_{f\in F}P[|f-Pf|^3]/\sigma(f)^3=\gamma<\infty$, then, Berry-Esseen theorem implies that 
\[
\sup_{f\in F}g(n,f):=g(n)\leqslant \frac{\gamma}{\sqrt{n}}\enspace.
\]
Assume moreover that $D(F)\leq \Delta^2 \sqrt{N}$. 
Let  $K=C\sqrt{N}$, where $C$ is a sufficiently large absolute constant.
Then Theorem~\ref{Thm:SMOMMinsker} implies that, simultaneously for all $s\leq C'\sqrt{N}$, with probability larger than $1-e^{-s}$,
\[
\sup_{f\in F}|\hat{P}_Kf-Pf|\leqslant C\bigg(\sqrt{\frac{D(F)}N}+\Delta\sqrt{\frac{\gamma^2+s}N}\bigg)\enspace.
\]
\end{remark}
\begin{proof}
Let us first remark that $(\hat{P}_K-P)f$ is solution of $Q^{(n)}_{K}(f,\cdot)=0$, where
\[
Q^{(n)}_{K}(f,z):=\frac1K\sum_{k=1}^K\rho\bigg(\sqrt{n}\frac{(P_{B_k}-P)f-z}{\Delta}\bigg)\enspace.
\]
The strategy is then to find a deterministic function $U(\cdot)$ such that, for any $z\in \R$, w.h.p., for any $f\in F$, 
\[
Q^{(n)}_{K}(f,z)\leqslant U(z)\enspace.
\]
Then, if $z_+$ denotes the smallest solution of $U(z)=0$, on the event $\sup_{f\in F}Q^{(n)}_K(f,z_+)\leqslant U(z_+)$, for any $f\in F$,
\[
Q^{(n)}_K(f,z_+)\leqslant U(z_+)=0=Q^{(n)}_K(f,(\hat{P}_K-P)f)\enspace.
\]
As $Q^{(n)}_K(f,\cdot)$ is non-increasing, this implies 
\[
\P\big(\forall f\in F,\ (\hat{P}_K-P)f\leqslant z_+\big)\geqslant \P\big(\sup_{f\in F}Q^{(n)}_K(f,z_+)\leqslant U(z_+)\big)\enspace.
\]
Fix $z$ and bound uniformly from above $Q^{(n)}_K(f,z)$.
Let $G_{k}(f)=\sqrt{n}(P_{B_k}-P)f/\sigma(f)$, then
 \[
Q^{(n)}_K(f,z)=\frac1K \sum_{k=1}^K\rho\bigg(\frac{\sigma(f)}{\Delta}G_{k}(f)-\frac{\sqrt{n}z}{\Delta}\bigg)\enspace.
 \]
Therefore
\begin{align*}
Q^{(n)}_K(f,z)\leqslant (Q^{(n)}_K(f,z)-Q^{(n)}(f,z))+(Q^{(n)}(f,z)-Q(f,z))+Q(f,z)\enspace, 
\end{align*}
where $G$ is a standard Gaussian random variable and 
\begin{gather*}
Q^{(n)}(f,z):= \E\bigg[\rho\bigg(\frac{\sigma(f)}{\Delta}G_{k}(f)-\frac{\sqrt{n}z}{\Delta}\bigg)\bigg]\enspace,\\
Q(f,z):=\E\bigg[\rho\bigg(\frac{\sigma(f)}{\Delta}G-\frac{\sqrt{n}z}{\Delta}\bigg)\bigg]\enspace.
\end{gather*}
Let $\psi_1(f,z)=Q^{(n)}_K(f,z)-Q^{(n)}(f,z)$.
As $\rho$ takes values in $[-1,1]$, the bounded difference inequality grants that, with probability larger than $1-e^{-s}$
\[
\sup_{f\in F}\psi_1(f,z)\leqslant \E\big[\sup_{f\in F}\psi_1(f,z)\big]+\sqrt{\frac{2s}K}\enspace.
\]
By symmetrization and contraction $(\rho(\cdot-x)-\rho(-x))$ being $1$-Lipshitz,
\[
\E\big[\sup_{f\in F}\psi_1(f,z)\big]\leqslant 16\frac{\sqrt{n}}{\Delta}\sqrt{\frac{D(F)}N}\enspace.
\]
For any real numbers $\alpha$ and $\beta$ and any real valued random variable $X$ with c.d.f. $F_X$, 
\begin{align}
\notag \E[\rho(\alpha X-\beta)]&=-F_X\bigg(\frac{\beta-1}{\alpha}\bigg)+1-F_X\bigg(\frac{\beta+1}{\alpha}\bigg)+\int_{\frac{\beta-1}{\alpha}}^{\frac{\beta+1}{\alpha}}(\alpha t-\beta)dF_X(t)\\
\label{eq:LemTech} &=-2F_X\bigg(\frac{\beta-1}{\alpha}\bigg)+1-\alpha\int_{\frac{\beta-1}{\alpha}}^{\frac{\beta+1}{\alpha}}F_X(t)dt\enspace.
\end{align}
By definition of $g(n)$, it follows that
\[
\sup_{f\in F}|Q^{(n)}(f,z)-Q(f,z)|\leqslant 4g(n)\enspace.
\]
Now, let $f\in F$, $\alpha=\sigma(f)/\Delta$ and $\beta=\sqrt{n}z/\Delta$,
\begin{align*}
Q(f,z)&=\P\big(\alpha G-\beta\geqslant 1\big)-\P\big(\alpha G-\beta\leqslant -1\big)+\E\big[\big(\alpha G-\beta\big){\bf 1}_{|\alpha G-\beta|\leqslant 1}\big]\\
 &\leqslant \alpha\E\big[G{\bf 1}_{|\alpha G-\beta|\leqslant 1}\big]-\beta\P\big(|\alpha G-\beta|\leqslant 1)\\
 &=\frac{\alpha}{\sqrt{2\pi}} [e^{-(\beta-1)^2/2\alpha^2}-e^{-(\beta+1)^2/2\alpha^2}]-\beta\P\big(|\alpha G-\beta|\leqslant 1)\\
 &\leqslant\sqrt{\frac{2}{\pi}}\frac{\beta}{\alpha} e^{-(\beta-1)^2/2\alpha^2}-\beta\P\big(|\alpha G-\beta|\leqslant 1)\\
 &=-\frac{\beta}{\alpha\sqrt{2\pi}}\int_{\beta-1}^{\beta+1}e^{-x^2/2\alpha^2}-e^{-(\beta-1)^2/2\alpha^2}{\rm d}x\enspace.
\end{align*}
Assume now that $\beta\leqslant1/16$, and, as $\alpha\leqslant 1$, write
\begin{align*}
 \int_{\beta-1}^{\beta+1}e^{-x^2/2\alpha^2}-e^{-(\beta-1)^2/2\alpha^2}\frac{{\rm d}x}{\alpha\sqrt{2\pi}}&\geqslant \int_{-1/2}^{1/2}e^{-x^2/2\alpha^2}(1-e^{x^2-(\beta-1)^2/2\alpha^2})\frac{{\rm d}x}{\alpha\sqrt{2\pi}}-\frac{2\beta e^{-(\beta-1)^2/2\alpha^2}}{\alpha\sqrt{2\pi}}\\
 &\geqslant (1-e^{-161/512\alpha^2)}\P(-1/2\leqslant G\leqslant 1/2)-\frac{e^{-225/512\alpha^2}}{8\alpha\sqrt{2\pi}} \\
 &\geqslant (1-e^{-161/512)}\P(-1/2\leqslant G\leqslant 1/2)-\frac{e^{-1/2}}{8\sqrt{2\pi}}
\geqslant 0.06\enspace.
\end{align*}
It follows that, if $z\leqslant \Delta/16\sqrt{n}$,
\[
Q(f,z)\leqslant -0.06\sqrt{n}z/\Delta\enspace.
\]
Overall, for any $z\leqslant  \Delta/16\sqrt{n}$, with probability larger than $1-e^{-s}$,
\[
Q^{(n)}_{K}(f,z)\leqslant 16\frac{\sqrt{n}}{\Delta}\sqrt{\frac{D(F)}N}+\sqrt{\frac{2s}K}+4g(n)-0.06\frac{\sqrt{n}z}{\Delta}\enspace.
\]
As a conclusion, one can pick
\[
z_+=300\sqrt{\frac{D(F)}N}+20\Delta\bigg(\sqrt{\frac{2s}N}+4\frac{g(n)}{\sqrt{n}}\bigg)\enspace,
\]
since, by assumption, this quantity is smaller than $\Delta/16\sqrt{n}$.
\end{proof}
\chapter{Multivariate mean estimation}\label{Chap:MME}

Let $\|\cdot\|$ denote the Euclidean norm on $\R^d$.
Let $\cP_2$ denote the set of probability distributions on $\R^d$ such that 
\[
P[\|X\|^2]<\infty\enspace.
\] 
For any $P\in \cP_2$, denote by 
\[
\mu_P=PX\in \R^d,\qquad \Sigma_P=P[(X-\mu_P)(X-\mu_P)^T]\in\R^{d\times d}\enspace.
\]
The goal of the chapter is to build estimators of $\mu_P$ based on an i.i.d. sample of $P$, $\cD_N=(X_1,\ldots,X_N)$, with deviations bounded for all $P\in \cP_2$ by those of the empirical mean $\muh_e=N^{-1}\sum_{i=1}^NX_i$ when the vectors $X_i$ are Gaussian.
To compute the deviation bounds in the Gaussian case, we need to define the trace of $\Sigma_P$, $\text{Tr}(\Sigma_P)$ and its largest eigenvalue $\|\Sigma_P\|_{\text{op}}$.

\paragraph{Example: Least-squares density estimation}
The multivariate mean estimation problem is highly connected to a particular instance of unsupervized learning where one wants to recover, from an i.i.d. sample $X_1,\ldots,X_N$ taking values in a measurable space $\cX$, the distribution $P$ of $X$. 
Assume that $P$ has density ${\bar f}$ with respect to a known reference measure $\mu$, so recovering $P$ is equivalent to recover ${\bar f}$.
Assume that ${\bar f}\in L^2(\mu)$.
To estimate ${\bar f}$, choose an orthonormal basis $(\varphi_i)_{i\in \N}$ of $L^2(\mu)$.
The function ${\bar f}$ can be decomposed onto this basis ${\bar f}=\sum_{i\in \N}\beta_i\varphi_i$ (the  convergence of the series being in $L^2(\mu)$-sense). 
Moreover, the coefficient $\beta_i$ in this decomposition is the inner product in $L^2(\mu)$ between $\varphi_i$ and ${\bar f}$, $\beta_i=\int\varphi_i{\bar f}\rmd\mu$, that is, $\beta_i=P[\varphi_i]$.
Overall
\[
{\bar f}=\sum_{i\in \N}P[\varphi_i]\varphi_i\enspace.
\]
The projection method proceeds by cutting the sum and estimate the projections of ${\bar f}$ onto finite dimensional subspaces: 
\[
{\bar f}_d=\sum_{i=1}^dP[\varphi_i]\varphi_i\enspace.
\]
Estimating ${\bar f}_d$ is then equivalent to estimate the vector 
\[
\mu_P=P\bX,\qquad \text{where}\qquad \bX=\begin{bmatrix}
 \varphi_1(X)\\
 \vdots\\
 \varphi_d(X)
\end{bmatrix}\in \R^d
\enspace.
\]
The $2$-moment assumption is equivalent to the assumption that $P[\varphi_i^2]<\infty$ for all $i\in \{1,\ldots,d\}$.
It is a weaker requirement than the connection between $L^\infty$ and $L^2$-norms that is made to analyse the empirical mean estimator of $\mu_P$, $\forall \ba\in\R^d$, $\|\ba^T\bX\|_\infty\leqslant L\sqrt{d}\|\ba^T\bX\|_{L^2(\mu)}$, which, as $\varphi_1,\ldots,\varphi_d$ is an orthonormal system in $L^2(\mu)$ reduces to $\|\ba^T\bX\|_\infty\leqslant L\sqrt{d}\|\ba\|$.

\section{Deviations of the empirical mean in the Gaussian case}
In order to establish a relevant benchmark, start by computing the deviations of the empirical mean in the Gaussian case.

\begin{theorem}[Hanson-Wright]\label{thm:HW}
If the dataset $\cD_N=(X_1,\ldots,X_N)$ is a collection of i.i.d. Gaussian vectors with common distribution $\gauss(\mu,\Sigma)$, the empirical mean $\muh_e=N^{-1}\sum_{i=1}^NX_i$ satisfies
 \[
 \forall t>0,\qquad \P\bigg(\|\muh_e-\mu\|>\sqrt{\frac{\text{Tr}(\Sigma)}N}+\sqrt{\frac{2\|\Sigma\|_{\text{op}}t}N}\bigg)\leqslant e^{-t}\enspace.
 \]
\end{theorem}
\begin{proof}
Let $\bS=\{\bu\in\R^d: \|\bu\|=1\}$ and, for any $\bu\in \bS$, let $X_{\bu}=\bu^T\muh_e$. 
The random variables $X_{\bu}$ are Gaussian with expectation $\mu_{\bu}=\bu^T\mu$ and variance $\sigma_{\bu}^2=\bu^T\Sigma \bu/N$.
It follows that 
\begin{equation}\label{eq:Sigmaop}
\sigma^2=\sup_{\bu\in \bS}\sigma_{\bu}^2=\frac{\|\Sigma\|_{\text{op}}}N\enspace. 
\end{equation}
Moreover, 
\[
\|\muh_e-\mu\|=\sup_{\bu\in \bS}\bu^T(\muh_e-\mu)=\sup_{\bu\in \bS}(X_{\bu}-\mu_{\bu})\enspace.
\]
It comes from the concentration theorem for suprema of Gaussian processes that
\[
\forall t>0,\qquad \P\bigg(\|\muh_e-\mu\|>\E[\|\muh_e-\mu\|]+\sqrt{\frac{2\|\Sigma\|_{\text{op}}t}N}\bigg)\leqslant e^{-t}\enspace.
\]
Now, by Cauchy-Schwarz inequality,
\begin{align*}
 \E[\|\muh_e-\mu\|]&\leqslant \sqrt{\E\big[\|\muh_e-\mu\|^2\big]}=\sqrt{\frac1{N^2}\sum_{1\leqslant i,j\leqslant N}\E[(X_i-\mu)^T(X_j-\mu)]}\\
 &=\sqrt{\frac1{N}\E[(X-\mu)^T(X-\mu)]}\\
 &=\sqrt{\frac1{N}\text{Tr}(\E[(X-\mu)(X-\mu)^T])}=\sqrt{\frac{\text{Tr}(\Sigma)}N}\enspace.
\end{align*}
\end{proof}

\section{A first glimpse at minmax strategies}\label{sec:FirstMinmax}
Recall that $\bS$ denotes the unit sphere in $\R^d$: $\bS=\{\bu\in \R^d:\|\bu\|=1\}$.
The proof of Hanson-Wright theorem is based on the following representation of the risk:
\[
\|\muh_e-\mu\|=\sup_{\bu\in\bS}\bu^T(P_NX-\mu_P)\enspace.
\]
As $P_N$ is linear, this can be rewritten
\[
\|\muh_e-\mu\|=\sup_{\bu\in\bS}\{P_N[\bu^TX]-P[\bu^TX]\}\enspace.
\]
The risk bound is then based on the fact the empirical estimators $P_N[\bu^TX]$ of the univariate expectations $P[\bu^TX]$ have uniform deviations over the sphere $\bS$.
Therefore, there are three ingredients to prove Hanson-Wright's inequality:
\begin{itemize}
 \item[(i)] build estimators $\widehat{P}[\bu^TX]$ of the \emph{univariate} expectations $P[\bu^TX]$,
 \item[(ii)] bound the deviations of $|\widehat{P}[\bu^TX]-P[\bu^TX]|$ uniformly over the unit sphere $\bS$,
 \item[(iii)] deduce from the collection $\{\widehat{P}[\bu^TX],\ \bu\in\bS\}$ an estimator of $\mu_P$.
\end{itemize}
In Chapter~\ref{Chap:UME}, we presented various constructions that can be used to estimate the univariate expectations with sub-Gaussian guarantee when $P\in\cP_2$, therefore, extending step (i) will not be difficult.
In Chapter~\ref{Chap:ConcIn}, we showed uniform deviation bounds for these estimators that will be sufficient to extend step (ii).
Step (iii) is obvious for the empirical mean since, by linearity, $P_N[\bu^TX]=\bu^TP_N[X]$. However, none of the ``robust" estimators presented in Chapter~\ref{Chap:UME} is linear.
Therefore, extending step (iii) requires a new idea. 
A first idea, that appeared independently in various works such as \cite{CatGiu2017,LugosiMendelson2017-2} for example, is to consider the minmax estimator
\[
\muh\in\argmin_{\mu\in \R^d}\sup_{\bu\in\bS}|\bu^T\mu-\hat{P}[\bu^TX]|\enspace.
\]
If the minimum is not achieved, then $\muh$ in this definition can be replaced by any $\muh_N$ satisfying 
\[
\sup_{\bu\in\bS}|\bu^T\muh_N-\hat{P}[\bu^TX]|\leqslant \inf_{\mu\in \R^d}\sup_{\bu\in\bS}|\bu^T\mu-\hat{P}[\bu^TX]|+\frac1N\enspace.
\]
This would not affect the results of this section.

A very nice feature of this construction is that this estimator has a risk bounded from above by the uniform deviations of $\widehat{P}[\bu^TX]$ around $P[\bu^TX]$.
Actually, using successively the representation of the Euclidean norm as a supremum, the triangle inequality and the definition of $\muh$, it holds
\begin{align*}
 \|\muh-\mu_P\|&=\sup_{\bu\in\bS}|\bu^T(\muh-\mu_P)|\leqslant \sup_{\bu\in\bS}|\bu^T\muh-\hat{P}[\bu^TX]|+\sup_{\bu\in\bS}|\bu^T\mu_P-\hat{P}[\bu^TX]|\\
 &\leqslant 2\sup_{\bu\in\bS}|\bu^T\mu_P-\hat{P}[\bu^TX]|=2 \sup_{\bu\in\bS}\{|(P-\hat{P})[\bu^TX]|\}\enspace.
\end{align*}
We deduce from these remarks the following result.
\begin{lemma}\label{lem:MinMax1}
For any $\bu\in\bS$, let $\widehat{P}[\bu^TX]$ denote an estimator of the univariate expectation $P[\bu^TX]$. 
On the event $\Omega_r$ where these estimators have uniform deviations bounded from above by $r$,
\begin{equation}
\Omega_r=\bigg\{ \sup_{\bu\in\bS}\{|(P-\hat{P})[\bu^TX]|\}\leqslant r\bigg\}\enspace,
\end{equation}
the minmax estimator 
\[
\muh\in\argmin_{\mu\in \R^d}\sup_{\bu\in\bS}|\bu^T\mu-\hat{P}[\bu^TX]|\enspace,
\]
satisfies $\|\muh-\mu_P\|\leqslant 2r$.
\end{lemma}
Lemma~\ref{lem:MinMax1} shows that the risk of minmax estimators is bounded from above by $2r$ on the event $\Omega_r$. 
To show that the risk of the minmax estimator is bounded by $2r$ with high probability, it is therefore sufficient to compute $r$ such that $\Omega_r$ has high probability.
As a first example, consider the case where $\hat{P}[\bu^TX]=\MOM{K}{\bu^TX}$.
The following result is a corollary of the concentration theorem for suprema of MOM processes given in Theorem~\ref{thm:ConcSupMOM}.
\begin{theorem}\label{thm:MOMExpRd}
Let $P\in \cP_2$, $K\in\{1,\ldots,N\}$ and 
\[
r_K=128\sqrt{\frac{\text{Tr}(\Sigma_P)}N}\vee 4\sqrt{\frac{2\|\Sigma_P\|_{\text{op}}K}N}\enspace.
\]
Then,
\begin{equation}\label{eq:concboundMOMMeanRd}
 \P\bigg(\sup_{\bu\in\bS}|\MOM{K}{\bu^TX}-P[\bu^TX]|>r_K\bigg)\leqslant e^{-K/32}\enspace.
\end{equation}
In particular, the minmax MOM estimator 
\[
\muh\in\argmin_{\mu\in \R^d}\sup_{\bu\in\bS}|\bu^T\mu-\MOM{K}{\bu^TX}|\enspace
\]
satisfies
\[
\P(\|\muh-\mu_P\|\leqslant2r_K)\geqslant 1-e^{-K/32}\enspace.
\]
\end{theorem}
\begin{proof}
 The second result comes from \eqref{eq:concboundMOMMeanRd} and Lemma~\ref{lem:MinMax1}.
To prove~\eqref{eq:concboundMOMMeanRd}, consider the class of functions $F=\{\bu^T\cdot,\; \bu\in\bS\}$.
 By \eqref{eq:D(F)Lin}, this class satisfies $D(F)\leqslant \text{Tr}(\Sigma)$.
 Moreover, for any $\bu\in\bS$,
\begin{equation}\label{eq:maxvarLinfunc}
 \text{Var}(\bu^TX)=\bu^T\Sigma_P\bu\leqslant \|\Sigma_P\|_{\text{op}}\enspace.
\end{equation}
 Therefore, Theorem~\ref{thm:ConcSupMOM} shows \eqref{eq:concboundMOMMeanRd}.
\end{proof}

As a second application, consider smoothed MOM estimators.
\begin{theorem}\label{thm:SMOMExpRd}
Assume that there exists a known constant $v$ such that $v\geqslant \sqrt{\|\Sigma_P\|_{\text{op}}}$, let $\rho$ denote the function defined in \eqref{eq:DefRho} and consider the estimator $\hat{P}_K[\bu^TX]$ to be the solution of 
\[
\sum_{k=1}^K\rho\bigg(\sqrt{\frac NK}\frac{P_{B_k}[\bu^TX]-z}{v}\bigg)=0\enspace.
\]
Let 
\[
r_{K,s}=\sqrt{\frac{\text{Tr}(\Sigma_P)}N}+v\bigg(\sqrt{\frac{s}N} +g(N/K)\sqrt{\frac{K}{N}}\bigg)\enspace.
\]
Then, there exists an absolute constant $C>0$ such that, if 
\[
K\geqslant C\bigg(\frac{\text{Tr}(\Sigma_P)}{v^2}\vee g(N/K)^2\bigg)\enspace,
\]
then, for any $s\leqslant K/C$,
\begin{equation}\label{eq:concboundsMOMMeanRd}
 \P\bigg(\sup_{\bu\in\bS}|\hat{P}_K[\bu^TX]-P[\bu^TX]|>Cr_{K,s}\bigg)\leqslant e^{-s}\enspace.
\end{equation}
In particular, the minmax smooth-MOM estimator 
\[
\muh\in\argmin_{\mu\in \R^d}\sup_{\bu\in\bS}|\bu^T\mu-\hat{P}_K[\bu^TX]|\enspace
\]
satisfies, for any $s\leqslant K/C$,
\[
\P(\|\muh-\mu_P\|\leqslant2Cr_{K,s})\geqslant 1-e^{-s}\enspace.
\]
\end{theorem}
\begin{remark}
Theorem~\ref{thm:SMOMExpRd} improves upon Theorem~\ref{thm:MOMExpRd} as it shows, for example, that, when $g(N/K)\leqslant \gamma\sqrt{K/N}$, the minmax estimator $\muh$ based on smoothed MOM preliminary estimates with $K=\sqrt{N}$ is sub-Gaussian at any level $t\lesssim \sqrt{N}$.
On the other hand, this improvement holds under $L^3/L^2$ comparison to bound $g(N/K)$ and requires the knowledge of an upper bound $v$ on $\|\Sigma_P\|_{\text{op}}$.
\end{remark}
\begin{proof}
  The second result comes from \eqref{eq:concboundsMOMMeanRd} and Lemma~\ref{lem:MinMax1}.
Eq~\eqref{eq:concboundsMOMMeanRd} comes from Theorem~\ref{Thm:SMOMMinsker} applied to the class $F$ of linear functionals $F=\{\bu^T\cdot,\ \bu\in\bS\}$.
For this class of functions, $D(F)=\text{Tr}(\Sigma_P)$, see \eqref{eq:D(F)Lin} and $\sigma^2=\|\Sigma_P\|_{\text{op}}\leqslant v$, see Eq~\eqref{eq:maxvarLinfunc}.
\end{proof}

\section{Working with other norms}
Suppose here that one wants to estimate $\mu_P$ and that we measure the risk of an estimator $\muh$ by $|\muh-\mu_P|_*$, where $|\cdot|_*$ denote the dual norm of a norm $|\cdot|$ in $\R^d$.
In this section, denote the sphere of the norm $|\cdot|$ by
\[
\bS=\{u\in \R^d: |u|=1\}\enspace.
\]
Recall that the dual norm $|\cdot|_*$ is defined, for any $v\in \R^d$, by 
\[
|v|_*=\sup_{u\in \bS}u^Tv\enspace.
\]
Let $\bS_*$ denote the unit sphere for the dual norm $\bS_*=\{v\in \R^d:|v|_*=1\}$.
The construction of the previous section naturally extends to this framework.
Define the estimator 
\[
\muh\in\argmin_{\mu\in \R^d}\sup_{\bu\in\bS}|\bu^T\mu-\hat{P}[\bu^TX]|\enspace.
\]
The triangle inequality and the definition of $\muh$ show that
\begin{align*}
 |\muh-\mu_P|_*&=\sup_{\bu\in\bS}|\bu^T(\muh-\mu_P)|\leqslant \sup_{\bu\in\bS}|\bu^T\muh-\hat{P}[\bu^TX]|+\sup_{\bu\in\bS}|\bu^T\mu_P-\hat{P}[\bu^TX]|\\
 &\leqslant 2\sup_{\bu\in\bS}|\bu^T\mu_P-\hat{P}[\bu^TX]|=2 \sup_{\bu\in\bS}\{|(P-\hat{P})[\bu^TX]|\enspace.
\end{align*}
Assume that $\hat{P}[u^TX]=\MOM{K}{u^TX}$ for any $u\in \bS$. 
We have 
\[
\sup_{u\in \bS}\text{Var}(u^TX)=\sup_{u\in \bS}u^T\Sigma u=\|\Sigma\|_{**}\enspace.
\]
Moreover, the Rademacher complexity of the class $F$ of linear functions $x\mapsto u^Tx$, for all $u\in\bS$, can be computed as follows
\[
\sqrt{D(F)}=\E\bigg[\sup_{u\in \bS}\frac1{\sqrt{N}}\sum_{i=1}^N\epsilon_iX_i^Tu\bigg]=\E\bigg[\bigg|\frac1{\sqrt{N}}\sum_{i=1}^N\epsilon_iX_i\bigg|_*\bigg]\enspace.
\]
Proceeding as in Lemma~\ref{lem:MinMax1} yields the following result.
\begin{lemma}\label{lem:MeanRdGenNorm}
 Let $|\cdot|$ denote a norm on $\R^d$, let $\bS$ denote the unit sphere for $|\cdot|$ and let $|\cdot|_*$ denote the dual norm of $|\cdot|$.
The minmax estimator 
\[
\muh\in\argmin_{\mu\in \R^d}\sup_{\bu\in\bS}|\bu^T\mu-\MOM{K}{\bu^TX}|\enspace,
\]
 satisfies, with probability larger than $1-e^{-K/32}$,
 \[
|\muh-\mu_P|_*\leqslant  128\E\bigg[\bigg|\frac1{N}\sum_{i=1}^N\epsilon_iX_i\bigg|_*\bigg]\vee 4\sqrt{\frac{2\|\Sigma\|_{**} K}{N}}
 \]
\end{lemma}
\paragraph{Example: error rates in $\ell_1$-norm}
Assume that 
\[
|u|=\|u\|_\infty=\max_{i\in\{1,\ldots,d\}}|u_i|
\] so $|u|_*=\|u\|_1=\sum_{i=1}^d|u_i|$. As the term $\|\Sigma\|_{**}=\sup_{\bu:\|\bu\|_\infty\leqslant 1}\bu^T\Sigma\bu$ would also appear in the concentration of the empirical mean using Theorem~\ref{lem:supGauss}, it is sufficient to bound the main term in Lemma~\ref{lem:MeanRdGenNorm}.
We have
\[
\E\bigg[\bigg\|\frac1{N}\sum_{i=1}^N\epsilon_i(X_i-\mu_P)\bigg\|_1\bigg]=\sum_{j=1}^d\E\bigg[\bigg|\frac1{N}\sum_{i=1}^N\epsilon_i(X_{i,j}-\mu_{P,j})\bigg|\bigg]\enspace.
\]
Applying Cauchy-Schwarz inequality, we get
\[
\E\bigg[\bigg|\frac1{N}\sum_{i=1}^N\epsilon_i(X_{i,j}-\mu_{P,j})\bigg|\bigg]\leqslant \sqrt{\frac1{N^2}\sum_{i=1}^N\E[(X_{i,j}-\mu_{P,j})^2]}=\sqrt{\frac{\text{Var}(X_{1,j})}N}\enspace.
\]
Hence,
\[
\E\bigg[\bigg\|\frac1{N}\sum_{i=1}^N\epsilon_iX_i\bigg\|_1\bigg]\leqslant \frac{\sum_{j=1}^d\sqrt{\text{Var}(X_{i,j})}}{\sqrt{N}}\enspace.
\]
This is, up to multiplicative numerical constant the order of $\E[\|X-\mu_P\|_1]$ when $X\sim \gauss(\mu_P,\Sigma_P/N)$

\paragraph{Example: rates in sup-norm}
Assume that $|u|=\|u\|_1=\sum_{i=1}^d|u_i|$ so $|u|_*=\|u\|_\infty=\max_{i\in\{1,\ldots,d\}}|u_i|$.
For any $u\in \bS$,
\[
u^T\Sigma u\leqslant \max_{1\leqslant i\leqslant d}\big|\sum_{j=1}^d\Sigma_{i,j}u_j\big|\leqslant \max_{1\leqslant i,j\leqslant d}|\Sigma_{i,j}|\enspace.
\]
By Cauchy-Schwarz inequality,
\[
\max_{1\leqslant i,j\leqslant d}|\Sigma_{i,j}|=\max_{1\leqslant i\leqslant d}\Sigma_{i,i}\enspace.
\]
As this upper bound is achieved for $u$ a vector in the canonical basis, it follows that 
\[
\|\Sigma\|_{**}=\|\Sigma\|_{\infty}\enspace.
\]
The main term
\[
\E\bigg[\max_{j=1,\ldots,d}\bigg|\frac1{N}\sum_{i=1}^N\epsilon_iX_{i,j}\bigg|\bigg]\enspace,
\]
can be evaluated using higher moment assumptions on $X_{i,j}$.
Assume that $\E[|X_{i,j}|^p]<\infty$, for some $p\geqslant 2$. 
Then, for any $q\in\{2,\ldots, p\}$, Pisier's trick applies and gives
\begin{align*}
 \E\bigg[\max_{j=1,\ldots,d}\bigg|\frac1{N}\sum_{i=1}^N\epsilon_iX_{i,j}\bigg|\bigg]&\leqslant \bigg(\E\bigg[\max_{j=1,\ldots,d}\bigg|\frac1{N}\sum_{i=1}^N\epsilon_iX_{i,j}\bigg|^q\bigg]\bigg)^{1/q}\\
 &\leqslant \frac1N\bigg(\sum_{j=1}^d\E\bigg[\bigg|\sum_{i=1}^N\epsilon_iX_{i,j}\bigg|^q\bigg]\bigg)^{1/q}
\end{align*}
Now, apply Khinchine's inequality on moments of order $p$ for sums of independent random variables, see for examples \cite[Chapter 15]{BouLugMass13}.
It shows that
\begin{align*}
\bigg(\E\bigg[\bigg|\sum_{i=1}^N\epsilon_iX_{i,j}\bigg|^q\bigg]\bigg)^{1/q}&\leqslant 3\sqrt{q}\E\bigg[\bigg(\sum_{i=1}^NX_{i,j}^2\bigg)^{q/2}\bigg]^{1/q}\enspace.
\end{align*}
Then, by convexity of $x\mapsto x^{q/2}$,
\[
\bigg(\sum_{i=1}^NX_{i,j}^2\bigg)^{q/2}=N^{q/2}\bigg(\frac1N\sum_{i=1}^NX_{i,j}^2\bigg)^{q/2}\leqslant N^{q/2-1}\sum_{i=1}^N|X_{i,j}|^q\enspace.
\]
Therefore,
\begin{align*}
 \E\bigg[\max_{j=1,\ldots,d}\bigg|\frac1{N}\sum_{i=1}^N\epsilon_iX_{i,j}\bigg|\bigg]&\leqslant \frac{3\sqrt{q}}{N^{1/2+1/q}}\bigg(\sum_{j=1}^d\sum_{i=1}^N\E\big[\big|X_{i,j}\big|^{q}\big]\bigg)^{1/q}\enspace.
\end{align*}
As $X_i$ are i.i.d., this bound reduces to 
\begin{equation}\label{eq:BoundSup}
  \E\bigg[\max_{j=1,\ldots,d}\bigg|\frac1{N}\sum_{i=1}^N\epsilon_iX_{i,j}\bigg|\bigg]\leqslant \frac{3\sqrt{q}}{N^{1/2}}\bigg(\sum_{j=1}^d\E\big[\big|X_{1,j}\big|^{q}\big]\bigg)^{1/q}\enspace.
\end{equation}

We have proved the following result:
\begin{theorem}
Assume that there exists $p\geqslant 2$ such that $\E[|X_{i,j}|^p]<\infty$ and, for any $q\leqslant p$, let $M_q=\big(\sum_{j=1}^d\E\big[\big|X_{1,j}\big|^{q}\big]\big)^{1/q}$.
Let $\bS_1$ denote the sphere for the $\ell_1$-norm on $\R^d$. 
The estimator 
\[
\muh\in\argmin_{\mu\in \R^d}\sup_{\bu\in\bS_1}|\bu^T\mu-\MOM{K}{\bu^TX}|\enspace,
\]
 satisfies, with probability larger than $1-e^{-K/32}$,
 \[
\|\muh-\mu_P\|_{\infty}\leqslant  \frac{384}{\sqrt{N}}\inf_{q\leqslant p}(\sqrt{q}M_q)\vee 4\sqrt{\frac{2\|\Sigma\|_{\infty} K}{N}}\enspace.
 \]
\end{theorem}
\begin{remark}
Since $p\geqslant 2$, $\inf_{q\leqslant p}(\sqrt{q}M_q)\leqslant \sqrt{2}M_2=\sqrt{2\text{Tr}(\Sigma)}$. 

If the coordinates $X_{1,j}$ have a finite moment of order $p\geqslant 2\log d$.
Denote by $C_d=\max_{j\in\{1,\ldots,d\}}\E[|X_{1,j}|^{2\log d}]^{1/(2\log d)}$.
We have, for any $q\leqslant 2\log d$, $M_q\leqslant C_{d}d^{1/q}$, hence
\[
\inf_{q\leqslant p}(\sqrt{q}M_q)\leqslant C_d\sqrt{2\log d}d^{1/(2\log d)}=C_d\sqrt{2e\log d}\enspace.
\]
In this case, the optimal sub-Gaussian inequality is therefore recovered only under stronger moment assumption on the vectors $X_i$.
\end{remark}

\section{PAC-Bayesian analysis}\label{Sec:PacBayes}
Applying the minmax strategy with MOM (or smoothed MOM) estimators $\MOM{K}{\bu^TX}$ of univariate expectations yields estimators $\muh$ of $PX$ with sub-Gaussian tails but the constants involved in the deviation property are a bit loose compared to the Gaussian case of Hanson-Wright theorem.
As for univariate mean estimation, there exist alternatives with much better performance from this perspective.
The material of this section is borrowed from \cite{CatGiu2017}.
Let $\psi$ denote a function such that 
\begin{equation}\label{def:psi}
\forall t\in \R,\qquad  -\log(1-t+t^2/2)\leqslant \psi(t)\leqslant \log(1+t+t^2/2)\enspace. 
\end{equation}
For example, one can verify that the following function satisfies \eqref{def:psi},
\[
\psi(t)=
\begin{cases}
 -2\sqrt{2}/3&\ \text{ if }\ t<-\sqrt{2}\\
 t-t^3/6&\ \text{ if }\ t\in [-\sqrt{2},\sqrt{2}]\\
 2\sqrt{2}/3&\ \text{ if }\ t>\sqrt{2}
\end{cases}
\enspace.
\]
Let $\lambda>0$, $\beta>0$, $\bI_d$ denote the identity matrix in $\R^d$ and, for any $\bu\in \bS$,  $\rho_{\bu}=\gauss(u,\beta \bI_d)$. 
Define the estimators of univariate expectations $P[\bu^TX]$ as 
\[
\hat{P}_{\lambda,\beta}[\bu^TX]=\frac1{N\lambda}\sum_{i=1}^N\int \psi(\lambda \bv^TX_i){\rm d}\rho_{\bu}(\bv)\enspace.
\]
These new estimators are not translation invariant which means that, if $b$ denotes a deterministic quantity, one cannot guarantee that $\hat{P}_{\lambda,\beta}[\bu^TX+b]=\hat{P}_{\lambda,\beta}(\bu^TX)+b$. 
In particular, $\hat{P}_{\lambda,\beta}[\bu^TX]-P[\bu^TX]$ may not be equal to $\hat{P}_{\lambda,\beta}[\bu^TX-P[\bu^TX]]$.
Therefore, the analysis of these estimators is a bit more tricky than for MOM.
The following result shows the deviation properties of the minmax estimators based on the preliminary estimates $\hat{P}_{\lambda,\beta}[\bu^TX]$.
\begin{theorem}\label{thm:PrelimPBEst}
Assume that there exist \emph{known} constants $T$ and $v$, $v\leqslant T$, such that the matrix $\overline{\Sigma}_P=P[XX^T]$ satisfies
\begin{gather*}
\text{Tr}(\overline{\Sigma}_P)\leqslant T,\qquad  \|\overline{\Sigma}_P\|_{\text{op}}\leqslant v\enspace.
\end{gather*}
Let $\lambda$ and $\beta$ denote the following quantities
\begin{gather*}
 \lambda=\sqrt{\frac{2\log(\delta^{-1})}{Nv}},\qquad \beta=\frac1{\lambda\sqrt{NT}}=\sqrt{\frac{v}{2T\log(1/\delta)}}\enspace.
\end{gather*}
Then, with probability at least $1-\delta$, 
\begin{equation}\label{eq:supdevMestRd}
\sup_{u\in \bS}|\hat{P}_{\lambda,\beta}[\bu^TX]-P[\bu^TX]|\leqslant \sqrt{\frac TN}+\sqrt{\frac{2v\log(1/\delta)}N}\enspace. 
\end{equation}
In particular, the minmax estimator
\[
\muh\in\argmin_{\mu\in \R^d}\sup_{\bu\in\bS}|\bu^T\mu-\hat{P}_{\lambda,\beta}[\bu^TX]|\enspace
\]
satisfies
\[
\P\bigg(\|\muh-\mu_P\|\leqslant 2\bigg(\sqrt{\frac TN}+\sqrt{\frac{2v\log(1/\delta)}N}\bigg)\bigg)\geqslant 1-\delta\enspace.
\]
\end{theorem}

\begin{proof}
The second result comes from \eqref{eq:supdevMestRd} and Lemma~\ref{lem:MinMax1}.
Let us focus on proving \eqref{eq:supdevMestRd}.
Fix $\mu=\rho_0$. Let $\Gamma:\bS\times \R^d\to\R$, $(\bv,\bx)\mapsto \Gamma_{\bv}(\bx)=\psi(\lambda \bv^T\bx)$, so 
 \[
 P_N\bigg[\int\Gamma_{\bv}{\rm d}\rho_{\bu}(\bv)\bigg]=\lambda\hat{P}_{\lambda,\beta}[\bu^TX]\enspace.
 \]
 Moreover, by definition of $\psi$, \eqref{def:psi},
\begin{align*}
  \int \log P[e^{\Gamma_{\bv}}]{\rm d}\rho_{\bu}(\bv)&= \int \log P[e^{\psi(\lambda \bv^TX)}]{\rm d}\rho_{\bu}(\bv)\\
  &\leqslant \int \log \bigg(1+\lambda P[\bv^TX]+\frac{\lambda^2}2P[(\bv^TX)^2]\bigg){\rm d}\rho_{\bu}(\bv)\enspace.
\end{align*}
As $\log(1+x)\leqslant x$ for any $x> -1$,
\begin{align*}
 \int \log \bigg(1+\lambda P[\bv^TX]&+\frac{\lambda^2}2P[(\bv^TX)^2]\bigg){\rm d}\rho_{\bu}(\bv)\\
 &\leqslant \lambda P\bigg[\int \bv^TX{\rm d}\rho_{\bu}(\bv)\bigg]+\frac{\lambda^2}2P\bigg[\int(\bv^TX)^2{\rm d}\rho_{\bu}(\bv)\bigg]\enspace.
 \end{align*}
 Conditionally on $X$, when $\bv$ is distributed as $\rho_{\bu}$, $\lambda \bv^TX$ is distributed according to a Gaussian $\gauss(\lambda \bu^TX,\beta\lambda^2\|X\|^2)$. 
 Hence,
\begin{align*}
 \int \log \bigg(1+\lambda P[\bv^TX]&+\frac{\lambda^2}2P[(\bv^TX)^2]\bigg){\rm d}\rho_{\bu}(\bv)\\
 &\leqslant\lambda\bigg(P\big[\bu^TX\big]+\frac{\lambda}2P\big[(\bu^TX)^2+\beta\|X\|^2\big]\bigg)\\
 &\leqslant \lambda\bigg(P\big[\bu^TX\big]+\frac{\lambda}2\big(v+\beta T\big)\bigg)\enspace.
\end{align*}
 Moreover, as $K(\rho_{\bu},\rho_0)=1/(2\beta)$, it follows from the PAC-Bayesian inequality, see Theorem~\ref{thm:PBIneq}, that, with probability $1-\delta$, for any $\bu\in\bS$,
\begin{equation}\label{eq:PBB1}
 \hat{P}_{\lambda,\beta}[\bu^TX]\leqslant P[\bu^TX]+\frac{\lambda}2\big(v+\beta T\big)+\frac{(1/2\beta)+\log(1/\delta)}{\lambda N}\enspace.  
\end{equation}
 The choice of parameters now ensures that
\begin{align*}
 \frac{\lambda v}2&=\frac12\sqrt{\frac{2v\log(1/\delta)}N}\enspace,\\
 \frac{\lambda \beta T}2&=\frac12\sqrt{\frac TN}\enspace,\\
 \frac1{2\beta\lambda N}&=\frac12\sqrt{\frac TN}\enspace,\\
 \frac{\log(1/\delta)}{\lambda N}&=\frac12\sqrt{\frac{2v\log(1/\delta)}N}\enspace.
\end{align*}
Plugging these estimates into \eqref{eq:PBB1} shows that
\[
\P\bigg(\sup_{\bu\in \bS}(\hat{P}_{\lambda,\beta}[\bu^TX]-P[\bu^TX])\leqslant \sqrt{\frac TN}+\sqrt{\frac{2v\log(1/\delta)}N}\bigg)\geqslant 1-\delta\enspace.
\]
As $\bS$ is symmetric, Eq~\eqref{eq:supdevMestRd} is proved and therefore the theorem is established.
%
\end{proof}

The problem with Theorem~\ref{thm:PrelimPBEst} is that it involves upper bounds on the $L^2$ moments $\overline{\Sigma}_P$ rather than on the covariance matrix $\Sigma_P$. 
Fortunately, there is a simple trick to deduce from the estimator $\muh$ an estimator with sub-Gaussian deviations based on the actual covariance matrix $\Sigma_P$.
\begin{theorem}
 Assume that there exist \emph{known} constants ${\bar T}$, ${\bar v}$ and $b$ such that
\begin{equation}\label{eq:HypPBGen}
 \text{Tr}(\Sigma_P)\leqslant {\bar T},\qquad \|\Sigma_P\|_{\text{op}}\leqslant{\bar v},\qquad \|\mu_P\|^2\leqslant b\enspace.
\end{equation}
Fix $\delta\in(0,1)$ and let 
\[
A:=4\bigg(\sqrt{{\bar T}+b}+\sqrt{2({\bar v}+b)\log(1/\delta)}\bigg)^2\enspace.
\]
For any $k\in\{1,\ldots,N-1\}$, there exists an estimator $\muh$ of $\mu$ such that
\[
\P\bigg(\|\muh-\mu_P\|\leqslant 2\bigg(\sqrt{\frac{{\bar T}+A/k}{N-k}}+\sqrt{\frac{2({\bar v}+A/k)\log(1/\delta)}{N-k}}\bigg)\bigg)\leqslant 1-2\delta\enspace.
\]
\end{theorem}
\begin{remark}
If $d$ is fixed and $N\to\infty$, one can choose $k\asymp \sqrt{N}$ to deduce that, for any $\epsilon>0$, there exists $N_0$ such that, for any $N\geqslant N_0$, there exists an estimator $\muh$ of $\mu$ such that
\[
\P\bigg(\|\muh-\mu_P\|\leqslant (2+\epsilon)\bigg(\sqrt{\frac{{\bar T}}{N}}+\sqrt{\frac{2{\bar v}\log(1/\delta)}{N}}\bigg)\bigg)\leqslant 1-2\delta\enspace.
\]
Therefore, $\muh$ achieves much better constants than the minmax estimator based on MOM preliminary estimators, see Theorem~\ref{thm:MOMExpRd} or on smoothed MOM estimators, see Theorem~\ref{thm:SMOMExpRd}.
Actually, these constants are asymptotically not worse than twice the optimal constants of the Hanson-Wright theorem.
On the other hand, the knowledge of upper bounds on both $\text{Tr}(\Sigma_P)$ and $\|\Sigma_P\|_{\text{op}}$ is mandatory for the construction of these estimators.
\end{remark}
\begin{proof}
Under Assumption~\eqref{eq:HypPBGen}, the constants $T$ and $v$ satisfy the requirements of  Theorem~\ref{thm:PrelimPBEst}, where
\[
T\geqslant {\bar T}+b,\qquad v\geqslant{\bar v}+b\enspace.
\]
To build an estimator with correct sub-Gaussian deviations, split the sample in two parts $(X_1,\ldots,X_k)$ and $(X_{k+1},\ldots,X_N)$.
With the first sample $X_1,\ldots,X_k$, build the minmax estimator ${\bar \mu}$ of $\mu_P$ of Theorem~\ref{thm:PrelimPBEst} with the constants $T$ and $v$. 
According to Theorem~\ref{thm:PrelimPBEst}, $\P(\Omega_1)\geqslant 1-\delta$, where
\[
\Omega_1=\bigg\{\|{\bar \mu}-\mu_P\|\leqslant \sqrt{\frac{A}k}\bigg\}\enspace.
\]
The estimator $\muh$ will be a minmax estimator of $\mu_P$ based on Theorem~\ref{thm:PrelimPBEst}, using the sample $(X_{k+1}-{\bar \mu},\ldots,X_N-{\bar \mu})$. 
To choose appropriate constants $T'$ and $v'$ in this theorem, one has to bound the $L^2$-moments of the sample $(X_{k+1}-{\bar \mu},\ldots,X_N-{\bar \mu})$. 
The idea is to work conditionally on $\cF_k$, the $\sigma$-algebra generated by $X_1,\ldots,X_k$.
It holds
\[
P[(X-{\bar \mu})(X-{\bar \mu})^T|\cF_k]=\Sigma_P+({\bar \mu}-\mu_P)({\bar \mu}-\mu_P)^T\enspace.
\]
In particular,
\begin{gather*}
\Tr[P[(X-{\bar \mu})(X-{\bar \mu})^T|\cF_k]]\leqslant \Tr(\Sigma_P)+\|{\bar \mu}-\mu_P\|^2\enspace,\\
\|P[(X-{\bar \mu})(X-{\bar \mu})^T|\cF_k]\|_{\text{op}}\leqslant \|\Sigma_P\|_{\text{op}}+\|{\bar \mu}-\mu_P\|^2\enspace. 
\end{gather*}
This bound cannot be used to build $\muh$ but it holds
\begin{gather*}
\Tr[P[(X-{\bar \mu})(X-{\bar \mu})^T|\cF_k,\Omega_1]]\leqslant \Tr(\Sigma_P)+\frac{A}k\enspace,\\
\|P[(X-{\bar \mu})(X-{\bar \mu})^T|\cF_k,\Omega_1]\|_{\text{op}}\leqslant \|\Sigma_P\|_{\text{op}}+\frac{A}k\enspace. 
\end{gather*}
This suggests to build the minmax estimator using Theorem~\ref{thm:PrelimPBEst}, based on the sample $(X_{k+1}-{\bar \mu},\ldots,X_N-{\bar \mu})$, with the constants ${\bar T}+A/k$ and ${\bar v}+A/k$. 
Let 
\[
r=2\bigg(\sqrt{\frac{{\bar T}+A/k}N}+\sqrt{\frac{2({\bar v}+A/k)\log(1/\delta)}N}\bigg)\enspace,
\]
According to these preliminary computations, it holds
\[
\P\bigg(\bigg\{\|\muh-\mu_P\|>2\bigg(\sqrt{\frac{{\bar T}+A/k}N}+\sqrt{\frac{2({\bar v}+A/k)\log(1/\delta)}N}\bigg)\bigg\} \bigg|\cF_k,\Omega_1\bigg)\leqslant \delta\enspace.
\]
Therefore, 
\begin{align*}
\P\big(\|\muh-\mu_P\|> r\big) &\leqslant 1-\P(\Omega_1)+\P\big(\Omega_1\cap\big\{\|\muh-\mu_P\|> r\}\big)\\
&\leqslant \delta+\E\big[\P\big(\Omega_1\cap\big\{\|\muh-\mu_P\|> r\}\big|\cF_k\big)\big]\leqslant 2\delta\enspace.
\end{align*}
\end{proof}

\section{Toward a generic minmax strategy}\label{sec:GenStrat}
This section introduces a minmax strategy that can be extended more easily than the one presented in Section~\ref{sec:FirstMinmax} to any learning problem where ERM can be used.
The starting point of this generic construction is that the multivariate expectation $\mu_P$ is solution of a minimization problem where the objective function is a univariate expectation
\[
\mu_P\in \argmin_{\mu\in \R^d}P[\|X-\mu\|^2]\enspace.
\]
A first ``natural" idea to build an estimator from this formulation would be to consider
\[
\muh_{\text{nat}}\in\argmin_{\mu\in \R^d}\MOM{K}{\|X-\mu\|^2}\enspace.
\]
This construction would be similar to that of the empirical mean $\ERM_{\text{emp}}$ which is the ERM associated to the loss $\|f-x\|^2$, i.e. $\ERM_{\text{emp}}\in \argmin_{f\in F}P_N[\|X-f\|^2]$.
It turns out that min MOM estimators have suboptimal deviation bounds even under stronger assumption on $F$, see for example~\cite{lecue2018robust}. 
The reason is that the lack of linearity of the median prevents from using localization ideas that yields optimal deviation rates of the ERM.
Instead of simply minimizing MOM, the idea is to reformulate the problem in order to build an estimator based on estimators of the expectations of the \emph{increments} of loss rather than on the losses themselves. 
To clarify this idea, remark that, by linearity of $P$, the target $\mu_P$ is also solution of the following minmax problem:
\[
\mu_P\in \argmin_{\mu\in \R^d}\sup_{\nu\in \R^d}P[\|X-\mu\|^2-\|X-\nu\|^2]\enspace.
\]
Now, one can obtain an estimator of $\mu_P$ simply by plugging in this formulation estimators $\widehat{P}[\|X-\mu\|^2-\|X-\nu\|^2]$ of the univariate expectations $P[\|X-\mu\|^2-\|X-\nu\|^2]$.
Contrary to the previous minmax strategy that used the specific form of the risk function, this new construction is completely generic and can be extended to many learning tasks.
In the remaining of this section, we present as a first example of application an analysis of the minmax MOM estimator 
\[
\muh\in \argmin_{\mu\in \R^d}\sup_{\nu\in \R^d}\MOM{K}{\|X-\mu\|^2-\|X-\nu\|^2}\enspace.
\]
The minmax MOM estimator differs from the min MOM since the median is not linear.
In the following chapters, we shall extend the analysis of minmax MOM estimators to much more general learning tasks and show sub-Gaussian oracle inequalities in several classical problems.
\begin{theorem}\label{thm:MOMMean1}
 Let $X_1,\ldots,X_N$ denote i.i.d. realizations of a distribution $P\in \cP_2$. Let $K\leqslant N$ and let 
 \[
 \muh\in\argmin_{\mu\in \R^d}\sup_{\nu\in \R^d}\MOM{K}{\|X-\mu\|^2-\|X-\nu\|^2}\enspace.
 \]
 Then
 \[
\P\bigg(\|\muh-\mu_P\|>(\sqrt{2}+1)\bigg(128\sqrt{\frac{\text{Tr}(\Sigma_P)}N}\vee4\sqrt{\frac{\|\Sigma_P\|_{\text{op}}K}N}\bigg)\bigg)\leqslant e^{-K/32}\enspace.
 \]
\end{theorem}
\begin{remark}
Remark that the constants here are slightly worse than for the first minmax strategy presented in Theorem~\ref{thm:MOMExpRd}.
\end{remark}
\begin{proof}
 Define for any $\mu\in \R^d$ its score
 \[
 S(\mu)=\sup_{\nu\in \R^d}\MOM{K}{\|X-\mu\|^2-\|X-\nu\|^2}\enspace.
 \]
 On one hand, we have
 \[
 S(\muh)\leqslant S(\mu_P)=\sup_{\nu\in \R^d}\MOM{K}{\|X-\mu_P\|^2-\|X-\nu\|^2}\enspace.
 \]
 On the other hand, for any $\nu\in \R^d$,
\begin{align}
\notag S(\nu)&\geqslant \MOM{K}{\|X-\nu\|^2-\|X-\mu_P\|^2}\\
\label{eq:LBSnu}&=-\MOM{K}{\|X-\mu_P\|^2-\|X-\nu\|^2}\enspace. 
\end{align}
This suggests to analyse the process
 \[
\{\MOM{K}{\|X-\mu_P\|^2-\|X-\nu\|^2},\ \nu\in\R^d\}\enspace.
 \]
 As, for any $\nu\in \R^d$,
 \[
 \|X-\mu_P\|^2-\|X-\nu\|^2=2(X-\mu_P)^T(\nu-\mu_P)-\|\nu-\mu_P\|^2\enspace,
 \]
 we have 
\begin{equation}\label{eq:SmuP}
 \MOM{K}{\|X-\mu_P\|^2-\|X-\nu\|^2}=2\|\nu-\mu_P\|\MOM{K}{(X-\mu_P)^T\frac{\nu-\mu_P}{\|\nu-\mu_P\|}}-\|\nu-\mu_P\|^2\enspace.
\end{equation}
  Therefore, it is sufficient to analyse the process 
 \[
 \{\MOM{K}{(X-\mu_P)^T\bu},\ \bu\in \bS\},\qquad \bS=\{\bu\in \R^d: \|u\|=1\}\enspace.
 \]
 Recall that in Theorem~\ref{thm:MOMExpRd}, we showed that, for  
\[
r_K=128\sqrt{\frac{\text{Tr}(\Sigma_P)}N}\vee 4\sqrt{\frac{\|\Sigma_P\|_{\text{op}}K}N}\enspace,
\]
then,
\[
 \P\bigg(\sup_{\bu\in\bS}|\MOM{K}{\bu^TX}-P[\bu^TX]|>r_K\bigg)\leqslant e^{-K/32}\enspace.
\]
%
Let $\Omega=\{\forall \bu\in \bS, \ \sup_{\bu\in \bS}\MOM{K}{(X-\mu_P)^T\bu}\leqslant r_K\}$, so $\P(\Omega)\geqslant 1-e^{-K/32}$.
On $\Omega$, by \eqref{eq:SmuP}, 
\[
S(\mu_P)\leqslant \sup_{a\in \R}\{2ar_K-a^2\}\leqslant r_K^2\enspace.
\]
Therefore, by definition of $\muh$, $S(\muh)\leqslant r_K^2$.
On the other hand, on $\Omega$, by \eqref{eq:LBSnu} and \eqref{eq:SmuP}, for any $\nu\in \R^d$, 
\[
S(\nu)\geqslant -2\|\nu-\mu_P\|r_K+\|\nu-\mu_P\|^2
\]
In particular, therefore, on $\Omega$,
\[
-2\|\muh-\mu_P\|r_K+\|\muh-\mu_P\|^2\leqslant r_K^2\enspace.
\]
Solving this inequality shows that, on $\Omega$,
\[
\|\muh-\mu_P\|\leqslant (\sqrt{2}+1)r_K\enspace.
\]
This concludes the proof of the theorem.
\end{proof}

\section{Resistance to outliers}\label{sec:Outliers}
In this section, consider the $O\cup I$ framework where $(X_i)_{i\in I}$ are independent random variables such that 
\[
\forall i\in I,\quad \E[X_i]=\mu_P,\quad \E[(X_i-\mu)(X_i-\mu)^T]=\Sigma_P\enspace.
\]
No assumption is granted on the outliers $(X_i)_{i\in O}$. 
Denote by $\epsilon=|O|/N$.
\subsection{Resistance of MOM estimators}
This section investigate minmax MOM estimator in this framework.
\begin{theorem}
Assume that $K\geqslant 20N\epsilon/9$. Denote by $\ERM_K$ minmax MOM estimator
 \[
 \muh_K\in\argmin_{\mu\in \R^d}\sup_{\nu \in \R^d}\MOM{K}{\|X-\mu\|^2-\|X-\nu\|^2}\enspace.
 \]
 Then there exists an absolute constant $C$ such that, with probability at least $1-e^{-K/C}$,
 \[
 \|\muh_K-\mu\|^2\leqslant C\bigg(\frac{\text{Tr}(\Sigma)}{N}\vee \frac{\|\Sigma\|_{\text{op}}K}N\bigg)\enspace.
 \]
\end{theorem}
\begin{proof}
 The proof uses intensively results obtained in the proof of Theorem~\ref{thm:MOMMean1}.
 Proceeding as in this proof, denote, for any $\xi\in\R^d$, by 
 \[
 S(\xi)=\sup_{\nu\in \R^d}\MOM{K}{\|X-\xi\|^2-\|X-\nu\|^2}\enspace,
 \]
and recall that $S(\mu)\leqslant R_K^2$, where
 \[
 R_K=\sup_{\bu\in\bS}\MOM{K}{\bu^T(X-\mu)}\enspace.
 \]
 Denote by $\cK$ the indexes of blocks $B_k\subset I$.
 It is clear that $|\cK|\geqslant K-N\epsilon$. 
 Applying the general version of Lugosi and Mendelson concentration bound for median-of-means processes, with probability at least $1-e^{-(K-N\epsilon)/c^*}$, there exists at leat $9(K-N\epsilon)/10\geqslant K/2$ blocks $B_k\subset I$ where 
 \[
\forall \bu\in \bS,\qquad P_{B_k} [\bu^T(X-\mu)]\leqslant c^*\bigg(\sqrt{\frac{20\text{Tr}(\Sigma)}{9N}}\vee \sqrt{\frac{\|\Sigma\|_{\text{op}}K}N}\bigg)\enspace.
 \]
 This implies in particular that, with probability at least $1-e^{-K/2c^*}$,
 \[
 r_K\leqslant c^*\bigg(\sqrt{\frac{20\text{Tr}(\Sigma)}{9N}}\vee \sqrt{\frac{\|\Sigma\|_{\text{op}}K}N}\bigg)\enspace.
 \]
 The proof terminates as the one of Theorem~\ref{thm:MOMMean1}.
\end{proof}
\begin{remark}
 The condition $K\gtrsim N\epsilon$ implies that the convergence rate of the minmax MOM estimator $\muh_K$ is bounded from above by 
 \[
 C\bigg(\frac{\text{Tr}(\Sigma)}{N}\vee \frac{\|\Sigma\|_{\text{op}}K}N\vee \|\Sigma\|_{\text{op}}\epsilon\bigg)\enspace.
 \]
 In particular, these rates match those obtained on clean datasets in Theorem~\ref{thm:MOMMean1} as long as $\epsilon\lesssim r(\Sigma)/N$, where $r(\Sigma)$ is the effective rank of $\Sigma$, $r(\Sigma)=\text{Tr}(\Sigma)/\|\Sigma\|_{\text{op}}$.
\end{remark}
\subsection{Depth}
The purpose of this section is to investigate optimality of the proportion of outliers $\epsilon \lesssim r(\Sigma)/N$ allowed by MOM estimators by comparing with the Gaussian case.
The material of this section is an adaptation of results obtained in \cite{Chao2017}.
Assume that inliers have Gaussian distribution $P_I=\gauss(\mu,\sigma^2\bI_d)$. 
Consider the following Gaussian $O\cup I$ framework where the dataset $\cD_N$ contains $|I|$ data $(X_i)_{i\in I}$ i.i.d. with common distribution $P_I$ and $|O|$ outliers $(X_i)_{i\in O}$ that can be anything. Tuckey's depth (hereafter called depth) of any $\nu\in \R^d$ relatively to a distribution $\P$ on $\R^d$ is defined by 
\[
D(\nu,\P)=\inf_{\bu\in \bS}\P(\bu^TX\leqslant \bu^T\nu)\enspace.
\]
Tuckey's depth (hereafter called depth) of any $\nu\in \R^d$ relatively to the dataset $\cD_N$ on $\R^d$ is the empirical version of $D(\nu,\P)$:
\[
D(\nu,\cD_N)=\inf_{\bu\in \bS}\frac1N\sum_{i=1}^N{\bf 1}_{\{\bu^TX_i\leqslant \bu^T\nu\}}\enspace.
\]
In other words, $D(\nu,\cD_N)$ is Tuckey's depth relative to the empirical measure $P_N$.
Tuckey's median is the deepest point in $\R^d$, that is 
\[
\muh_{\text{Tuc}}\in\argmax_{\nu\in \R^d}D(\nu,\cD_N)\enspace.
\]
The purpose of this section is to establish the following result.
\begin{theorem}\label{thm:TucDepth}
Denote by $\muh_{\text{Tuc}}$ Tuckey's median. 
Assume the Gaussian $O\cup I$ framework, denote by $\epsilon=|O|/N$. 
There exist absolute constants $C_1,C_2$ such that, for any $\delta\in (0,1)$ satisfying $C_1(d+\epsilon^2+\log(1/\delta))/N<1$,
\[
\P\bigg(\|\muh_{\text{Tuc}}-\mu\|^2\leqslant C_2\bigg(\frac{d}N+\epsilon^2+\frac{\log(1/\delta)}N \bigg)\bigg)\enspace.
\]
\end{theorem}
\begin{remark}
 In this example, the covariance matrix $\Sigma$ of the inliers is the identity matrix $\bI_d$, so $\text{Tr}(\Sigma)=d$, $\|\Sigma\|_{\text{op}}=1$ and the effective rank $r(\Sigma)=d$. 
 It follows that the rates of the convergence of MOM estimators in the clean case are not downgraded if the proportion of outliers $\epsilon\lesssim d/N$. 
 It comes from Theorem~\ref{thm:TucDepth} that Tuckey's median tolerates much outliers since a proportion $\epsilon\lesssim \sqrt{d/N}$ is allowed here.
 Of course, the result for Tuckey's median only holds when inliers are Gaussian and it provides the optimal sub-Gaussian dependence on the covariance matrix of $X$ only in the case where this covariance is bounded from bellow by the identity.
 These conditions are way more restrictive than those required for MOM estimators. 
 Moreover, one can show that the dependence $\epsilon\leqslant d/N$ is optimal if we allow inliers with heavier tails than Gaussian.
 Nevertheless, this shows that the number of outliers allowed by minmax MOM estimators is not optimal in general and opens an interesting question: is there some estimator achieving optimal sub-Gaussian deviation bounds assuming only that $P\in \cP_2$ and whose dependency in the number of outliers is always optimal? 
\end{remark}
\begin{proof}
 Assume, without loss of generality, that $\mu=0$. 
 Define, for any $\bu\in\bS$ and any $\nu\in\R^d$, the half space $H_{\bu,\nu}=\{\bx\in \R^d: \bu^T\bx\leqslant \bu^T\nu\}$.
 Define 
 \[
 P_N^{(I)}(H_{u,\nu})=\frac1{|I|}\sum_{i\in I}{\bf 1}_{\{X_i\in H_{u,\nu}\}}\enspace.
 \]
 The set of Half spaces $H_{u,\nu}$ is the set of all affine half spaces in $\R^d$, it has Vapnik-Chervonenkis $\text{VC}$ dimension $d+1$.
 It comes from standard VC theory, see \cite{MR1719582} that there exists an absolute constant $C$ such that, with probability larger than $1-\delta$,
 \[
\sup_{\bu\in\bS,\nu\in\R^d}( P_N^{(I)}-P_I)(H_{u,\nu})\leqslant C\bigg(\sqrt{\frac{d}N}+\sqrt{\frac{\log(1/\delta)}N}\bigg)\enspace.
 \]
 This result implies that
\begin{equation}\label{eq:VCB}
\sup_{\bu\in\bS,\nu\in\R^d}(D(\nu,(X_i)_{i\in I})-D(\nu,P_I))\leqslant C\bigg(\sqrt{\frac{d}N}+\sqrt{\frac{\log(1/\delta)}N}\bigg)\enspace. 
\end{equation}
 
 In particular,
 \[
 D(\muh_{\text{Tuc}},P_I)\geqslant D(\muh_{\text{Tuc}},(X_i)_{i\in I})- C\bigg(\sqrt{\frac{d}N}+\sqrt{\frac{\log(1/\delta)}N}\bigg)\enspace.
 \]
 Now, for any $\nu\in \R^d$, 
\begin{align*}
 D(\nu,(X_i)_{i\in I})&=\inf_{\bu\in \bS}\frac1{|I|}\sum_{i\in I}{\bf 1}_{\{\bu^TX_i\leqslant \bu^T\nu\}}\\
 &\geqslant \frac{N}{|I|}D(\nu,\cD_N)-\frac{|O|}{|I|}\\
 &=\frac{N}{|I|}(D(\nu,\cD_N)-\epsilon)\geqslant \frac1{1-\epsilon}(D(\nu,\cD_N)-\epsilon)\enspace.
\end{align*}
 Hence,
 \[
 D(\muh_{\text{Tuc}},P_I)\geqslant \frac1{1-\epsilon}(D(\muh_{\text{Tuc}},\cD_N)-\epsilon)- C\bigg(\sqrt{\frac{d}N}+\sqrt{\frac{\log(1/\delta)}N}\bigg)\enspace.
 \]
 By definition of $\muh_{\text{Tuc}}$, this implies that
 \[
  D(\muh_{\text{Tuc}},P_I)\geqslant \frac1{1-\epsilon}(D(\mu,\cD_N)-\epsilon)- C\bigg(\sqrt{\frac{d}N}+\sqrt{\frac{\log(1/\delta)}N}\bigg)\enspace.
 \]
 As for any $\nu\in \R^d$, $ND(\nu,\cD_N)\geqslant |I|D(\nu,(X_i)_{i\in I})$, this implies $D(\nu,\cD_N)\geqslant (1-\epsilon)D(\nu,(X_i)_{i\in I})$, thus
 \[
  D(\muh_{\text{Tuc}},P_I)\geqslant D(\mu,(X_i)_{i\in I})-\frac{\epsilon}{1-\epsilon}- C\bigg(\sqrt{\frac{d}N}+\sqrt{\frac{\log(1/\delta)}N}\bigg)\enspace.
 \]
 Applying \eqref{eq:VCB} one more time shows that, with probability larger than $1-2\delta$,
 \[
 D(\muh_{\text{Tuc}},P_I)\geqslant D(\mu,P_I)-\frac{\epsilon}{1-\epsilon}- 2C\bigg(\sqrt{\frac{d}N}+\sqrt{\frac{\log(1/\delta)}N}\bigg)\enspace.
 \]
 Introduce $\Phi$, the c.d.f of the standard Gaussian distribution on $\R$, $\gauss(0,1)$. It holds that, for any $\nu\in \R^d$, 
 \[
 D(\nu,P_I)=\inf_{u\in\R^d}P_I(u^TX\leqslant u^T\nu)=\inf_{u\in\R^d}\Phi(u^T\nu)=1-\Phi(\|\nu\|)\enspace.
 \]
 In particular, as $\mu=0$, $D(\mu,P_I)=1/2$, so
 \[
  D(\muh_{\text{Tuc}},P_I)=1-\Phi(\|\muh_{\text{Tuc}}\|)\geqslant \frac12-\frac{\epsilon}{1-\epsilon}- 2C\bigg(\sqrt{\frac{d}N}+\sqrt{\frac{\log(1/\delta)}N}\bigg)\enspace.
 \]
 Equivalently, with probability larger than $1-2\delta$,
 \[
 \Phi(\|\muh_{\text{Tuc}}\|)\leqslant \frac12+\frac{\epsilon}{1-\epsilon}+ 2C\bigg(\sqrt{\frac{d}N}+\sqrt{\frac{\log(1/\delta)}N}\bigg)\enspace.
 \]
 Now, the proof terminates since there exists an absolute constant $c$ such that $\Phi(x)\geqslant 1/2+x/4$ for any $0<x<c$. 
\end{proof}


\chapter{The homogeneity lemma}\label{Chap:MinMax}

The homogeneity lemma is one of the most important tools in these notes.
Roughly speaking, it allows to reduce the analysis of minmax estimators to deviation bounds of the underlying process on localized classes of functions.
It is an alternative to the peeling argument that has been repeatedly used in the analysis of the ERM to benefit from localization ideas \cite{MR2829871} and prove fast rates of convergence in statistical learning theory.
It is particularly well adapted to problems where deviation inequalities are only available up to a certain confidence parameter, as it is the case of MOM processes, see Section~\ref{sec:DevMOMProc} in Chapter~\ref{Chap:ConcIn}.
The version presented here is an extension of the ``deterministic argument" presented in \cite{Chinot2018robust} that allows to deal with convex losses as will be done in Chapters~\ref{Chap:LipConvLoss} and \ref{Chap:LSR} and with the tests of $\rho$-estimation presented in Chapter~\ref{Sec:RhoEst}.

\section{Learning, ERM, minmax aggregation of tests}\label{sec:Generalities}
Consider the statistical learning framework of Vapnik.
Let $Z$ denote a random variable taking values in a measurable space $\cZ$, with distribution $P$.
Let $F$ denote a set of parameters and let $\ell:F\times \cZ\to \R,\ (f,z)\mapsto\ell_f(z)$ denote a function called loss.
Assume that there exists $f_0\in F$ such that, for any $f\in F$, $\ell_f(\cdot)-\ell_{f_0}(\cdot)\in L^1(P)$. Under this assumption, $\ell_f-\ell_g\in L^1(P)$ for any $f,g\in F$. We want to estimate 
\begin{equation}\label{eq:TargetLearning}
f^*\in \argmin_{f\in F}P[\ell_f-\ell_{f_0}]\enspace. 
\end{equation}
It is clear that, for any $g\in F$, we also have $f^*\in \argmin_{f\in F}P[\ell_f-\ell_g]$.
The arguably most simple example of such problem is multivariate mean estimation where one wants to estimate the expectation $\mu_P$ of a measure $P$ on $\R^d$. In this example, let $\cZ=F=\R^d$, $\|\cdot\|$ denote the Euclidean norm and $\ell:(f,z)\mapsto \|f-z\|^2$, $f_0=0$, then $\ell_f(z)-\ell_{f_0}(z)=\|f\|^2-2f^Tz\in L^1(P)$ when $Z\in L^1(P)$ and $P[\ell_f-\ell_{f_0}]=\|f-\mu_P\|^2-\|\mu_P\|^2$ is obviously minimized when $f=\mu_P$, that is $f^*=\mu_P$.

To estimate $f^*$, a dataset $Z_1,\ldots,Z_N$ i.i.d. with common distribution $P$ is available.
Let $P_N$ denote the empirical measure of the sample $Z_1,\ldots,Z_N$ defined for any function $g:\cZ\to\R$ by $P_Ng=N^{-1}\sum_{i=1}^Ng(Z_i)$ One way to handle problem \eqref{eq:TargetLearning} is to use Empirical Risk Minimizers defined by
\[
\hat{f}^{\text{ERM}}\in \argmin_{f\in F}P_N[\ell_f]\enspace.
\]
For multivariate mean estimation, this yields for example, the empirical mean estimator $\hat{f}^{\text{ERM}}=N^{-1}\sum_{i=1}^NZ_i$.
This estimator is not robust to heavy-tailed data, or the presence of outliers in the dataset.

To build robust alternative, the empirical mean could be replaced by any robust estimator seen in Chapter 1 in the mean problem to get a robust estimator for learning task. 
As for multivariate mean estimation considered in Chapter~\ref{Chap:MME}, this strategy is suboptimal in general. 
Instead, following the strategy introduced in Section~\ref{sec:GenStrat} of Chapter~\ref{Chap:MME}, one can rewrite the min problem \eqref{eq:TargetLearning} as follows 
\begin{equation}\label{eq:TargetLearning2}
f^*\in \argmin_{f\in F}P[\ell_f-\ell_{f_0}]=\argmin_{f\in F}\sup_{g\in F}P[\ell_f-\ell_{g}]\enspace. 
\end{equation}
Then, one can plug into this definition any robust estimator of the increments $P[\ell_f-\ell_{g}]$. For example, the minmax MOM estimator is defined by
\[
\hat{f}^{\text{MOM}}_K\in\argmin_{f\in F}\sup_{g\in F}\MOM{K}{\ell_f-\ell_{g}}\enspace.
\]
The ERM could also be be obtained this way since
\[
\hat{f}^{\text{ERM}}\in \argmin_{f\in F}\sup_{g\in F}P_N[\ell_f-\ell_g]\enspace.
\]
Notice also that min MOM and minmax MOM estimators differ in general since MOM processes are not linear.

The ideas that we develop in this chapter intend to analyse the following extension of these minmax strategies. 
The building blocks of the general construction are \emph{tests statistics} or increment estimators which are random variables $T(f,g)$ where $f$ and $g$ belong to $F$. 
For ERM, $T(f,g)=P_N[\ell_f-\ell_g]$ and for minmax MOM estimators $T(f,g)=\MOM{K}{\ell_f-\ell_g}$. 
In both examples, $T(f,g)$ is an estimator of $P[\ell_f-\ell_g]$. 
An other example is presented later in Section~\ref{sec:RhoEst}.
As this property is satisfied in all examples, it is always assumed that $T(f,g)=-T(g,f)$ and in particular that $T(f,f)=0$ for any $f\in F$.
The heuristic is that $T(f,g)>0$ means that $g$ is better than $f$ to estimate $f^*$.
The estimator we want to analyse is the minmax estimator
\begin{equation}\label{eq:LearnFromTests}
 \hat{f}\in\argmin_{f\in F}\sup_{g\in F}T(f,g)\enspace.
\end{equation}
In the following, $\cE:F\to \R$ denotes a real valued function such that, for any $f\in F$, $\cE(f)$ evaluates the performance of $f$ as an estimator of $f^*$.
As this evaluation function is not involved in the definition of $\hat{f}$, it may perfectly depend on the unknown distribution $P$.
In the examples, $\cE$ will usually denote the \emph{excess risk} $\cE(f)=P[\ell_f-\ell_{f^*}]$ (which is non negative by definition and obviously null if $f=f^*$) or some distance between $f$ and $f^*$. 
In any case, large values of $\cE(f)$ indicate that $f$ is \emph{far} from the target $f^*$, so $f$ is not a desirable estimator of $f^*$.

\section{General results}
This section gathers the main results of this chapter.
The goal is to reduce the analysis of minmax estimators to concentration inequalities for suprema of test processes $\sup_{f\in \cV(f^*)}T(f^*,f)$ over subsets $\cV(f^*)\subset F$ localized around the oracle $f^*$.
When applied to the tests $T(f,g)=P_N[\ell_f-\ell_g]$ defining ERM, these results extend well known localization ideas widely used to prove fast rates of convergence for this estimator, see for example \cite{MR2829871} for an overview on this topic.

\subsection{Link with multiple testing theory}
The first result extracts the idea underlying the proof of the risk bound for minmax MOM estimator of multivariate expectations in Theorem~\ref{thm:MOMMean1}.
Interestingly, it establishes a link between learning or estimation from tests, which is the analysis of estimators built as in \eqref{eq:LearnFromTests} and multiple testing theory, which is an extension of the classical theory of tests in statistics where one is interested in testing several null hypotheses at the same time.
\begin{lemma}\label{lem:FWER+FWSRimplyERGen}[Link with Multiple Testing]
Let $B$ and $r$ denote positive real numbers.
On $\Omega$, $\cE(\hat{f})\leqslant r$, where $\Omega$ denotes the event where the following equations hold.
\begin{gather}
\label{def:FWERGen} \sup_{g\in F}T(f^*,g)\leqslant B\enspace,\\
\label{def:FWSRGen}\sup_{g\in F: \cE(g)> r}T(f^*,g)< -B\enspace.
\end{gather}
\end{lemma}
\begin{proof}
 By definition of $\hat{f}$, on $\Omega$, by \eqref{def:FWERGen}.
 \[
\sup_{g\in F}T(\hat{f},g)\leqslant \sup_{g\in F}T(f^*,g)\leqslant B\enspace.
 \]
 On the other hand, by \eqref{def:FWSRGen}, any $f$ such that $\cE(f)> r$ satisfies
 \[
 \sup_{g\in F}T(f,g)\geqslant T(f,f^*)>B\enspace.
 \]
 As a consequence, $\cE(\hat{f})\leqslant r$ on $\Omega$.
\end{proof}
Conditions~\eqref{def:FWERGen} and~\eqref{def:FWSRGen} are intuitively clear.
\eqref{def:FWERGen} means that the tests between $f^*$ and any $g\in F$ should not become too large ($B=0$ in \eqref{def:FWERGen} for the ideal test $T(f,g)=P[\ell_f-\ell_g]$). $B$ controls typically the fluctuations of the process $\sup_{g\in F}T(f^*,g)$ which has negative drift.
\eqref{def:FWSRGen} means that $f^*$ is preferred to any $g$ with large drift with a margin larger than the noise level.

Let us clarify in which sense Lemma~\ref{lem:FWER+FWSRimplyERGen} makes a link with multiple testing theory.
This link uses the formalism and definitions borrowed from \cite{MR3576553}.
Let $\cP$ denote a class of probability distributions on $\cZ$ and for any $P\in \cP$, denote by 
\[
F^*_P=\argmin_{f\in F}P[\ell_f-\ell_{f_0}]\subset F\enspace.
\]
Define, for any $f\in F$, the hypothesis $H_f=\{P\in \cP: f\in F_P^*\}\subset \cP$.
We want to test simultaneously all assumptions $\cH=\{H_f,\ f\in F\}$.
A multiple test of $\cH$ is a random subset $\cR\subset \cH$ of \emph{rejected} hypotheses. 
To evaluate the multiple testing procedure $\cR$, introduce, for any $P\in \cP$, the sets 
\[
\cF(P)=\{H_f\in \cH : \ f\notin F^*_P\}, \qquad \cF_r(P)=\{H_f\in \cH : \ \cE(f)>r\}\enspace.
\]
$\cF(P)$ is called the set of \emph{false} hypotheses and $\cF_r(P)$ is the set of assumptions that are $r$-separated from the true assumptions $\cT(P)=\cH\setminus\cF(P)$.
The \emph{family-wise error rate} (FWER) of the multiple testing $\cR$ is defined by 
\[
\text{FWER}(\cR):=\sup_{P\in \cP}P\big(\cR\cap \cT(P)\ne \emptyset)=1-\inf_{P\in\cP}P(\cR\subset \cF(P))\enspace.
\]
It is the (maximal) probability to reject at least one true hypothesis.
To understand this definition, consider the situation where $\cH$ is reduced to a singleton $\cH=\{H\}$.
In this case, one can consider a simple test $\phi_H$ of the assumption $H$ against the complementary $H^c=\cP/H$: $\phi_H=1$ means that $H$ is rejected and $\phi=0$ that $H$ is not rejected.
The multiple test associated with the simple test $\phi_H$ is $\cR=\{ H: \phi_H=1\}$. 
In words, $\cR=\{H\}$ if $H$ is rejected by the simple test $\phi_H$ and $\cR=\emptyset$ if $H$ is not rejected by $\phi_H$.
In this case, $\text{FWER}(\cR)$ is the size of the simple test $\phi_H$.
In particular, $\phi_H$ has level $\alpha$ iff $\text{FWER}(\cR)\leqslant \alpha$.
In this sense, $\text{FWER}(\cR)$ is an extension of the first type error rate for simple tests.
The \emph{family-wise separation rate} (FWSR) of the test $\cR$ is defined by 
\[
\text{FWSR}_{\beta}(\cR)=\inf\{r>0: \inf_{P\in \cP} P(\cR\supset\cF_r(P))\geqslant 1-\beta\}\enspace.
\]
FWSR measures the minimal distance between $H_f$ and $\cT(P)$ such that $H_f$ is rejected with given confidence level. 
FWSR extends the notion of \emph{separation rates} for simple tests, see \cite{MR1935648} for a definition of separation rates and \cite{MR3576553} for more details on this extension.
FWSR is a measure of the second type error rate for multiple testing that allows to define a minimax theory for these tests.

Going back to learning from tests, one can use the family of test statistics $T(f,g)$ to build a multiple testing on $\cH$. 
The idea is to use the score 
\[
\hat{\cE}(f)=\sup_{g\in F}T(f,g)
\]
as a test statistic to build a simple test of the assumption $H_f$.
Small values of $\hat{\cE}(f)$ indicate that $H_f$ might be true and large values that it seems false. This suggests to consider, for some threshold $B>0$, the multiple testing
\[
\cR_B=\{H_f\in \cH: \hat{\cE}(f)>B\}\enspace.
\]
This test satisfies $\text{FWER}(\cR_B)\leqslant \alpha$ if, for any $P\in\cP$, $\cR_B$ does not contain any $H_{f^*_P}$, where $f^*_P\in F_P^*$, with $P$-probability larger than $\alpha$.
In words, to bound the FWER, we have to bound from bellow the probability that any $f^*_P\in F_P^*$ is not rejected, that is the probability of the event
\[
\sup_{g\in F}T(f^*_P,g)\leqslant B\enspace.
\]
Bounding from above the FWER of the multiple testing $\cR_B$ by $\alpha$ is equivalent to bound from bellow the probability of the event~\eqref{def:FWERGen} by $1-\alpha$.

Consider now the FWSR of $\cR_B$.
This FWSR is bounded by $r$ if the probability that any $f\in F$ such that $\cE(f)>r$ is rejected with probability at least $1-\beta$.
Formally, $\text{FWSR}_\beta(\cR)\leqslant r$ if, for any $P\in\cP$, $P(\Omega_{B,r})\geqslant 1-\beta$, where $\Omega_{B,r}$ is the event
\begin{equation}\label{eq:ActualFWSR}
\forall f\in F : \cE(f)>r,\qquad \sup_{g\in F}T(f,g)>B\enspace. 
\end{equation}
Remark that $\Omega_{B,r}$ clearly contains the event $\Omega'_{B,r}$ defined by 
\[
\forall f\in F : \cE(f)>r,\qquad T(f,f^*)>B\enspace.
\]
Therefore, if \eqref{def:FWSRGen} holds with probability $1-\beta$, then the FWSR of the test $\cR_B$ is bounded from above by $r$.

It transpires from the proof of Lemma~\ref{lem:FWER+FWSRimplyERGen} that Assumption~ \eqref{eq:ActualFWSR} can replace Assumption \ref{def:FWSRGen} with the same conclusion: $\cE(\hat{f})\leqslant r$. 
Therefore, if $\text{FWER}(\cR_B)\leqslant \alpha$ and $\text{FWSR}_\beta(\cR_B)\leqslant r$, then $\cE(\hat{f})\leqslant r$ with probability $1-\alpha-\beta$ for any choice of the probability distribution $P\in\cP$.
However, besides this application, we will always use the restricted form of this result given by Lemma~\ref{lem:FWER+FWSRimplyERGen} which is why we presented this version.

\subsection{The homogeneity lemma}

The following result is the most important of this chapter and one the most fundamental tool of these lectures. 
It is called the ``homogeneity lemma" and it shows that risk bounds for minmax estimators follow from concentration of suprema of test processes over sub-classes $\cV(\bayes)\subset F$ \emph{localized} around the oracle $\bayes$. 
This result holds under abstract conditions on the test statistics that can easily be checked in the applications developed in the following chapters.
It will be at the heart of the proofs of all risk bounds given afterwards.
The idea is to show that conditions \eqref{def:FWERGen} and \eqref{def:FWSRGen} in Lemma~\ref{lem:FWER+FWSRimplyERGen} are met if suprema of test processes over localized classes are controled. 

\begin{lemma}[homogeneity lemma]\label{lem:DetArg}
Assume that the tests $T(f,g)$ satisfy the homogeneity property. 

\medskip
{\bf (HP)}: There exists $r_0>0$ such that, for any $r>r_0$ and any $f\in F$ satisfying $\cE(f)>r$, there exists $f_r\in F$ such that
\begin{equation}\label{eq:LBT}
\cE(f_r)=r,\qquad T(f,f^*)\geqslant  T(f_r,f^*)\enspace. 
\end{equation}

Let $B:\R_+\to\R_+$ and $d:F^2\to\R$ such that $d(f,g)=-d(g,f)$. Consider, for any $r>0$, the event
 \[
 \Omega_r=\bigg\{\sup_{f\in F: \cE(f)\leqslant r}T(f^*,f)-d(f^*,f)\leqslant B(r)\bigg\}\enspace.
 \]
Assume that there exists $r_1>r_0$ such that
\begin{gather}
 \label{def:r1}B(r_1)-\inf_{f\in F:\cE(f)=r_1}d(f,f^*)\leqslant 0\enspace,
 \end{gather}
 Let $\cB\geqslant B(r_1)+\sup_{f:\cE(f)\leqslant r_1}d(f^*,f)$.
 Assume that there exists $r_2>r_0$ such that
\begin{gather}
\label{def:r2} B(r_2)-\inf_{f\in F:\cE(f)=r_2}d(f,f^*)\leqslant -\cB\enspace.
\end{gather}
On the event $\Omega_{r_1}\cap\Omega_{r_2}$, \eqref{def:FWERGen} and \eqref{def:FWSRGen} hold with $B=\cB$ and $r=r_2$.
In particular, $\P(\cE(\hat{f})\leqslant r_2)\geqslant \P(\Omega_{r_1}\cap\Omega_{r_2})$.
\end{lemma}
\begin{proof}
 On $\Omega_{r_1}$, by definition, for any $f\in F$ such that $\cE(f)\leqslant r_1$,
\begin{equation}\label{eq:close}
 T(f^*,f)=d(f^*,f)+(T(f^*,f)-d(f^*,f))\leqslant \cB\enspace.
\end{equation}
 Moreover, for any $r>r_0$ and any $f\in F$ such that $\cE(f)> r$, 
there exist $f_r\in F$ such that $\cE(f_r)=r$ and 
\[
T(f^*,f)\leqslant T(f^*,f_r)\enspace.
\]
 It follows that, for any $r>r_0$ and any $f\in F$ such that $\cE(f)>r$, on $\Omega_r$,
\begin{align*}
 T(f^*,f)&\leqslant -d(f_r,f^*)+(T(f^*,f_r)-d(f^*,f_r))\\
 &\leqslant -\inf_{f\in F:\cE(f)=r}\{d(f,f^*)\}+B(r)]\enspace.
\end{align*}
Hence, on $\Omega_{r_1}$, 
$T(f^*,f)\leqslant 0$ for any $f\geqslant r_1$, and by \eqref{eq:close}, \eqref{def:FWERGen} holds with $B=\cB$.
Moreover, on $\Omega_{r_2}$, for any $f\in F$ such that $\cE(f)\geqslant r_2$,
\[
T(f^*,f)\leqslant -\inf_{f\in F:\cE(f)=r_2}\{d(f,f^*)\}+B(r_2)]\leqslant -\cB\enspace.
\]
Therefore, \eqref{def:FWSRGen} holds with $r=r_2$.
\end{proof}
\begin{remark}
In some applications, the set $F$ is discrete and the requirement $\cE(f_r)=r$ may be restrictive. The interested reader can check that a direct adaptation of the proof allows to relax slightly Condition {\bf (HP)} into 

{\bf (HPr)} there exist $r_0>0$ and an absolute constant $c$ such that, for any $r>r_0$ and any $f\in F$ satisfying $\cE(f)>r$, there exists $f_r\in F$ such that
\begin{equation*}
\cE(f_r)\in[cr,r],\qquad T(f,f^*)\geqslant T(f_r,f^*)\enspace. 
\end{equation*}
Under this relaxed condition, the conclusion of Lemma~\ref{lem:DetArg} holds, with the minor modification that $\inf_{f\in F:\cE(f)=r}d(f,f^*)$ has to be replaced by $\inf_{f\in F:\cE(f)\in[cr,r]}d(f,f^*)$ in the definition of $r_1$ and $r_2$.
Remark that Lemma~\ref{lem:DetArg} is a particular instance of this extended result in the ideal case where $c=1$.
\end{remark}

\begin{remark}
The function $d$ will be $d(f,g) =P[\ell_f-\ell_g]$ in learning problems.
This function clearly satisfies the requirements $d(f,g)=-d(g,f)$.
Moreover, in this case, $d(f^*,f)\leqslant 0$ so $\cB=B(r_1)$.
\end{remark}

\subsection{Convex losses}
When working in Vapnik's learning framework, $\cE(f)$ will usually denote a distance between $f$ and $f^*$ derived from a norm $\cE(f)=\|f-f^*\|$.
In this setting, the concentration inequalities of the previous chapter allow to bound the probability of the events $\Omega_r$.
To conclude this section, we show that some assumptions of the homogeneity lemma are met when $T(f,g)=\hat{P}[\ell_f-\ell_g]$ when the process $\hat{P}$ is positive and homogeneous and when the loss $\ell$ is convex. 
\begin{lemma}[convex losses]\label{lem:ConvLoss}
 Assume that $F$ is convex and that
 \[
 \forall z\in \cZ,\qquad f\mapsto \ell_{f}(z)\text{ is convex}\enspace.
 \]
 Assume that there exists a norm $\|\cdot\|$ such that $\cE(f)=\|f-f^*\|$.
 Assume that the estimators $\hat{P}[g]$ are well defined for any real valued function $g$ and satisfy the following requirement:
\begin{itemize}
 \item[(i)] $\hat{P}$ is non-decreasing: for any $g\leqslant g'$, $\hat{P}[g]\leqslant \hat{P}[g']$,
 \item[(ii)] $\hat{P}$ is homogeneous: for any $a\in\R$, $\hat{P}[ag]=a\hat{P}[g]$,
\end{itemize}
Then the tests $T(f,g)=\hat{P}[\ell_f-\ell_g]$ satisfy the homogeneity property of Lemma~\ref{lem:DetArg} with $r_0=0$: for any $r>0$ and any $f\in F$ satisfying $\cE(f)>r$, there exists $f_r\in F$ such that
\[
\cE(f_r)=r,\qquad T(f,f^*)\geqslant T(f_r,f^*)\enspace.
\]
\end{lemma}
\begin{proof}
 Let $r>0$ and $f\in F$ satisfying $\cE(f)>r$. 
 Let $\alpha=\cE(f)/r>1$ and $f_r=f^*+\alpha^{-1}(f-f^*)=\alpha^{-1}f+(1-\alpha^{-1})f^*$. 
 By convexity of $F$, $f_r\in F$. Moreover, 
 \[
 \cE(f_r)=\|f^*-f_r\|=\bigg\|\frac{f^*-f}{\alpha}\bigg\|=\frac{\|f^*-f\|}{\alpha}=r\enspace.
 \] 
 Now, for any $z\in \cZ$, the function $\psi:u\mapsto \ell_{f^*+u}(z)-\ell_{f^*}(z)$ is convex so 
 \[
\psi(f_r-f^*)=\psi(\alpha^{-1}(f-f^*)+(1-\alpha^{-1})0)\leqslant \alpha^{-1}\psi(f-f^*)+(1-\alpha^{-1})\psi(0)\enspace.
 \]
As $\psi(0)=0$, this can be rewritten $\ell_{f}(z)-\ell_{f^*}(z)\leqslant \alpha^{-1}(\ell_{f_r}(z)-\ell_{f^*}(z))$ or,
 \[
 \ell_f-\ell_{f^*}\geqslant \alpha(\ell_{f_r}-\ell_{f^*})\enspace.
 \]
 It follows that 
 \[
 T(f,f^*)=\hat{P}[\ell_f-\ell_{f^*}]\geqslant \hat{P}[\alpha(\ell_{f_r}-\ell_{f^*})]=\alpha\hat{P}[\ell_{f_r}-\ell_{f^*}]=\alpha T(f_r,f^*)\enspace.
 \]
\end{proof}

\paragraph{Examples of operator $\hat{P}$.}
We will use repeatedly Lemma~\ref{lem:ConvLoss} when $\hat{P}$ denotes the empirical mean $P_N$ or the median-of-means operator $\MOM{K}{\cdot}$.
As both empirical means and the median satisfy conditions (i) and (ii) of  Lemma~\ref{lem:ConvLoss}, these estimators can actually safely be used when applying this lemma.
It is worth noticing though that neither smoothed median-of-means nor $M$-estimators in general satisfy the homogeneity condition (ii).

\subsection{The tests of $\rho$-estimation.}\label{sec:RhoEst}
$\rho$-estimators have been introduced in \cite{MR3595933} and further extended in \cite{BarBir2017RhoAgg2}.
The idea is to estimate a distribution $P^*$ on a measurable space $\cX$ from an i.i.d. sample $X_1,\ldots,X_N$ with common distribution $P^*$.
The risk is measured for any estimator $\hat{P}$ by the squared Hellinger distance between $P$ and $\hat{P}$: $h^2(\hat{P},P)$, where, for all distributions $P$ and $Q$, for any measure $\mu$ such that $P\ll \mu$, $Q\ll\mu$, denoting by $p=\rmd P/\rmd\mu$ and $q=\rmd P/\rmd\mu$,
\[
h^2(P,Q)=\frac12\int (\sqrt{p}-\sqrt{q})^2\rmd\mu\enspace.
\]
It is easy to check that $h^2(P,Q)$ is well defined for any $P$ and $Q$, does not depend on $\mu$ and always satisfy
\[
0\leqslant h^2(P,Q)\leqslant 1\enspace.
\]
This problem does not directly falls into Vapnik's learning framework.
Nevertheless, the homogeneity lemma may be used in this problem.
Let $\mu$ denote a measure on $\cX$ and let $F$ denote a closed convex set of densities with respect to $\mu$, that is, for any $f\in F$, $f\geqslant 0$ $\mu$-a.s. and $\int f\rmd\mu=1$. 
For any $f\in F$, let also $P_f$ denote the distribution with density $f$ w.r.t. $\mu$.
To compare elements $f$ and $g$ in $F$, Baraud and Birg\'e defined in \cite{BarBir2017RhoAgg2} the following tests:
\begin{equation}\label{eq:DefRhoTests}
T(f,g)=\sum_{i=1}^N\rho\bigg(\sqrt{\frac{g(Z_i)}{f(Z_i)}}\bigg)\enspace. 
\end{equation}
Here, the function $\rho=(x-1)/(x+1)$ is non-decreasing $[0,+\infty]\to[-1,1]$, $2$-Lipschitz, it satisfies $\rho(1/x)=-\rho(x)$ for any $x\in[0,+\infty)$.
This last property implies that
\[
T(f,g)=\sum_{i=1}^N\rho\bigg(\sqrt{\frac{g(Z_i)}{f(Z_i)}}\bigg)=-\sum_{i=1}^N\rho\bigg(\sqrt{\frac{f(Z_i)}{g(Z_i)}}\bigg)=-T(g,f)\enspace. 
\]
Hence, $T(f,g)$ are test statistics in the sense of Section~\ref{sec:Generalities}.
To conclude this chapter, we show that these test statistics satisfy the homogeneity property.

\begin{lemma}\label{lem:RhoTestAreHomogeneous}
 The $\rho$-test defined in \eqref{eq:DefRhoTests} satisfy the homogeneity property of Lemma~\ref{lem:DetArg} with the evaluation function $\cE(f)=h(P^*,P_f)$ and minimal radius $r_0=\min_{f\in F}h(P^*,P_f)$.
\end{lemma}

\begin{proof}
 Let $f^*\in F$ denote an oracle, that is a function such that $r_0= \cE(f^*)$. For any $r>r_0$, any $f\in F$ such that $\cE(f)>r$ and any $\epsilon\in(0,1)$, let 
 \[
 f_{\epsilon}=\epsilon f^*+(1-\epsilon)f\enspace.
 \]
By convexity of $F$, $f_\epsilon\in F$.
Moreover, if $P^*$ is absolutely continuous with respect to $\mu$ (otherwise, one can change $\mu$ to $\mu+P^*$), and denoting by $p^*$ its density,
\begin{align*}
 \cE(f_{\epsilon})=\frac12\int (\sqrt{p^*}-\sqrt{\epsilon f^*+(1-\epsilon)f})^2\rmd\mu\enspace.
\end{align*}
The map $\epsilon\mapsto \cE(f_{\epsilon})$ is continuous and takes value $\cE(f)>r$ when $\epsilon=0$, $\cE(f^*)<r$ when $\epsilon=1$.
Therefore, there exists $\epsilon\in (0,1)$ such that $\cE(f_\epsilon)=r$.
 Elementary calculus shows that, for any $a\geqslant0$, the functions $\eta_a=(2a)/(a+\sqrt{x})$ are convex.
Therefore, for any $x\in \cX$,
\begin{align*}
 \rho\bigg(\sqrt{\frac{f^*(x)}{f_\epsilon(x)}}\bigg)&=\frac{\sqrt{\frac{f^*(x)}{f_\epsilon(x)}}-1}{\sqrt{\frac{f^*(x)}{f_\epsilon(x)}}+1}=\frac{\sqrt{f^*(x)}-\sqrt{f_\epsilon(x)}}{\sqrt{f^*(x)}+\sqrt{f_\epsilon(x)}}\\
 &=\frac{2\sqrt{f^*(x)}}{\sqrt{f^*(x)}+\sqrt{f_\epsilon(x)}}-1\\
 &=\eta_{\sqrt{f^*(x)}}(\epsilon f^*(x)+(1-\epsilon)f(x))-1\\
 &\leqslant \epsilon\eta_{\sqrt{f^*(x)}}( f^*(x))+(1-\epsilon)\eta_{\sqrt{f^*(x)}}(f(x))-1\\
 &=(1-\epsilon)(\eta_{\sqrt{f^*(x)}}(f(x))-1)=(1-\epsilon)\rho\bigg(\sqrt{\frac{f^*(x)}{f(x)}}\bigg)\enspace.
\end{align*}
It follows that
\[
T(f,f^*)=\sum_{i=1}^N\rho\bigg(\sqrt{\frac{f^*(X_i)}{f(X_i)}}\bigg)\geqslant \frac1{1-\epsilon}\sum_{i=1}^N\rho\bigg(\sqrt{\frac{f^*(X_i)}{f_\epsilon(X_i)}}\bigg)=\frac1{1-\epsilon}T(f_\epsilon,f^*)\enspace.
\]
In words, $T$ satisfies Eq~\eqref{eq:LBT}.
\end{proof}

\section{Back to multivariate mean estimation.}\label{sec:MMEst}
As a first example of application of the freshly introduced general methodology, let us go back to the problem of estimating a multivariate expectation discussed in Chapter~\ref{Chap:MME}.

Recall that $\norm{\cdot}$ denotes the Euclidean norm on $\cZ=\R^d$ and $\cP_2$ denote the set of distributions on $\R^d$ such that $P[\norm{X}^2]<\infty$. 
For any $P\in \cP_2$, let $f^*_P=PX\in \R^d$ and $\Sigma_P=P[(X-\mu_P)(X-\mu_P)^T]\in \R^{d\times d}$. 
Recall that estimating $\bayes_P$ is a learning problem that falls into Vapnik's framework: let $F=\R^d$ and $\ell_f(z)=\|z-f\|^2$. 
Then, the quadratic loss satisfies the quadratic/multiplier decomposition:
\[
\forall f,g\in F, \forall z\in \cZ,\qquad \ell_f(z)-\ell_{g}(z)=-2(f-g)^T(z-g)+\|f-g\|^2\enspace.
\]
In particular, 
\begin{equation}\label{eq:ERMeanEst1}
\ell_f(Z)-\ell_{\bayes_P}(Z)=-2(f-\bayes_P)^T(Z-\bayes_P)+\|f-\bayes_P\|^2,
\end{equation}
so
\begin{equation}\label{eq:ERMeanEst}
 P[\ell_f-\ell_{f^*_P}]=\|f-f^*_P\|^2\enspace.
\end{equation}
Therefore,
\[
\{f^*_P\}=\argmin_{f\in F}P\ell_f\enspace.
\]
The loss satisfies the convexity assumption in Lemma~\ref{lem:ConvLoss}.
Moreover, as discussed after this lemma, the empirical mean $P_N$ or MOM processes $\MOM{K}{\cdot}$ satisfy conditions (i) and (ii) on the mean estimators $\hat{P}$.
It follows that Lemma~\ref{lem:ConvLoss} applies to the tests $T(f,g)=\hat{P}[\ell_f-\ell_g]$. 
In particular, these tests satisfy the homogeneity property in the homogeneity lemma. 

To deduce risk bounds for the associated minmax estimators, it remains to compute the function $B$ for a choice of evaluation function $\cE$ and pseudo-distance function $d$ in the homogeneity lemma. 
By \eqref{eq:ERMeanEst}, $d(f,f^*_P):=P[\ell_f-\ell_{f^*_P}]=\|f-f^*_P\|^2$.
Pick $\cE(f)=\|f-f_P^*\|$ so, for any $r>0$, $\inf_{f\in F:\cE(f)=r}d(f,f^*_P)=r^2$.
It follows therefore from \eqref{eq:ERMeanEst1} that 
\begin{align*}
T(f^*_P,f)-d(f,f^*_P)&=\hat{P}[2(f-\bayes_P)^T(Z-\bayes_P)-\|f-\bayes_P\|^2]+\|f-f^*_P\|^2\enspace. 
\end{align*}
Therefore by homogeneity and translation invariance of $\hat{P}$: for any function $g$ and any $b\in \R$, $\hat{P}[g+b]=\hat{P}[g]+b$,
\begin{align}
\notag T(f^*_P,f)-d(f,f^*_P)&=2\|f-f^*_P\|\hat{P}\bigg[\bigg(\frac{f-f^*_P}{\|f-f^*_P\|}\bigg)^T(Z-f^*_P)\bigg]\\
\label{eq:ERMeanEst3}&\leqslant 2\|f-f^*_P\|R\enspace,
\end{align}
where 
\begin{equation}\label{eq:defR}
R=\sup_{\bu\in\bS}\hat{P}[\bu^T(Z-f^*_P)]\enspace,
\end{equation}
where $\bS=\{\bu\in\R^d:\|\bu\|=1\}$.

\subsection{ERM in the Gaussian case}
Start with an application to the Gaussian case.
The purpose here is to show that one can recover Hanson-Wright result (up to constants) using our general methodology.
\begin{theorem}[ERM]
If $Z$ is Gaussian, the ERM $\hat{f}^{\text{ERM}}=N^{-1}\sum_{i=1}^NZ_i$ satisfies
 \[
 \forall s>0,\qquad \P\bigg(\|\hat{f}^{\text{ERM}}-f_P^*\|>(1+\sqrt{5})\bigg(\frac{\sqrt{\text{Tr}(\Sigma)}+\sqrt{2\|\Sigma\|_{\text{op}}s}}{\sqrt{N}}\bigg)\bigg)\leqslant e^{-s}\enspace.
 \]
\end{theorem}
\begin{proof}
The proof of Theorem~\ref{thm:HW} shows that, with probability at least $1-e^{-s}$, the random variable $R$ defined in \eqref{eq:defR} with $\hat{P}=P_N$ satisfies $R\leqslant r_s$, where
\begin{equation}\label{eq:ConcSupGauss}
r_s=\sqrt{\frac{\text{Tr}(\Sigma_P)}N}+\sqrt{\frac{2\|\Sigma_P\|_{\text{op}}s}N}\enspace. 
\end{equation}
On the event $\Omega_{\text{good}}=\{R\leqslant r_s\}$, it follows from \eqref{eq:ERMeanEst3} that, choosing $B(r)=2r_sr$, all events $\{\Omega_r,r>0\}$, where $\Omega_r$ is defined in Lemma~\ref{lem:DetArg} hold simultaneously.
Recall that the choice of $\cE(f)=\|f-f^*_P\|$ and $d(f,g)=P[\ell_f-\ell_g]$ imply that 
\[
\inf_{f:\cE(f)=r}P[\ell_f-\ell_{f^*_P}]=\inf_{f:\|f-f_P^*\|=r}\|f-f^*_P\|^2=r^2\enspace.
\]
Therefore, Condition \eqref{def:r1} defining $r_1$ in Lemma~\ref{lem:DetArg} is satisfied for $r_1$ the largest solution of 
\[
2r_sr-r^2=0,\qquad \text{i.e. for}\qquad r_1=2r_s\enspace.
\]
This gives $B(r_1)=4r_s^2$, thus, Condition~\eqref{def:r2} defining $r_2$ in Lemma~\ref{lem:DetArg} is satisfied for $r_2$ the largest solution of 
\[
2r_sr-r^2=-4r_s^2,\qquad \text{i.e. for}\qquad r_2=(1+\sqrt{5})r_s\enspace.
\]
The theorem follows from the homogeneity lemma.
\end{proof}

\subsection{Minmax MOM estimators}
This section shows that the general methodology can easily be used to analyse minmax MOM estimators also.
As for Hanson-Wright result, the result obtained via a direct approach can be recovered from the general principles.
\begin{theorem}\label{thm:minmaxMOMBasic}
 Assume that $P\in \cP_2$ then the minmax MOM estimator
 \[
 \hat{f}_K\in\argmin_{f\in \R^d}\sup_{g\in \R^d}\MOM{K}{\|X-f\|^2-\|X-g\|^2}
 \]
 satisfies, 
 \[
 \P\bigg(\|\hat{f}_K-f^*_P\|>(1+\sqrt{5})\bigg(128\sqrt{\frac{\text{Tr}(\Sigma)}N}\vee 4\sqrt{\frac{2\|\Sigma\|_{\text{op}}K}{ N}}\bigg)\bigg)\leqslant e^{-K/32}\enspace.
 \]
\end{theorem}
\begin{proof}
From Eq~\eqref{eq:concboundMOMMeanRd} in Theorem~\ref{thm:MOMExpRd}, with probability $1-e^{-K/32}$, the random variable $R$ defined in Eq~\eqref{eq:defR} satisfies $R\leqslant r_K$, where
\[
r_K=128\sqrt{\frac{\text{Tr}(\Sigma_P)}N}\vee 4\sqrt{\frac{2\|\Sigma_P\|_{\text{op}}K}N}\enspace.
\]
On the event $\Omega_{\text{good}}=\{R\leqslant r_K\}$, it follows from \eqref{eq:ERMeanEst3} that, choosing $B(r)=2r_Kr$, all events $\{\Omega_r,r>0\}$, where $\Omega_r$ is defined in Lemma~\ref{lem:DetArg} hold simultaneously.
Recall that the choice of $\cE$ and $d$ imply that 
\[
\inf_{f:\cE(f)=r}d(f,f^*_P)=r^2\enspace.
\]
Therefore, Condition \eqref{def:r1} defining $r_1$ in Lemma~\ref{lem:DetArg} is satisfied for $r_1$ solution of 
\[
2r_Kr-r^2=0,\qquad \text{i.e. for}\qquad r_1=2r_K\enspace.
\]
This gives $B(r_1)=4r_K^2$, thus, Condition~\eqref{def:r2} defining $r_2$ in Lemma~\ref{lem:DetArg} is satisfied for $r_2$ solution of 
\[
2r_Kr-r^2=-4r_K^2,\qquad \text{i.e. for}\qquad r_2=(1+\sqrt{5})r_K\enspace.
\]
The theorem follows from the homogeneity lemma.
\end{proof}


\chapter{Learning from Lipschitz-convex losses}\label{Chap:LipConvLoss}
This chapter presents results that have been proved in \cite{Chinot2018robust}.
Following \cite{pierre2017estimation}, we first investigate the ERM in a general statistical learning setting where the loss function is assumed to be both convex and Lipschtiz in its first variable, see Assumption~\eqref{cdt:LipConv}.
This setting include several losses that have been considered for convex relaxation of the $0-1$ loss in classification as the hinge loss that is used in the SVM algorithm and the logistic loss that is used in the Boosting algorithm.
It also includes classical losses in robust regression as the famous Huber's loss.
This analysis is conducted under sub-Gaussian assumption on the design $X$.
We also provide an analysis of minmax MOM estimators which holds under moment conditions only on the design.

\section{General setting}

Consider the supervized learning framework where one observes a dataset $\cD_N=(Z_1,\ldots,Z_N)$ of random variables taking values in a measurable space $\cZ$.
The space $\cZ$ is a product space $\cZ=\cX\times \cY$ and a data $z\in \cZ$ is a couple $z=(x,y)$, where $x$, called the input, takes values in a measurable space and $y$, called the output, takes value in $\cY\subset\R$.
The goal is to predict the value of the output $Y$ from the input $X$ when $Z=(X,Y)$ is drawn from $P$, independently of $\cD_N$.
The parameters $f\in F$ are functions $f:\cX\to\R$ and the loss function $\ell_f(z)$ takes the form $\ell_f(z)=c(f(x),y)$ for some \emph{cost} function $c$ measuring the accuracy of the prediction of $y$ by $f(x)$.

All along the chapter, the function $c$ is defined on ${\bar \cY}\times\cY$, where $\cY\subset {\bar \cY}\subset \R$ is a convex set containing all possible values of $f(x)$ for $f\in F$ and $x\in \cX$, $F$ is a convex set of functions and the following assumption always holds.
\begin{equation}\label{cdt:LipConv}
\exists L>0:\ \forall y\in \cY,\qquad c(\cdot,y)\text{ is convex and $L$-Lipschitz}\enspace.
\end{equation}

\section{Examples of loss functions}
Before analysing estimators based on these losses, we proceed to give a few examples of problems in machine learning where Condition \eqref{cdt:LipConv} is met.

\paragraph{Huber regression}
Let $\alpha>0$, the Huber function is defined by 
\[
h_{\alpha}(x)=
\begin{cases}
 \frac{x^2}2&\text{ if } x\leqslant \alpha\enspace,\\
 \alpha|x|-\frac{\alpha^2}2&\text{ if } x> \alpha\enspace.
\end{cases}
\]
This function is convex and continuously differentiable, with derivative bounded by $\alpha$. 
It interpolates between quadratic function $x\mapsto x^2/2$ and absolute value $x\mapsto |x|$. 
In the 1960's, to build robust alternatives to least-squares minimizers, Huber proposed to estimate the regression function by 
\[
\hat{f}_{\text{Hub},\alpha}\in \argmin_{f\in F}\sum_{i=1}^Nh_\alpha(f(x)-y)\enspace.
\]
This estimator typically interpolates between the (unbiased but non robust) least-squares estimator that would be obtained for the function $h(x)=x^2/2$ and the (robust but biased) empirical median that would be obtained for the function $h(x)=|x|$.
It transpires from this definition that $\hat{f}_{\text{Hub},\alpha}$ is the ERM associated to the loss function $\ell_f(x,y)=c(f(x),y)$, with $c(u,y)=h_\alpha(u-y)$.
In this case, for any subsets $\cY\subset {\bar \cY}= \R$, this cost function satisfies Assumption~\eqref{cdt:LipConv} with $L=\alpha$.

\paragraph{Logistic regression}
Here $\cY=\{-1,1\}$. 
The most classical loss in classification is the $0-1$ loss defined by ${\bf 1}_{\{y\ne f(x)\}}$, which is used in the work of Vapnik for example. 
The problem with this loss is that the minimization problem defining the ERM 
\[
\hat{f}\in \argmin_{f\in F}\sum_{i=1}^N{\bf 1}_{\{Y_i\ne f(X_i)\}}
\] 
is at best computationally demanding, and cannot even be solved in most interesting cases. 
The problem is that neither $F$ nor the function $f\mapsto P_N\ell_{f}$ are convex. 
To bypass this issue, several convex surrogates to the $0-1$ loss have been considered. 
Logistic loss is among the most famous.
Define the logistic function
\begin{equation}\label{def:LogLoss}
 \cL(u)=\log_2(1+e^{u})\enspace.
\end{equation}
The logistic function $\cL$ is convex, non-increasing and $L$-Lipschitz with $L=1/\log(2)$. 
It is used to define the logistic loss $\ell_f(x,y)=\cL(-yf(x))$.
This loss has the form $c(f(x),y)$, with 
\[
c(u,y)=\cL(-yu)\enspace.
\]
It is clear that $c(\cdot,y)$ satisfies Assumption~\ref{cdt:LipConv} with $L=1/\log 2$.

\paragraph{Hinge loss}
As in the previous example $\cY=\{-1,1\}$. 
The hinge loss is another convex surrogate to the $0-1$ loss, which is used for example in the SVM algorithms.
Define the hinge function 
\begin{equation}\label{def:HingeLoss}
 H(u)=(1+u)_+,\qquad \text{where}\qquad \forall x\in \R,\ x_+=\max(x,0)\enspace.
\end{equation}
The hinge function defines the hinge loss $\ell_f(x,y)=H(-yf(x))$. 
This loss has the form $c(f(x),y)$ with $c(u,y)=H(-uy)$.
It satisfies Assumption \eqref{cdt:LipConv} with $L=1$.

\section{Examples of classes of functions}
This section presents three classes of functions $F$.
\subsection{SVM}
Recall the definition of reproducing kernel Hilbert spaces.
\begin{definition}
 Let $W$ denote a Hilbert space of functions $f:\cX\to{\bar \cY}$, with $\cX$ separable and endowed with a continuous function $K:\cX^2\to\R$ such that
\begin{itemize}
 \item[(i)] $K$ is symmetric $K(x,x')=K(x',x)$, for any $x,x'\in \cX$,
 \item[(ii)] for any $x\in \cX$, $K(x,\cdot)\in W$,
 \item[(iii)] for any $f\in W$ and any $x\in \cX$, $\psh{f}{K(x,\cdot)}_W=f(x)$.
\end{itemize}
The space $W$ is called reproducing kernel Hilbert space (RKHS) with kernel $K$.
\end{definition}

\noindent
Let $W$ denote a RKHS with kernel $K$, $\cD_N=((X_1,Y_1),\ldots,(X_N,Y_N))$ and 
\[
F=\big\{f\in W:\|f\|_{W}\leqslant \theta\big\}\enspace.
\]
The class $F$ is used in the SVM algorithm.
Let $\ell_f$ denote the hinge loss: $\ell_f(z)=H(-yf(x))$ (the function $H$ being defined in \eqref{def:HingeLoss}). 
The support vector machine (SVM) estimator is defined as 
\begin{equation}\label{def:SVM}
\hat{f}_{\text{svm}}\in\argmin_{f\in F}P_N\ell_f\enspace. 
\end{equation}
The SVM estimator $\hat{f}_{\text{svm}}$ is an ERM based on a convex and Lipschitz loss.
SVM algorithm \eqref{def:SVM} can be equivalently defined as a solution of the minmax problem:
if $T_{\text{emp}}(f,g)=P_N[\ell_f-\ell_g]$ denotes the usual empirical test, then
\[
\hat{f}_{\text{svm}}\in \argmin_{f\in F}P_N\ell_f=\argmin_{f\in F}\sup_{f\in F}T_{\text{emp}}(f,g)\enspace.
\]
A natural alternative to SVM would therefore be the MOM SVM estimators: if $T_{\text{mom}}(f,g)=\MOM{K}{\ell_f-\ell_g}$ denotes the MOM tests,
\begin{equation}\label{def:MOMSVM}
\hat{f}_{\text{msvm}}\in \argmin_{f\in F}\sup_{f\in F}T_{\text{mom}}(f,g)\enspace. 
\end{equation}
%

\paragraph{Computational issues}
To actually compute the SVM estimator, the representer theorem shows that SVM equivalently solves $\min_{f\in F_0}P_N\ell_f$, where 
\[
F_0=\bigg\{\ba^T\bK,\ :\ba^T\K\ba\leqslant \theta^2\bigg\},\qquad \bK(x)=\begin{bmatrix}
 K(X_1,x)\\
 \vdots\\
 K(X_N,x)
\end{bmatrix}\enspace.
\]
Here, $\K$ denotes the (random) $N\times N$ matrix with entries $K(X_i,X_j)$.
Likewise, for computational issues, the representer theorem can be used to show that
\[
\hat{f}_{\text{momSVM}}\in \argmin_{f\in F_0}\sup_{f\in F_0}T_{\text{mom}}(f,g)\enspace.
\]

\subsection{Boosting}
Let $f_1,\ldots,f_d$ denote functions $f_i:\cX\to{\bar \cY}$ and let $\Delta_d$ denote the simplex in $\R^d$:
\[
\Delta_d=\bigg\{\ba\in\R^d_+:  \sum_{i=1}^da_i=1\bigg\}\enspace.
\]
The Boosting estimator is defined as
\begin{equation}\label{def:Boost}
 \hat{f}_{\text{Boost}}=\widehat{\ba}_b^T\bff,\qquad \text{where}\qquad \widehat{\ba}_b\in\argmin_{\ba\in \Delta_d}P_N\ell_{\ba},\ \bff(x)=
\begin{bmatrix}
 f_1(x)\\
 \vdots\\
 f_d(x)
\end{bmatrix}\enspace.
\end{equation}
Here, pick $\varphi\in\{\cL,H\}$ where the hinge function $H$ and the logistic function $\cL$ have been defined respectively in \eqref{def:HingeLoss} and \eqref{def:LogLoss} and define
\[
 \ell_{\ba}(z)=\varphi(-y\ba^T\bff(x))
\enspace.
\]
Clearly $\widehat{\ba}_b$ is an ERM based on Lipschitz and convex losses $\ell_{\ba}$. 
Alternatively, one can consider MOM Boosting estimators, simply by considering 
\begin{equation}\label{def:MOMBoost}
 \hat{f}_{\text{mBoost}}=\widehat{\ba}_{\text{mb}}^T\bff,\qquad \text{where}\qquad \widehat{\ba}_{\text{mb}}\in\argmin_{\ba\in \Delta_d}\sup_{\bfb\in\Delta_d}T_{\text{mom}}(\ba,\bfb)\enspace.
\end{equation}
%

\section{Non-localized bounds}
Start with a lemma extending Vapnik's bound for ERM.
Recall that this elementary upper bound states that
\[
P[\ell_{\hat{f}_{\text{erm}}}-\ell_{f^*}]\leqslant 2\sup_{f\in F}|(P_N-P)\ell_f|\enspace.
\]
This comes from the following fact.
\begin{lemma}\label{lem:VapMin}
 Let $\hat{P}$ denote any estimator of the operator $P$ and let 
 \[
\hat{f}\in\argmin_{f\in F}\hat{P}\ell_f\enspace.
\]
Then,
\[
P[\ell_{\hat{f}}-\ell_{f^*}]\leqslant 2\sup_{f\in F}|(\hat{P}-P)\ell_f|\enspace.
\]
\end{lemma}
\begin{proof}
\begin{align*}
 P[\ell_{\hat{f}}-\ell_{f^*}]&=(P-\hat{P})\ell_{\hat{f}}+(\hat{P}-P)\ell_{\bayes}+[\hat{P}\ell_{\hat{f}}-\hat{P}\ell_{\bayes}]\enspace.
\end{align*}
The third term is non-positive by definition of $\hat{f}$ while the two first terms are upper bounded by $\sup_{f\in F}|(\hat{P}-P)\ell_f|$. 
\end{proof}

The following lemma extends this bound for minmax estimators.
\begin{lemma}\label{lem:Rough}
Let $\hat{f}$ denote a minmax estimator:
\[
\hat{f}\in\argmin_{f\in F}\sup_{g\in F}\hat{P}[\ell_f-\ell_g]\enspace.
\]
Then, almost surely,
\begin{align*}
 P[\ell_{\hat{f}}-\ell_{f^*}]&\leqslant 2\sup_{f\in F}(\hat{P}-P)[\ell_{f^*}-\ell_{f}]\enspace.
\end{align*}
\end{lemma}
\begin{proof}
 Start with basics:
 
\begin{align*}
P[\ell_{\hat{f}}-\ell_{f^*}]&\leqslant \hat{P}[\ell_{\hat{f}}-\ell_{f^*}]+(\hat{P}-P)[\ell_{f^*}-\ell_{\hat{f}}]\\
&\leqslant \hat{P}[\ell_{\hat{f}}-\ell_{f^*}]+\sup_{f\in F}(\hat{P}-P)[\ell_{f^*}-\ell_{f}]\enspace. 
\end{align*}
 Then, by definition of $\hat{f}$,
 \[
\hat{P} [\ell_{\hat{f}}-\ell_{f^*}]\leqslant \sup_{g\in F}\hat{P}[\ell_{\hat{f}}-\ell_{g}]\leqslant \sup_{g\in F}\hat{P}[\ell_{f^*}-\ell_{g}]\enspace.
 \]
 Finally, by definition of $f^*$, $P[\ell_{f^*}-\ell_{g}]\leqslant 0$ for any $g\in F$, so 
 \[
 \hat{P} [\ell_{\hat{f}}-\ell_{f^*}]\leqslant\sup_{g\in F}(\hat{P}-P)[\ell_{f^*}-\ell_{g}]\enspace.
 \]
 This concludes the first inequality of Lemma~\ref{lem:Rough}. The other result is immediate.
\end{proof}

Together with concentration bounds of Chapter~\ref{Chap:ConcIn}, Lemma~\ref{lem:Rough} allows to obtain first basic bounds that can be useful in some examples.

\begin{theorem}\label{thm:RBCLL1}
 Assume that $\ell_f(z)=c(f(x),y)$ where $c$ satisfies Assumption~\ref{cdt:LipConv} and that all $f\in F$ have finite $L^2(P)$-moments. 
 Let $\sigma^2(F)=\sup_{f\in F}\text{Var}(f(X))$.
 Then, the min MOM estimator $\hat{f}_{\text{mom}}\in\argmin_{f\in F}\MOM{K}{\ell_f}$
 satisfies 
 \[
 \P\bigg(P[\ell_{\hat{f}_{\text{mom}}}-\ell_{f^*}]\leqslant8\bigg(64\sqrt{\frac{D_N(F)}N}\vee \sqrt{\frac{2\sigma^2(F)K}N}\bigg)\bigg)\geqslant 1-e^{-K/32}\enspace.
 \]
 If all $f(X)$ are Gaussian random variables, then, the ERM $\hat{f}_{\text{erm}}\in\argmin_{f\in F}P_N\ell_f$ satisfies
 \[
\forall s>0,\qquad \P\bigg(P[\ell_{\hat{f}}-\ell_{f^*}]\leqslant 8L\sqrt{\frac{D_N(F)}N}+2L\sqrt{\frac{2\sigma^2(F)s}N}\bigg)\geqslant 1-e^{-s}\enspace.
 \]
\end{theorem}
\begin{proof}
By Lemma~\ref{lem:Rough},
\[
P[\ell_{\hat{f}_{\text{mom}}}-\ell_{f^*}]\leqslant 2\sup_{f\in F}|\MOM{K}{\ell_f-P\ell_f}|\enspace.
\]
By Theorem~\ref{thm:ConcSupMOM}, 
\[
\P\bigg(\sup_{f\in F}|\MOM{K}{\ell_f-P\ell_f}|>128\sqrt{\frac{D_N(F)}N}\vee4\sqrt{\frac{2\sigma^2(F)K}N}\bigg)\leqslant e^{-K/32}\enspace.
\]
By Lemma~\ref{lem:VapMin},
\[
P[\ell_{\hat{f}_{\text{erm}}}-\ell_{f^*}]\leqslant 2\sup_{f\in F}|(P_N-P)[\ell_f]|\enspace.
\]
By Assumption~\ref{cdt:LipConv}, for any $f\in F$, $\ell_f(z)=c(f(x),y)$ is a $L$-Lipschitz function.
By Theorem~\ref{thm:LipTransfoGauss}, it follows that
\[
\P\bigg(\sup_{f\in F}|(P_N-P)[\ell_f]|\leqslant \E[\sup_{f\in F}|(P_N-P)[\ell_f]|]+\sqrt{\frac{2\sigma^2(F)s}N}\bigg)\geqslant 1-e^{-s}\enspace.
\]
Moreover, by symmetrization, 
\[
\E[\sup_{f\in F}|(P_N-P)[\ell_f]|]\leqslant 2\E\bigg[\sup_{f\in F}\frac1N\sum_{i=1}^N\epsilon_i\ell_f(Z_i)\bigg]=2\sqrt{\frac{D_N(F)}N}\enspace.
\]
\end{proof}

For example, Theorem~\ref{thm:RBCLL1} applies to SVM and Boosting and yields the following corollaries.

\begin{corollary}
Assume that the kernel $K$ is a \emph{trace norm} operator, which means that 
\begin{equation}\label{eq:TraceNormAss}
 P[K(X,X)]:=k_2\leqslant \infty\enspace.
\end{equation}
Let $\Sigma=P[K\otimes K]$, where $K\otimes K:W\to W$ is the random operator defined by
\[
\forall f\in W,\qquad K\otimes K(f)= \psh{K(X,\cdot)}{f}_WK(X,\cdot)=f(X)K(X,\cdot)\enspace.
\] 
Then, the min MOM SVM estimator satisfies
 \[
 \P\bigg(P[\ell_{\hat{f}_{\text{msvm}}}-\ell_{f^*}]\leqslant16L\theta\bigg(64\sqrt{\frac{\text{Tr}(\Sigma)}N}\vee \sqrt{\frac{2\|\Sigma\|_{\text{op}}K}N}\bigg)\bigg)\geqslant 1-e^{-K/32}\enspace.
 \]
 If $X$ is a Gaussian vector in $\R^d$, then, the SVM estimator $\hat{f}_{\text{svm}}$ defined in \eqref{def:SVM} satisfies
 \[
\forall s>0,\qquad \P\bigg(P[\ell_{\hat{f}}-\ell_{f^*}]\leqslant 8L\theta\sqrt{\frac{\text{Tr}(\Sigma)}N}+2L\theta\sqrt{\frac{2\|\Sigma\|_{\text{op}}s}N}\bigg)\geqslant 1-e^{-s}\enspace.
 \]
\end{corollary}
\begin{remark}
Assumption~\ref{eq:TraceNormAss} relaxes the boundedness assumption $\sup_{x\in \cX}K(x,x):=k_{\infty}<+\infty$ usually considered to analyse SVM.
 The expectation defining $\Sigma$ is understood in Bochner sense, see for example \cite{steinwart08support}.
\end{remark}

\begin{proof}
The result is a combination of Theorem~\ref{thm:RBCLL1} with the following lemma.
\end{proof}
\begin{lemma}\label{def:UsefulQuantSVM}
Assume that $K$ is a trace norm operator and let $\Sigma=P[K\otimes K]$.
Then,
\begin{gather*}
 D_N(F)\leqslant \theta^2k_2=\theta^2\text{Tr}(\Sigma)\enspace,\\
 \sigma^2(F)=\sup_{f\in F}\text{Var}(\ell_f(Z))\leqslant 2L^2 \sup_{f\in F}P[f^2(X)]=
2L^2 \theta^2\|\Sigma\|_{\text{op}}\enspace. 
\end{gather*}
\end{lemma}

\begin{proof}
 Start with the variance.
 Let $Z'$ denote an independent copy of $Z$.
 By Jensen's inequality,
\begin{align*}
 \text{Var}(\ell_f(Z))&=\E[(\ell_f(Z)-\E[\ell_{f}(Z')|Z])^2]\leqslant \E[(\ell_f(Z)-\ell_f(Z'))^2]\\
 &\leqslant L^2\E[(f(X)-f(X'))^2]\leqslant 2L^2\text{Var}(f(X))\leqslant 2L^2P[f^2]\enspace.
\end{align*}
 
 The operator $K\otimes K$ is a.s. symmetric: for any $f,g$ in $W$,
\[
\psh{K\otimes K(f)}{g}_W=\psh{K(X,\cdot)}{f}_W\psh{K(X,\cdot)}{g}_W=\psh{f}{K\otimes K(g)}_W\enspace.
\]
Therefore, $\Sigma$ is symmetric and, as $W$ is separable under the assumptions that $\cX$ is separable and $K$ continuous, see for example \cite[Lemma~4.33]{steinwart08support}, this implies that there exists an orthonormal basis of $W$ made of eigenvectors of $\Sigma$.
Moreover, for any $f\in W$,
\[
P[f^2(X)]=P[\psh{f}{K\otimes K(f)}_W]=\psh{f}{\Sigma(f)}_W\enspace.
\]
Therefore, 
\begin{equation}\label{eq:SupVRKHS}
 \sup_{f\in F}P[f^2]=\theta^2\|\Sigma\|_{\text{op}}\enspace.
\end{equation}
Let us now turn to the Rademacher complexity of $F$.
Using successively the representation property (iii) and Cauchy-Schwarz inequality twice,
\begin{align*}
 D_N(F)&=\bigg(\E\bigg[\sup_{f\in W:\|f\|_W\leqslant \theta}\frac1{\sqrt{N}}\sum_{i=1}^N\epsilon_if(X_i)\bigg]\bigg)^2\\
 &=\bigg(\E\bigg[\sup_{f\in W:\|f\|_W\leqslant \theta}\psh{f}{\frac1{\sqrt{N}}\sum_{i=1}^N\epsilon_iK(X_i,\cdot)}_W\bigg]\bigg)^2\\
 &\leqslant \theta^2\bigg(\E\bigg[\bigg\|\frac1{\sqrt{N}}\sum_{i=1}^N\epsilon_iK(X_i,\cdot)\bigg\|_W\bigg]\bigg)^2\\
 &=\theta^2\E\bigg[\bigg\|\frac1{\sqrt{N}}\sum_{i=1}^N\epsilon_iK(X_i,\cdot)\bigg\|_W^2\bigg]\enspace.
\end{align*}
Moreover, developing the square-norm, using the representation property (iii) shows that
\begin{align*}
 \E\bigg[\bigg\|\frac1{\sqrt{N}}\sum_{i=1}^N\epsilon_iK(X_i,\cdot)\bigg\|_W^2\bigg]&=\frac1N\sum_{1\leqslant i,j\leqslant N}\E\big[\epsilon_i\epsilon_j\psh{K(X_i,\cdot)}{K(X_j,\cdot)}_W\big]\\
 &=\frac1N\sum_{i=1}^N\E\big[\psh{K(X_i,\cdot)}{K(X_i,\cdot)}_W\big]=k_2\enspace.
\end{align*}
Hence, 
\begin{equation*}
 D_N(F)\leqslant k_2\theta^2\enspace.
\end{equation*}
Finally, the random operator $K\otimes K$ has clearly rank $1$ with $K(X,X)$ as single singular value. 
By Fubbini-Tonelli theorem, it yields
\[
k_2=P[K(X,X)]=P[\text{Tr}(K\otimes K)]=\text{Tr}(P[K\otimes K])=\text{Tr}(\Sigma)\enspace.
\]
The trace-norm assumption therefore states that the trace of $K\otimes K$ is finite.
\end{proof}

\begin{corollary}
Consider the boosting class based on a collection of functions satisfying the following assumptions.
Let $\sigma^2=\max_{1\leqslant i\leqslant d}P[f_i^2]$.
For $p=\log d$, there exists a constant $\gamma>0$ such that
\begin{equation}\label{Hyp:subGaussMom}
 \forall j\in\{1,\ldots,d\},\qquad  P[f_j^p]\leqslant (\gamma\sigma)^p \enspace.
\end{equation}
 The min MOM estimator satisfies
 \[
\P\bigg(P[\ell_{\hat{f}_{\text{mBoost}}}-\ell_{f^*}]\leqslant \frac{16L\sigma}{\sqrt{N}}\big(192e\gamma\sqrt{\log d}\vee \sqrt{2K}\big)\bigg)\geqslant 1-e^{-K/32}\enspace.
 \]
 If $X$ is a Gaussian vector in $\R^d$, then, the Boosting estimator $\hat{f}_{\text{Boost}}$ defined in \eqref{def:Boost} satisfies
 \[
\forall s>0,\qquad \P\bigg(P[\ell_{\hat{f}_{\text{Boost}}}-\ell_{f^*}]\leqslant 2L\sqrt{\frac{\|\Sigma\|_\infty}N}\big(12e\gamma \sqrt{\log d}+\sqrt{2s}\big)\bigg)\geqslant 1-e^{-s}\enspace.
 \]
\end{corollary}
\begin{proof}
The result is a combination of Theorem~\ref{thm:RBCLL1} with the following result.
\end{proof}
\begin{lemma}\label{lem:usefulQuantBoost}
Assume that $P[\|\bff(X)\|^2]<\infty$ and let 
\[
\Sigma=P[\bff(X)\bff(X)^T],\qquad \|\Sigma\|_\infty=\max_{1\leqslant i,j\leqslant d}|\Sigma_{i,j}|\enspace.
\]
Then,
\[
\sup_{\ba\in\Delta_d}\text{Var}(\ba^T\bff(X))\leqslant P[(\ba^T\bff(X))^2]\leqslant \|\Sigma\|_\infty\enspace.
\]
Moreover, for any $p\geqslant 2$ such that $\max_{1\leqslant j\leqslant d}P[|f_j|^p]<\infty$, if $\Theta_p=\sum_{i=1}^dP[|f_j|^p]$, then 
\begin{equation}\label{eq:BoundDBoost1}
D_N(F)\leqslant 9p\Theta_p^{2/p}\enspace.
\end{equation}
In particular, if \eqref{Hyp:subGaussMom} holds, then
\begin{equation}\label{eq:RadCompBoost*}
D_N(F)\leqslant 9e^2\gamma^2\|\Sigma\|_\infty\log d\enspace.
\end{equation}
\end{lemma}
\begin{proof}
 Start with the variance. 
Let $\ba\in \Delta_d$, 
\begin{align*}
 P[(\ba^T\bff(X))^2]&\leqslant \sup_{\ba\in\Delta_d}\ba^T\Sigma\ba\enspace.
\end{align*}
It is not hard not see that, for any $\ba\in\Delta_d$,
\begin{align*}
\ba^T\Sigma\ba&\leqslant \max_{i=1,\ldots,d}(\Sigma\ba)_i\leqslant \max_{1\leqslant i,j\leqslant d}|\Sigma_{i,j}| =\|\Sigma\|_\infty
\enspace.
\end{align*}
Hence,
\[
\sup_{\ba\in\Delta_d}P[(\ba^T\bff(X))^2]\leqslant \|\Sigma\|_\infty\enspace.
\]
Regarding the Rademacher complexity. 
\begin{align}
\notag D_N(F)&= \bigg(\E\bigg[\sup_{\ba\in \Delta_d}\frac1{\sqrt{N}}\sum_{i=1}^N\epsilon_i\ba^T\bff(X_i)\bigg]\bigg)^2\\
\notag &=\bigg(\E\bigg[\sup_{\ba\in \Delta_d}\ba^T\bigg(\sum_{i=1}^N\epsilon_i\frac{\bff(X_i)}{\sqrt{N}}\bigg)\bigg]\bigg)^2\\
\label{eq:RadCompBoost0}&=\bigg(\E\bigg[\max_{1\leqslant j\leqslant d}\bigg|\bigg(\sum_{i=1}^N\epsilon_i\frac{\bff(X_i)}{\sqrt{N}}\bigg)_j\bigg|\bigg]\bigg)^2\enspace.
\end{align}
Under the assumption $\max_{1\leqslant j\leqslant d}P[|f_j|^p]<\infty$, the random variables 
\[
Z_j=\bigg(\sum_{i=1}^N\epsilon_i\frac{\bff(X_i)}{\sqrt{N}}\bigg)_j
\]
have finite moments of order $p$.
Moreover, by Jensen's inequality,
\begin{equation}\label{eq:RadCompBoost}
 \E\big[\max_{1\leqslant j\leqslant d}\big|Z_j\big|\big]\leqslant \bigg(\E\big[\max_{1\leqslant j\leqslant d}\big|Z_j\big|^p\big]\bigg)^{1/p}\leqslant \bigg(\sum_{j=1}^d\E\big[\big|Z_j\big|^p\big]\bigg)^{1/p}\enspace.
\end{equation}
Now, apply Khinchine's inequality on moments of order $p$ for sums of independent random variables, see for examples \cite[Chapter 15]{BouLugMass13}.
It shows that
\begin{align*}
 \E\big[\big|Z_j\big|^p\big]^{1/p}&\leqslant 3\sqrt{p\sum_{i=1}^N\E\bigg[\frac{|f_j(X_i)|^p}{N^{p/2}}\bigg]^{2/p}}=3\sqrt{\frac{p}N\sum_{i=1}^N\E\big[|f_j(X_i)|^p\big]^{2/p}}\\
 &=3\sqrt{pP\big[|f_j|^p\big]^{2/p}}=3\sqrt{p}P\big[|f_j|^p\big]^{1/p}\enspace.
\end{align*}
This shows \eqref{eq:BoundDBoost1}.
By Assumption~\ref{Hyp:subGaussMom}, it follows that
\[
\E\big[\big|Z_j\big|^p\big]^{1/p}\leqslant 3\gamma\sqrt{pP[f_j^2]}\leqslant 3\gamma\sqrt{p\|\Sigma\|_{\infty}}\enspace.
\]
Plugging this inequality into \eqref{eq:RadCompBoost} yields
\[
\E\big[\max_{1\leqslant j\leqslant d}\big|Z_j\big|\big]\leqslant 3\gamma\sqrt{p\|\Sigma\|_\infty}d^{1/p}\enspace.
\]
As $p=\log d$, this yields
\[
\E\big[\max_{1\leqslant j\leqslant d}\big|Z_j\big|\big]\leqslant 3e\gamma\sqrt{\|\Sigma\|_\infty\log d}\enspace.
\]
Plugging this bound into \eqref{eq:RadCompBoost0} shows \eqref{eq:RadCompBoost*}.
\end{proof}

\section{Localized bounds: preliminary results}\label{sec:Prelim}
Theorem~\ref{thm:RBCLL1} is useless when $D(F)=\infty$, which happens for example with  classes of linear functions indexed by unbounded subsets of $\R^d$, for example:
\[
F=\{\bff^T\cdot,\ \bff\in \R^d\}\enspace.
\] 
The following sections develop a general strategy that allows to deal with these examples.
Hereafter, assume that $\cX= \R^d$ and $F$ is the set of all linear functions $\bff^T\cdot$ with $\bff\in \R^d$.
Assume also that the distribution $P$ of $Z=(X,Y)$ has a first marginal $X$ satisfying $P[\|X\|^2]<\infty$ and $ \cY\subset \R$. 
Denote by $\Sigma=P[XX^T]$.
Both the ERM and minmax MOM estimators will be analysed thanks to the homogeneity lemma, Lemma~\ref{lem:DetArg}. 
The convexity of $c(\cdot,y)$ implies the convexity of $\ell_f$ therefore, Lemma~\ref{lem:ConvLoss} applies and shows that the tests 
\[
T_{\text{erm}}(f,g)=P_N[\ell_f-\ell_g],\qquad T_{\text{mom}}(f,g)=\MOM{K}{\ell_f-\ell_g}
\]
satisfy the homogeneity assumption {\bf (HP)} of the homogeneity lemma, provided that the evaluation function $\cE$ derives from a norm.
Hereafter, for any $f\in F$, let 
\[
\cE(f)=\sqrt{P[(f-f^*)^2]}\enspace.
\]
Finally, as in every learning problem
\[
d(f,g)=P[\ell_f-\ell_g]\enspace,
\]
so $d(f,g)=-d(g,f)$ and $d(f^*,f)\leqslant 0$ so $\cB=B(r_1)$ in the homogeneity lemma (Lemma~\ref{lem:DetArg}).
The homogeneity lemma will be used under a technical assumption that we introduce and discuss in the following section.

\section{Bernstein's condition}\label{SecExBernstein}
To check \eqref{def:r1} and \eqref{def:r2}, the following ``local" Bernstein condition will be useful: there exist $A>0$ and $B>0$ such that
\begin{equation}\label{Ass:Bern}
\forall f\in F: \cE(f)\leqslant A,\qquad P[\ell_f-\ell_{f^*}]\geqslant B\cE(f)^2\enspace.
\end{equation}

Relationships between between $\cE(f)$ and the excess risk $P[\ell_f-\ell_{f^*}]$ are usually called Bernstein's condition. 
These are convenient to prove ``fast rates" of convergence for ERM with bounded losses, see for example \cite{MR2051002} for a discussion on fast and slow rates.
To the best of our knowledge, this assumption first appeared in \cite[Hyp A2 of Theorem 4.2]{MR1813803}. 
This form of Assumption~\ref{Ass:Bern} was first introduced in \cite{Chinot2018robust}.
The relationship between $\cE(f)$ and $P[\ell_f-\ell_{f^*}]$ is only assumed in a neighborhood of $f^*$.
This is a necessary constraint to deal with unbounded classes of functions.
Actually, by the Lipschitz assumption of $c$, it holds, by Cauchy-Schwarz inequality,
\[
P[\ell_f-\ell_{f^*}]\leqslant LP|f-f^*|\leqslant L\cE(f)\enspace.
\]
Hence, the Bernstein's assumption \eqref{Ass:Bern} can only be true if 
\[
B\cE(f)^2\leqslant L\cE(f),\qquad\text{that is, if}\qquad \cE(f)\leqslant \frac{L}B\enspace.
\]
Let us present some examples where Assumption~\eqref{Ass:Bern} holds. 
To proceed, we assume in the remaining of this sections that 
\begin{equation}\label{Hyp:OracleInModel}
 f^*\text{ is a minimizer of }P\ell_f \text{ among all measurable functions }f:\cX\to\cY\enspace.
\end{equation}
This assumption is quite restrictive as it implies that the model $F$ is ``exact".
It is convenient to make explicite computations. 
Indeed, it ensures that 
\[
\forall x\in \cX,\qquad f^*(x)\in\argmin_{u\in \R}\E[c(u,Y)|X=x]\enspace.
\]
In particular, it allows to show results on $f^*$ based on assumption on the c.d.f. of $Y$ conditionally on $X=x$.

The second assumption that will be done all along the examples is an hypothesis comparing $L^4(P)$ and $L^2(P)$ norms of functions in $F$.
For any $p\geqslant 1$, for any function $f:\cX\to\R$ for which it makes sense, let
\[
\|f\|_{L^p(P)}=\big(P[|f|^p]\big)^{1/p}\enspace.
\]
The $L^4/L^2$ assumption states that there exists $\Delta\geqslant 1$ such that
\begin{equation}\label{Ass:L4L2}
\forall f\in F:\qquad \|f-f^*\|_{L^4(P)}\leqslant \Delta \|f-f^*\|_{L^2(P)}\enspace.
\end{equation}
Let us comment this assumption.
First, by Cauchy-Schwarz inequality 
\[
\|f-f^*\|_{L^2(P)}\leqslant \|f-f^*\|_{L^4(P)}\enspace,
\]
hence, the restriction $\Delta\geqslant 1$ in Assumption~\eqref{Ass:L4L2} holds without loss of generality.
The following proposition gives an example where Assumption~\eqref{Ass:L4L2} holds.
\begin{proposition}
 Assume that $X\in \R^d$ is a vector with centered, independent entries $X_i$, $i\in\{1,\ldots,d\}$ with kurtosis bounded by $\kappa$, i.e. such that $P[X_i^4]^{1/4}\leqslant \kappa P[X_i^2]^{1/2}$.
 Then, any linear function $f(\cdot)=\bff^T\cdot$ satisfies $\|f\|_{L^4(P)}\leqslant \kappa\|f\|_{L^2(P)}$. 
\end{proposition}
\begin{proof}
 One can assume w.l.o.g. that $\kappa\geqslant 1$.
 Using independence of $X_i$ and the fact that $P[X_i]=0$, 
\begin{gather*}
  \|f\|_{L^2(P)}=\bigg(\sum_{i=1}^d\bff_i^2P[X_i^2]\bigg)^{1/2}\enspace,\\
    \|f\|_{L^4(P)}=\bigg(\sum_{i=1}^d\bff_i^4P[X_i^4]+\sum_{1\leqslant i\ne j\leqslant d}\bff_i^2\bff_j^2P[X_i^2]P[X_j]^2\bigg)^{1/4}\enspace.
\end{gather*}
 Using that $P[X_i^4]\leqslant \kappa^4 P[X_i^2]^2$ and $\kappa\geqslant 1$, it yields
\begin{align*}
 \|f\|_{L^4(P)}&\leqslant \kappa\bigg(\sum_{i=1}^d\bff_i^4P[X_i^2]^2+\sum_{1\leqslant i\ne j\leqslant d}\bff_i^2\bff_j^2P[X_i^2]P[X_j]^2\bigg)^{1/4}\\
 &=\kappa\bigg(\sum_{1\leqslant i, j\leqslant d}\bff_i^2\bff_j^2P[X_i^2]P[X_j]^2\bigg)^{1/4}\\
 &=\kappa\bigg(\sum_{i=1}^d\bff_i^2P[X_i^2]\bigg)^{1/2}=\kappa\|f\|_{L^2(P)}\enspace.
\end{align*}
\end{proof}

The $L^4/L^2$ should be used with care as shown by the following example.
\begin{proposition}
Let $X$ denote a random variables taking values in a measurable space $\cX$.
Let $I_1,\ldots,I_d$ denote a partition of $\cX$ such that $P[I_j]=1/d$ for any $j\in\{1,\ldots,d\}$.
Let $\bX\in \R^d$ denote the vector
\[
\bX=
\begin{bmatrix}
 {\bf 1}_{\{X\in I_1\}}\\
 \vdots\\
 {\bf 1}_{\{X\in I_d\}}
\end{bmatrix}\in \R^d\enspace.
\]
Then, for any $\bff\in \R^d$, $P[(\bff^T\bX)^4]^{1/4}\leqslant d^{1/4}P[(\bff^T\bX)^2]^{1/2}$.
\end{proposition}
\begin{remark}
 In words, any class of linear functions $f(\cdot)=\bff^T\cdot$ satisfies Assumption~\eqref{Ass:L4L2}, but with a parameter $\Delta$ that is \emph{not} a constant, but depends on the dimension $d$.
\end{remark}

\begin{proof}
For any $\bff\in \R^d$,
\begin{align*}
 P[(\bff^T\bX)^4]&=\sum_{j=1}^d\bff^4_iP[I_j]\leqslant d\sum_{j=1}^d\bff^4_iP[I_j]^2\leqslant d\big(\sum_{j=1}^d\bff_j^2P[I_j]\big)^2=dP[(\bff^T\bX)^2]^2\enspace.
\end{align*}
\end{proof}

\paragraph{Huber loss}
Denote by $F_x$ the conditional c.d.f. of $Y$ given $X=x$. 
Assume that there exists $\nu>0$ such that
\begin{equation}\label{Ass:Hub}
 \forall (x,y)\in \cX\times\cY:\ |y-f^*(x)|\leqslant 2A\Delta^2 ,\qquad F_x(y+\alpha)-F_x(y-\alpha)\geqslant \nu\enspace.
\end{equation}
For example, Assumption~\eqref{Ass:Hub} holds if the conditional density $f_x$ of $Y$ given $X=x$ is bounded away from $0$ in a neighborhood of $f^*(x)$.

\begin{proposition}
 Assume \eqref{Hyp:OracleInModel}, \eqref{Ass:L4L2} and \eqref{Ass:Hub}.
Then, 
\[
 \forall f\in F:\cE(f)\leqslant A,\qquad P[\ell_f-\ell_{f^*}]\geqslant \frac{\nu}4\cE(f)^2\enspace.
\]
\end{proposition}

\begin{proof}
 Let 
 \[
 H_x(u)=\E[h_\alpha(Y-u)|X=x]=\int h_\alpha(y-u)\rmd F_x(y)\enspace.
 \]
 The function $H_x$ is differentiable, with 
 
\begin{align*}
  H'_x(u)&=-\int h'_{\alpha}(y-u)F_x(y)\\
  &=\alpha\int_{-\infty}^{u-\alpha}\rmd F_x(y)-\int_{u-\alpha}^{u+\alpha}(y-u)\rmd F_x(y)-\alpha\int_{u+\alpha}^{+\infty}\rmd F_x(y)\\
  &=\alpha(F_x(u-\alpha)-1+F_x(u+\alpha))-\int_{u-\alpha}^{u+\alpha}(y-u)\rmd F_x(y)\\
  &=\alpha(F_x(u-\alpha)-1+F_x(u+\alpha))-[(y-u)F_x(y)]_{u-\alpha}^{u+\alpha}+\int_{u-\alpha}^{u+\alpha} F_x(y)\rmd y\\
  &=\int_{u-\alpha}^{u+\alpha} F_x(y)\rmd y-\alpha\enspace.
\end{align*}
In particular, as $f^*(x)\in \argmin_{u\in \R}H_x(u)$, it follows that $H'_x(f^*(x))=0$. 
Moreover, 
\[
H''_x(u)=F_x(u+\alpha)-F_x(u-\alpha)\enspace.
\]
Let $\cX_{\text{loc}}=\{x\in\cX:|f(x)-f^*(x)|\leqslant 2A\Delta^2\}$.
For any $x\in \cX_{\text{loc}}$, it follows that
\begin{align*}
H_x(f(x))-H_x(f^*(x))&=\int_{f^*(x)}^{f(x)}H_x'(u)\rmd u=\int_{f^*(x)}^{f(x)}(H_x'(u)-H'_x(f^*(x)))\rmd u\\
&=\int_{f^*(x)}^{f(x)}\int_{f^*(x)}^uH_x''(v)\rmd v\rmd u\enspace.
\end{align*}
For any $v$ in the segment with extremities $f^*(x)$ and $u$, by Assumption~\eqref{Ass:Hub},
\[
H_x''(v)=F_x(v+\alpha)-F_x(v-\alpha)\geqslant\nu\enspace.
\]
Therefore, if $f(x)\geqslant f^*(x)$, 
\begin{align*}
H_x(f(x))-H_x(f^*(x))&\geqslant \int_{f^*(x)}^{f(x)}\int_{f^*(x)}^u\nu\rmd v\rmd u\\
&=\int_{f^*(x)}^{f(x)}\nu(u-f^*(x))\rmd u\\
&=\frac{\nu}2(f(x)-f^*(x))^2\enspace.
\end{align*}
Likewise, if $f(x)\leqslant f^*(x)$, 
\begin{align*}
H_x(f(x))-H_x(f^*(x))&\geqslant \int_{f(x)}^{f^*(x)}\int_u^{f^*(x)}\nu\rmd v\rmd u\\
&=\int_{f(x)}^{f^*(x)}\nu(f^*(x)-u)\rmd u\\
&=\frac{\nu}2(f(x)-f^*(x))^2\enspace.
\end{align*}
Overall, by definition of $f^*(x)$, $H_x(f(x))-H_x(f^*(x))\geqslant 0$ for any $x\in \cX$ and
\[
\forall x\in \cX_{\text{loc}},\qquad H_x(f(x))-H_x(f^*(x))\geqslant \frac{\nu}2(f(x)-f^*(x))^2\enspace.
\]
It follows that
\begin{align}
\notag P[\ell_f-\ell_{f^*}]&=\E[H_X(f(X))-H_X(f^*(X))]\\
\notag &\geqslant \E[\{H_X(f(X))-H_X(f^*(X))\}{\bf 1}_{\{X\in \cX_{\text{loc}}\}}]\\
\notag &\geqslant \frac{\nu}2\E[(f(X)-f^*(X))^2{\bf 1}_{\{X\in \cX_{\text{loc}}\}}]\\
\label{eq:Hub->BI} &=\frac{\nu}2(\cE(f)-\E[(f(X)-f^*(X))^2{\bf 1}_{\{X\notin \cX_{\text{loc}}\}}])\enspace.
 \end{align}
 By Cauchy-Schwarz, 
\begin{equation}\label{eq:CSHub}
 \E[(f(X)-f^*(X))^2{\bf 1}_{\{X\notin \cX_{\text{loc}}\}}]\leqslant \|f-f^*\|_{L^4(P)}^2\sqrt{\P(X\notin \cX_{\text{loc}})}\enspace.
\end{equation}
 By Markov's inequality,
 \[
 \P(X\notin \cX_{\text{loc}})=\P(|f(x)-f^*(x)|> 2A\Delta^2)\leqslant \frac{\|f-f^*\|_{L^2(P)}^2}{4A^2\Delta^4}=\frac{\cE(f)^2}{4A^2\Delta^4}\enspace.
 \]
 If $\cE(f)\leqslant A$, it follows that 
 \[
 \P(X\notin \cX_{\text{loc}})\leqslant \frac1{4\Delta^4}\enspace. 
 \]
 Plugging this into \eqref{eq:CSHub} shows that, for any $f\in F$ such that $\cE(f)\leqslant A$.
 \[
 \E[(f(X)-f^*(X))^2{\bf 1}_{\{X\notin \cX_{\text{loc}}\}}]\leqslant \frac{\|f-f^*\|_{L^4(P)}^2}{2\Delta^2}\enspace.
 \]
Using \eqref{Ass:L4L2}, we get 
\[
 \E[(f(X)-f^*(X))^2{\bf 1}_{\{X\notin \cX_{\text{loc}}\}}]\leqslant \frac{\|f-f^*\|_{L^2(P)}^2}{2}=\frac{\cE(f)^2}{2}\enspace.
\]
Plugging this inequality into \eqref{eq:Hub->BI} concludes the proof.
\end{proof}

\paragraph{Logistic regression}
Denote by $\eta:\cX\to\cY$ the regression function satisfying $\E[Y\vartheta(X)]=P[\eta\vartheta]$ for any bounded measurable function $\vartheta$.
Recall that 
\[
\log\bigg[\frac{\eta(x)}{1-\eta(x)}\bigg]\in\argmin_{u\in \R}\E[\cL(-Yu)|X=x] \enspace.
\]
Assume that there exists $\nu>0$, such that 
\begin{equation}\label{eq:HypLogReg}
\P\bigg(\frac1{1+e^\nu}\leqslant \eta(X)\leqslant \frac1{1+e^{-\nu}}\bigg)\geqslant 1-\frac1{8\Delta^4}\enspace.
\end{equation}
This is equivalent to 
\[
\P\bigg(\log\bigg[\frac{\eta(X)}{1-\eta(X)}\bigg]>\nu\bigg)\leqslant \frac1{8\Delta^4}\enspace.
\]
\begin{proposition}
 Assume \eqref{Hyp:OracleInModel}, \eqref{Ass:L4L2}, \eqref{eq:HypLogReg}.
 Then, there exists a constant $B=\bB(A,\nu,\Delta)>0$ such that, for all $f\in F$ such that $\cE(f)\leqslant A$, $P[\ell_f-\ell_{f^*}]\geqslant B\cE(f)^2$.
\end{proposition}
\begin{remark}
A value of the constant $B$ is given in Eq~\eqref{eq:defBLog} in the proof.
\end{remark}
\begin{proof}
 Let $f\in F$ such that $\cE(f)\leqslant A$.
 Let $H_x(u)=\eta(x)\log_2(1+e^{-u})+(1-\eta(x))\log_2(1+e^u)$.
 The function $H_x$ is continuously twice differentiable with
\begin{align*}
 H_x'(u)&=\frac{\eta(x)}{\log(2)}\frac{-e^{-u}}{1+e^{-u}}+\frac{1-\eta(x)}{\log 2}\frac{e^u}{1+e^{u}}\\
 &=\frac{-\eta(x)+(1-\eta(x))e^u}{(\log 2)(1+e^u)}\enspace.
 \end{align*}
 \begin{align*}
 H_x''(u)&=\frac{(1-\eta(x))e^u(1+e^u)-(-\eta(x)+(1-\eta(x))e^u)e^u}{(\log 2)(1+e^u)^2}\\
 &=\frac{e^u}{(\log 2)(1+e^u)^2}\enspace.
 \end{align*}
Fix $\zeta>0$ and let $\cX_{\text{loc}}=\{x\in \cX:|f^*(x)|\leqslant \nu, |f(x)-f^*(x)|\leqslant \sqrt{8}A\Delta^2\}$.
For any $x\in \cX_{\text{loc}}$, $\max\{|f(x)|,|f^*(x)|\}\leqslant \nu+\sqrt{8}A\Delta^2$.
Therefore, as $H_x'(f^*(x))=0$, for any $x\in \cX$, $H_x(f(x))-H_x(f^*(x))\geqslant 0$ and 
\[
\forall x\in \cX_{\text{loc}},\qquad H_x(f(x))-H_x(f^*(x))\geqslant 2B(f(x)-f^*(x))^2\enspace,
\]
where
\begin{equation}\label{eq:defBLog}
 B=\frac{e^{-(\nu+\sqrt{8}A\Delta^2)}}{2(\log 2)(1+e^{\nu+\sqrt{8}A\Delta^2})^2}\enspace.
\end{equation}
It follows that
\begin{align}
\notag P[\ell_f-\ell_{f^*}]&=\E[H_X(f(X))-H_X(f^*(X))]\\
\notag &\geqslant \E[\{H_X(f(X))-H_X(f^*(X))\}{\bf 1}_{\{X\in \cX_{\text{loc}}\}}]\\
\notag &\geqslant \frac{\nu}2\E[(f(X)-f^*(X))^2{\bf 1}_{\{X\in \cX_{\text{loc}}\}}]\\
\label{eq:Log->BI} &=2B(\cE(f)-\E[(f(X)-f^*(X))^2{\bf 1}_{\{X\notin \cX_{\text{loc}}\}}])\enspace.
 \end{align}
 By Cauchy-Schwarz, 
\begin{equation}\label{eq:CSLog}
 \E[(f(X)-f^*(X))^2{\bf 1}_{\{X\notin \cX_{\text{loc}}\}}]\leqslant \|f-f^*\|_{L^4(P)}^2\sqrt{\P(X\notin \cX_{\text{loc}})}\enspace.
\end{equation}
 By Markov's inequality,
\begin{align*}
 \P(X\notin \cX_{\text{loc}})&\leqslant \P(|f^*(X)|>\nu)+\P(|f(X)-f^*(X)|> \sqrt{8} A\Delta^2)\\
 &\leqslant \frac1{8\Delta^4}+\frac{\|f-f^*\|_{L^2(P)}^2}{8A^2\Delta^4}\\
 &=\frac1{8\Delta^4}+\frac{\cE(f)^2}{8A^2\Delta^4}\enspace.
\end{align*}
 If $\cE(f)\leqslant A$, it follows that 
 \[
 \P(X\notin \cX_{\text{loc}})\leqslant \frac1{4\Delta^4}\enspace. 
 \]
 Plugging this into \eqref{eq:CSLog} shows that, for any $f\in F$ such that $\cE(f)\leqslant A$.
 \[
 \E[(f(X)-f^*(X))^2{\bf 1}_{\{X\notin \cX_{\text{loc}}\}}]\leqslant \frac{\|f-f^*\|_{L^4(P)}^2}{2\Delta^2}\enspace.
 \]
Using \eqref{Ass:L4L2}, we get 
\[
 \E[(f(X)-f^*(X))^2{\bf 1}_{\{X\notin \cX_{\text{loc}}\}}]\leqslant \frac{\|f-f^*\|_{L^2(P)}^2}{2}=\frac{\cE(f)^2}{2}\enspace.
\]
Plugging this inequality into \eqref{eq:Log->BI} concludes the proof.
\end{proof}

\paragraph{Exercise} Find conditions sufficient to prove the ``local" Bernstein's condition for the Hinge loss.

\section{ERM in the Gaussian case}
We consider a cost function $c$ satisfying Assumption~\eqref{cdt:LipConv} and study the estimator 
\begin{equation}\label{def:ERMConvLip}
\hat{f}\in \argmin_{f\in \R^d}\sum_{i=1}^Nc(f^TX_i,Y_i)\enspace. 
\end{equation}

\begin{theorem}\label{thm:LCGauss}
Assume that $X$ is Gaussian with $\Sigma=P[XX^T]$ positive definite.
Assume that the Bernstein assumption \eqref{Ass:Bern} holds for constants $A$ and $B$ such that
\[
AB\sqrt{N}\geqslant 2(1+\sqrt{5})L\sqrt{d}\enspace.
\]
Then, for any $s$ such that
\[
(1+\sqrt{5})L(4\sqrt{d}+\sqrt{2s})\leqslant 2AB\sqrt{N}\enspace,
\]
 the empirical risk minimizer \eqref{def:ERMConvLip} satisfies
 \[
\P\bigg( \cE(\hat{f})\leqslant \frac{1+\sqrt{5}}2\frac{ L}B\frac{4\sqrt{d}+\sqrt{2s}}{\sqrt{N}}\bigg)\geqslant 1-2e^{-s}\enspace.
 \]
\end{theorem}
\begin{remark}
 This result is ``robust" as it does not involve assumptions on the outputs $Y_i$.
\end{remark}
\begin{proof}
Recall that the empirical risk minimizer 
\[
\hat{f}\in\argmin_{f\in F}P_N\ell_f=\argmin_{f\in F}\sup_{g\in F}T_{\text{emp}}(f,g)\enspace,
\]
with 
\[
T_{\text{emp}}(f,g)=P_N[\ell_f-\ell_g]\enspace.
\]
As explained in Section~\ref{sec:Prelim}, the test $T_{\text{erm}}(f,g)$ satisfy Assumption {\bf (HP)} of the homogeneity lemma (Lemma~\ref{lem:DetArg}).
Moreover, recall that we want to apply this lemma with $d(f,g)=P[\ell_f-\ell_g]$.
It remains to compute the function $B$ in the homogeneity lemma, and for this, we look for a bound $B(r)$ such that, with high probability,
\[
\sup_{f\in F:\cE(f)\leqslant r}(P_N-P)[\ell_{f^*}-\ell_f]\leqslant B(r)\enspace.
\]
Assume to simplify the argument that $f^*=0$.
This case can be solved with the basic Gaussian concentration inequality. 
The general case involves more elaborated tools on Gaussian processes, see \cite[Lemma 8.1.]{pierre2017estimation}.

Let $F_r=\{f\in F:\cE(f)\leqslant r\}$.
By Theorem~\ref{thm:LipTransfoGauss},
with probability larger that $1-e^{-s}$,
\[
\sup_{f\in F_r}(P_N-P)[\ell_{f^*}-\ell_f]\leqslant E_N(F_r)+\sqrt{\frac{2\sigma^2(F_r)s}N}\enspace.
\]
Here, $\sigma^2(F_r)=\sup_{f\in F_r}\text{Var}((\ell_f-\ell_{f^*})(Z))$ and
\begin{align*}
E_N(F_r)&=\E\bigg[\sup_{f\in F:\|f-f^*\|\leqslant r}(P_N-P)[\ell_{f^*}-\ell_f]\bigg]\enspace.
\end{align*}
Let us first bound the variance.
\[
\text{Var}((\ell_f-\ell_{f^*})(Z))\leqslant P[(\ell_f-\ell_{f^*})^2]\leqslant L^2P[(f-f^*)^2]=L^2\cE(f)^2\enspace.
\]
Hence, $\sigma^2(F_r)\leqslant L^2r^2$.
Using the symmetrization trick, 
\[
E_N(F_r)\leqslant 2\sqrt{\frac{D_N(F_r)}N}\enspace,
\]
where 
\[
D_N(F_r)=\bigg(\E\bigg[\sup_{f\in F_r}\frac1{\sqrt{N}}\sum_{i=1}^N\epsilon_i(\ell_f-\ell_{f^*})(Z_i)\bigg]\bigg)^2\enspace.
\]
By the contraction lemma,
\begin{align*}
D_N(F_r)&\leqslant 4L^2\bigg(\E\bigg[\sup_{f\in F_r}\frac1{\sqrt{N}}\sum_{i=1}^N\epsilon_i(f-f^*)(X_i)\bigg]\bigg)^2\enspace.
\end{align*}
Now, $F_r=\{f=f^*+rg,\ g\in \cB\}$, with $\cB=\{f=\bff^T\cdot:\  P[(\bff^TX)^2]=1\}$.
Hence,
\begin{align*}
D_N(F_r)&\leqslant 4L^2r^2\bigg(\E\bigg[\sup_{f=\bff^T\cdot\in \cB}\bff^T\bigg(\frac1{\sqrt{N}}\sum_{i=1}^N\epsilon_iX_i\bigg)\bigg]\bigg)^2\enspace.
\end{align*}
Assume that $\Sigma$ is positive definite. 
In this case, one can define a positive definite square root $\Sigma^{1/2}$ of $\Sigma$. 
Therefore, for any $\ba$, $\bfb$ in $\R^d$, 
\[
\ba^T\bfb=(\Sigma^{1/2}\ba)^T(\Sigma^{-1/2}\bfb)\leqslant (\ba^T\Sigma\ba)^{1/2}(\bfb^T\Sigma^{-1}\bfb)^{1/2}\enspace.
\]
Defining, for any positive semi-definite matrix $\cM$ and any vector $\ba\in \R^d$, $\|\ba\|_{\cM}=\ba^T\cM\ba$, it follows that
\begin{align*}
D_N(F_r)&\leqslant 4L^2r^2\bigg(\E\bigg[\bigg\|\frac1{\sqrt{N}}\sum_{i=1}^N\epsilon_iX_i\bigg\|_{\Sigma^{-1}}\bigg]\bigg)^2\enspace.
\end{align*}
By Cauchy Schwarz inequality, 
\begin{align*}
D_N(F_r)&\leqslant 4L^2r^2\E\bigg[\bigg\|\frac1{\sqrt{N}}\sum_{i=1}^N\epsilon_iX_i\bigg\|_{\Sigma^{-1}}^2\bigg]\\
&=\frac{4L^2r^2}N\sum_{1\leqslant i,j\leqslant N}\E[\epsilon_i\epsilon_j X_i^T\Sigma^{-1}X_j]\\
&=4L^2r^2\E[X^T\Sigma^{-1}X]=4L^2r^2d\enspace.
\end{align*}

Hence, with probability larger that $1-e^{-s}$,
\[
\sup_{f\in F:\|f-f^*\|\leqslant r}(P_N-P)[\ell_{f^*}-\ell_f]\leqslant Lr\frac{4\sqrt{d}+\sqrt{2s}}{\sqrt{N}}\enspace.
\]
This suggests to use in the homogeneity lemma the function 
\[
B(r)=Lr\frac{4\sqrt{d}+\sqrt{2s}}{\sqrt{N}}:=rr_s\enspace.
\]
By the Bernstein Assumption \eqref{Ass:Bern}, \eqref{def:r1} would hold for $r_1$ solution of the equation 
\[
rr_s-Br^2=0\enspace.
\]
This is possible if 
\[
r_1=\frac{r_s}B\leqslant A\enspace.
\]
If this assumption is met, then \eqref{def:r2} would hold for $r_2$ solution of the equation
\[
\frac{r_s^2}{B}+rr_s-Br^2=0\enspace,
\]
that is for 
\[
r_2=\frac{1+\sqrt{5}}2\frac{r_s}B\enspace.
\]
\end{proof}
\section{Minmax MOM estimators}
This section extends the previous result to the case where the \emph{design} is not assumed to be sub-Gaussian anymore.
Indeed, Lipshitz losses are classically considered in robust statistics.
This success, as explained after Theorem~\ref{thm:LCGauss}, is due to the fact that the ERM can be analysed in this framework \emph{without assumptions on the outputs} $Y$.
However, this analysis highly depends on the sub-Gaussian assumption made on the design.
The extension is even more important to handle possibly corrupted datasets.
Indeed, these data are likely to be corrupted, specially in high dimensional settings.

Consider the minmax MOM estimator
\begin{equation}\label{eq:minmaxMOMK}
\hat{f}_{K}\in \argmin_{f\in F}\sup_{g\in F}\MOM{K}{\ell_f-\ell_g}. 
\end{equation}
The main result here is the following.
\begin{theorem}
Assume that the Bernstein assumption \eqref{Ass:Bern} holds for constants $A$ and $B$ such that
\[
AB\sqrt{N}\geqslant 128(1+\sqrt{5})L\sqrt{d}\enspace.
\]
Then, for any $K$ such that
\[
2(1+\sqrt{5})L\big(64\sqrt{d}\vee\sqrt{2K}\big)\leqslant AB\sqrt{N}\enspace,
\]
 the empirical risk minimizer \eqref{def:ERMConvLip} satisfies
 \[
\P\bigg(\cE( \hat{f}_K)\leqslant 2(1+\sqrt{5})\frac{L}B\bigg(64\frac{\sqrt{d}\vee\sqrt{2K}}{\sqrt{N}}\bigg)\bigg)\geqslant 1-2e^{-K/32}\enspace.
 \]
\end{theorem}
\begin{proof}
 
%
The key is to compute the function $B$ in the homogeneity lemma. 
Let $r>0$ fixed. Apply the concentration bound for suprema of MOM processes on the class of functions 
\[
F_r=\{\ell_{f^*}-\ell_f-P[\ell_{f^*}-\ell_f], f\in B(f^*,r)\}\enspace.
\]
With probability at least $1-e^{-K/32}$,
\[
\sup_{f\in B(f^*,r)}\MOM{K}{\ell_{f^*}-\ell_f-P[\ell_{f^*}-\ell_f]}\leqslant 128\sqrt{\frac{D(F_r)}N}\vee4\sigma(F_r)\sqrt{2\frac{K}N}\enspace.
\]
The computations of the previous proof show that $D(F_r)\leqslant 4L^2r^2d$, $\sigma^2(F_r)\leqslant L^2r^2$, hence, with probability at least $1-e^{-K/32}$,
\[
\sup_{f\in B(f^*,r)}\MOM{K}{\ell_{f^*}-\ell_f-P[\ell_{f^*}-\ell_f]}\leqslant 4Lr\bigg(64\frac{\sqrt{d}\vee\sqrt{2K}}{\sqrt{N}}\bigg)\enspace.
\]
This suggests to use 
\[
B(r)=rr_K,\qquad \text{with}\qquad r_K=4Lr\bigg(64\frac{\sqrt{d}\vee\sqrt{2K}}{\sqrt{N}}\bigg)\enspace.
\]
The proof is concluded with the same arguments as the previous one.
\end{proof}

\chapter{Least-squares regression}\label{Chap:LSR}

This chapter considers the classical least-squares linear regression problem.
This problem has attracted a lot of attention recently in the case where both the inputs $X$ and the outputs $Y$ may be heavy-tailed.
The first paper proving oracle inequalities in this setting is \cite{MR2906886}. 
The estimator there was derived from $M$-estimators of univariate expectations.
Recent articles, in particular the seminal paper \cite{LugosiMendelson2016}, see also \cite{LugosiMendelson2017, MOM1, MOM2}, also investigate median-of-mean approaches in both small and large dimension least-squares regression.
This analysis is reproduced in this chapter in the simplified setting of linear least-squares regression.
The pros and cons of these approaches are the same as in the multivariate mean estimation problems, see the discussion in Section~\ref{Sec:PacBayes}.
All these results rely on either a $L^4/L^2$ or a $L^2/L^1$ comparison between the functions in the hypothesis class $F$ that should hold uniformly for a constant that should not depend on the dimension of $F$.
This last restriction typically fails in many important classes of functions of interest as explained in \cite{MR3757527}.
Section~\ref{Sec:Saumard} presents two analyses of minmax MOM estimators, proving the statistical optimality of these estimators in a toy example in small dimension ($d\leqslant \sqrt{N})$ where the uniform $L^2/L^1$ comparison fails.

\section{Setting}
Consider the supervized statistical learning framework where the data space $\cZ$ is a product $\cZ=\cX\times \cY$, with $\cX\subset \R^d$ and $\cY\subset \R$, so data $z\in \cZ$ are couples $z=(x,y)$ and the goal is to predict as best as possible an output $Y$ from an input $X$ when $Z=(X,Y)$ is drawn from an unknown distribution $P$. 
For any $f\in \R^d$ and $z\in \cZ$, let $\ell$ denote the square loss
\[
\ell_f(z)=(y-x^Tf)^2\enspace.
\]
Hereafter, we assume that $P$ satisfies $P[Y^2]<\infty$ and $P[\|X\|^2]<\infty$ and measure the risk of any $f\in \R^d$ by 
\[
P\ell_f=P[(Y-X^Tf)^2]\enspace.
\]
As usual, $f^*\in\argmin P\ell_f$ denotes an oracle.
Let $\Sigma=P[XX^T]$ and assume that 
\begin{equation}\label{Hyp:InvSigma}
 \Sigma\text{ is positive definite}\enspace.
\end{equation}
Let $F\subset \R^d$ denote a convex subset of $\R^d$.

This chapter studies both ERM and minmax MOM estimators.
As $f\mapsto \ell_f(x,y)$ is convex for any $z=(x,y)$, Lemma~\ref{lem:ConvLoss} applies and shows that the tests 
\[
T_{\text{emp}}(f,g)=P_N[\ell_f-\ell_g],\qquad T_{\text{mom}}(f,g)=\MOM{K}{\ell_f-\ell_g}\enspace,
\]
satisfy Assumption {\bf (HP)} of the homogeneity lemma, see Lemma~\ref{lem:DetArg}, provided that the evaluation function $\cE$ derives from a norm.
Hereafter in this section, for any $f$ and $g$ in $\R^d$, let $d(f,g)=P[\ell_f-\ell_g]$ which satisfies $d(f,g)=-d(g,f)$ and $d(f^*,f)\leqslant 0$, so the functions $B$ and $\cB$ in Lemma~\ref{lem:DetArg} are equal.
For the evaluation function, for any $f\in F$, let
\begin{align*}
\cE(f)&=\|f-f^*\|_{L^2(P)}:=\sqrt{P[(X^T(f-f^*))^2]}\\
&=\sqrt{(f-f^*)^T\Sigma(f-f^*)}=\|f-f^*\|_{\Sigma}\enspace, 
\end{align*}
where, for any $d\times d$ matrix $\cM$ and any $\ba\in \R^d$, $\|\ba\|_{\cM}=\sqrt{\ba^T\cM\ba}$.
Developing the square $((y-x^Tg)-x^T(f-g))^2$ shows the so called quadratic/multiplier decomposition of the square loss:
\begin{equation}\label{eq:QuadMult}
 \ell_f(x,y)-\ell_g(x,y)=[x^T(f-g)]^2-2x^T(f-g) (y-x^Tg)\enspace.
\end{equation}
Let $\xi=Y-X^Tf^*$, the process
\[
f\mapsto \xi X^T(f-f^*)\enspace,
\]
is called the \emph{multiplier} process and 
\[
f\mapsto (X^T(f-f^*))^2
\]
the \emph{quadratic} process.
Let $t\in (0,1)$ and $f\in F$. 
By definition of $f^*$, as $(1-t)f^*+tf\in F$,
\begin{align*}
P[(Y-X^Tf^*)^2]&\leqslant P[(Y-X^T((1-t)f^*+tf))^2]\\
&=P[(\xi-tX^T(f-f^*))^2]\\
&=P[\xi^2]-2tP[\xi X^T(f-f^*)]+t^2P[(X^T(f-f^*))^2]\enspace.
\end{align*}
It follows that, for any $t\in (0,1)$,
\[
P[\xi X^T(f-f^*)]\leqslant \frac{t}2P[(X^T(f-f^*))^2]\enspace.
\]
Letting $t\to 0$ shows that
\begin{equation}\label{eq:Conv}
 P[\xi X^T(f-f^*)]\leqslant 0\enspace.
\end{equation}
Together with \eqref{eq:QuadMult}, this implies in particular that 
\[
P[\ell_f-\ell_{f^*}]=P[(X^T(f-f^*))^2]-2P[\xi X^T(f-f^*)]\geqslant \|f-f^*\|_{\Sigma}^2\enspace.
\]
In particular, the following ``global" Bernstein condition is satisfied for least-squares regression:
\begin{equation}\label{Ass:GlobBern}
\forall f\in F,\qquad d(f,f^*)=P[\ell_f-\ell_{f^*}]\geqslant \cE(f)^2\enspace.
\end{equation}
\section{ERM in the Gaussian case}\label{sec:ERMGaussReg}
To establish the Benchmark, let us first consider the ERM estimator when $Z=(X,Y)$ is a Gaussian vector.
Let 
\[
\overline{\Sigma}=P[(X-P[X])(X-P[X])^T], 
\quad \sigma^2=P[(Y-X^Tf^*)^2]\enspace.
\]
\begin{theorem}
Assume that $F=\R^d$ and $Z=(X,Y)$ is a Gaussian vector such that the covariance matrix of $X$ in $\R^{d\times d}$ is non-degenerate.
Assume moreover that $64d\leqslant \gamma N$ for $\gamma=\sqrt{2/\pi e}$.
 Then, for any $s>0$ such that
 \[
8\sqrt{d}+2\sqrt{2s}\leqslant \sqrt{\gamma N}\enspace,
\]
the empirical risk minimizer $\hat{f}\in \argmin_{f\in F}P_N[(Y-f^TX)^2]$ satisfies
\[
\P\bigg(\cE(\hat{f})\leqslant \frac{24\sigma}{\sqrt{\gamma^3N}}\big(3\sqrt{d}+4\sqrt{s}\big)\bigg)\geqslant 1-5e^{-s}\enspace.
\]
\end{theorem}
\begin{proof}
 The key is to compute the function $B$ in the homogeneity lemma. 
 Start with algebraic computations.
 Let $r>0$ and $f\in F$ such that $\cE(f)=\|f-f^*\|_\Sigma\leqslant r$.
 Then $f=f^*+rg$ with $g=(f-f^*)/r$ satisfying $\|g\|_{\Sigma}\leqslant 1$.
 Then
 \begin{align}
\notag&(P_N-P)\big[2\xi X^T(f-f^*) -[X^T(f-f^*)]^2\big]\\
\label{dec:HomQM}&\qquad=2r(P_N-P)\big[\xi X^Tg)\big]-r^2(P_N-P)\big[(X^Tg)^2\big]
\enspace.
\end{align}
Let $\bB=\{f\in\R^d:\|f\|_{\Sigma}\leqslant 1\}$,
\begin{gather*}
M= \sup_{f\in \bB}(P_N-P)\big[\xi X^Tf \big]\\
Q=\inf_{f\in \bB}(P_N-P)\big[(X^Tf)^2\big]\enspace.
\end{gather*}
With these notations, from \eqref{dec:HomQM},
\begin{align}\label{eq:IntRegGauss1}
 \sup_{f\in F:\cE(f)\leqslant r}(P_N-P)\big[\ell_f-\ell_{f^*}\big]\leqslant 2rM-r^2Q\enspace.
\end{align}
\begin{lemma}\label{lem:MultGauss}
 For any $s\in[0,N]$, with probability $1-4e^{-s}$,
 \[
 \P\bigg(M\leqslant\frac{\sigma}{\sqrt{N}}\big(3\sqrt{d}+4\sqrt{s}\big)\bigg)\geqslant 1-4e^{-s}\enspace.
 \]
\end{lemma}
\begin{proof}
Let $f\in\bB$. 
As $Z=(X,Y)$ is a Gaussian vector and $F=\R^d$, $X^Tf^*$ is the projection of $Y$ onto the linear span of $X$ in $L^2$.
Therefore, $X^Tf$ is, conditionally on $\xi$, a Gaussian random variable, with mean $P[X^Tf]$ and variance $P[((X-P[X])^Tf)^2]=f^T\overline{\Sigma}f=\|f\|_{\overline{\Sigma}}^2\leqslant 1$.
Let $\cF_N$ denote the $\sigma$-algebra spanned by $\xi_1,\ldots,\xi_N$.
Conditionally on $\cF_N$, the random variables $X_f=P_N[\xi X^Tf]$ are Gaussian random variables centered at $P_N[\xi]P[X^Tf]$ with variance 
\begin{equation}\label{eq:VarMultGauss}
 \sigma_f^2=\frac{ P_N[\xi^2]}N f^T\overline{\Sigma} f\leqslant V\enspace,
\end{equation}
where $V=N^{-1}P_N[\xi^2]$.
 By concentration of suprema of Gaussian processes, for any $s>0$,
 \[
\P\bigg(\sup_{f\in \bB}(X_f-\E[X_f|\cF_N])\leqslant\E[\sup_{f\in F}(X_f-\E[X_f|\cF_N])|\cF_N]+\sqrt{2Vs}|\cF_N\bigg)\leqslant 1-e^{-s}\enspace.
 \]
Now, 
\begin{align*}
 \E[\sup_{f\in \bB}(X_f-\E[X_f|\cF_N])|\cF_N]&=\E[\sup_{f\in \bB}P_N[\xi (X-P[X])^Tf ]|\cF_N]
 \end{align*}
 Now, as $\Sigma$ is non degenerate, $\|\cdot\|_{\Sigma}$ is a norm whose dual norm is $\|\cdot\|_{\Sigma^{-1}}$. 
 Hence, 
\begin{align*}
 \E[\sup_{f\in \bB}(X_f-\E[X_f|\cF_N])|\cF_N]
 &=\E\bigg[\bigg\|\frac1N\sum_{i=1}^N\xi_i(X_i-P[X])\bigg\|_{\Sigma^{-1}}|\cF_N\bigg]\\
\text{by Cauchy-Schwarz } &\leqslant \sqrt{\E\bigg[\bigg\|\frac1N\sum_{i=1}^N\xi_i(X_i-P[X])\bigg\|_{\Sigma^{-1}}^2|\cF_N\bigg]}\enspace.
\end{align*}
Now, developing the square-norm and using the independence between $\xi$ and $X$,
\begin{align*}
&\E\bigg[\bigg\|\frac1N\sum_{i=1}^N\xi_i(X_i-P[X])\bigg\|_{\Sigma^{-1}}^2|\cF_N\bigg]\\
&=\frac1{N^2}\sum_{1\leqslant i,j\leqslant N}\xi_i(Y_j-X_j^Tf^*)\E[(X_i-P[X])^T\Sigma^{-1}(X_j-P[X])]\\
&=\frac1{N^2}\sum_{i=1}^N\xi_i^2\E[(X_i-P[X])^T\Sigma^{-1}(X_i-P[X])]\\
&\leqslant\frac{P_N[\xi^2]}{N}P[X^T\Sigma^{-1}X]\enspace.
\end{align*}
Finally,
\begin{equation}\label{eq:PXSigmad}
P[X^T\Sigma^{-1}X]=P[\text{Tr}(X^T\Sigma^{-1}X)]=P[\text{Tr}(\Sigma^{-1}XX^T)]=\text{Tr}(\bI_d)=d\enspace. 
\end{equation}
Therefore,
\[
\E\bigg[\bigg\|\frac1N\sum_{i=1}^N\xi_i(X_i-P[X])\bigg\|_{\Sigma^{-1}}^2|\cF_N\bigg]\leqslant \frac{P_N[\xi^2]}{N}d\enspace,
\]
so 
\[
 \E[\sup_{f\in \bB}(X_f-\E[X_f|\cF_N])|\cF_N]\leqslant \sqrt{Vd}\enspace.
\]
%
Overall, with probability at least $1-e^{-s}$, 
\begin{equation}\label{eq:BoundMultGauss1}
M=\sup_{f\in \bB}X_f\leqslant |P_N[\xi]|+\sqrt{V}\big(\sqrt{d}+\sqrt{2s}\big)\enspace.
\end{equation}
Now, with probability $1-2e^{-s}$, the centered Gaussian random variable $P_N[\xi]$ satisfies
\begin{equation}\label{eq:BoundT1MultGauss}
|P_N[\xi]|\leqslant \sigma\sqrt{\frac{2s}N}\enspace.
\end{equation}
Moreover, $\xi\sim\gauss(0,\sigma^2)$, so $\E[\xi^{2k}]=(2k)!\sigma^{2k}/2^kk!$ and, for any $u<1/2\sigma^2$, 
\begin{align*}
\E[e^{u\xi^{2}}]&=1+u\sigma^2+\sum_{k\geqslant 2}\frac{u^k(2k)!\sigma^{2k}}{2^k(k!)^2}\\
&\leqslant 1+u\sigma^2+\sum_{k\geqslant 2}(2u\sigma^2)^k\\
&=1+u\sigma^2+\frac{4u^2\sigma^4}{1-2u\sigma^2}\enspace.
\end{align*}
Hence, for any $u<N/2\sigma^2$,
\[
\log\E\big[e^{u(NV-\sigma^2)}\big]\leqslant N\log\bigg(1+\frac{4(u/N)^2\sigma^4}{1-2(u/N)\sigma^2}\bigg)\leqslant \frac{u^28\sigma^4/N }{2(1-u2\sigma^2/N)}\enspace.
\]
It follows therefore from Bernstein's inequality that, for any $s>0$,
\[
\P\bigg(NV-\sigma^2>2\sigma^2\bigg(2\sqrt{\frac{s}N}+\frac{s}N\bigg)\bigg)\leqslant e^{-s}\enspace.
\]
Plugging this bound and \eqref{eq:BoundT1MultGauss} into \eqref{eq:BoundMultGauss1} shows that, for any $s\leqslant N$, with probability at least $1-4e^{-s}$,
\begin{align*}
M&\leqslant \sigma\sqrt{\frac{2s}N}+\bigg(\sqrt{\frac{d}N}+\sqrt{\frac{2s}N}\bigg)\sigma\sqrt{1+4\sqrt{\frac sN}+\frac{2s}N}
\leqslant \sigma\bigg(3\sqrt{\frac{d}N}+4\sqrt{\frac{s}N}\bigg)\enspace. 
\end{align*}
\end{proof}
Let us now bound the quadratic process.
\begin{lemma}\label{lem:QuadGauss}
 \[
 \forall s>0,\qquad \P\bigg(Q<\frac{\gamma}2-1-\bigg(4\sqrt{\frac{d}N}+\sqrt{\frac{2s}N}\bigg)^2\bigg)\leqslant e^{-s}\enspace.
 \]
\end{lemma}
\begin{proof}
Elementary calculus shows that, for any real valued Gaussian random variable $Z$,
 \[
 \E[|Z|]\geqslant \gamma\sqrt{\E[Z^2]}\enspace,
 \]
where $\gamma=\sqrt{2/\pi e}$.
 From this remark follows that the class of linear functions satisfies the \emph{small ball assumption} of Mendelson: as $X^Tf$ is Gaussian for any $f\in\R^d$,
\begin{equation}\label{eq:SBA}
 \forall f\in F,\qquad P[|X^Tf|]\geqslant \gamma\sqrt{P[(X^Tf)^2]}=\gamma\|f\|_{\Sigma}\enspace.
\end{equation}
By Jensen's inequality,
\[
P_N\big[(X^Tf)^2\big]\geqslant \big(P_N[|X^Tf|]\big)^2\enspace.
\]
Let $f\in\bB$, 
\[
\text{Var}(X^Tf)=f^T\overline{\Sigma}f\leqslant 1\enspace.
\]
Now, by Borel's concentration inequality, with probability at least $1-e^{-s}$,
\begin{align*}
\sup_{f\in \bB}\big|(P_N-P)|X^Tf|\big|&\leqslant \E\bigg[\sup_{f\in \bB}\big|(P_N-P)|X^Tf|\big|\bigg]+\sqrt{\frac{2s}N}
\end{align*}
By symmetrization and contraction,
\begin{align*}
\sup_{f\in \bB}\big|(P_N-P)|X^Tf|\big|&\leqslant 4\E\bigg[\sup_{f\in \bB}\big|\frac1N\sum_{i=1}^N\epsilon_iX_i^Tf\big|\bigg]+\sqrt{\frac{2s}N}\\
&=4\E\bigg[\sup_{f\in \bB}\big|f^T\bigg(\frac1N\sum_{i=1}^N\epsilon_iX_i\bigg)\big|\bigg]+\sqrt{\frac{2s}N}\enspace.
\end{align*}
Using that $\|\cdot\|_{\Sigma}$ is a norm with dual norm $\|\cdot\|_{\Sigma^{-1}}$,
\begin{align*}
\sup_{f\in \bB}\big|(P_N-P)|f^TX|\big|&=4\E\bigg[\bigg\|\frac1N\sum_{i=1}^N\epsilon_iX_i\bigg\|_{\Sigma^{-1}}\bigg]+\sqrt{\frac{2s}N}\\
\text{by Cauchy-Schwarz }&\leqslant 4\sqrt{\E\bigg[\bigg\|\frac1N\sum_{i=1}^N\epsilon_iX_i\bigg\|_{\Sigma^{-1}}^2\bigg]}+\sqrt{\frac{2s}N}
\enspace.
\end{align*}
Developing the square, using independence between the $\epsilon_i$ and $X_i$ and that $\E[\epsilon_i]=0$,
\begin{align*}
\sup_{f\in \bB}\big|(P_N-P)|X^Tf|\big|&=4\sqrt{\frac{1}{N^2}\sum_{1\leqslant i,j\leqslant N}\E\bigg[\epsilon_i\epsilon_jX_i^T\Sigma^{-1}X_j\bigg]}+\sqrt{\frac{2s}N}\\
&=4\sqrt{\frac{1}{N}P\big[X^T\Sigma^{-1}X\big]}+\sqrt{\frac{2s}N}\enspace.
\end{align*}
By \eqref{eq:PXSigmad}, it follows that $\P(\Omega_s)\geqslant 1-e^{-s}$, where
\begin{align*}
\Omega_s=\bigg\{\sup_{f\in \bB}\big|(P_N-P)|X^Tf|\big|&\leqslant4\sqrt{\frac{d}{N}}+\sqrt{\frac{2s}N}\bigg\}\enspace.
\end{align*}
On $\Omega_s$, for any $f\in F$,
\begin{align*}
P_{N}\big[(X^Tf)^2\big]&\geqslant \bigg(P[|X^Tf|]-4\sqrt{\frac{d}N}-\sqrt{\frac{2s}N}\bigg)^2\\
&\geqslant \frac12(P[|X^Tf|])^2-\bigg(4\sqrt{\frac{d}N}+\sqrt{\frac{2s}N}\bigg)^2\\
&\geqslant \frac{\gamma}2\|f\|_\Sigma^2-\bigg(4\sqrt{\frac{d}N}+\sqrt{\frac{2s}N}\bigg)^2\enspace.
\end{align*}
It follows that 
\[
(P_N-P)\big[(X^Tf)^2\big]\geqslant \bigg(\frac{\gamma}2-1\bigg)\|f\|_\Sigma^2-\bigg(4\sqrt{\frac{d}N}+\sqrt{\frac{2s}N}\bigg)^2\enspace.
\]
Therefore, as $\gamma/2-1<0$, 
\[
\forall s>0,\qquad \P\bigg(Q\geqslant \frac{\gamma}2-1-\bigg(4\sqrt{\frac{d}N}+\sqrt{\frac{2s}N}\bigg)^2\bigg)\geqslant 1-e^{-s}\enspace.
\]
\end{proof}

Let 
\begin{gather*}
 m_s=\frac{\sigma}{\sqrt{N}}\big(3\sqrt{d}+4\sqrt{s}\big),\qquad q_s=\frac{\gamma}2-\bigg(4\sqrt{\frac{d}N}+\sqrt{\frac{2s}N}\bigg)^2\enspace.
\end{gather*}
It follows from Lemmas~\ref{lem:MultGauss} and \ref{lem:QuadGauss} that the event $\Omega=\{M\leqslant m_s\}\cap\{Q\geqslant q_s-1\}$ has probability larger than $1-5e^{-s}$.
Moreover, from \eqref{eq:IntRegGauss1}, $\Omega$ contains $\cap_{r>0}\Omega_r$, where 
\[
\Omega_r=\bigg\{ \sup_{f\in F:\cE(f)\leqslant r}(P_N-P)\big[\ell_f-\ell_{f^*}\big]\leqslant B(r)\bigg\}\enspace, 
 \]
 with $B(r)=2rm_s-r^2(q_s-1)$.
With this choice of function $B$, by \eqref{Ass:GlobBern}, it follows that \eqref{def:r1} holds if $q_s>0$ and
 \[
 2r_1m_s-r^2_1q_s\leqslant 0,\quad\text{i.e.}\quad r_1\geqslant \frac{2m_s}{q_s}\enspace.
 \]
 Let 
 \[
 r_1= \frac{2m_s}{q_s}\quad \text{so}\quad B(r_1)=\frac{4m_s^2}{q_s}-\frac{4m_s^2(q_s-1)}{q_s^2}=\frac{4m_s^2}{q_s^2}\enspace.
 \]
 Then, \eqref{def:r2} holds if $q_s>0$ and
 \[
 \frac{4m_s^2}{q_s^2}+2r_2m_s-r^2_2q_s=\frac{4m_s^2}{q_s^2}+\frac{m^2_s}{q_s}-q_s(r_2-\frac{m_s}{q_s})^2\leqslant 0\enspace,
 \]
 that is if 
 \[
 r_2=\frac{2m_s}{q_s}\bigg(1+\frac{1}{\sqrt{q_s}}\bigg)\enspace.
 \]
 As
 \[
\bigg(4\sqrt{\frac{d}N}+\sqrt{\frac{2s}N}\bigg)^2\leqslant \frac{\gamma}4,\qquad q_s\geqslant \frac{\gamma}4\enspace.
 \]
 Therefore, 
 \[
 r_2\leqslant \frac{24}{\gamma^{3/2}}m_s=\frac{24\sigma}{\sqrt{\gamma^3N}}\big(3\sqrt{d}+4\sqrt{s}\big)\enspace.
 \]
 The proof is concluded by Lemma~\ref{lem:DetArg}.
\end{proof}

\section{Minmax MOM estimators}\label{sec:UnconstrainedMOMReg}
In the previous section, we used several features of Gaussian distributions to prove the deviation bound in least-squares regression. 
The first of these properties is that, since $(X,Y)$ was Gaussian, the vector $f^*\in\argmin_{f\in F}P[\|Y-X^Tf\|^2]$ satisfies $X^Tf^*=\E[Y|X]$ and one can write  
\[
Y=X^Tf^*+\xi\enspace,
\]
where $\xi/\sigma$ is a standard Gaussian \emph{independent from} $X$.
It follows that 
\[
\forall f\in \R^d,\qquad \text{Var}\big((X^Tf)\xi\big)=\sigma^2\text{Var}\big(X^Tf\big)=\sigma^2\|f\|_{\Sigma}^2\enspace.
\]
Therefore
\begin{equation}\label{eq:cdtsigma1}
 \sigma^2\geqslant \sup_{f\in F:\|f\|_{\Sigma}=1}P\big(\xi^2(X^Tf)^2\big)\enspace.
\end{equation}
The independence between $\xi$ and $X$ also allows to show that
\[
P[\xi^2\sup_{f\in F:\|f\|_{\Sigma}\leqslant 1}(X^Tf)^2]=P[\xi^2\|X\|_{\Sigma^{-1}}^2]=\sigma^2P[\|X\|_{\Sigma^{-1}}^2]=\sigma^2d\enspace.
\]
The last inequality comes from \eqref{eq:PXSigmad}.
Hence,
\begin{equation}\label{eq:cdtsigma2}
 \sigma^2\geqslant \frac{P[\xi^2\|X\|_{\Sigma^{-1}}^2]}d\enspace.
\end{equation}
It turns out that independence between the noise $\xi$ and the inputs $X$ can be removed provided that $\sigma=P[\xi^2]$ is replaced by the adequate quantity in conditions \eqref{eq:cdtsigma1} and \eqref{eq:cdtsigma2}.
Hereafter, denote by $\olsigma$ a positive real number such that
\begin{equation}\label{def:L4L2}
 \olsigma^2\geqslant\sup_{f\in F:\|f\|_{\Sigma}=1}P\big(\xi^2(X^Tf)^2\big)\vee\frac{P[\xi^2\|X\|_{\Sigma^{-1}}^2]}d\enspace.
\end{equation}
The parameter $\olsigma$ does not appear in the construction of the minmax estimator and may therefore be unknown from the statistician.
Notice that, as $\xi$ and $X$ may not be independent, it is implicitly assumed that $\olsigma<+\infty$ in the following.

\subsection{The small ball hypothesis}
The second property of Gaussian distributions was the small ball property \cite{MR3431642,Shahar-COLT}, see Eq~\eqref{eq:SBA} in the previous proof.
To extend the Gaussian case, we will \emph{assume} that this property holds for the distribution of the vector $X$.
Formally, there exists an absolute constant $\gamma>0$ such that
\begin{equation}\label{eq:SBAg}
 \forall f\in F,\qquad P[|X^Tf|]\geqslant \gamma \sqrt{P[(X^Tf)^2]}=\|f\|_{\Sigma}\enspace.
\end{equation}
This assumption is checked in the following example.
\begin{lemma}
 Assume that the random vector $X$ has coordinates $(X^{(i)})_{1\leqslant i\leqslant N}$ satisfying the following property.
 There exist constants $C_1,C_2>0$ such that, $\forall 1\leqslant i,j\leqslant N$,
\begin{gather}
\label{eq:checkL2L1i}\E[X^{(i)}X^{(j)}]\leqslant C_1\E[|X^{(i)}|]\E[|X^{(j)}|]\enspace,\\
\label{eq:checkL2L1ii} \sum_{i=1}^N|f_i|\E[|X^{(i)}|]\leqslant C_2P[|X^Tf|]\enspace. 
\end{gather}
Then, \eqref{eq:SBAg} holds with $\gamma=1/\sqrt{C_1}C_2$.
\end{lemma}
\begin{proof}
 Let $f\in\R^d$.
\begin{align*}
\|f\|_{\Sigma}^2&=\sum_{1\leqslant i,j\leqslant N}f_if_j \E[X^{(i)}X^{(j)}]\\
&\leqslant C_1\bigg(\sum_{i=1}^N|f_i|\E[|X^{(i)}|]\bigg)^2 \text{ by \eqref{eq:checkL2L1i}}\\
&\leqslant C_1C_2^2(P[|X^Tf|])^2\text{ by \eqref{eq:checkL2L1ii}}\enspace.
\end{align*}
Therefore, \eqref{eq:SBAg} holds with $\gamma=1/\sqrt{C_1}C_2$.
\end{proof}

\noindent
Another example where one can check the small ball property is the following.
\begin{lemma}
Assume that the $L^4/L^2$ comparison holds.
\begin{equation}\label{eq:L4L2}
 \exists C>0 : \forall f\in F,\qquad P[(X^Tf)^4]\leqslant C P[(X^Tf)^2]^2\enspace.
\end{equation}
Then, \eqref{eq:SBAg} holds with $\gamma=\sqrt{2}/8C$.
\end{lemma}
\begin{remark}
Assumption is discussed in Section~\ref{SecExBernstein} of the previous chapter.
It was used there to check the Bernstein assumption.
\end{remark}
\begin{proof}
The proof relies on the following simple Paley-Zigmund argument. 
Let $f\in F$,
\begin{align*}
 P[(X^Tf)^2]&=P[(X^Tf)^2{\bf 1}_{|X^Tf|\leqslant 4CP[|X^Tf|]}]+P[(X^Tf)^2{\bf 1}_{|X^Tf|> 4CP[|X^Tf|]}]\\
 \text{ by Cauchy-Schwarz }&\leqslant 16C^2P[|X^Tf|]^2+\sqrt{P[(X^Tf)^4]\P\big(|X^Tf|> 4CP[|X^Tf|]\big)}\\
 \text{ by Markov }&\leqslant 16C^2P[|X^Tf|]^2+\sqrt{\frac{P[(X^Tf)^4]}{4C}}\\
 \text{by~\eqref{eq:L4L2}}&\leqslant 16C^2P[|X^Tf|]^2+\frac{P[(X^Tf)^2]}2\enspace.
\end{align*}
It follows that 
\[
 P[(f^TX)^2]\leqslant 32C^2P[|f^TX|]^2\enspace.
\]
In other words, \eqref{eq:SBAg} holds with $\gamma=\sqrt{2}/8C$.
\end{proof}

 \subsection{Main results}
Recall that we observe $(X_1,Y_1),\ldots,(X_N,Y_N)$ i.i.d. copies of $(X,Y)$, a random vector taking values in $\R^d\times \R$, that $\Sigma=P[XX^T]$, $\E[Y^2]<\infty$, $f^*\in \argmin_{f\in F}P\ell_f$, where $\ell_f(x,y)=(y-f^Tx)^2$ and $\xi=Y-X^Tf^*$.
\begin{theorem}\label{thm:MinMaxMOMSB}
 Let $\olsigma$ be defined in \eqref{def:L4L2} and assume that \eqref{eq:SBAg} holds. There exists an absolute constant $C$ such that, if $Cd\leqslant \gamma^2N$, then, for any $K$ such that $CK\leqslant \gamma^2N$, the minmax MOM estimator
 \[
 \ERM_K\in\argmin_{f\in F}\sup_{g\in F}\MOM{K}{\ell_f-\ell_g}
 \]
 satisfies
 \[
 \P\bigg(\|\ERM_K-f^*\|_\Sigma\leqslant \frac{C\olsigma}{\gamma^3\sqrt{N}}\big(\sqrt{d}\vee\sqrt{K}\big)\bigg)\geqslant 1-4e^{-K/C}\enspace.
 \]
\end{theorem}

\begin{proof}
 The key is to compute the function $B$ in the homogeneity lemma.
 Let $r>0$ and $F_r=\{f\in F:\|f-f^*\|_{\Sigma}\leqslant r\}$.
 By the quadratic/multiplier decomposition of the quadratic loss \eqref{eq:QuadMult}, one wants to bound from above
 \[
 \cM_r=\sup_{f\in F_r}\MOM{K}{2\xi X^T(f-f^*)-(X^T(f-f^*))^2+P[(X^T(f-f^*))^2]}\enspace.
 \]
 As $F_r=\{f=f^*+ru,\ u\in\bB\}$, with $\bB=\{u\in \R^d:\|u\|_{\Sigma}\leqslant 1\}$, it holds
\begin{align*}
 \cM_r
=\sup_{u\in\bB}\MOM{K}{2r\xi[X^Tu]+r^2(\|u\|_{\Sigma}^2-(X^Tu)^2)}\enspace.
\end{align*}
To bound $\cM_r$, the following lemmas will prove useful.
\begin{lemma}\label{lem:MultMOM}
 There exists an absolute constant $c^*$ such that, with probability larger than $1-e^{-K/c^*}$, there exist at least $9K/10$ blocks $B_k$ where 
\[
\sup_{u\in \bB}(P_{B_k}-P)[\xi X^Tu]=\sup_{u\in \bB}P_{B_k}[\xi X^Tu]\leqslant \frac{c^*\olsigma}{\sqrt{N}}\bigg(\sqrt{d}\vee\sqrt{K}\bigg)\enspace.
\]
\end{lemma}
\begin{proof}
 Consider the set of functions $\cF_M=\{(x,y)\mapsto (y-x^Tf^*)(x^Tu),\ u\in\bB\}$. 
By definition of $\olsigma$, see \eqref{def:L4L2},
 \[
 \forall u\in\bB,\qquad \sigma^2_u=P[\xi^2(X^Tu)^2]\leqslant \olsigma^2\enspace.
 \]
Hence, $\sigma^2(\cF_M)=\sup_{f\in \cF_M}\text{Var}(f(Z))\leqslant \olsigma^2$.
Moreover, as $\|\cdot\|_{\Sigma^{-1}}$ is the dual norm of $\|\cdot\|_{\Sigma}$, one can bound the Rademacher complexity of $\cF_M$ as follows.
\begin{align*}
D(\cF_M)&=\bigg(\E\bigg[\sup_{u\in\bB}\frac1{\sqrt{N}}\sum_{i=1}^N\epsilon_i\xi_i(X_i^Tu)\bigg]\bigg)^2\\
&=\bigg(\E\bigg[\sup_{u\in\bB}u^T\bigg(\frac1{\sqrt{N}}\sum_{i=1}^N\epsilon_i\xi_iX_i\bigg)\bigg]\bigg)^2\\
&=\bigg(\E\bigg[\bigg\|\frac1{\sqrt{N}}\sum_{i=1}^N\epsilon_i\xi_iX_i\bigg\|_{\Sigma^{-1}}\bigg]\bigg)^2\enspace.
\end{align*}
 By Cauchy-Schwarz,
\[
D(\cF_M)\leqslant \E\bigg[\bigg\|\frac1{\sqrt{N}}\sum_{i=1}^N\epsilon_i\xi_iX_i\bigg\|_{\Sigma^{-1}}^2\bigg]\enspace.
\] 
Developing the square norm and using independence and centering,
\begin{equation}\label{eq:RadCompMultRd}
D(\cF_M)\leqslant \E\big[\big\|\xi X\big\|_{\Sigma^{-1}}^2\big]=P\big[\xi^2\big\|X\big\|_{\Sigma^{-1}}^2\big]\leqslant \olsigma^2d\enspace.
\end{equation}
The last inequality uses the definition of $\olsigma$, see \eqref{def:L4L2}.
By the general concentration result for quantile of means processes, see Theorem~\ref{thm:ConcSupMOMg}, there exists an absolute constant $c^*$ such that, with probability larger than $1-e^{-K/c^*}$, there exist at least $9K/10$ blocks $B_k$ where 
\[
\sup_{u\in \bB}(P_{B_k}-P)[\xi X^Tu]\leqslant \frac{c^*\olsigma}{\sqrt{N}}\bigg(\sqrt{d}\vee\sqrt{K}\bigg)\enspace.
\]
\end{proof}
\begin{lemma}\label{lem:QuadMOM}
 There exists an absolute constant $c^*$ such that, with probability larger than $1-e^{-K/c^*}$, there exist at least $9K/10$ blocks $B_k$ where, for any $u\in\bB$,
\[
P_{B_k}[(X^Tu)^2]\geqslant \bigg(\gamma\|u\|_{\Sigma}-\frac{c^*\olsigma}{\sqrt{N}}\big(\sqrt{d}\vee\sqrt{K}\big)\bigg)_+^2\enspace.
\]\end{lemma}
\begin{proof}
 Consider $\cF_Q=\{x\mapsto |x^Tu|,\  u\in\bB\}$. 
 \[
 \forall u\in\bB,\qquad \sigma^2_u=P[(X^Tu)^2]\leqslant 1\enspace.
 \]
 Hence, $\sigma^2(\cF_Q)=\sup_{f\in \cF_Q}\text{Var}(f(X))\leqslant 1$.
 Moreover, by the contraction principle, the Rademacher complexity of $\cF_Q$ can be upper bounded as follows.
\begin{align*}
 D(\cF_Q)&=\bigg(\E\bigg[\sup_{u\in\bB}\frac1{\sqrt{N}}\sum_{i=1}^N\epsilon_i|X_i^Tu|\bigg]\bigg)^2\\
 &\leqslant 4\bigg(\E\bigg[\sup_{u\in\bB}\frac1{\sqrt{N}}\sum_{i=1}^N\epsilon_iX_i^Tu\bigg]\bigg)^2\enspace.
\end{align*}
 As $\|\cdot\|_{\Sigma^{-1}}$ is the dual norm of $\|\cdot\|_{\Sigma}$,
 \[
 D(\cF_Q)\leqslant 4\bigg(\E\bigg[\bigg\|\frac1{\sqrt{N}}\sum_{i=1}^N\epsilon_iX_i\bigg\|_{\Sigma^{-1}}\bigg]\bigg)^2\leqslant 4\E\bigg[\bigg\|\frac1{\sqrt{N}}\sum_{i=1}^N\epsilon_iX_i\bigg\|_{\Sigma^{-1}}^2\bigg]\enspace.
 \]
 Developing the square and using independence yields
\begin{equation}\label{eq:RadCompQuadRd}
 D(\cF_Q)\leqslant 4P[\|X\|_{\Sigma^{-1}}^2]=4d\enspace.
\end{equation}
  The last equality comes from \eqref{eq:PXSigmad}.
By Theorem~\ref{thm:ConcSupMOMg}, there exists an absolute constant $c^*$ such that, with probability larger than $1-e^{-K/c^*}$, at least $9K/10$ blocks $B_k$ satisfy
\[
\sup_{u\in \bB}|(P_{B_k}-P)[|X^Tu|]|\leqslant \frac{c^*}{\sqrt{N}}\big(\sqrt{d}\vee\sqrt{K}\big)\enspace.
\]
Moreover, $P[|u^TX|]\geqslant \gamma\|u\|_{\Sigma}$, therefore, with probability at least $1-e^{-K/c^*}$, for any $u\in \bB$, there exist at least $9K/10$ blocks $B_k$ where 
\[
P_{B_k}[(X^Tu)^2]\geqslant (P_{B_k}[|X^Tu|])^2\geqslant \bigg(\gamma\|u\|_{\Sigma}-\frac{c^*}{\sqrt{N}}\big(\sqrt{d}\vee\sqrt{K}\big)\bigg)_+^2\enspace.
\]
\end{proof}
Denote by $c^*$ the largest of the absolute constants appearing in Lemmas~\ref{lem:MultMOM} and~\ref{lem:QuadMOM}.
Define $\Omega$ as the event where, simultaneously, there exist $9K/10$ blocks $B_k$ where 
\[
\sup_{u\in \bB}P_{B_k}[\xi X^Tu]\leqslant \frac{c^*\olsigma}{\sqrt{N}}\bigg(\sqrt{d}\vee\sqrt{K}\bigg)=:m_K\enspace,
\]
and $9K/10$ blocks $B_k$ where, for any $u\in \bB$, 
\[
P_{B_k}[(X^Tu)^2]\geqslant \bigg(\gamma\|u\|_{\Sigma}-\frac{c^*}{\sqrt{N}}\big(\sqrt{d}\vee\sqrt{K}\big)\bigg)_+^2\enspace.
\]
By Lemmas~\ref{lem:MultMOM} and~\ref{lem:QuadMOM}, $\P(\Omega)\geqslant 1-2e^{-K/c^*}$.
On $\Omega$, there exist at least $9K/10$ blocks where, for any $u\in \bB$,
\[
\|u\|_{\Sigma}^2-P_{B_k}[(X^Tu)^2]\leqslant \|u\|_{\Sigma}^2-\bigg(\gamma\|u\|_{\Sigma}-\frac{c^*}{\sqrt{N}}\big(\sqrt{d}\vee\sqrt{K}\big)\bigg)_+^2\enspace.
\]
Assume that 
\[
\frac{c^*}{\sqrt{N}}\big(\sqrt{d}\vee\sqrt{K}\big)\leqslant \frac{\gamma}2\enspace,
\]
As the functions $u\mapsto u^2-(\alpha u-\beta)^2_+$, for $\alpha<1$ are non-decreasing on $[0,1]$, it follows that, on $\Omega$, there exist at least $9K/10$ blocks where, for any $u\in \bB$,
\begin{align*}
 \|u\|_{\Sigma}^2-P_{B_k}[(X^Tu)^2]&\leqslant \|u\|^2_{\Sigma}-\bigg(\gamma\|u\|_{\Sigma}-\frac{c^*}{\sqrt{N}}\big(\sqrt{d}\vee\sqrt{K}\big)\bigg)_+^2\\
 &\leqslant 1-\bigg(\gamma-\frac{c^*}{\sqrt{N}}\big(\sqrt{d}\vee\sqrt{K}\big)\bigg)_+^2\leqslant 1-\frac{\gamma^2}4\enspace.
\end{align*}
It follows that, on $\Omega$, there exists at least $8K/10$ blocks where, simultaneously, for any $u\in\bB$,
\begin{gather*}
 P_{B_k}[\xi X^Tu]\leqslant m_K,\qquad  \|u\|_{\Sigma}^2-P_{B_k}[(X^Tu)^2]\leqslant 1-\frac{\gamma^2}4\enspace.
\end{gather*}
On these blocks,
\[
\forall r>0,\qquad P_{B_k}\big[2r\xi[X^Tu]+r^2(\|u\|_{\Sigma}^2-(X^Tu)^2)\big]\leqslant 2rm_K+(1-\gamma^2/4)r^2\enspace.
\]
As this relationship holds on more than $K/2$ blocks, it holds for the median, so $\Omega$ contains $\cap_{r>0}\Omega_r$, where 
\[
\forall r>0,\qquad \Omega_r=\{\cM_r\leqslant B(r)\},\qquad B(r)=2rm_K+\bigg(1-\frac{\gamma^2}4\bigg)r^2\enspace.
\]
With this choice of function $B$, by \eqref{Ass:GlobBern}, it follows that \eqref{def:r1} holds if
 \[
 2r_1m_K-r^2_1\frac{\gamma^2}4\leqslant 0,\quad\text{i.e.}\quad r_1\geqslant \frac{8m_K}{\gamma^2}\enspace.
 \]
 Let 
 \[
 r_1= \frac{8m_K}{\gamma^2}\quad \text{so}\quad B(r_1)=\frac{16m_K^2}{\gamma^2}+\bigg(1-\frac{\gamma^2}4\bigg)\frac{64m_K^2}{\gamma^4}=\frac{64m_K^2}{\gamma^4}\enspace.
 \]
 Then, \eqref{def:r2} holds if 
 \[
\frac{64m_K^2}{\gamma^4}+2r_2m_K-r^2_2\frac{\gamma^2}4=\frac{64m_K^2}{\gamma^4}+\frac{4m_K^2}{\gamma^2}-\frac{\gamma^2}4\bigg(r_2-\frac{4m_K}{\gamma^2}\bigg)^2\leqslant 0\enspace,
 \]
 that is if 
 \[
 r_2=\frac{8m_K}{\gamma^2}\bigg(1+\frac{1}{\gamma}\bigg)\leqslant \frac{16m_K}{\gamma^3}=\frac{16c^*\olsigma}{\gamma^3\sqrt{N}}\bigg(\sqrt{d}\vee\sqrt{K}\bigg)\enspace.
 \]
 The proof is concluded by Lemma~\ref{lem:DetArg}.
 \end{proof}
 

\section{Saumard's problem}\label{Sec:Saumard}
This section discusses the problem of least-squares regression in the case where the (elegant as it only involves $L^1$ and $L^2$ moments) Assumption~\ref{eq:SBAg} does not hold uniformly.
Let us first get convinced that this problem naturally arises in important examples.
Consider the following toy-model where the observations $(\tilde{X},Y)$ take values in $\cX\times \cY$ and denote by $I_1,\ldots,I_d$ a partition of $\cX$ such that, for any $i\in\{1,\ldots,d\}$, $P[I_i]=1/d$.
Let $\varphi_i(x)={\bf 1}_{x\in I_i}$, for any $x\in \cX$, $i\in \{1,\ldots,d\}$.
Let 
\begin{equation}\label{def:DesignHist}
 X=
\begin{bmatrix}
 \varphi_1(\tilde{X})\\
 \vdots\\
 \varphi_d(\tilde{X})
\end{bmatrix}\in\R^d\enspace.
\end{equation}
Let $\cD_N=(Z_1,\ldots,Z_N)$ denote a dataset of i.i.d. copies of $Z=(\tilde{X},Y)$ and, for any $i\in\{1,\ldots,N\}$, let 
\[
X_i=
\begin{bmatrix}
 \varphi_1(\tilde{X}_i)\\
 \vdots\\
 \varphi_d(\tilde{X}_i)
\end{bmatrix}\in\R^d\enspace.
\]
For any $f\in \R^d$, denoting by $\|f\|_p$ its $\ell_p$ norm,
\begin{gather*}
P[|X^Tf|]=P[|\sum_{i=1}^df_i\varphi_i|]=\sum_{i=1}^d|f_i|P\varphi_i= \frac{\|f\|_1}{d}\enspace,\\
P[(X^Tf)^2]=P[(\sum_{i=1}^df_i\varphi_i)^2]=P[\sum_{i=1}^df_i^2\varphi_i]= \frac{\|f\|_2^2}d\enspace. 
\end{gather*}
As $\|f\|_1^2\geqslant \|f\|_2^2$ (this bound is tight if $f$ is the first element of the canonical basis of $\R^d$), Assumption~\ref{eq:SBAg} holds with $\gamma=1/\sqrt{d}$ and $\delta=1$.
Therefore, Theorem~\ref{thm:MinMaxMOMSB} does not provide optimal rates of convergence in this example.
In \cite{MR3757527}, Saumard showed that this problem does not hold only on histogram or localized basis, but basically on any space generated by functions $\varphi_1,\ldots,\varphi_d$ with reasonable approximation propoerties.
The reason is that these spaces are naturally designed to be able to reproduce many functions, in particular ``spiky ones" for which the $L^2/L^1$ comparison does not hold uniformly.

\subsection{First least-squares analysis of histograms.}
The suboptimality in the rates provided in Theorem~\ref{thm:MinMaxMOMSB} comes from the analysis of the quadratic process. 
Improving these rates require modifications of Lemma~\ref{lem:QuadMOM} using properties of histogram spaces that will be the subject of this section.
Start with the following rough alternative.
The vector $X$ defined in \eqref{def:DesignHist} satisfies 
\[
\Sigma=P[XX^T]=\frac1d\bI_d\enspace.
\]
Therefore, for any $u\in \R^d$, $\|u\|_{\Sigma}=\|u\|/\sqrt{d}$.
Let 
\[
\bB=\{u\in\R^d:\|u\|_{\Sigma}\leqslant 1\}=\{u\in\R^d:\|u\|\leqslant \sqrt{d}\}\enspace.
\]
\begin{lemma}\label{lem:QuadMOMHist0}
 Consider the design vector $X$ defined in \eqref{def:DesignHist}. There exists an absolute constant $c^*$ such that, with probability larger than $1-e^{-K/c^*}$, there exist at least $9K/10$ blocks $B_k$ where, for any $u\in\bB$,
\[
P_{B_k}[(X^Tu)^2]\geqslant P[(X^Tu)^2]-\frac{c^*\olsigma}{\sqrt{N}}\big(d\vee\sqrt{dK}\big)\enspace.
\]
\end{lemma}
\begin{proof}
Consider $\cF_Q=\{x\mapsto (x^Tu)^2,\  u\in\bB\}$. 
The sup norm of any function in $\cF_Q$ can be bounded from above as follows:
\[
\|(x^Tu)^2\|_{\infty}=\sup_{\tilde{x}\in\cX}\big(\sum_{i=1}^du_i\varphi_i(\tilde{x})\big)^2=\sup_{\tilde{x}\in\cX}\sum_{i=1}^du^2_i\varphi_i(\tilde{x})=\max_{1\leqslant i\leqslant d}u_i^2\enspace.
\]
As $u\in\bB$, $\max_{i\in \{1,\ldots,d\}}u_i^2\leqslant \|u\|^2\leqslant d$.
Hence, 
\[
\forall u\in\bB,\qquad \sigma^2_u\leqslant P[(X^Tu)^4]\leqslant \|(x^Tu)^2\|_{\infty}P[(X^Tu)^2]\leqslant d\enspace.
\]
Hence, $\sigma^2(\cF_Q)=\sup_{f\in \cF_Q}\text{Var}(f(X))\leqslant d$.
Moreover, the functions $x\mapsto x^Tu$ take values in $[-\sqrt{d},\sqrt{d}]$ and the function $x\mapsto x^2$ is $2\sqrt{d}$ Lipschitz on $[-\sqrt{d},\sqrt{d}]$.
Therefore, by the contraction principle, the Rademacher complexity of $\cF_Q$ can be upper bounded as follows.
\begin{align*}
D(\cF_Q)&=\bigg(\E\bigg[\sup_{u\in\bB}\frac1{\sqrt{N}}\sum_{i=1}^N\epsilon_i(X_i^Tu)^2\bigg]\bigg)^2\\
&\leqslant 8d\bigg(\E\bigg[\sup_{u\in\bB}\frac1{\sqrt{N}}\sum_{i=1}^N\epsilon_iX_i^Tu\bigg]\bigg)^2\enspace.
\end{align*}
By \eqref{eq:RadCompQuadRd}, it follows that 
\[
D(\cF_Q)\leqslant 32d^2\enspace.
\]
By Theorem~\ref{thm:ConcSupMOMg}, there exists an absolute constant $c^*$ such that, with probability larger than $1-e^{-K/c^*}$, at least $9K/10$ blocks $B_k$ satisfy
\[
\sup_{u\in \bB}|(P_{B_k}-P)[(X^Tu)^2]|\leqslant \frac{c^*}{\sqrt{N}}\big(d\vee\sqrt{Kd}\big)\enspace.
\]
\end{proof}

Lemma~\ref{lem:QuadMOMHist0} implies the following corollary.
\begin{corollary}\label{cor:BoundHisto}
  Consider the design vector $X$ defined in \eqref{def:DesignHist}. There exists an absolute constant $c^*$ such that, if 
  \[
c^*\olsigma\big(d\vee\sqrt{dK}\big)\leqslant \sqrt{N}\enspace,
  \]
  the minmax MOM estimator $\ERM_K$ satisfies, with probability larger than $1-2e^{-K/c^*}$, 
  \[
  \|\ERM_K-\bayes\|_{\Sigma}\leqslant \frac{c^*\olsigma}{\sqrt{N}}\bigg(\sqrt{d}\vee\sqrt{K}\bigg)\enspace.
  \]
\end{corollary}
\begin{proof}
 Denote by $c^*$ the largest of the absolute constants appearing in Lemmas~\ref{lem:MultMOM} and~\ref{lem:QuadMOMHist0}.
Define $\Omega$ as the event where, simultaneously, there exist $9K/10$ blocks $B_k$ where 
\[
\sup_{u\in \bB}P_{B_k}[\xi X^Tu]\leqslant \frac{c^*\olsigma}{\sqrt{N}}\bigg(\sqrt{d}\vee\sqrt{K}\bigg)=:m_K\enspace,
\]
and $9K/10$ blocks $B_k$ where, for any $u\in \bB$, 
\[
P_{B_k}[(X^Tu)^2]\geqslant P[(X^Tu)^2]-\frac{c^*\olsigma}{\sqrt{N}}\big(d\vee\sqrt{dK}\big)\enspace.
\]
By Lemmas~\ref{lem:MultMOM} and~\ref{lem:QuadMOMHist0}, $\P(\Omega)\geqslant 1-2e^{-K/c^*}$.
On $\Omega$, there exist at least $9K/10$ blocks where, for any $u\in \bB$,
\[
\|u\|_{\Sigma}^2-P_{B_k}[(X^Tu)^2]\leqslant \frac{c^*\olsigma}{\sqrt{N}}\big(d\vee\sqrt{dK}\big)\enspace.
\]
Assume that 
\[
\frac{c^*\olsigma}{\sqrt{N}}\big(d\vee\sqrt{dK}\big)\leqslant \frac12\enspace.
\]
It follows that, on $\Omega$, there exists at least $8K/10$ blocks where, simultaneously, for any $u\in\bB$,
\begin{gather*}
 P_{B_k}[\xi X^Tu]\leqslant m_K,\qquad  \|u\|_{\Sigma}^2-P_{B_k}[(X^Tu)^2]\leqslant \frac{1}2\enspace.
\end{gather*}
On these blocks,
\[
\forall r>0,\qquad P_{B_k}\big[2r\xi[X^Tu]+r^2(\|u\|_{\Sigma}^2-(X^Tu)^2)\big]\leqslant 2rm_K+\frac{r^2}2\enspace.
\]
As this relationship holds on more than $K/2$ blocks, it holds for the median, so $\Omega$ contains $\cap_{r>0}\Omega_r$, where 
\[
\forall r>0,\qquad \Omega_r=\{\cM_r\leqslant B(r)\},\qquad B(r)=2rm_K+\frac{r^2}2\enspace.
\]
With this choice of function $B$, by \eqref{Ass:GlobBern}, it follows that \eqref{def:r1} holds if
 \[
 2r_1m_K-\frac{r^2_1}2\leqslant 0,\quad\text{i.e.}\quad r_1\geqslant 4m_K\enspace.
 \]
 Let 
 \[
 r_1= 4m_K\quad \text{so}\quad B(r_1)=16m_K^2\enspace.
 \]
 Then, \eqref{def:r2} holds if 
 \[
16m_K^2+2r_2m_K-\frac{r^2_2}2=18m_K^2-\frac{1}2\bigg(r_2-2m_K\bigg)^2\leqslant 0\enspace,
 \]
 that is if 
 \[
 r_2=8m_K=\frac{8c^*\olsigma}{\sqrt{N}}\bigg(\sqrt{d}\vee\sqrt{K}\bigg)\enspace.
 \]
 The proof is concluded by Lemma~\ref{lem:DetArg}.
 \end{proof}

\subsection{An alternative analysis}
To conclude this section, let us provide an alternative analysis that can be used on histograms too.
All along this section $F=\R^d$ and $\bS=\{f\in F:\|f\|_{\Sigma}=1\}$.
Here, the evaluation function $\cE$ is defined as $\cE(f)=P[\ell_f-\ell_{f*}]=\|f-f^*\|^2_{\Sigma}$ in the linear regression problem.
Let $\gamma$ denote a constant such that
\begin{equation}\label{eq:SBAgGen2}
 \forall f\in \bS,\qquad P[[X^Tf]^4]\leqslant \gamma^2\enspace.
\end{equation}
The minmax MOM estimator is studied
\[
\ERM_K\in\argmin_{f\in F}\sup_{g\in F}\MOM{K}{\ell_f-\ell_g}\enspace.
\]
To express the results, the following complexity is used.
\begin{gather}
\label{def:QCompRegGen} \cC_Q(F):=\E\bigg[\sup_{f\in \bS}\sum_{i=1}^N\epsilon_i(X_i^Tf)^2\bigg]\enspace. 
\end{gather}
Recall also that
\[
\E\bigg[\sup_{f\in \bS}\frac1{\sqrt{N}}\sum_{i=1}^N\epsilon_i(\xi X-P[\xi X])^Tf\bigg]=\sqrt{D_N(\bS)}\enspace,
\]
where $D_N(\bS)$ is the Rademacher complexity computed in the proof of Lemma~\ref{lem:MultMOM}.
By \eqref{eq:RadCompMultRd}, it holds that 
\begin{equation}\label{def:MCompRegGen}
\E\bigg[\sup_{f\in \bS}\sum_{i=1}^N\epsilon_i(\xi X-P[\xi X])^Tf\bigg]=\olsigma\sqrt{d N}\enspace. 
\end{equation}
This shows that that $\cC_M(F)$
 is a measure of complexity that extends the dimension used in the previous sections.

\begin{theorem}\label{thm:RegGen}
 Assume \eqref{eq:SBAgGen2}. There exists an absolute constant $c>0$ such that the following holds.
  If
  \[
c\cC_Q(F)\leqslant N\quad\text{and}\quad c\gamma\sqrt{K}\leqslant \sqrt{N}\enspace,
  \]
  then, the minmax MOM estimator satisfies
  \[
  \P\bigg(\cE(\ERM_K)\leqslant c\olsigma^2 \frac{d\vee K}{N}\bigg)\bigg)\geqslant 1-2e^{-K/c}\enspace.
  \]
\end{theorem}

\begin{proof}
By \eqref{def:MCompRegGen},
\begin{align*}
 \E\bigg[\sup_{f\in \bS}\frac1{\sqrt{N}}\sum_{i=1}^N\epsilon_i(\xi X-P[\xi X])^Tf\bigg]&\leqslant \olsigma\sqrt{d}\enspace.
\end{align*}
By Theorem~\ref{thm:ConcSupMOMg}, 
there exists an absolute constant $c^*$ such that, with probability larger than $1-e^{-K/c^*}$, there exist at least $9K/10$ blocks $B_k$ where 
\[
\forall f\in F,\qquad (P_{B_k}-P)[\xi X^Tf]\leqslant c^*\olsigma \sqrt{\frac{d\vee K}{N}}\|f\|_{\Sigma}\enspace. 
\]
By \eqref{eq:Conv}, for any $f\in F$, $P[\xi X^T(f-f^*)]\leqslant 0$.
This implies that, with probability larger than $1-e^{-K/c^*}$, for any $f\in F$, there exist more than $9K/10$ blocks $B_k$ where
\begin{equation}\label{eq:MultGen}
P_{B_k}[\xi X^T(f-f^*)]\leqslant  c^*\olsigma \sqrt{\frac{d\vee K}{N}}\|f-f^*\|_{\Sigma}\enspace. 
\end{equation}
For any $f\in \bS$, 
\[
\text{Var}((X^Tf)^2)\leqslant \gamma^2\enspace.
\]
Moreover,
\[
\E\bigg[\sup_{f\in \bS}\sum_{i=1}^N\epsilon_i(X^Tf)^2\bigg]\leqslant \cC(Q)\enspace.
\]
By Theorem~\ref{thm:ConcSupMOMg}, there exists an absolute constant $c^*$ such that, with probability larger than $1-e^{-K/c^*}$, there exist at least $9K/10$ blocks $B_k$ where, for any $f\in F$, 
\[
P_{B_k}[(X^Tf)^2]\geqslant \|f\|_\Sigma^2(1-\theta_K), \qquad \theta_K= c^*\bigg(\frac{\cC(Q)}{N}\vee\gamma\sqrt{\frac KN}\bigg)\enspace.
\]
Combined with \eqref{eq:MultGen}, this shows that, 
with probability larger than $1-2e^{-K/c^*}$, 
there exist at least $9K/10$ blocks $B_k$ where
\[
\forall f\in F: P_{B_k}[\xi X^Tf]\leqslant \|f\|_{\Sigma}m_K,\qquad \text{where}\qquad m_K=c^*\olsigma \sqrt{\frac{d\vee K}{N}}\bigg)
\]
and at least $9K/10$ blocks $B_k$ where
\[
P_{B_k}[(X^Tf)^2]\geqslant \|f\|_\Sigma^2(1-\theta_K\big)\enspace.
\]
Let $m'_K=m_K\vee \cC(Q)\geqslant \cC(Q)\vee \cC(M)$. 
If $c$ in the theorem is chosen such that $\theta_K\leqslant 1/2$, on this event, there is at least $8K/10$ blocks where, for any $f\in F$
\[
P_{B_k}[2\xi X^Tf-(X^Tf)^2]\leqslant 2\|f\|_\Sigma m_K-\|f\|_\Sigma^2(1-\theta_K)\leqslant 2m_K^2\enspace.
\]
It follows that 
\[
\sup_{f\in F}\MOM{K}{\ell_{\bayes}-\ell_f}\leqslant 2m_K^2\enspace.
\]
The proof terminates with the non-localized bound Lemma~\ref{lem:Rough}.
\end{proof}

\begin{proof}[A second proof of Corollary~\ref{cor:BoundHisto}.]
In the histogram example, for any vector $\bu\in \R^d$, $\|\bu\|_{\Sigma}=\|\bu\|/\sqrt{d}$.
Moreover,
\[
P[(X^T\bu)^4]=P[(\sum_{i=1}^du_i\varphi_i(\tilde{X}))^4]=\sum_{i=1}^du_i^4P\varphi_i=\frac1d\|\bu\|_4^4\leqslant \frac{\|\bu\|^4}d=d\|\bu\|_{\Sigma}^4\enspace.
\]
Hence, \eqref{eq:SBAgGen2} holds with $\gamma=\sqrt{d}$.
Moreover, for any $f\in \bS$, 
\begin{align*}
 \sum_{i=1}^N\epsilon_i(X_i^Tf)^2&=\sum_{i=1}^N\epsilon_i(\sum_{j=1}^du_j\varphi_j(\tilde{X}_i))^2\\
 &=\sum_{i=1}^N\epsilon_i\sum_{j=1}^du_j^2\varphi_j(\tilde{X}_i)\\
 &\leqslant d\|\bu\|_{\Sigma}^2\max_{j\in\{1,\ldots,d\}}\bigg|\sum_{i=1}^N\epsilon_i\varphi_j(\tilde{X}_i)\bigg|\enspace.
\end{align*}
Hence, 
\[
\cC_Q(F)\leqslant d\E\bigg[\max_{j\in\{1,\ldots,d\}}\bigg|\sum_{i=1}^N\epsilon_i\varphi_j(\tilde{X}_i)\bigg|\bigg]\enspace.
\]
By \eqref{eq:BoundSup},
\[
\cC_Q(F)\leqslant 5d\sqrt{N}\enspace.
\]
Therefore, the conditions of Theorem~\ref{thm:RegGen} reduce to those of Corollary~\ref{cor:BoundHisto}. It follows from Theorem~\ref{thm:RegGen} that Corollary~\ref{cor:BoundHisto} holds.
\end{proof}

\chapter{Density estimation with Hellinger loss}\label{Sec:RhoEst}

This chapter presents basic properties of $\rho$-estimators that have been introduced in \cite{MR3565484, MR3595933,BarBir2017RhoAgg2}.
The purpose is not to make a complete presentation of this rich theory, the interested reader is invited to read the mentioned references for this.
Instead, I try to stress some links between robust learning theory and this extension of Le Cam and Birgé's works on estimation from robust tests, see \cite{MR2219712} for an account on this theory and references.
In particular, one can see that these estimators are built from the minmax principle presented in Section~\ref{sec:Generalities} of Chapter~\ref{Chap:MinMax} and can be analysed with Talagrand's inequality and the homogeneity lemma instead of the peeling argument used in the original proofs of the main result of this chapter.
It provides an example of estimation problem that does not fall into Vapnik's theory presented in the introduction where the homogeneity lemma in its general form is useful.
Besides this minor modification, all the material presented here is borrowed from \cite{BarBir2017RhoAgg2}.

\section{Setting}
This chapter deals with a particular instance of \emph{unsupervised learning} where the dataset $\cD_N=(Z_1,\ldots,Z_N)$ is a set of i.i.d. random variables taking values in a measurable space $\cZ$, with common distribution $P^*$. 
Let $\mu$ denote a measure on $\cZ$.
The parameters $f\in F$ are real valued functions defined on $\cZ$.
These functions are densities with respect to $\mu$ and define the measures $P_f$ on $\cZ$, $P_f$ being the distribution with density $f$ with respect to $\mu$.
To measure distances between probability distributions and evaluate the distributions $P_f$ as estimators of $P^*$, we use the Hellinger distance $h$.
Let $P$ and $Q$ denote two probability measures and let $\lambda$ denote a measure dominating both $P$ and $Q$, the Hellinger distance between $P$ and $Q$ is defined by 
\[
h(P,Q)=\frac1{\sqrt{2}}\sqrt{\int (\sqrt{p}-\sqrt{q})^2{\rm d}\lambda}\enspace.
\]
It is clear that $0\leqslant h(P,Q)\leqslant1$ for any probability measures $P$ and $Q$ and that $h(P,Q)$ does not depend on the dominating measure $\lambda$.
The evaluation function $\cE$ is defined on $F$ as $\cE(f)=h(P_f,P^*)$.

The purpose of this chapter is to analyse $\rho$-estimators of $P^*$ introduced in \cite{MR3595933} and defined by $P_{\ERM}$, where
\begin{equation}\label{def:rhoEst}
 \ERM\in\argmin_{f\in F}\sup_{g\in F}T(f,g),\qquad \text{where}\qquad T(f,g)=\sum_{i=1}^N\rho\bigg(\sqrt{\frac{g(Z_i)}{f(Z_i)}}\bigg)\enspace.
 \end{equation}
Here, the function $\rho=(x-1)/(x+1)$ is non-decreasing $[0,+\infty]\to[-1,1]$, $2$-Lipschitz, it satisfies $\rho(1/x)=-\rho(x)$ for any $x\in[0,+\infty)$.

\section{Preliminary results}
This section presents the first results on the tests defining $\rho$-estimators.
The goal is to understand the intuition behind the construction of these estimators.
The choice of function $\rho$ is justified by the following remarkable property.
The material of this section is borrowed from \cite{BarBir2017RhoAgg2}.
\begin{theorem}\label{thm:LinkWithHellinger}
 For any $f\in F$, let $P_f$ denote the probability distribution with density $f$ w.r.t. the measure $\mu$, then
\begin{gather*}
\int \rho\bigg(\sqrt{\frac{f}g}\bigg){\rm d} R\leqslant 4 h^2(R,P_g)-(3/8)h^2(R,P_f)\enspace,\\
\int \rho^2\bigg(\sqrt{\frac{f}g}\bigg){\rm d} R\leqslant 3\sqrt{2}[ h^2(R,P_g)+h^2(R,P_f)]\enspace. 
\end{gather*}
\end{theorem}
\begin{remark}
 The strength of this result is that it is valid for any distributions $P_f$, $P_g$ and $R$.
 It implies in particular that the sign of the expectation $\E[T(f,g)]$, where $T(f,g)$ is defined in \eqref{def:rhoEst}, provides relevant informations regarding which distribution between $P_f$ and $P_g$ is the closest to $P^*$.
\end{remark}
\begin{proof}
The proof proceeds in two steps.
\begin{lemma}\label{lem:LinkWithHell1}
 Theorem~\ref{thm:LinkWithHellinger} holds for any $R$ absolutely continuous w.r.t. $\mu$.
\end{lemma}
\begin{proof}
The proof is quite technical and not very intuitive.
It uses repeatedly the following relation: for any distributions $P$, $Q$ and any measure $\lambda$ dominating $P$ and $Q$,
\[
h^2(P,Q)=\frac1{2}\int (\sqrt{p}-\sqrt{q})^2{\rm d}\lambda=1-\int\sqrt{pq}{\rm d}\lambda\enspace.
\]
 Let $r$ denote the density $r=\delta^{-2}(\sqrt{f}+\sqrt{g})^2$, where
 \[
 \delta^2=\int(\sqrt{f}+\sqrt{g})^2{\rm d}\mu=4\bigg(1-\frac{h^2(P_f,P_g)}2\bigg)\enspace.
 \]
As $h^2(P_f,P_g)\in[0,1]$, this implies that $\sqrt{2}\leqslant \delta\leqslant 2$.
Moreover, by convexity of the map $\vartheta:u\mapsto 1/\sqrt{1-u}$ on $(0,1)$,
\begin{equation}\label{eq:2/delta}
 \frac2{\delta}=\vartheta\bigg(\frac{h^2(P_f,P_g)}2\bigg)\geqslant \vartheta(0)+\vartheta'(0)\frac{h^2(P_f,P_g)}2=1+\frac{h^2(P_f,P_g)}4\enspace. 
\end{equation}
 Denote by $s$ the density of $R$ with respect to $\mu$. It follows that
\begin{align}
\label{eq:BoundhRPR2} h^2(R,P_r)&=1-\int\sqrt{sr}{\rm d}\mu =1-\frac1{\delta}\bigg(\int\sqrt{sf}{\rm d}\mu+\int\sqrt{sg}{\rm d}\mu\bigg)\\
\notag &=1-\frac2{\delta}+\frac{h^2(R,P_f)+h^2(R,P_g)}{\delta}\\
\label{eq:BoundhRPR} &\leqslant \frac{h^2(R,P_f)+h^2(R,P_g)}{\delta}-\frac{h^2(P_f,P_g)}4\enspace.
\end{align}
Elementary calculus shows that
\begin{align*}
 \int\rho^2\bigg(\sqrt{\frac{f}g}\bigg)s{\rm d} \mu&=\int_{r>0}\bigg(\frac{\sqrt{f}-\sqrt{g}}{\sqrt{f}+\sqrt{g}}\bigg)^2(\sqrt{s}-\sqrt{r}+\sqrt{r})^2{\rm d} \mu\enspace.
\end{align*}
Using the inequality $(a+b)^2\leqslant (1+\alpha)a^2+(1+\alpha^{-1})b^2$, valid for any real numbers $a$ and $b$ and any $\alpha>0$ to $a=\sqrt{s}-\sqrt{r}$ and $b=\sqrt{r}$ shows that, for any $\alpha>0$,
\begin{align*}
 \int\rho^2\bigg(\sqrt{\frac{f}g}\bigg)s{\rm d} \mu&\leqslant (1+\alpha)\int_{r>0}\bigg(\frac{\sqrt{f}-\sqrt{g}}{\sqrt{f}+\sqrt{g}}\bigg)^2(\sqrt{s}-\sqrt{r})^2{\rm d} \mu\\
 &\quad +(1+\alpha^{-1})\int_{r>0}\bigg(\frac{\sqrt{f}-\sqrt{g}}{\sqrt{f}+\sqrt{g}}\bigg)^2\bigg(\frac{\sqrt{f}+\sqrt{g}}{\delta}\bigg)^2{\rm d} \mu\enspace.
\end{align*}
In this expression, as $\big((\sqrt{f}-\sqrt{g})(\sqrt{f}+\sqrt{g})\big)^2\leqslant 1$, the first item in the right hand side is bounded from above by $2(1+\alpha)h^2(R,P_r)$. 
The second item in the right hand side is equal to $(1+\alpha^{-1})(2/\delta^2)h^2(P_f,P_g)$.
Combining these upper bounds yields
\[
\int\rho^2\bigg(\sqrt{\frac{f}g}\bigg)s{\rm d} \mu\leqslant 2(1+\alpha)h^2(R,P_r)+\frac{2(1+\alpha^{-1})}{\delta^2}h^2(P_f,P_g)\enspace.
\]
Then by \eqref{eq:BoundhRPR},
\begin{multline*}
  \int\rho^2\bigg(\sqrt{\frac{f}g}\bigg)s{\rm d} \mu\\
  \leqslant \frac{2(1+\alpha)}\delta\big(h^2(R,P_f)+h^2(R,P_g)\big)-\frac{\delta^2(1+\alpha)-4(1+\alpha^{-1})}{2\delta^2}h^2(P_f,P_g)\enspace.
\end{multline*}
If $(1+\alpha)\delta^2=4(1+\alpha^{-1})$, it implies
\begin{align*}
 \int\rho^2\bigg(\sqrt{\frac{f}g}\bigg)s{\rm d} \mu=\frac{2(1+\alpha)}{\delta}\big(h^2(R,P_f)+h^2(R,P_g)\big)\enspace.
\end{align*}
Solving the equation $(1+\alpha)\delta^2=4(1+\alpha^{-1})$ in $\alpha$ gives $\alpha=4/\delta^2$, thus $2(1+\alpha)/\delta=2/\delta+4/\delta^3\leqslant 3\sqrt{2}$ since $\delta\geqslant \sqrt{2}$.
This proves the second item of Theorem~\ref{thm:LinkWithHellinger} when $R$ is absolutely continuous with respect to $\mu$.

Moving to the first item, define, for any $f\in F$, 
\[
\rho_r(R,P_f)=\frac12\bigg[\int \sqrt{fr}{\rm d}\mu+\int \sqrt{\frac fr}s{\rm d}\mu\bigg]\enspace.
\]
The increments of $\rho_r(R,\cdot)$ are intimately related to the expectation of $T$: for any $f$ and $g$ in $F$,
\begin{align}
\notag\rho_r(R,P_f)-&\rho_r(R,P_g)\\
\notag&=\frac12\bigg[\frac1\delta\int (\sqrt{f}-\sqrt{g})(\sqrt{f}+\sqrt{g}){\rm d}\mu+\delta\int \frac{\sqrt{f}-\sqrt{g}}{\sqrt{f}+\sqrt{g}}s{\rm d}\mu\bigg]\\
\label{eq:RhoR1} &=\frac{\delta}2\int\rho\bigg(\sqrt{\frac fg}\bigg)s{\rm d}\mu=\frac{\delta}{2N}\E[T(f,g)]\enspace.
\end{align}
Moreover,
\begin{align*}
\int \sqrt{\frac fr}s{\rm d}\mu&= \int \sqrt{\frac fr}(\sqrt{s}-\sqrt{r}+\sqrt{r})^2{\rm d}\mu\\
&=\int \sqrt{\frac fr}(\sqrt{s}-\sqrt{r})^2{\rm d}\mu+\int \sqrt{fr}{\rm d}\mu+2\int \sqrt{ f }(\sqrt{s}-\sqrt{r}){\rm d}\mu\\
&=\int \sqrt{\frac fr}(\sqrt{s}-\sqrt{r})^2{\rm d}\mu-\int \sqrt{fr}{\rm d}\mu+2\int \sqrt{ f s}{\rm d}\mu\enspace.
\end{align*}
As $f=0$ on the event $r=0$, it follows that
\begin{align*}
 \rho_r(R,P_f)&=\int \sqrt{ f s}{\rm d}\mu+\frac12\int_{r>0}\sqrt{\frac fr}(\sqrt{s}-\sqrt{r})^2{\rm d}\mu\\
 &=\int \sqrt{ f s}{\rm d}\mu+\frac{\delta}2\int_{r>0}\frac{\sqrt{f}}{\sqrt{f}+\sqrt{g}}(\sqrt{s}-\sqrt{r})^2{\rm d}\mu
\end{align*}
Thus \eqref{eq:RhoR1} implies that
\begin{align*}
\frac{\delta}2\int\rho\bigg(\sqrt{\frac fg}\bigg)s{\rm d}\mu&=\int (\sqrt{f}-\sqrt{g})\sqrt{s}{\rm d}\mu+\frac{\delta}2\int_{r>0}\frac{\sqrt{f}-\sqrt{g}}{\sqrt{f}+\sqrt{g}}(\sqrt{s}-\sqrt{r})^2{\rm d}\mu\\
&=\int (\sqrt{f}-\sqrt{g})\sqrt{s}{\rm d}\mu+\frac{\delta}2\int_{r>0}\rho\bigg(\sqrt{\frac fg}\bigg)(\sqrt{s}-\sqrt{r})^2{\rm d}\mu\enspace.
\end{align*}
As $\rho$ takes values in $[-1,1]$, 
\[
\frac{\delta}2\int\rho\bigg(\sqrt{\frac fg}\bigg)s{\rm d}\mu\leqslant \int (\sqrt{f}-\sqrt{g})\sqrt{s}{\rm d}\mu+\delta h^2(R,P_r)\enspace.
\]
By \eqref{eq:BoundhRPR2},
\[
\frac{\delta}2\int\rho\bigg(\sqrt{\frac fg}\bigg)s{\rm d}\mu\leqslant \delta-2\int \sqrt{ g s}{\rm d}\mu\enspace.
\]
By \eqref{eq:2/delta},
\begin{align*}
\int\rho\bigg(\sqrt{\frac fg}\bigg)s{\rm d}\mu&\leqslant 2\bigg[1- \int \sqrt{ g s}{\rm d}\mu\bigg(1+\frac{h^2(P_f,P_g)}4\bigg)\bigg]\\
&\leqslant 2\bigg[h^2(R,P_g)\bigg(1+\frac{h^2(P_f,P_g)}4\bigg)- \frac{h^2(P_f,P_g)}4\bigg]\\
&\leqslant \frac12[5h^2(R,P_g)-h^2(P_f,P_g)]\enspace.
\end{align*}
By the triangular inequality, $h(P_f,P_g)\geqslant |h(R,P_g)-h(R,P_f)|$, hence, 
\begin{align*}
h^2(P_f,P_g)&\geqslant h^2(R,P_g)+h^2(R,P_f)-2h(R,P_g)h(R,P_f)\\
&\geqslant \frac34h^2(R,P_f)-3h^2(R,P_g)\enspace.
\end{align*}
Therefore,
\begin{align*}
\int\rho\bigg(\sqrt{\frac fg}\bigg)s{\rm d}\mu\leqslant 4h^2(R,P_g)-\frac38h^2(R,P_f)\enspace.
\end{align*}
The first item of Theorem~\ref{thm:LinkWithHellinger} is established in the case where $R$ is absolutely continuous with respect to $\mu$. This concludes the proof of Lemma~\ref{lem:LinkWithHell1}.
\end{proof}

The second result shows that it is sufficient to show Theorem~\ref{thm:LinkWithHellinger} when $R$ is absolutely continuous with respect to $\mu$ to prove it in general.
\begin{lemma}\label{lem:absConIsEnough}
If Theorem~\ref{thm:LinkWithHellinger} holds for any $R$ absolutely continuous with respect to $\mu$, it holds for any $R$.
\end{lemma}
\begin{proof}
 Write $R=\delta^2R'+(1-\delta^2)R''$, with $R'$ absolutely continuous with respect to $\mu$, $R''$ orthogonal to $\mu$ and $\delta\in(0,1)$.
 Let $\overline{\mu}=R+P_f$ which dominates both $R$ and $P_f$. 
 As $R''$ is orthogonal to $\mu$, it holds that $(\rmd R''/\rmd\overline{\mu})(\rmd P_f/\rmd\overline{\mu})=0$.
 Therefore, the following fundamental relationship between the Hellinger distances $h^2(R,P_f)$ and $h^2(R',P_f)$ holds,
 \[
 1-h^2(R,P_f)=\int\sqrt{\bigg(\delta^2\frac{\rmd R'}{\rmd\overline{\mu}}+(1-\delta^2)\frac{\rmd R''}{\rmd\overline{\mu}}\bigg)\frac{\rmd P_f}{\rmd\overline{\mu}}}\rmd\overline{\mu}=\delta(1-h^2(R',P_f))\enspace.
 \]
As this holds for any $f\in F$, in particular, 
\begin{gather}
\notag h^2(R,P_f)=1-\delta+\delta h^2(R',P_f)\geqslant 1-\delta\enspace.\\
\label{eq:FundRel}h^2(R,P_g)=1-\delta+\delta h^2(R',P_g)\geqslant 1-\delta\enspace.
\end{gather}
By hypothesis, the second item of Theorem~\ref{thm:LinkWithHellinger} applies to $R'$ that is absolutely continuous with respect to $\mu$, so
\begin{align*}
 \int \rho^2\bigg(\sqrt{\frac{f}g}\bigg){\rm d} R&\leqslant \delta^2\int \rho^2\bigg(\sqrt{\frac{f}g}\bigg){\rm d} R'+1-\delta^2\\
 &\leqslant 3\sqrt{2}\delta^2[ h^2(R',P_g)+h^2(R',P_f)]+1-\delta^2\enspace.
 \end{align*}
 Applying the fundamental relations \eqref{eq:FundRel} yields
\begin{align*}
 \int \rho^2\bigg(\sqrt{\frac{f}g}\bigg){\rm d} R &\leqslant 3\sqrt{2}\delta(h^2(R,P_g)+h^2(R,P_f)-2(1-\delta))+(1-\delta^2)\\
 &=3\sqrt{2}(h^2(R,P_g)+h^2(R,P_f))+\text{Rem}(\delta)\enspace,
\end{align*}
where the remainder term satisfies, according to the fundamental relations \eqref{eq:FundRel},
\begin{align*}
\text{Rem}(\delta)&= (1-\delta^2)-3\sqrt{2}(2\delta(1-\delta)+(1-\delta)(h^2(R,P_g)+h^2(R,P_f)))\\
&\leqslant (1-\delta^2)-3\sqrt{2}(2\delta(1-\delta)+2(1-\delta)^2)\\
&= (1-\delta)(1+\delta-6\sqrt{2}(2-\delta))\leqslant  (1-\delta)(2-6\sqrt{2})\leqslant 0\enspace.
\end{align*}
This proves the second item of Theorem~\ref{thm:LinkWithHellinger} for $R$.
 
By hypothesis, the first item of Theorem~\ref{thm:LinkWithHellinger} applies to $R'$ that is absolutely continuous with respect to $\mu$, so
\begin{align*}
 \int \rho\bigg(\sqrt{\frac{f}g}\bigg){\rm d} R&\leqslant \delta^2\int \rho\bigg(\sqrt{\frac{f}g}\bigg){\rm d} R'+(1-\delta^2)\\
 &\leqslant \delta^2(4 h^2(R',P_g)-(3/8)h^2(R',P_f))+(1-\delta^2)\\
 &=\delta(4 h^2(R,P_g)-(3/8)h^2(R,P_f)-(29/8)(1-\delta))+(1-\delta^2)\\
 &=4 h^2(R,P_g)-(3/8)h^2(R,P_f)+\text{Rem}(\delta)\enspace,
\end{align*}
where the remainder term 
\begin{align*}
 \text{Rem}(\delta)&=(1-\delta^2)-(29/8)\delta(1-\delta)-(1-\delta)(4 h^2(R,P_g)-(3/8)h^2(R,P_f))\\
 &\leqslant (1-\delta)(11/8+\delta-29/8\delta-4(1-\delta))\\
 &=(1-\delta)(-21/8+11/8\delta)\leqslant 0\enspace.
\end{align*}
This concludes the proof of the first item of Theorem~\ref{thm:LinkWithHellinger}.
Therefore, Lemma~\ref{lem:absConIsEnough} is proved.
\end{proof}
Theorem~\ref{thm:LinkWithHellinger} is a direct consequence of Lemmas~\ref{lem:LinkWithHell1} and \ref{lem:absConIsEnough}.
\end{proof}

\section{Main result}
The remaining of the chapter is devoted to the proof of the following theorem.
\begin{theorem}\label{thm:RhoEstMain}
Let $f^*\in\argmin_{f\in F} h(P^*,P_f)$ and, for any $f\in F$, let $U_{i,f}=\rho(\sqrt{f(X_i)/f^*(X_i)})$.
Define the complexity of the model $F$ as a fixed point of the following local Rademacher complexity of $F$:
\[
D(F)=1\vee N\bigg(\sup\bigg\{r>0:\E\bigg[\sup_{f\in F:\cE(f)\leqslant r}\frac1{N}\sum_{i=1}^N\epsilon_iU_{i,f}\bigg]>\frac{r^2}{80}\bigg\}\bigg)^2\enspace.
\] 
 There exists an absolute constant $C$ such that any $\rho$-estimator $\ERM$ defined in \eqref{def:rhoEst} satisfies, with probability larger than $1-2e^{-t}$,
 \[
 h^2(P^*,P_{\ERM})\leqslant C\bigg(\inf_{f\in F}h^2(P^*,P_{f})+\frac{D(F)+t}N\bigg)\enspace.
 \]
\end{theorem}
\begin{remark}
 Again, the remarkable feature here is that Theorem~\ref{thm:RhoEstMain} holds without assumptions on $P^*$ or the set $F$ of densities.
\end{remark}

\begin{proof}
Recall that the evaluation function is defined in this chapter, for any $f\in F$, by $\cE(f)=h(P^*,P_f)$ and that $f^*$ is defined as a density in $F$ such that
\[
\forall f\in F,\qquad \cE(f^*)\leqslant \cE(f)\enspace.
\]
Hereafter, define also
\[
\forall f,g\in F^2,\qquad d(f,g)=N \E_{P^*}\bigg[\rho\bigg(\sqrt{\frac{g}{f}}\bigg)\bigg]\enspace.
\]
Theorem~\ref{thm:LinkWithHellinger} shows in particular that, for any $f\in F$,
\begin{equation}\label{eq:dcE}
(3/8)\cE^2(f)-4\cE^2(f^*)\leqslant \frac{d(f,f^*)}N\leqslant 4\cE^2(f)-3/8\cE^2(f^*)\enspace. 
\end{equation}
Let $r_0=\cE(f^*)$. 
By Lemma~\ref{lem:RhoTestAreHomogeneous}, the test $T$ fulfils Condition {\bf (HP)} of the homogeneity lemma (Lemma~\ref{lem:DetArg}).
To bound the Hellinger distance between the associated minmax estimator and the unknown density $P^*$ of the observations, it remains to compute the function $B$ in the homogeneity Lemma.
 
Fix $r>r_0$. Recall that $U_{i,f}=\rho(\sqrt{f(X_i)/f^*(X_i)})$ are independent random variables, bounded by $1$ and that 
\[
T(f^*,f)=\sum_{i=1}^NU_{i,f}\enspace.
\] 
Moreover, Theorem~\ref{thm:LinkWithHellinger} shows that, for any $f\in F$ such that $\cE(f)\leqslant r$,
\[
\text{Var}(U_{i,f})\leqslant 6\sqrt{2}Nr^2\enspace.
\]
Therefore, it follows from Talagrand's concentration inequality (Theorem~\ref{thm:TalIneq}) that, for any $t>0$, the random variable $Z_r=\sup_{f\in F:\cE(f)\leqslant r}\sum_{i=1}^N(U_{i,f}-\E[U_{i,f}]$ satisfies
\[
\P\bigg(Z_r\leqslant 2\E\big[Z_r\big]+\frac{N}{20}r^2+(2+20\sqrt{6})t\bigg)\geqslant 1-e^{-t}\enspace.
\]
By the symmetrization trick, $\E\big[Z_r\big]\leqslant 2\E[Z_{\epsilon,r}]$, where 
\[
Z_{\epsilon,r}=\sup_{f\in F:\cE(f)\leqslant r}\sum_{i=1}^N\epsilon_iU_{i,f}\enspace.
\]
Hence, with probability at least $1-e^{-t}$,
\[
Z_r\leqslant4\E\big[Z_{\epsilon,r}\big]+\frac{N}{20}r^2+(2+20\sqrt{6})t\leqslant 4\E\big[Z_{\epsilon,r}\big]+ \frac{N}{20}r^2+51t\enspace.
\] 
By definition of $D(F)$, for any $r>\sqrt{D(F)/N}$, 
\[
\E[Z_{\epsilon,r}]\leqslant \frac{Nr^2}{80}\enspace.
\] 
Hence, for any $r>r_0\vee\sqrt{D(F)/N}$, it follows that, with probability at least $1-e^{-t}$,
\[
Z_r\leqslant \frac{N}{10}r^2+51t\enspace.
\] 
As a consequence, for any $t>0$ and $r>\sqrt{D(F)/N}\vee r_0$, one can choose 
\[
B(r)=\frac{Nr^2}{10}+51 t
\]
in the homogeneity lemma and get that the event $\Omega_r$ in Lemma~\ref{lem:DetArg} holds with probability at least $1-e^{-t}$.
 
With this value of $B(r)$, from \eqref{eq:dcE} that
\begin{align*}
B(r)-\inf_{f\in F:\cE(f)=r}d(f,f^*)&\leqslant \frac{Nr^2}{10}+51t-\frac{3N}8r^2+4N\cE^2(f^*)\\
&\leqslant  51 t+4N\cE^2(f^*)-\frac{N}{4}r^2\enspace.
\end{align*}
From this upper bound, one can choose $r_1=\sqrt{204 t/N+16\cE^2(f^*)}\vee \sqrt{D(F)/N}$ in \eqref{def:r1}. 
Then, 
\begin{align*}
B(r_1)&=(20.4t+1.6N\cE^2(f^*)+51 t)\vee\bigg(\frac{D(F)}{10}+51Nt\bigg)\\
&\leqslant (2N\cE^2(f^*)+72t)\vee \bigg(\frac{D(F)}{10}+51t\bigg)\enspace. 
\end{align*}
By \eqref{eq:dcE}, $\sup_{f\in F}d(f^*,f)\leqslant 4N\cE(f^*)$. Hence, one can choose the following upper bound $\cB$ in Lemma~\ref{lem:DetArg}:
\[
\cB=6N\cE^2(f^*)+\frac{D(F)}{10}+72t\enspace.
\]
Hence, \eqref{def:r2} holds for any $r$ such that
 \[
 51 t+4N\cE^2(f^*)-\frac{N}{4}r^2\leqslant -\bigg(6N\cE^2(f^*)+\frac{D(F)}{10}+72t\bigg)\enspace,
 \]
 i.e. for any
 \[
 r^2\geqslant 40\cE^2(f^*)+\frac{2D(F)}{5N}+492t\enspace.
 \]
\end{proof}

\chapter{Estimators computable in polynomial time}\label{Chap:CompEst}
 
In the previous chapters, we studied minmax MOM estimators in various contexts and showed that they achieved interesting theoretical performance under weak assumptions on the data.

For example, for multivariate mean estimation, they are proved to satisfy a sub-Gaussian deviation inequality 
\begin{equation}\label{eq:SGDev}
 \P\bigg(\|\muh_K-\mu_P\|>C\frac{\sqrt{\text{Tr}(\Sigma)}+\sqrt{\|\Sigma\|_{\text{op}}K}}{\sqrt{N}}\bigg)\leqslant e^{-K/C}\enspace,
\end{equation}
where $C$ is some absolute constant, assuming only that $P[\|X\|^2]<\infty$.

The first estimator that was shown to achieve this bound was proposed in \cite{LugosiMendelson2017-2}. 
The procedure there was closely related to minmax MOM, several other procedures achieving similar bounds have been proposed since then. 
Some of them are presented in Chapter~\ref{Chap:MME} for example, see also \cite{LugMen:2019} for a proof that a clever extension of the classical trimmed mean estimator on $\R$ has sub-Gaussian deviations and \cite{LugMen:2019Rev} for a review on the subject.
The problem with the minmax MOM construction or the one based on Le Cam's aggregation of tests in \cite{LugosiMendelson2017-2} is that these estimators cannot be computed in polynomial time.

In this chapter, we consider the problem of building estimators achieving sub-Gaussian deviations \eqref{eq:SGDev} that can be computed in polynomial time.
This problem was solved first in \cite{Hop:2018} using an estimator solving a semidefinite program (SDP).
Recall that these take the form of finding the minimizer in $\bX\in \R^{d\times d}$ of the functional $\psh{\bX}{\bC}=\text{Tr}(\bX\bC^T)$, subject to the constraints $\psh{\bA_1}{\bX}\geqslant 0, \ldots,\psh{\bA_k}{\bX}\geqslant 0$ and $\bX$ ranges over the symmetric positive semi-definite matrices $\bX\succeq 0$.
Under mild conditions on $\bC$ and $\bA_1,\ldots,\bA_k$, semidefinite programs (SDP) can be solved in polynomial time.
To find a SDP whose solution achieves \eqref{eq:SGDev}, \cite{Hop:2018} uses the sum-of-squares (SoS) method.
Let $p, q_1,\ldots,q_m$ denote multivariate polynomials in $\R[x_1,\ldots,x_n]$, the SoS method produces a SDP relaxation of the problem of finding a minimizer of $P(\bx)$ under the constraints $q_1(\bx)\geqslant 0,\ldots,q_m(\bx)\geqslant 0$.
This relaxation depends on an even integer $r\geqslant \max\{\deg(p),\deg(q_i),i=1,\ldots,m\}$. 
The relaxation is solvable in $O((Nm)^{O(r)})$ operations and, of course, the quality of the approximation improves with $r$.
The solution in \cite{Hop:2018} uses $r=8$ and produces an algorithm that runs in $O(N^{24})$ operations.
While this is actually polynomial time algorithm, it can still not be used in practice.

Using ideas related to \cite{Hop:2018}, \cite{CheFlaBar:2019} proposed an alternative SDP relaxation that improved considerably the running time.
The method goes as follows.
They first considered the problem $P0$ of finding the vectors $\bfb\in\{0,1\}^K$ and $\bv\in\bS$ such that $\sum_{k=1}^Kb_k$ is as large as possible under the constraint that, for all $k=1,\ldots K$, $b_k\bv^T(P_{B_k}X-\bx)\geqslant b_k^2r$.
If this problem could be solved, it could be used to estimate first the distance between $\bx$ and $\mu_P$ by $d_{\bx}$ the largest $r$ such that $\sum b_k\geqslant (1-\alpha)K$ and then an estimation $g_\bx$ of the direction $\bx-\mu_P$ by the optimal vector $\bv$ given for $r=d_{\bx}$.
Then, using these estimations, one can build a descent algorithm that moves from $\bx$ to $T(\bx)=\bx-\gamma d_{\bx}g_{\bx}$ and that stops when $d_{T(\bx)}>d_\bx$.

They first proved that this descent algorithm produces in $O(\log\|\muh^{(0)}\|/\epsilon)$ an estimator that is, with probability larger than $1-e^{-K/C}$ at distance from $\mu_P$ bounded from above by  
\[
\epsilon\vee C\frac{\sqrt{\text{Tr}(\Sigma)}+\sqrt{\|\Sigma\|_{\text{op}}K}}{\sqrt{N}}
\]
 
The key is thus to find an approximate solution of the basic problem $P0$.
For this, they used a SDP relaxation of $P0$.
They looked for a positive semidefinite matrix $\bX\succeq 0$ of size $K+d+1$ with entries $x_{i,j}$ such that $\sum_{k=1}^K x_{1,b_k}$ is as large as possible under the constraints that $x_{1,b_k}=x_{b_k,b_k}$, $X_{1,1}=1$, $\sum_{j=1}^dx_{v_j,v_j}=1$ and, for any $k=1,\ldots,K$, $\bX_{b_k,\bv}^T(P_{B_k}X-\bx)\geqslant x_{b_k,b_k}r$.
Here, the vectors 
\[
\bX_{b_k,\bv}=
\begin{bmatrix}
 x_{b_k,v_1}\\
 \vdots\\
 x_{b_k,v_d}
\end{bmatrix}\enspace.
\]
This SDP can be solved using an interior point method that runs in $O(k^{3.5})$ operations.
They also proved that a solution of this problem can be used to build a descent algorithm that, overall, runs in $\tilde{O}(kd+k^{3.5})$ operations.
Here and in the following $m=\tilde{O}(f(N,d,k))$ mean that there exists absolute constants $c_1,c_2$ such that 
\[
m\leqslant c_1f(N,d,k) (\log(dN))^{c_2}\enspace.
\]

The method detailed in this chapter comes from \cite{DepLec:2019}. 
It uses a convex relaxation of the problem that is closely related to a construction proposed in \cite{Cheng:2019:HRM:3310435.3310606}, to build estimators that are robust to a large number of outliers.
The key technical tool to solve this problem comes from \cite{Peng:2012:FSW:2312005.2312026}, it is reproduced here without a proof.
The main idea in \cite{DepLec:2019} is an extension of Theorem~\ref{thm:ConcSupMOMg} that is provided in Lemma~\ref{lem:PrelimAlgo}.
The material of Lemma~\ref{lem:GeomMean} is a simplification of the Geometric-MOM algorithm of \cite{MR3378468} also due to J. Depersin and G. Lecu\'e that provides a particularly simple and elegant construction that yields performance similar to \cite{MR3378468}, which are sadly slightly sub-optimal.
A competitive method, with similar complexity but using a spectral algorithm instead of a SDP relaxation was also proposed in \cite{LLVZ:2019}, it will be included in a future version of these notes.
The main result of the chapter is the following.
\begin{theorem}
 There exists a numerical constant $C$ and an algorithm that runs in $\tilde{O}(uK+Kd)$ operations and outputs an estimator $\muh$ of $\mu_P$ such that
 \[
\P\bigg( \|\muh-\mu_P\|\leqslant C\bigg(\sqrt{\frac{\text{Tr}(\Sigma)}N}+\sqrt{\frac{\|\Sigma\|_{\text{op}}K}N}\bigg)\bigg)\geqslant 1-e^{-u\wedge(K/C)}\enspace.
 \]
\end{theorem}

 All along the chapter, we consider the problem of estimating the multivariate expectation $\mu_P\in \R^d$ of a distribution $P$ from a data-set $X_{1},\ldots,X_N$ of independent random vectors with common expectations $\mu_P$ and common covariance matrix $\Sigma=P[(X-\mu_P)(X-\mu_P)^T]$.
Hereafter, $K$ denotes an integer, smaller than $N$.
All results can be extended to allow for a proportion $\gamma K$ of outliers, for some $\gamma\in (0,1/3)$. 
These outliers may be adversarial, they may not be independent nor independent from the inliers, without affecting the results.

 \section{Initialization of the algorithm}
Let $\cM=\{P_{B_k}X, k\in \{1,\ldots,K\}\}$ denote the set of means.
For any $k\in \{1,\ldots, K\}$, let 
\[
\cC_k=\text{median}\{\|P_{B_k}X-m\|, m\in \cM\}, \qquad 
\widehat{k}\in \argmin_{k\in \{1,\ldots,K\}}\cC_k\enspace.
\]
We initialize the algorithm with 
\[
\muh^{(0)}=P_{B_{\kh}}X\enspace.
\] 
\begin{lemma}\label{lem:GeomMean}
The estimator $\muh^{(0)}$ is computed in $O(K(d+K))$ operations and satisfies
\[
\P\bigg(\|\muh^{(0)}-\mu_P\|\leqslant 12\sqrt{\frac{\text{Tr}(\Sigma) K}N}\bigg)\geqslant 1-e^{-K/128}\enspace.
\]
\end{lemma}
\begin{proof}
To compute $\muh^{(0)}$, we need at most $Kd$ operations to compute each $P_{B_k}X$, $K^2$ operations to compare all differences $\|P_{B_k}X-P_{B_j}X\|$ and $O(K\log K)$ operations to rank the $\cC_k$.

We have
\[
P[\|P_{B_k}X-\mu_P\|^2]=\frac1{|B_k|^2}\sum_{(i,j)\in B_k}P[(X_i-\mu_P)^T(X_j-\mu_P)]=\frac{\text{Tr}(\Sigma) K}{N}\enspace.
\]
Define $r_K=\text{Tr}(\Sigma) K/N$. By Markov's inequality, it follows that, for any $k\in\{1,\ldots,K\}$,
\[
\forall x>0, \qquad \P\big(\|P_{B_k}X-\mu_P\|>xr_K\big)\leqslant \frac1{x^2}\enspace.
\]
By Hoeffding's inequality,
\begin{align*}
\forall \alpha>0,\qquad \P\bigg( \frac1K\sum_{i=1}^K\bigg({\bf 1}_{\{\|P_{B_k}X-\mu_P\|>xr_K\}}-\frac1{x^2}\bigg)>\alpha\bigg)\leqslant e^{-2K\alpha^2}\enspace.
\end{align*}
In words, for any $x>0$ and $\alpha>0$, the probability that there exist at least $(1-\alpha-1/x^2) K$ blocks $B_k$ where $\|P_{B_k}X-\mu_P\|\leqslant xr_K$ is larger than $1-e^{-2K\alpha^2}$. 
Choosing $\alpha=1/16$ and $x=4$ shows that, with probability at least $1-e^{-K/128}$, $|\cK|\geqslant 7K/8$, where $\cK$ is the set of indices $k$ of the blocks $B_k$ such that 
\[
\|P_{B_k}X-\mu_P\|\leqslant 4r_K\enspace.
\]
For any $k$ in $\cK$, by the triangular inequality, for any $j\in \cK$, 
\[
\|P_{B_k}X-P_{B_j}X\|\leqslant 8 r_K\enspace.
\]
Therefore, since $K\geqslant 3$, $7K/8-1\geqslant K/2$ and thus, on the event $|\cK|\geqslant 7K/8$,
\[
\forall k\in \cK,\qquad \cC_k\leqslant 8r_K\enspace.
\]
In particular thus, $\cC_{\kh}\leqslant 8r_K$.
Therefore, there is more than $K/2$ blocks where $\|P_{B_k}X-P_{B_{\kh}}X\|\leqslant 8r_K$, and $7K/8$ blocks where $\|P_{B_k}X-\mu_P\|\leqslant 4r_K$. 
As $K\geqslant 3$, it follows that there is at least one blocks $B_k$ such that both inequalities hold. 
Therefore, on the event $|\cK|\geqslant 7K/8$,
\[
\|\muh^{(0)}-\mu_P\|\leqslant \|P_{B_k}X-P_{B_{\kh}}X\|+\|P_{B_k}X-\mu_P\|\leqslant 12r_K\enspace.
\]
%

\end{proof}

\section{Technical tools}
Before going to the iteration step of the algorithm, we need a series of results that allow to understand the algorithm.

Let $\cS_1$ denote the set of matrices $\bM\in \R^{d\times d}$ which are symmetric positive semi-definite and satisfy $\text{Tr}(\bM)=1$.
For any $\bM\in \cS_1$, denote by $\bM^{1/2}$ a symmetric, positive semi-definite square-root of $\bM$.
The following result is the main new insight from \cite{DepLec:2019} that allows to apply the machinery in \cite{Cheng:2019:HRM:3310435.3310606}.
It is a non trivial consequence of the deviation theorem on suprema of median-of-means processes, in its general version (see Theorem~\ref{thm:ConcSupMOMg}).

\begin{lemma}\label{lem:PrelimAlgo}
For any $\alpha\in(0,1)$, there exists a constant $C_{\alpha}$ such that, for all $K\geqslant 1/\alpha$, with probability larger than $1-e^{-K/C_{\alpha}}$, for any $\bM\in \cS_1$, there exist more than $(1-\alpha)K$ blocs $B_k$ satisfying
\[
\|\bM^{1/2}(P_{B_k}X-\mu_P)\|\leqslant C_\alpha\frac{\sqrt{\text{Tr}(\Sigma)}+\sqrt{\|\Sigma\|_{\text{op}}K}}{\sqrt{N}}\enspace.
\]
\end{lemma}
\begin{remark}
 Notice that $\cS_1$ contains all matrices of the form $\bM=\bv\bv^T$, with $\bv\in \bS$, hence Lemma~\ref{lem:PrelimAlgo} implies that, with probability larger than $1-e^{-K/C_{\alpha}}$, for any $\bv\in \bS$, there exist more than $(1-\alpha)K$ blocs $B_k$ satisfying
\[
[\bv^T(P_{B_k}X-\mu_P)]^2=\|\bM^{1/2}(P_{B_k}X-\mu_P)\|^2\leqslant C_\alpha\frac{\text{Tr}(\Sigma)+\|\Sigma\|_{\text{op}}K}{N}\enspace.
\]
It is therefore an extension of Corollary~\ref{Cor:ConcSupLin}.
\end{remark}
\begin{proof}
%
Let
\[
r=c_{\alpha}\frac{\sqrt{\text{Tr}(\Sigma)}+\sqrt{\|\Sigma\|_{\text{op}}K}}{\sqrt{N}}\enspace.
\]
Let $\beta\in(0,\overline{\Phi}(1))$ and let $\Omega$ denote the event where, for any $\bv\in \bS$, $|\cK_{\bv}|\geqslant (1-\beta \alpha)K$, where $\cK_{\bv}$ denotes the set of indices $k$ such that
\[
|\bv^T(P_{B_k}X-\mu_P)|\leqslant r\enspace.
\]
Assume that $c_{\alpha}$ is chosen such that $\P(\Omega)\geqslant 1-e^{-K/c_{\alpha}}$. 
This is possible thanks to Corollary~\ref{Cor:ConcSupLin}.

Fix $\bM\in\cS_1$, $a>0$ and let 
\[
\cA_{\bM}=\{k\in \{1,\ldots,K\}: \|\bM^{1/2}(P_{B_k}X-\mu_P)\|>ar\}\enspace.
\]
Suppose that $|\cA_{\bM}|\geqslant \alpha K$.
Let $b\in (1,a)$, let $G$ denote a Gaussian vector with covariance matrix $\bM$, independent from $X_1,\ldots,X_N$ and let
\[
Z=\sum_{k\in \{1,\ldots,K\}}{\bf 1}_{\{|G^T(P_{B_k}-\mu_P)|>br\}}\enspace.
\]
For any $k\in \{1,\ldots,K\}$, conditionally on $\cD_N$, $G^T(P_{B_k}-\mu_P)$ is a Gaussian random variable, centered, with variance $\sigma_k^2=\|M^{1/2}(P_{B_k}X-\mu_P)\|^2$.
It follows that, for any $k\in \cA_M$,
\[
\P(|G^T(P_{B_k}-\mu_P)|>br|\cD_N)\geqslant \P(|N|>b/a)\geqslant \overline{\Phi}(1)\enspace,
\]
where $N$ denote a standard Gaussian random variable.
Therefore, 
\[
\E[Z|\cD_N]\geqslant \overline{\Phi}(1)|\cA_M|\geqslant \overline{\Phi}(1)\alpha K\enspace.
\]
The Paley-Zygmund inequality grants that, for any $\theta\in[0,1]$, any non-negative random variable $Y$ with finite variance satisfies ({\bf Exercise:} Prove it!)
\[
\P(Y>\theta\E[Y])\geqslant (1-\theta)^2\frac{\E[Y]^2}{\E[Y^2]}\enspace.
\]
As $0\leqslant Z\leqslant K$ almost surely, $\E[Z^2|\cD_N]\leqslant K^2$, so, for $\theta=\beta/\overline{\Phi}(1)$,
\[
\P(Z>\beta\alpha K|\cD_N)\geqslant (1-\theta)^2\frac{(\alpha \overline{\Phi}(1))^2K^2}{\E[K^2]}=\big(\overline{\Phi}(1)-\beta\big)^2\alpha^2\enspace.
\]
The Gaussian concentration inequality implies also that, with probability larger than $1-e^{-x}$, 
\[
\|G\|\leqslant \E[\|G\|]+\sqrt{2\|\bM\|_{\text{op}}x}\enspace.
\]
As $\E[\|G\|]\leqslant \text{Tr}(\bM)\leqslant 1$ and $\|\bM\|_{\text{op}}\leqslant \text{Tr}(\bM)\leqslant 1$, this implies that 
\[
\P(\|G\|\leqslant 1+\sqrt{2x})\geqslant 1-e^{-x}\enspace.
\]
For any $x>-2\log\big[\alpha\big(\overline{\Phi}(1)-\beta\big)\big]$, if $|\cA_{\bM}|\geqslant \alpha K$, the event 
\[
\{Z>\beta\alpha K\}\cap\{(\|G\|\leqslant 1+\sqrt{2x}\}\ne \emptyset\enspace.
\]
Hence, if $|\cA_{\bM}|\geqslant \alpha K$, there exists a vector $g$ such that $\|g\|=3\log\big[\alpha\big(\overline{\Phi}(1)-\beta\big)\big]$ and $\beta\alpha K$ blocks such that $|g^T(P_{B_k}X-\mu_P)|>br$.
Fix $b=3\log\big[\alpha\big(\overline{\Phi}(1)-\beta\big)\big]$ and $a=2b$, the vector $\bv=g/b\in \bS$ satisfies $|\cK_{\bv}|<(1-\beta\alpha)K$ on $|\cA_{\bM}|\geqslant \alpha K$.
Therefore, the event $|\cA_{\bM}|\geqslant \alpha K$ is by definition contained in $\Omega^c$, it has probability smaller than $e^{-K/c_\alpha}$.
\end{proof}

Fix $C_\alpha$ as in Lemma~\ref{lem:PrelimAlgo} and in the remaining of this section, fix 
\[
r=C_\alpha\frac{\sqrt{\text{Tr}(\Sigma)}+\sqrt{\|\Sigma\|_{\text{op}}K}}{\sqrt{N}}\enspace.
\]
Let $\Omega_{\alpha}$ denote the event where there exist more than $(1-\alpha)K$ blocs $B_k$ satisfying
\[
\sup_{\bM\in\cS_1}\|\bM^{1/2}(P_{B_k}X-\mu_P)\|\leqslant r\enspace.
\]
The triangular inequality gives the following corollary of Lemma~\ref{lem:PrelimAlgo}.
\begin{corollary}\label{cor:DistmuPPBk}
 On $\Omega_\alpha$, for any $\bM\in \cS_1$, there are more than $(1-\alpha)K$ blocks such that, for any $\bx\in \R^d$,
 \[
 \|\bM^{1/2}(\mu_P-x)\|-r\leqslant  \|\bM^{1/2}(P_{B_k}X-\bx)\|\leqslant \|\bM^{1/2}(\mu_P-x)\|+r\enspace.
 \]
\end{corollary}

\section{Toward a convex relaxation}\label{sec:SDPDef}
The section introduces an optimization problem whose solutions are proved in Section~\ref{sec:ItStep} are used in the iteration step of the algorithm.
This problem is solved in Section~\ref{sec:SolvingSDP}.

For any $\bw\in \R^K$ and $\bx\in \R^d$, let 
\[
\widehat{\bM}_{\bw}(\bx)=\sum_{k=1}^Kw_k(P_{B_k}X-\bx)(P_{B_k}X-\bx)^T\enspace.
\]
Let 
\[
\Delta_K=\{w\in [0,1/[(1-\alpha)K]]^K: \sum_{k=1}^Kw_k=1\}\enspace. 
\]
Denote by 
\[
\text{OPT}(\bx)=\sup_{\bM\in \cS_1}\inf_{\bw\in \Delta_K}\text{Tr}(\bM\widehat{\bM}_{\bw}(\bx))\enspace.
\]
Let also $h_{\bx}$ denote the following function.
\[
h_{\bx}:\bM\mapsto\inf_{\bw\in \Delta_K}\text{Tr}(\bM\widehat{\bM}_{\bw}(\bx))
\]
Let $I\subset \{1,\ldots,K\}$ denote the set of indices $k$ such that $(P_{B_k}X-\bx)^T\bM(P_{B_k}X-\bx)$ is one of the $(1-\alpha)K$ smallest values among the $((P_{B_j}X-\bx)^T\bM(P_{B_j}X-\bx))_{j\in \{1,\ldots,K\}}$.
The infimum in the definition of $h_{\bx}(\bM)$ is achieved (it is therefore a minimum) by the vector $\bw$ such that 
\begin{equation}\label{eq:defbwOpt}
w_k=
\begin{cases}
 1/[(1-\alpha)K]&\text{ if } k\in I\enspace,\\
 0&\text{ otherwise}\enspace.
\end{cases}
\end{equation}
The following lemma bounds $\text{OPT}(\bx)$ when $\bx$ is far from $\mu_P$.

\begin{lemma}\label{lem:BoundOPTx}
On $\Omega_{\alpha}$, for any $\bx\in \R^d$,
\[
\text{OPT}(\bx)\leqslant (\|\bx-\mu_P\|+r)^2\enspace.
\] 
Moreover, for any $\bx\in \R^d$ such that $\|\bx-\mu_P\|\geqslant r$, 
\[
\text{OPT}(\bx)\geqslant  \frac{1-2\alpha}{1-\alpha}(\|\bx-\mu_P\|-r)^2\enspace.
\]
\end{lemma}
\begin{proof}
 Fix $\bM\in \cS_1$ and $\bx\in\R^d$. 
 Let 
 \[
 \cK_\bM=\{k\in\{1,\ldots,K\}: \|\bM^{1/2}(P_{B_k}X-\mu_P)\|\leqslant r\}\enspace.
 \] 
 On $\Omega_{\alpha}$, $|\cK_\bM|\geqslant (1-\alpha)K$.
 By the triangular inequality, for any $k\in \cK_\bM$, 
\begin{equation}\label{eq:BoundMatTransf}
   \|\bM^{1/2}(\mu_P-\bx)\|-r\leqslant  \|\bM^{1/2}(P_{B_k}X-\bx)\|\leqslant \|\bM^{1/2}(\mu_P-\bx)\|+r\enspace.
\end{equation}
 Define the vector $\bw\in \R^K$ as follows:
 \[
w_k=
\begin{cases}
  \frac1{|\cK_\bM|}&\text{ if }k\in \cK_\bM\enspace,\\
  0&\text{ otherwise}\enspace.
\end{cases}
 \]
On $\Omega_{\alpha}$, $\bw\in \Delta_K$, so, by definition of $h_{\bx}$,
\begin{align*}
h_{\bx}(\bM)&\leqslant \text{Tr}(\bM\widehat{\bM}_{\bw}(\bx))=\sum_{k=1}^Kw_k(P_{B_k}X-\bx)^T\bM(P_{B_k}X-\bx)\\
&=\frac1{|\cK_\bM|}\sum_{k\in\cK_\bM}\|\bM^{1/2}(P_{B_k}X-\bx)\|^2\enspace.
\end{align*}
By \eqref{eq:BoundMatTransf}, this implies that
\begin{equation}\label{eq:UBh}
 h_{\bx}(\bM)\leqslant ( \|\bM^{1/2}(\mu_P-\bx)\|+r)^2\enspace.
\end{equation}
Taking the supremum over all $\bM\in\cS_1$ in this inequality shows that
\[
\text{OPT}(\bx)\leqslant (\|\bx-\mu_P\|+r)^2\enspace.
\]
Let now $\bx\in \R^d$ such that $\|\bx-\mu_P\|>r$. 
Fix also $\bM\in\cS_1$ and define $I$ as in the definition of the optimal weights $\bw$ in \eqref{eq:defbwOpt}.
On $\Omega_\alpha$, both $|I|$ and $|\cK_\bM|$ are larger than $(1-\alpha)K$, so $|I\cap\cK_\bM|\geqslant (1-2\alpha)K$, so
\begin{align*}
 h_{\bx}(\bM)&=\frac1{(1-\alpha)K}\sum_{k\in I}\|\bM^{1/2}(P_{B_k}X-\bx)\|^2\\
 &\geqslant \frac1{(1-\alpha)K}\sum_{k\in I\cap\cK_M}\|\bM^{1/2}(P_{B_k}X-\bx)\|^2\enspace.
\end{align*}
By \eqref{eq:BoundMatTransf}, this implies that
\[
 h_{\bx}(\bM)\geqslant \frac{1-2\alpha}{1-\alpha}(\|\bM^{1/2}(\mu_P-\bx)\|-r)^2\enspace.
\]
Taking the supremum over all $\bM\in\cS_1$ in this inequality shows that
\[
\text{OPT}(\bx)\geqslant \frac{1-2\alpha}{1-\alpha}(\|\bx-\mu_P\|-r)^2\enspace.
\]
\end{proof}

\begin{lemma}\label{Lem:LinkEigen}
Let $\beta\in[1/\sqrt{2},1]$.
On $\Omega_\alpha$, for any $\bM\in \cS_1$ such that $h_{\bx}(\bM)\geqslant (\beta\|\bx-\mu_P\|+r)^2$, the (normalized) top eigenvector $\bv$ of $\bM$ satisfies
 \[
 \bigg|\frac{\bv^T(\bx-\mu_P)}{\|\bx-\mu_P\|}\bigg|>\sqrt{2\beta^2-1}\enspace.
 \]
\end{lemma}
\begin{proof}
 Let $\bM$ satisfying the assumptions of the lemma.
 From \eqref{eq:UBh}, 
  \[
  h_{\bx}(\bM)\leqslant ( \|\bM^{1/2}(\mu_P-\bx)\|+r)^2\enspace.
  \]
  Let $\bu=(\bx-\mu_P)/\|\bx-\mu_P\|$. 
  This implies that
  \[
  \beta\|\bx-\mu_P\|\leqslant \|\bM^{1/2}(\mu_P-\bx)\|,\qquad\text{i.e.}\qquad \|\bM\|_{\text{op}}\geqslant \bu^T\bM\bu\geqslant \beta^2\enspace.
  \]
  Moreover, $\bu-(\bu^T\bv)\bv$ and $\bM[\bu-(\bu^T\bv)\bv]$ are orthogonal to $\bv$, hence, 
  \[
 \bu^T\bM\bu=(\bu^T\bv)^2\bv^T\bM\bv+[\bu-(\bu^T\bv)\bv]^T\bM[\bu-(\bu^T\bv)\bv]\enspace.
  \]
  First, $\bv^T\bM\bv=\|\bM\|_{\text{op}}\leqslant 1$. 
  Second, as $\bu-(\bu^T\bv)\bv$ is orthogonal to $\bv$, $[\bu-(\bu^T\bv)\bv]^T\bM[\bu-(\bu^T\bv)\bv]$ does not exceed the second largest eigenvalue $\lambda$ of $\bM$. As $\lambda+\|\bM\|_{\text{op}}\leqslant \text{Tr}(\bM)\leqslant 1$, it follows that 
\[
[\bu-(\bu^T\bv)\bv]^T\bM[\bu-(\bu^T\bv)\bv]\leqslant 1-\|\bM\|_{\text{op}}\leqslant 1-\beta^2\enspace.
\]
Hence, 
\[
\beta^2\leqslant (\bu^T\bv)^2+1-\beta^2\enspace,
\]
which proves Lemma~\ref{Lem:LinkEigen}.
%
\end{proof}

\section{The iteration step}\label{sec:ItStep}
We are now in position to show that a solution of the optimization problem defined in Section~\ref{sec:SDPDef} can be used to define the iteration step of our algorithm.

Given $\bx\in \R^d$, assume that we are given an approximation $\widehat{\bM}_{\bx}$ of 
\[
\widehat{\bM}_*\in\argmax_{\bM\in \cS_1}\inf_{\bw\in \Delta_K}\text{Tr}(\bM\widehat{\bM}_{\bw}(\bx))\enspace.
\]
This approximation should satisfy the following requirement:
There exists an absolute constant $A$ such that, if $\|\bx-\mu_P\|\geqslant Ar$, then, there exists $\beta\geqslant 0.8$ such that
\[
h_{\bx}(\widehat{\bM}_{\bx})\geqslant (\beta\|\bx-\mu_P\|+r)^2\enspace.
\]
Let then $\widehat{\bv}_{\bx}\in\bS$ denote a top eigenvector of $\widehat{\bM}_{\bx}$ and let 
\[
\theta_{\bx}=-\text{median}(\widehat{\bv}_{\bx}^T(P_{B_k}X-\bx), k=1,\ldots,K)\enspace.
\]
The algorithm moves at each iteration from $\bx$ to $T(\bx)$, where,
\begin{equation}\label{eq:DefItStep}
T(\bx)=\bx-\theta_{\bx} \widehat{\bv}_{\bx}\enspace. 
\end{equation}
\begin{remark}
 The vector $\bv_x$ is defined up to its sign.
 Whatever this sign, the function $T$ is well defined.
\end{remark}
\begin{proposition}\label{prop:udate}
If $\alpha\leqslant 1/2$, on $\Omega_\alpha$, 
 \[
 \|T(\bx)-\mu_P\|^2\leqslant \frac34\|\bx-\mu_P\|^2+(A^2+1)r^2\enspace.
 \]
 \end{proposition} 
\begin{proof}
The idea of the proof is the following decomposition of the distance between $T(\bx)$ and $\mu_P$.
Let $\bv=(\mu_P-\bx)/\|\mu_P-\bx\|$ denote the optimal descent direction and decompose $\bv=a\widehat{\bv}_{\bx}+b\widehat{\bv}_{\bx}^{\perp}$, where $\widehat{\bv}_{\bx}^{\perp}\in \bS\cap \{\widehat{\bv}_{\bx}\}^{\perp}$ and, therefore $a^2+b^2=1$.
We have, by Pythagoras relation
\begin{align}
\notag \|T(\bx)-\mu_P\|^2&= \|\bx-\mu_P-\theta_{\bx} \widehat{\bv}_{\bx}\|^2\\
\notag &=\|\|\mu_P-\bx\|(a\widehat{\bv}_{\bx}+b\widehat{\bv}_{\bx}^{\perp})-\theta_{\bx} \widehat{\bv}_{\bx}\|^2\\
\notag &=\|(a\|\mu_P-\bx\|-\theta_{\bx})\widehat{\bv}_{\bx}+b\|\mu_P-\bx\|\widehat{\bv}_{\bx}^{\perp}\|^2\\
\notag &=(a\|\mu_P-\bx\|-\theta_{\bx})^2+b^2\|\mu_P-\bx\|^2\enspace.
\end{align}
Since $a\|\mu_P-\bx\|=(\bx-\mu_P)^T\widehat{\bv}_{\bx}$ and $b=\bv^T\widehat{\bv}_{\bx}^{\perp}$, this relation can be written
\begin{align}
\label{eq:DecInc}  \|T(\bx)-\mu_P\|^2&=[(\bx-\mu_P)^T\widehat{\bv}_{\bx}-\theta_{\bx}]^2+[\bv^T\widehat{\bv}_{\bx}^{\perp}]^2\|\bx-\mu_P\|^2\enspace.
\end{align}
We bound separately each term in this decomposition.
 On $\Omega_{\alpha}$, there are more than $(1-\alpha)K$ blocks such that
 \[
 |\widehat{\bv}_{\bx}^T(P_{B_k}X-\mu_P)|\leqslant r\enspace.
 \]
On the same blocks
\[
|\widehat{\bv}_{\bx}^T(P_{B_k}X-\bx)-\widehat{\bv}_{\bx}^T(\mu_P-\bx)|\leqslant r\enspace.
\]
Therefore, if $\alpha<1/2$, on $\Omega_\alpha$,
\begin{equation}\label{eq:BoundFirst}
| \theta_{\bx}-\widehat{\bv}_{\bx}^T(\mu_P-\bx)|\leqslant r\enspace.
\end{equation}
If $\|\bx-\mu_P\|>Ar$, there exists $\beta\geqslant 0.8$ such that
\[
h_{\bx}(\widehat{\bM}_{\bx})\geqslant (\beta\|\bx-\mu_P\|+r)^2\enspace.
\]
By Lemma~\ref{Lem:LinkEigen}, this implies that, on $\Omega_{\alpha}$, $|\bv^T\widehat{\bv}_{\bx}|\geqslant \sqrt{2\beta^2-1}$, so $|a|\geqslant \sqrt{2\beta^2-1}$ and $b^2=1-a^2\leqslant 2(1-\beta^2)\leqslant 3/4$.
Plugging this inequality and \eqref{eq:BoundFirst} into \eqref{eq:DecInc} yields the result when $\|\bx-\mu_P\|>Ar$.

If $\|\bx-\mu_P\|\leqslant Ar$, as $|\bv_x^T\bv|\leqslant 1$, from \eqref{eq:BoundFirst} and \eqref{eq:DecInc},
\[
\|T(\bx)-\mu_P\|^2\leqslant (A^2+1)r^2\enspace.
\]
This proves the result when $\|\bx-\mu_P\|\leqslant Ar$.
\end{proof}

\section{Computation of $\widehat{\bM}_{\bx}$.}\label{sec:SolvingSDP}
It remains to compute an approximation $\widehat{\bM}_{\bx}$ of 
\[
\widehat{\bM}_*\in\argmax_{\bM\in \cS_1}\inf_{\bw\in \Delta_K}\text{Tr}(\bM\widehat{\bM}_{\bw}(\bx))=\argmax_{\bM\in \cS_1}h_{\bx}(\bM)\enspace,
\]
satisfying the requirement that there exists an absolute constant $A$ such that, if $\|\bx-\mu_P\|\geqslant Ar$, then, there exists $\beta\geqslant 0.8$ such that
\[
h_{\bx}(\widehat{\bM}_{\bx})\geqslant (\beta\|\bx-\mu_P\|+r)^2\enspace.
\]
We proceed in several steps.
Section~\ref{sec:FirstEqOptProb} presents an equivalent convex problem and Section~\ref{sec:ApproxPb} presents a convex problem whose solutions approximate those of the equivalent problem in the desired way.
This approximating problem is solved using the algorithms of \cite{Peng:2012:FSW:2312005.2312026, Cheng:2019:HRM:3310435.3310606} in Section~\ref{Sec:SolLinTime}.

The approximating depends on a parameter $\rho$ that should be carefully chosen. 
Section~\ref{ref:Lemprep} gathers technical lemmas that are used in Section~\ref{sec:CalAlgoSDP} to calibrate $\rho$ so as to ensure that the solution of the approximating satisfies the requirement that there exists an absolute constant $A$ such that, if $\|\bx-\mu_P\|\geqslant Ar$, then, there exists $\beta\geqslant 0.8$ such that
\[
h_{\bx}(\widehat{\bM}_{\bx})\geqslant (\beta\|\bx-\mu_P\|+r)^2\enspace.
\]

\subsection{An equivalent problem}\label{sec:FirstEqOptProb}
Consider the following convex maximization problem. 
The set of constraint $\cC$ is the set of triplets $z\geqslant 0$, $\by\in \R_+^K$ and $\bM\in \cS_1$ satisfying, for any $k\in\{1,\ldots,K\}$, $(P_{B_k}X-\mu_P)^T\bM(P_{B_k}X-\mu_P)+y_k\geqslant z$.
The goal is to find $(z,\by, \bM)\in \cC$ maximizing the objective function 
\[
z-\|\by\|_1/[(1-\alpha)K]\enspace.
\]
The link with $\widehat{\bM}_*$ is provided by the following lemma.

\begin{lemma}\label{lem:NewObj}
Fix $\bM\in \cS_1$ and define $\cC_\bM$, the set of couples $z\geqslant 0$ and $\by\in \R_+^K$ such that $(z,\by,\bM)\in \cC$ that is, such that, for any $k\in\{1,\ldots,K\}$, $(P_{B_k}X-\mu_P)^T\bM(P_{B_k}X-\mu_P)+y_k\geqslant z$.
Then
\[
 \max_{(z,y)\in \cC_{\bM}}\{z-\|\by\|_1/[(1-\alpha)K]\}=h_{\bx}(\bM)\enspace.
 \]
This maximal value is achieved when $z=z_{\bM}$, the $(1-\alpha)K$ largest values among the set $\{(P_{B_k}X-\mu_P)^T\bM(P_{B_k}X-\mu_P),k=\{1,\ldots,K\}\}$ and $\by=\by_{\bM}$, where the coordinates of $\by_{\bM}$ are defined by
\[
y_k=\big(z_{\bM}-(P_{B_k}X-\mu_P)^T\bM(P_{B_k}X-\mu_P)\big)_+\enspace.
\]
\end{lemma}
\begin{proof}
 To compute $\max_{(z,y)\in \cC_{\bM}}\{z-\|\by\|_1/[(1-\alpha)K]\}$, each $y_k\geqslant 0$ should be chosen as small as possible, thus the constraints imply that the maximum in $\by$ is achieved, for each given $z$, by 
\[
y_k=\big(z-(P_{B_k}X-\mu_P)^T\bM(P_{B_k}X-\mu_P)\big)_+\enspace.
\]
For this value of $\by$, one gets
\[
z-\frac{\|\by\|_1}{(1-\alpha)K}=z-\frac1{(1-\alpha)K}\sum_{k=1}^K\big(z-(P_{B_k}X-\mu_P)^T\bM(P_{B_k}X-\mu_P)\big)_+\enspace.
\]
Assume that $(P_{B_k}X-\mu_P)^T\bM(P_{B_k}X-\mu_P)$ are arranged in non-decreasing order. Then, if $z\in [(P_{B_k}X-\mu_P)^T\bM(P_{B_k}X-\mu_P),(P_{B_{k+1}}X-\mu_P)^T\bM(P_{B_{k+1}}X-\mu_P))$,
\[
z-\frac{\|\by\|_1}{(1-\alpha)K}=z\bigg(1-\frac{k}{(1-\alpha)K}\bigg)+\frac1{(1-\alpha)K}\sum_{j=1}^k(P_{B_j}X-\mu_P)^T\bM(P_{B_j}X-\mu_P)\enspace.
\]
This function of $k$ is 
\[
\begin{cases}
 \text{non-decreasing on the interval }[0,(P_{B_{(1-\alpha)K}}X-\mu_P)^T\bM(P_{B_{(1-\alpha)K}}X-\mu_P)]\enspace,\\
 \text{non-increasing on }[(P_{B_{(1-\alpha)K}}X-\mu_P)^T\bM(P_{B_{(1-\alpha)K}}X-\mu_P),+\infty)\enspace.
\end{cases}
\]
It maximal value is achieved for $z=z_{\bM}$ and is equal to 
\[
\frac1{(1-\alpha)K}\sum_{j=1}^{(1-\alpha)K}(P_{B_j}X-\mu_P)^T\bM(P_{B_j}X-\mu_P)=h_{\bx}(\bM)\enspace.
\]
\end{proof}

A first consequence of Lemma~\ref{lem:NewObj} is that $\widehat{\bM}_*$ satisfies 
\begin{equation}\label{eq:Link}
 (z_{\widehat{\bM}_*},\by_{\widehat{\bM}_*},\widehat{\bM}_*)\in \argmax_{(z,\by,\bM)\in \cC}\{z-\|\by\|_1/[(1-\alpha)K]\}\enspace.
\end{equation}

\subsection{An approximating problem.}\label{sec:ApproxPb}
Consider now the following convex optimization problem. 
Let $\rho>0$ and define the constraint set $\cC_\rho$ as the set of couples $\bM\in\R^{K\times K},\by\in \R^K$ where $\bM\succeq 0$ and, for any $k\in\{1,\ldots,K\}$,
\[
y_k\geqslant 0,\qquad \rho(P_{B_k}X-\bx)^T\bM(P_{B_k}X-\bx)+(1-\alpha)Ky_k\geqslant 1\enspace.
\]
The problem is to find a minimizer on the constraint set $\cC_\rho$ of the objective function
\[
\text{Tr}(\bM)+\|\by\|_1\enspace,
\]
A useful link between this problem and the one of Section~\ref{sec:FirstEqOptProb} is provided in the following lemma.
\begin{lemma}\label{lem:gval}
If we have built $(\bM,\by)\in \cC_\rho$ such that
\[
\text{Tr}(\bM)+\|\by\|_1\leqslant 1\enspace,
\]
then one can build in $O(K)$ operations $(z,\by',\bM')\in \cC$ such that
 \[
 z-\frac{\|\by'\|_1}{(1-\alpha)K}\geqslant \frac1{\rho}\enspace.
 \]
 Conversely, if we have built $(z,\by',\bM')\in \cC$ such that
 \[
 z-\frac{\|\by'\|_1}{(1-\alpha)K}\geqslant \frac1{\rho}\enspace,
 \]
then one can build in $O(1)$ operations $(\bM,\by)\in \cC_\rho$ such that
\[
\text{Tr}(\bM)+\|\by\|_1\leqslant 1\enspace.
\]
In both cases, the top eigenvectors of the matrices $\bM$ and $\bM'$ are equal.
\end{lemma}
\begin{proof}
Suppose that we have built $(\bM,\by)\in \cC_\rho$ such that
\[
\text{Tr}(\bM)+\|\by\|_1\leqslant 1\enspace.
\]
Then, one can define
\[
\bM'=\frac{\bM}{\text{Tr}(\bM)},\qquad \by'=\frac{(1-\alpha)K}{\rho\text{Tr}(\bM)}\by\enspace.
\]
Then $\bM'\in \cS_1$ and, for all $k\in \{1,\ldots,K\}$,
\[
y_k'\geqslant 0,\qquad (P_{B_k}X-\bx)^T\bM'(P_{B_k}X-\bx)+y'_k\geqslant \frac{1}{\rho\text{Tr}(\bM)}\enspace.
\]
Therefore, $(z, \by')\in \cC_{\bM'}$, where $z=1/(\rho\text{Tr}(\bM))$ and
\[
z-\frac{\|\by'\|_1}{(1-\alpha)K}=\frac{1-\|\by\|_1}{\rho\text{Tr}(\bM)}\geqslant \frac1{\rho}\enspace.
\]
Conversely, if we have found $\bM$ such that $z_{\bM}-\|\by_{\bM}\|_1/[(1-\alpha)K]\geqslant 1/\rho$, one can define
\[
\bM'=\frac{1}{\rho z_{\bM}} \bM,\qquad \by'=\frac{1}{(1-\alpha)Kz_{\bM}}\by_{\bM}\enspace.
\]
We have $\bM'\succeq 0$, and, for all $k\in\{1,\ldots,K\}$, 
\[
\rho(P_{B_k}X-\bx)^T\frac{1}{\rho z_{\bM}}\bM(P_{B_k}X-\bx)+(1-\alpha)K\frac{y_k}{(1-\alpha)Kz_{\bM}}\geqslant 1\enspace,
\]
that is $(\bM',\by')\in \cC_\rho$.
Moreover, 
\[
\text{Tr}(\bM')+\|\by'\|_1=\frac{1}{\rho z_{\bM}}+\frac{1}{(1-\alpha)Kz_{\bM}}\|\by_{\bM}\|_1\leqslant \frac{1}{\rho z_{\bM}}+\frac{z_{\bM}-1/\rho}{z_{\bM}}=1\enspace.
\]
\end{proof}

\subsection{Solving the approximating problem in nearly linear time}\label{Sec:SolLinTime}
The following Lemma comes from \cite{Peng:2012:FSW:2312005.2312026}, see also \cite[Section 4]{Cheng:2019:HRM:3310435.3310606}.

\begin{lemma}\label{lem:Peng}
For every $\rho>0$ and $\eta>0$, there exists an algorithm $\cA:(\R^d)^n\times [0,1]\to \cC_\rho$ such that, if $U$ is a uniform random variable on $[0,1]$ independent of $\cD_N$ and $\cA(\cD_N,U)=(\bM,\by)$,
\[
 \P(\text{Tr}(\bM)+\|\by\|_1\leqslant (1+\eta)g(\rho)|\cD_N)\geqslant 1-\frac1{e}\enspace.
\]
Moreover, $\cA(\cD_N,U)$ can be evaluated in $O(Kd)$ operations and a top eigenvector of $\bM$ can be computed in $\tilde{O}(Kd)$ operations.
\end{lemma}
A consequence of Lemma~\ref{lem:Peng} is the following result.
\begin{lemma}\label{lem:LBgrho}
 For every $\rho>0$, $\eta>0$, $R>0$, every positive integer $u$ and every $\bx$ in $\R^d$, one can build in $O((u+\log R)Kd)$ operations $\text{Alg}(\cD_N,\rho,\eta,u,\bx)=(\bM,\by)\in \cC_\rho$ such that
 \[
 \P(\text{Tr}(\bM)+\|\by\|_1\leqslant (1+\eta)g(\rho))\geqslant 1-\frac{e^{-u}}{R}\enspace.
 \]
\end{lemma}
\begin{proof}
Let $t=u+\log R$ and let $U_1,\ldots,U_t$ denote independent random variables with uniform distribution on $[0,1]$, independent from $\cD_N$.
Let $\cA(\cD_N,U_1)=(\bM_1,\by_1), \ldots,\cA(\cD_N,U_t)=(\bM_t,\by_t)$ denote the random variables built with the algorithm $\cA$ of Lemma~\ref{lem:Peng} and let $\hat{t}$ such that $\text{Tr}(\bM_{\hat{t}})+\|\by_{\hat{t}}\|_1$ is minimal.
As the random variables $\cA(\cD_N,U_1),\ldots,\cA(\cD_N,U_t)$ are i.i.d. conditionally on $\cD_N$, 
\[
\P(\text{Tr}(\bM_{\hat{t}})+\|\by_{\hat{t}}\|_1>(1+\eta)g(\rho))=(\P(\text{Tr}(\bM_{1})+\|\by_{1}\|_1>(1+\eta)g(\rho)))^t\enspace.
\]
The result then follows from Lemma~\ref{lem:Peng}.
\end{proof}

\subsection{The optimal solution of the approximating problem}\label{ref:Lemprep}
For any $\rho>0$, let 
\[
g(\rho)=\min_{(\bM,\by)\in \cC_\rho}\{\text{Tr}(\bM)+\|\by\|_1\}\enspace.
\]
\begin{lemma}\label{lem:Basicsg}
 For any $\rho'>\rho>0$, 
 \[
 g(\rho)\geqslant g(\rho')\geqslant \frac{\rho}{\rho'}g(\rho)\enspace.
 \]
\end{lemma}
\begin{proof}
 Clearly, $\cC_{\rho}\subset\cC_{\rho'}$, so $g(\rho)\geqslant g(\rho')$.
 Moreover, if $(\bM,\by)\in \cC_{\rho'}$, and $r=\rho'/\rho$, $(r\bM,r\by)\in \cC_\rho$, so $g(\rho')\geqslant rg(\rho)$.
\end{proof}

It follows from Lemma~\ref{lem:Basicsg} that $g$ is non-increasing and continuous.
Moreover, from Lemma~\ref{lem:gval}, it satisfies $g(\rho) \leqslant 1$ iff $1/\text{OPT}(\bx)\geqslant \rho$, so
\[
g(1/\text{OPT}(\bx))=1\enspace.
\]

\begin{lemma}\label{lem:UBgrho}
 On the event $\Omega_{\alpha}$, for all $\bx\in \R^d$ such that $\|\bx-\mu_P\|>r$,
 \[
 g(\rho)\leqslant \frac1{\rho\text{Opt}(\bx)}+\frac{9(\|\bx-\mu_P\|+r)^2}{8(\|\bx-\mu_P\|-r)^2}-1\enspace.
 \]
\end{lemma}
\begin{proof}
For any $\nu>0$, there exists $\bM_0\in\cS_1$ such that $z_{\bM_0},\by_{\bM_0}$ (as defined in Lemma~\ref{lem:NewObj}) satisfy
\begin{align*}
  z_{\bM_0}-\frac{\|\by_{\bM_0}\|_1}{(1-\alpha)K}&>\sup_{(z,\by,\bM)\in \cC}\bigg\{z-\frac{\|\by\|_1}{(1-\alpha)K}\bigg\}-\nu\\
  &=\sup_{\bM\in \cS_1}\sup_{(z,\by)\in \cC_{\bM}}\bigg\{z-\frac{\|\by\|_1}{(1-\alpha)K}\bigg\}-\nu\enspace.
\end{align*}
 By Lemma~\ref{lem:NewObj},
 \[
 \sup_{(z,\by)\in \cC_{\bM}}\bigg\{z-\frac{\|\by\|_1}{(1-\alpha)K}\bigg\}=h_{\bx}(\bM)\enspace.
 \]
Therefore,
 \[
 z_{\bM_0}-\frac{\|\by_{\bM_0}\|_1}{(1-\alpha)K}>\sup_{\bM\in \cS_1}h_{\bx}(\bM)-\nu=\text{OPT}(\bx)-\nu\enspace.
 \]
 Since $\|\bx-\mu_P\|>r$, from Lemma~\ref{lem:BoundOPTx}, on $\Omega_\alpha$,
 \[
\text{OPT}(\bx)\geqslant  \frac{1-2\alpha}{1-\alpha}(\|\bx-\mu_P\|-r)^2\enspace.
\]
From Lemma~\ref{lem:NewObj}, $z_{\bM_0}$ is the $(1-\alpha)K$ largest value in the set $\{(P_{B_k}X-\mu_P)^T\bM(P_{B_k}X-\mu_P),k=\{1,\ldots,K\}\}$.
It follows from Corollary~\ref{cor:DistmuPPBk} that, on $\Omega_\alpha$, 
\[
z_{\bM_0}\leqslant(\|\bM_0^{1/2}(\bx-\mu_P)\|+r)^2\leqslant (\|\bx-\mu_P\|+r)^2\enspace.
\]
Define
\[
\bM'=\frac1{\rho z_{\bM_0}}\bM_0,\qquad \by'=\frac1{(1-\alpha)Kz_{\bM_0}}\by_{\bM_0}\enspace.
\]
We have proved that
\begin{align*}
g(\rho)&\leqslant\text{Tr}(\bM')+\|\by'\|_1=\frac1{\rho z_{\bM_0}}+\frac{z_{\bM_0}+\nu-\text{Opt}(\bx)}{z_{\bM_0}}\\
&\leqslant\frac1{\rho(\text{Opt}(\bx)-\nu)}+\frac{\nu+\text{Opt}(\bx)\bigg(\frac{(1-\alpha)\|\bx-\mu_P\|+r)^2}{(1-2\alpha)\|\bx-\mu_P\|-r)^2}-1\bigg)}{\text{Opt}(\bx)-\nu}\enspace.
\end{align*}
As the result holds for any $\nu>0$, this concludes the proof.
\end{proof}

\subsection{Calibration of the approximating algorithm}\label{sec:CalAlgoSDP}
Assume that $\|\bx-\mu_P\|>Ar$.
On $\Omega_{\alpha}$, by Lemma~\ref{lem:UBgrho},
\begin{equation}\label{eq:UBg}
g(\rho)\leqslant \frac1{\rho\text{Opt}(\bx)}+\frac{9(A+1)^2}{8(A-1)^2}-1\leqslant \frac1{\rho\text{Opt}(\bx)}+b\enspace. 
\end{equation}
Here $b=2\alpha/(1-2\alpha)$.
The last inequality holds for any $A\geqslant A_\alpha$.

\begin{lemma}\label{lem:FirstRed}
 Fix $\epsilon\in (2\alpha/(1-\alpha),0.4(1-\alpha)/(1-2\alpha))$. Assume that $\rho$ satisfies $g(\rho)\geqslant 1-\epsilon+b$.
There exist constants $A_\alpha$ and $\beta>0.8$ such that, on $\Omega_\alpha$, for any $\bx\in \R^d$ satisfying $\|\bx-\mu_P\|>A_\alpha r$,
\[
 h_{\bx}(\bM')\geqslant (\beta\|\bx-\mu_P\|^2+r)^2\enspace.
\]
\end{lemma}
\begin{proof}
For any $\rho$ such that $g(\rho)\geqslant 1-\epsilon+b$, it follows from \eqref{eq:UBg} that 
\[
\frac1{\rho}\geqslant (1-\epsilon)\text{Opt}(\bx)\enspace.
\]
In this case, by Lemma~\ref{lem:gval}, one can build in $O(K)$ operations $(z,\by',\bM')\in \cC$ such that
 \[
 z-\frac{\|\by'\|_1}{(1-\alpha)K}\geqslant \frac1{\rho}\geqslant (1-\epsilon)\text{Opt}(\bx)\enspace.
 \]
 By Lemma~\ref{lem:NewObj}, the matrix $\bM'$ satisfies
 \[
 h_{\bx}(\bM')\geqslant z-\frac{\|\by'\|_1}{(1-\alpha)K}\geqslant (1-\epsilon)\text{Opt}(\bx)\enspace.
 \]
By Lemma~\ref{lem:BoundOPTx}, it follows that, on $\Omega_\alpha$,
\[
 h_{\bx}(\bM')\geqslant (1-\epsilon)\frac{1-2\alpha}{1-\alpha}(\|\bx-\mu_P\|-r)^2\enspace.
\]
The condition on $\epsilon$ implies the result provided that $A$ is large enough.
\end{proof}

It remains to find $\rho$ such that $g(\rho)\geqslant 1-\epsilon+b$.
To do this, we fix $\epsilon>b$, $\eta>0$ and build $\rho$ and $u$ such that the algorithm of Lemma~\ref{lem:LBgrho} outputs $\text{Alg}(\cD_N,\rho,\eta,u,\bx)=(\bM,\by)$ satisfying
\[
(1+\eta)(1-\epsilon+b)\leqslant \text{Tr}(\bM)+\|\by\|_1\leqslant (1+\eta)g(\rho)\leqslant 1\enspace.
\]

As $\rho\to\infty$, by Lemma~\ref{lem:UBgrho}, $g(\rho)\to b$.
If $\alpha<1/4$, $b<1$, so $g(\rho)<1$ for any $\rho$ large enough.
Fix
\[
\eta\in \bigg(0,\frac{1-4\alpha}{2\alpha}\wedge\frac{\alpha}{1-\alpha}\wedge\frac16\bigg)=\bigg(0,\bigg(\frac1b-1\bigg)\wedge\frac{\alpha}{1-\alpha}\wedge\frac16\bigg)\enspace.
\]
In particular, $1/(1+\eta)>b$.
Then, there exists $\rho_0$ such that $g(\rho_0)<1/(1+\eta)$.
For this value of $\rho_0$, by Lemma~\ref{lem:LBgrho}, the algorithm defined in this lemma outputs $\text{Alg}(\cD_N,\rho_0,\eta,u,\bx)=(\bM,\by)$ satisfying
\[
\P(\text{Tr}(\bM)+\|\by\|_1\leqslant 1|\cD_N)\geqslant 1-\frac{e^{-u}}R\enspace.
\]

On the other hand, when $\rho=0$, the minimal value of $\text{Tr}(\bM)+\|\by\|_1$ is achieved by the couple $(\bM,\by)\in \cC_{0}$ such that $\bM=0$ and $\by$ is the vector with all coordinates equal to $1/[(1-\alpha)K]$.
It follows that $g(0)=1/(1-\alpha)>1$.

\begin{lemma}\label{lem:BinarySurch}
For any fixed constant $\nu\in (6\eta,1)$, it is possible to build in $O(T(u+\log R)Kd)$ operations, where $T=O(\log (\rho_0))$, a couple $(\bM,\by)$ satisfying 
\[
\P(1-\nu\leqslant \text{Tr}(\bM)+\|\by\|_1\leqslant1\wedge (1+\eta)g(\eta)|\cD_N)\geqslant 1-\frac{Te^{-u}}{R}\enspace.
\] 
If $R=\tau T$, this gives that, for any fixed constant $\nu\in (5\eta,1)$, it is possible to build in $O(\log (\rho_0)(u+\log(\tau)+\log \log (\rho_0))K)$ operations, a couple $(\bM,\by)$ satisfying 
\[
\P(1-\nu\leqslant \text{Tr}(\bM)+\|\by\|_1\leqslant1\wedge (1+\eta)g(\eta)|\cD_N)\geqslant 1-\frac{e^{-u}}{\tau }\enspace.
\] 
\end{lemma}
\begin{proof}
Fix $\nu_0$ and $\nu_1$ such that $(1-\nu_0)>(1+\eta)(1-\nu_1)$ and
\[
\frac{2\eta}{1+\eta}< \nu_0<\nu_1<\frac{\nu-\eta+\nu\eta}{2}\enspace.
\]
This is possible since $\nu>5\eta$.
$\nu_0<\nu_1$ are chosen such that
\[
\frac{1-2\nu_1}{1+\eta}\geqslant 1-\nu\quad\text{and}\quad(1+\eta)\bigg(1-\frac{\nu_0}{2}\bigg)\leqslant 1\enspace.
\]
Moreover, as $(1-\nu_0)>(1+\eta)(1-\nu_1)$, we have, for $\rho$ satisfying $g(\rho)=1-\nu_1$ (which exists by continuity of $g$, $1-\nu_1=g(\rho)\leqslant (1+\eta)g(\rho)<1-\nu_0$. 
Hence, by Lemma~\ref{lem:LBgrho}, the algorithm defined in this lemma outputs $\text{Alg}(\cD_N,\rho,\eta,u,\bx)=(\bM,\by)$ satisfying
\[
\P(1-\nu_1\leqslant \text{Tr}(\bM)+\|\by\|_1\leqslant 1-\nu_0|\cD_N)\geqslant 1-e^{-u}/R\enspace.
\]
Consider the following recursive algorithm:
\begin{enumerate}
\item Initialize $\rho_{-}=0$, $\rho_+=\rho_0$, $\text{Alg}(\cD_N,(\rho_++\rho_-)/2,\eta,u,\bx)=(\bM,\by)$, $V=\text{Tr}(\bM)+\|\by\|_1$.
 \item if $V\in[1-\nu_1,1-\nu_0]$, stop,
 \item if $V<1-\nu_1$, update $\rho_+=(\rho_++\rho_-)/2$,
 \item if $V>1-\nu_0$, update $\rho_-=(\rho_++\rho_-)/2$,
 \item update $\text{Alg}(\cD_N,(\rho_++\rho_-)/2,\eta,u,\bx)=(\bM,\by)$, $V=\text{Tr}(\bM)+\|\by\|_1$.
\item Return $\rho_*=(\rho_++\rho_-)/2$.
\end{enumerate}
Clearly, on the event where $\text{Alg}(\cD_N,\rho_0,\eta,u,\bx)=(\bM,\by)$ satisfies
\[
\text{Tr}(\bM)+\|\by\|_1\leqslant 1\enspace,
\]
this algorithm outputs $\rho_*$ such that $\text{Alg}(\cD_N,\rho_*,\eta,u,\bx)=(\bM_*,\by_*)$ satisfies $\text{Tr}(\bM_*)+\|\by_*\|_1\in[1-\nu_1,1-\nu_0]$.
Such a value exists, at least on an event with large probability, by the discussion preceding the algorithm.
Let $T$ denote a number of steps to be defined later.
Using a union bound in Lemma~\ref{lem:LBgrho}, with probability larger than $1-Te^{-u}/R$, for all $\rho$ in the set of all $(\rho_++\rho_-)/2$ along the first $T$ steps of the algorithm, $\text{Alg}(\cD_N,\rho,\eta,u,\bx)=(\bM,\by)$ satisfies
\[
g(\rho)\leqslant \text{Tr}(\bM)+\|\by\|_1\leqslant (1+\eta)g(\rho)\enspace.
\]
Moreover, after $k$ steps of the algorithm,
\[
\rho_+-\rho_-=\frac{\rho_0}{2^k}\enspace.
\]
On this event $\rho_*$ satisfies $g(\rho_*)\in [(1-\nu_1)/(1+\eta), 1-\nu_0]$.
Now, by continuity of $g$ (see Lemma~\ref{lem:Basicsg}), there exists a constant $\delta>0$ such that, for any $\rho\in [\rho_*-\delta,\rho_*+\delta]$, $g(\rho)\in [(1-2\nu_1)/(1+\eta), 1-\nu_0/2]$, so $\text{Alg}(\cD_N,\rho,\eta,u,\bx)=(\bM,\by)$ satisfies
\[
\frac{1-2\nu_1}{1+\eta}\leqslant g(\rho)\leqslant \text{Tr}(\bM)+\|\by\|_1\leqslant (1+\eta)g(\rho)\leqslant (1+\eta)\bigg(1-\frac{\nu_0}{2}\bigg)\enspace.
\]
In particular, after $\log(\rho_0/\delta)/\log(2)$ steps,
$\text{Alg}(\cD_N,\rho,\eta,u,\bx)=(\bM,\by)$ satisfies
\[
1-\nu\leqslant\frac{1-2\nu_0}{1+\eta}\leqslant \text{Tr}(\bM)+\|\by\|_1\leqslant  (1+\eta)\bigg(1-\frac{\nu_0}{2}\bigg)\leqslant 1\enspace.
\]
Hence, the algorithm outputs in $T=\log(\rho_0/\delta)/\log(2)$ steps a solution satisfying the conclusions of Lemma~\ref{lem:BinarySurch}.
\end{proof}

Let now $\bx_0=\muh^{(0)}$ defined in Lemma~\ref{lem:GeomMean}.
This lemma grants that, with probability larger than $1-e^{-K/128}$, $\Omega_0$ holds, where $\Omega_0=\{\|\muh^{(0)}-\mu_P\|\leqslant 12\sqrt{\text{Tr}(\Sigma)K/N}\}$.
Consider the event $\Omega=\Omega_{0}\cap\Omega_{\alpha}$. 
We consider two cases.
On $\Omega$, either
\[
\|\muh^{(0)}-\mu_P\|\leqslant2r\quad \text{or}\quad \|\muh^{(0)}-\mu_P\|>2r\enspace.
\]
By Lemma~\ref{lem:BoundOPTx}, on $\Omega$,
\[
\text{OPT}(\muh^{(0)})\geqslant  \frac{1-2\alpha}{1-\alpha}(\|\muh^{(0)}-\mu_P\|-r)^2\enspace.
\]
If $\|\muh^{(0)}-\mu_P\|>2r$, it follows that
\[
\text{OPT}(\muh^{(0)})\geqslant  \frac{1-2\alpha}{1-\alpha}r^2\enspace.
\]
By \eqref{eq:UBg}, if $\|\muh^{(0)}-\mu_P\|>2r$, therefore
\[
g(\rho)\leqslant  \frac{(1-\alpha)}{(1-2\alpha)\rho r^2}+b\enspace. 
\]
Hence, $g(\rho)<1/(1+\eta)$ if $\rho_0=C_{\alpha}/r^2\leqslant C_\alpha N$.

\subsection{Final algorithm}
\begin{enumerate}
 \item Compute $\muh^{(0)}$, fix $R=\log N$, $\tau=\log[144 K]/\log[3/4]$, $\epsilon\in (2\alpha/(1-\alpha),0.4(1-\alpha)/(1-2\alpha))$, $\nu=\epsilon-b$, $B=0$.
 \item While $t\leqslant \tau$ and $B=0$,
 \begin{enumerate}
  \item Run the algorithm of Lemma~\ref{lem:BinarySurch} with $\rho_0=N$ and output $(\bM,\by)$ and $\rho^*$.
  \item If $\text{Tr}(\bM)+\|\by\|_1\notin [1-\nu,1]$, then $\muh^{(\kappa+1)}=\muh^{(t)}$, $B=1$.
 \item If $\text{Tr}(\bM)+\|\by\|_1\in [1-\nu,1]$, then apply the algorithm of Lemma~\ref{lem:gval} to output $(x,\by',\bM')$ satisfying, for any $\bx$ such that $\|\bx-\mu_P\|>A_{\alpha}r$,
 \[
h_{\bx}(\bM')\geqslant (\beta\|\bx-\mu_P\|+r)^2\enspace .
 \]
 \item Update $\muh^{(t+1)}$ according to \eqref{eq:DefItStep}, with $\bx=\muh^{(t)}$ and $\widehat{\bM}_{\bx}=\bM'$. 
\end{enumerate}
\item Output $\muh=\muh^{(\tau+1)}$.
\end{enumerate}

The algorithm terminates either after $\tau$ operations or when $\muh^{(t)}$ satisfies $\text{Tr}(\bM)+\|\by\|_1\notin [1-\nu,1]$.
Using a union bound in $t\in\{1,\ldots,\tau\}$, this last situation either happen if $\|\muh^{(t)}-\mu_P\|\leqslant 2 r$ or on an event of probability at least $1-e^{-u}$.
If the algorithm runs $\tau$ steps without stopping, on $\Omega=\Omega_0\cap \Omega_\alpha$, the output satisfies, by Proposition~\ref{prop:udate}
\[
\|\muh-\mu_P\|^2\leqslant (3/4)^{\tau} \bigg(144\frac{\text{Tr}(\Sigma)K}N\bigg)+(A_\alpha^2+1)r^2\sum_{i=0}^{+\infty}\bigg(\frac34\bigg)^i\leqslant C_\alpha r^2\enspace.
\]
Overall, choosing for example $\alpha=1/10$, we have obtained the following result.
\begin{theorem}
 There exists a numerical constant $C$ and an algorithm that runs in $\tilde{O}(uK+Kd)$ operations and outputs an estimator $\muh$ of $\mu_P$ such that
 \[
\P\bigg( \|\muh-\mu_P\|\leqslant C\bigg(\sqrt{\frac{\text{Tr}(\Sigma)}N}+\sqrt{\frac{\|\Sigma\|_{\text{op}}K}N}\bigg)\bigg)\geqslant 1-e^{-u\wedge(K/C)}\enspace.
 \]
\end{theorem}
%
%
%
%

\bibliographystyle{plain}
\bibliography{biblio}

\end{document}